\documentclass{article}

\PassOptionsToPackage{numbers, compress}{natbib}
\usepackage[main, final]{neurips_2026}

\usepackage[utf8]{inputenc} 
\usepackage[T1]{fontenc}    
\usepackage{hyperref}       
\usepackage{url}            
\usepackage{booktabs}       
\usepackage{amsfonts}       
\usepackage{nicefrac}       
\usepackage{microtype}      
\usepackage[dvipsnames]{xcolor}

\usepackage{amsmath,amsfonts,bm}

\def\eqref#1{equation~\ref{#1}}

\def\1{\bm{1}}

\def\A{\operatorname{A}}
\def\SA{\operatorname{SA}}
\def\LN{\operatorname{LN}}

\def\vzero{{\bm{0}}}
\def\vone{{\bm{1}}}

\def\vb{{\bm{b}}}

\def\vt{{\bm{t}}}
\def\vu{{\bm{u}}}
\def\vv{{\bm{v}}}
\def\vw{{\bm{w}}}
\def\vx{{\bm{x}}}

\def\vz{{\bm{z}}}

\def\mA{{\bm{A}}}
\def\mB{{\bm{B}}}

\def\mD{{\bm{D}}}

\def\mG{{\bm{G}}}
\def\mH{{\bm{H}}}
\def\mI{{\bm{I}}}

\def\mK{{\bm{K}}}

\def\mP{{\bm{P}}}
\def\mQ{{\bm{Q}}}
\def\mR{{\bm{R}}}
\def\mS{{\bm{S}}}
\def\mT{{\bm{T}}}

\def\mV{{\bm{V}}}
\def\mW{{\bm{W}}}
\def\mX{{\bm{X}}}
\def\mY{{\bm{Y}}}

\def\mZeros{\vzero} 

\def\mLambda{{\bm{\Lambda}}}

\DeclareMathAlphabet{\mathsfit}{\encodingdefault}{\sfdefault}{m}{sl}
\SetMathAlphabet{\mathsfit}{bold}{\encodingdefault}{\sfdefault}{bx}{n}

\def\sP{{\mathbb{P}}}

\def\sX{{\mathbb{X}}}

\newcommand{\R}{\mathbb{R}}

\DeclareMathOperator*{\argmax}{arg\,max}

\newcommand{\RR}{\mathbb{R}}

\usepackage{hyperref}
\usepackage{url}
\usepackage{enumitem}
\usepackage{subcaption}

\usepackage{bbm}
\usepackage{titlesec}
\usepackage{amsfonts}
\usepackage{hyperref}
\usepackage{mathtools}
\usepackage{tikz}
\usetikzlibrary{matrix}
\usepackage{verbatim}
\usepackage{caption,booktabs}
\usepackage{tabularx,booktabs}
\usepackage{bbm}
\usepackage{wrapfig,lipsum,booktabs}
\usepackage{capt-of}
\usepackage[original]{pict2e}
\usepackage{diagbox}
\usepackage{cancel}
\usepackage{natbib}
\usepackage{amsthm} 
\usepackage{paracol}
\usepackage{multirow}
\usepackage{tcolorbox}

\usepackage{amssymb} 
\usepackage{graphicx} 
\usepackage{wrapfig}
\usepackage{cleveref}
\Crefname{figure}{Fig.}{Figs.}

\usepackage[export]{adjustbox}

\newtheorem{theorem}{Theorem}

\newtheorem{lemma}[theorem]{Lemma}

\usepackage{titletoc}

\definecolor{tumblue}{RGB}{0, 101, 189}
\definecolor{tumdarkblue}{RGB}{0, 82, 147}
\definecolor{tumlightblue}{RGB}{100, 160, 200}
\definecolor{tumlighterblue}{RGB}{152, 198, 234}
\definecolor{tumgray}{RGB}{153, 153, 153}
\definecolor{tumorange}{RGB}{227,114,34}
\definecolor{tumgreen}{RGB}{162,173,0}
\definecolor{tumlightgray}{RGB}{218, 215, 203}

\definecolor{darkblue}{rgb}{0.0,0.0,0.65}
\definecolor{darkred}{rgb}{0.68,0.05,0.0}
\definecolor{darkgreen}{rgb}{0.0,0.5,0.3}
\definecolor{darkpurple}{rgb}{0.47,0.09,0.29}

\colorlet{mycolor}{Gray}

\hypersetup{
   colorlinks=true,
   citecolor=tumblue,
   linkcolor=tumblue,
   filecolor=tumblue,
   urlcolor=tumblue,
}

\title{How Do Transformers Learn to Represent Symmetries?}

\author{
    Eduardo Santos-Escriche\thanks{Equal contribution}\\
    Technical University of Munich\\ Munich Center for Machine Learning 
      \And
    Valerie Engelmayer\footnotemark[1]\\
    Technical University of Munich\\ Munich Center for Machine Learning 
    \AND
    Ya-Wei Eileen Lin \\
    Technical University of Munich\\ Munich Center for Machine Learning 
    \And
    Stefanie Jegelka \\
    Technical University of Munich\\ Munich Center for Machine Learning \\
    Massachusetts Institute of Technology
}

\begin{document}
\setlength{\abovedisplayskip}{3pt}
\setlength{\belowdisplayskip}{3pt}
\frenchspacing

\maketitle

\begin{abstract}
Training Transformer-based architectures with finite data augmentation has become an increasingly popular approach in geometric machine learning.  Despite its empirical success, the interplay between the Transformer architecture, invariance to different symmetries, and augmentation budgets remains underexplored.  In this paper, we study the ability of a vanilla Transformer to learn various symmetries through finite data augmentation for point cloud datasets.  We identify an ordering of increasing \textit{learnability} across the following symmetry groups: \textit{(i)} non-angle-preserving symmetries,  \textit{(ii)} angle-preserving symmetries, and \textit{(iii)} base angle-preserving subgroups, such as translation, rotation, and scale. For the base angle-preserving groups, we further investigate the Transformer's extrapolation behavior and conduct a structural analysis of the trained models, allowing us to identify interpretable mechanisms that induce invariance.  Finally, we extend our analysis to equivariant functions and show that the detected mechanisms for approximate invariance can also provide a key building block for learned equivariance.
Our project page is available at \href{https://transformers-learn-symmetries.github.io/}{this https URL}.
\end{abstract}

\section{Introduction}
\label{sec:introduction}

Data augmentation has been increasingly used as a common technique to induce approximate symmetry in machine learning models, since it provides practitioners with a flexible and model-agnostic approach that can be straightforwardly combined with off-the-shelf architectures.
In particular, it has become popular in various scientific fields such as materials science \cite{gibson2022data}, computational chemistry \cite{elhag2025learning}, structural biology \cite{abramson2024accurate, deng2026mutppi+}, as well as in computer vision \cite{shorten2019survey}, robotics \cite{laskin2020reinforcement} and reinforcement learning \cite{ma2025comprehensive}.
A notable example is the departure of AlphaFold3 \cite{abramson2024accurate} from the hard architectural equivariance of AlphaFold2 \cite{jumper2021highly}, and many other works have recently followed suit, e.g., MCFs \cite{wang2024swallowing}, TransIP \cite{elhag2025learning}, Zatom-1 \cite{morehead2026zatom}, or TABASCO \cite{vonessen2025tabasco}.

While this technique has been extensively studied in the literature assuming an infinite number of augmentations \cite{chen2020group, lyle2020benefits}, this assumption strays far from reality: compute budgets and thus the possible number of augmentations are \textit{finite}. 
A model only sees a limited number of transformed views of each sample, and the resulting symmetry is \textit{approximate}.
Recent theoretical work has begun to study this finite data augmentation regime \cite{tahmasebi2025achieving, tahmasebi2025data}, however, a central question still remains open: \textit{which architectures can efficiently learn which symmetries from finite data augmentation?}

This question is especially relevant for \textit{Transformers} \cite{vaswani2023attentionneed, elhage2021mathematical}.
The popularity and scalability of Transformers, together with their empirical success on downstream tasks, have been widely accepted as sufficient justification for their consistent adoption as the core architecture for many augmentation-based approaches.
However, the interplay between their architecture, data, and augmentation budgets for the learned symmetries has not yet been explored.
Their broad applicability does not address whether Transformers are equally capable of learning all types of symmetries from data augmentation, or whether their architecture is better aligned with some transformation groups than others.

In this paper, we investigate this question 
through a systematic study of a vanilla Transformer’s ability to learn different symmetries under limited data augmentation budgets, 
and examine how these learned symmetries are represented in trained models.
Our main contributions are the following:
\vspace{-1mm}
\begin{itemize}
[leftmargin=2.6em, labelsep=0.5em]
    \item We systematically investigate the ability of a vanilla Transformer to learn invariance to different symmetries from limited data augmentation and identify its special affinity for angle-preserving symmetries and, in particular, \textit{base groups}, e.g., translation, rotation, and scale.  
    \item We evaluate the generalization of learned angle-preserving symmetry to transformation ranges that are not observed during training, demonstrating that Transformers can acquire invariant behaviors that extend to unseen transformation ranges.
    \item We analyze trained Transformers and uncover interpretable mechanisms that implement invariance to each base angle-preserving symmetry, which we further connect to invariant theories and to previously observed empirical phenomena.
    \item We show that the identified mechanisms underlying approximate invariance in Transformers can also provide key building blocks for learning equivariance.
\end{itemize}

\begin{figure}[t]
    \centering
    \includegraphics[width=0.9\linewidth]{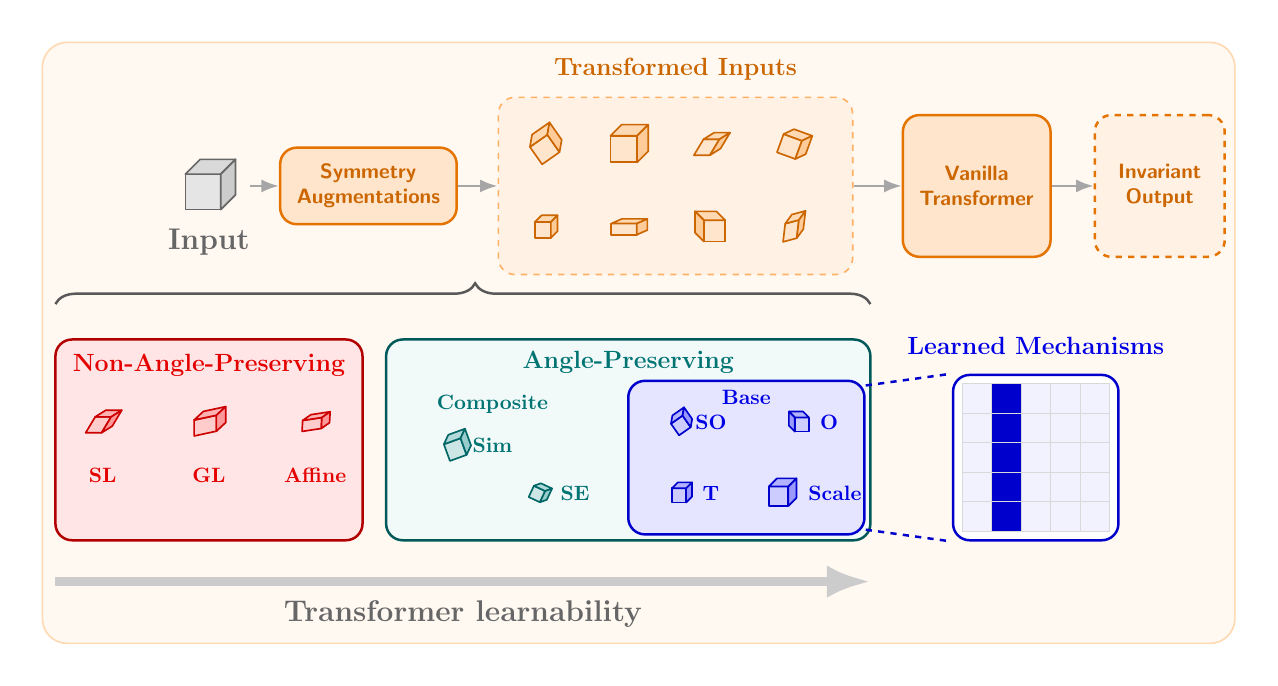}
    \vspace{-3mm}
    \caption{
    \textbf{Overview of our study and main findings.}
    Under finite data augmentation, vanilla Transformers learn invariance more efficiently for angle-preserving symmetries, especially their base groups, than for non-angle-preserving transformations. For these base symmetries, we identify \textit{interpretable mechanisms} underlying the learned approximate invariance in the trained models.
    }
    \label{fig:teaser}
    \vspace{-1em}
\end{figure}

\section{Related work}
\label{sec:related_work}

Incorporating symmetries into machine learning models has been explored for a wide range of symmetry groups, including permutations in sets \cite{zaheerDeepSets2017}, translations \cite{krizhevsky2012imagenet}, roto-translations \cite{fuchs2020se}, or local gauge transformations \cite{cohen2019gauge}. 
Among the multiple equivariant machine learning paradigms, the two most common ones are: \textit{(i)} \textit{hard equivariance through architectural constraints}, where layers are designed to be equivariant
to the desired symmetry 
(e.g., GCNN \cite{cohenGroupEquivariantConvolutional2016} or EMLP \cite{finzi2021practical}), and \textit{(ii)} \textit{approximate equivariance through data augmentation} \cite{hu2021forcenet, gerken2022equivariance, wang2022data, quiroga2020revisiting, chen2020group, xu2023comprehensive}, where symmetry-transformed samples are incorporated into the training distribution.
Beyond these paradigms, frame averaging \citep{puny2022frame, duval2023faenet, lin2024equivariance} seeks to achieve equivariance by averaging over a small input-dependent subset of the full group. Another approach, canonicalization \citep{kaba2023equivariance, ma2024canonicalization, lin2024equivariant, lin2026adaptive}, implements equivariance by defining a map from all inputs that are equivalent under the symmetry group to the same canonical form.
Symmetry can also be incorporated through multi-task learning \cite{elhag2025relaxed}.

Especially related to our work are the numerous extensions of standard Transformer architectures designed to enforce equivariance with respect to different symmetries, including
SE(3) \cite{fuchs2020se, liaoequiformer, liao2023equiformerv2, liao2026equiformerv3}, Lie groups \cite{hutchinson2021lieTransformer}, geometric algebras \cite{brehmer2023geometric, spinner2024lorentz}, local gauge transformations \cite{he2021gauge}, and the combination of continuous translations and Platonic symmetries \cite{islam2025platonic}. 
Such approaches can improve data efficiency and generalization by incorporating known structural priors into the model
\cite{dehmamyAutomaticSymmetryDiscovery2021, cohenGroupEquivariantConvolutional2016, bulusu2021generalization}.
However, enforcing exact equivariance also introduces several limitations, including increased implementation complexity, scalability challenges, and an overly restrictive hypothesis space. The latter is specifically relevant when the underlying symmetry is only approximate, for example because of noise, discretization effects, or other imperfections in the data.
These limitations have motivated a closer examination of the practical benefits of exact equivariance at scale \cite{brehmerDoesEquivarianceMatter2024} and have contributed to a recent shift toward alternative formulations of symmetry-aware models in scientific tasks with large amounts of data.
Besides the data augmentation methods discussed in \S\ref{sec:introduction}, one increasingly explored direction is to encourage equivariance through additional loss terms to explicitly bias the models towards equivariance,
e.g., Orb-v3 \cite{rhodes2025orb}, REMUL \cite{elhag2025learning}, and LieAugmenter \cite{santos2025lieaugmenter}.

\begin{table}[t]
    \centering
    \fbox{%
        \begin{minipage}{\dimexpr\textwidth-2\fboxsep-2\fboxrule\relax}
            \vspace{0.5em}
            \centering
            \renewcommand{\arraystretch}{1.5}
            
            \resizebox{0.98\linewidth}{!}{%
                \begin{tabular}{@{} l @{\hspace{4em}} l @{}}
                    $\operatorname{O}(d) = \{\mQ \in \mathbb{R}^{d \times d} : \mQ^\top \mQ = \mI_d\}$ &
                    $\operatorname{Scale}(d) = \{s \mI_d : s \in \mathbb{R}_{>0}\}$ \\
                    
                    $\operatorname{SO}(d) = \{\mQ \in \mathbb{R}^{d \times d} : \mQ^\top \mQ = \mI_d, \det(\mQ) = 1\}$ &
                    $\operatorname{SE}(d) = \operatorname{T}(d) \rtimes \operatorname{SO}(d)$ \\
                    
                    $\operatorname{T}(d) = \{\vw \in \mathbb{R}^d\}$ &
                    $\operatorname{Sim}(d) = \operatorname{T}(d) \rtimes (\operatorname{Scale}(d) \times \operatorname{SO}(d))$
                \end{tabular}%
            }
            
            \vspace{0.5em}
            \hrule
            \vspace{0.5em}
            
            \resizebox{0.98\linewidth}{!}{%
                \begin{tabular}{@{} l l l @{}}
                    $\operatorname{SL}(d) = \{\mA \in \mathbb{R}^{d \times d} : \det(\mA) = 1\}$ &
                    $\operatorname{GL}(d) = \{\mA \in \mathbb{R}^{d \times d} : \det(\mA) \neq 0\}$ &
                    $\operatorname{Aff}(d) = \operatorname{T}(d) \rtimes \operatorname{GL}(d)$
                \end{tabular}%
            }
            \vspace{0.5em}
        \end{minipage}%
    }
    
    \vspace{0.5em}
    \caption{The symmetry groups studied in this work. Symmetries are separated into angle-preserving (\textit{top}) and non-angle-preserving transformations (\textit{bottom}).}
    \label{tab:symmetry_groups}
   \vspace{-1.5 em}
\end{table}

\section{Finite augmentation reveals a symmetry inductive bias in Transformers}
\label{sec:motivation}

Motivated by recent work increasingly combining Transformer-based models with data augmentation to learn representations that remain unchanged under specific transformations,
we investigate whether Transformers can learn different types of symmetries in controlled settings.
Specifically, we ask whether increasing the number of augmented views during training is sufficient for a Transformer to learn invariance across a broad range of transformation groups (see Fig.~\ref{fig:sym_inductive_bias} for the groups we study).

\noindent\textbf{Architecture.}
Seeking that our analysis faithfully reflects the ability of Transformers to learn these symmetries, we consider a small vanilla Transformer architecture, without positional encodings or masks, consisting of two blocks, each with a single attention head (see Eq.~(\ref{eq:Transformer-block}) for a precise formulation). We explore the effect of incorporating additional architectural components in \S\ref{app:additional_components}.

\noindent\textbf{Symmetries.}
In this analysis, we study multiple symmetry groups that are common across many different domains, which we characterize in Table~\ref{tab:symmetry_groups}.
Invariance to any of those groups $G$ is defined as the output of the Transformer $\Psi$ remaining unchanged under the action of said group on the input domain $\mathcal{X} \subseteq \mathbb{R}^d$, i.e., $\forall g \in G, x \in \mathcal{X}$ we have $\Psi(g \cdot x) = \Psi(x)$.

\begin{figure}
    \centering
    \begin{minipage}[c]{0.45\textwidth}
        \includegraphics[width=\linewidth]{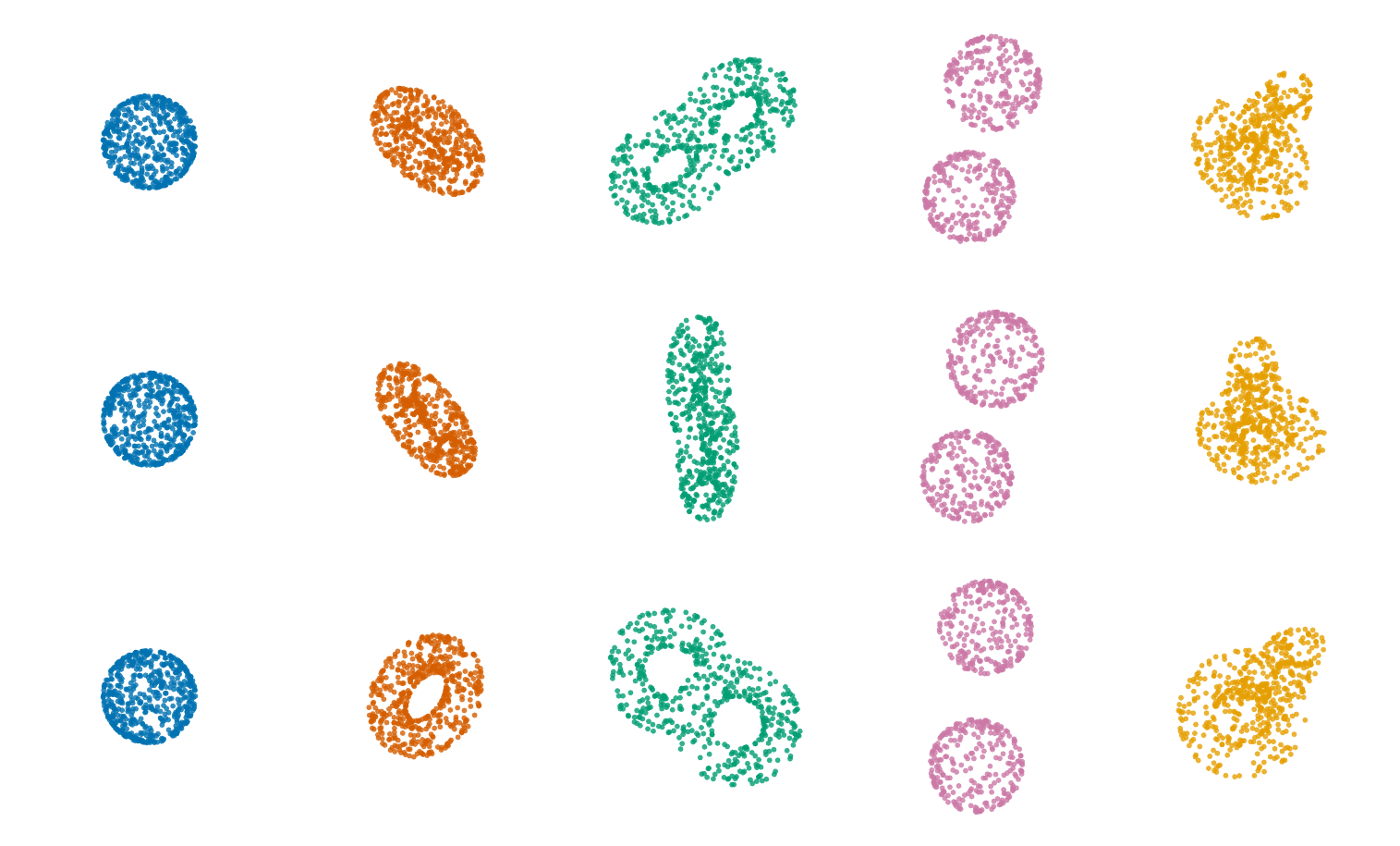}
        \includegraphics[width=\linewidth]{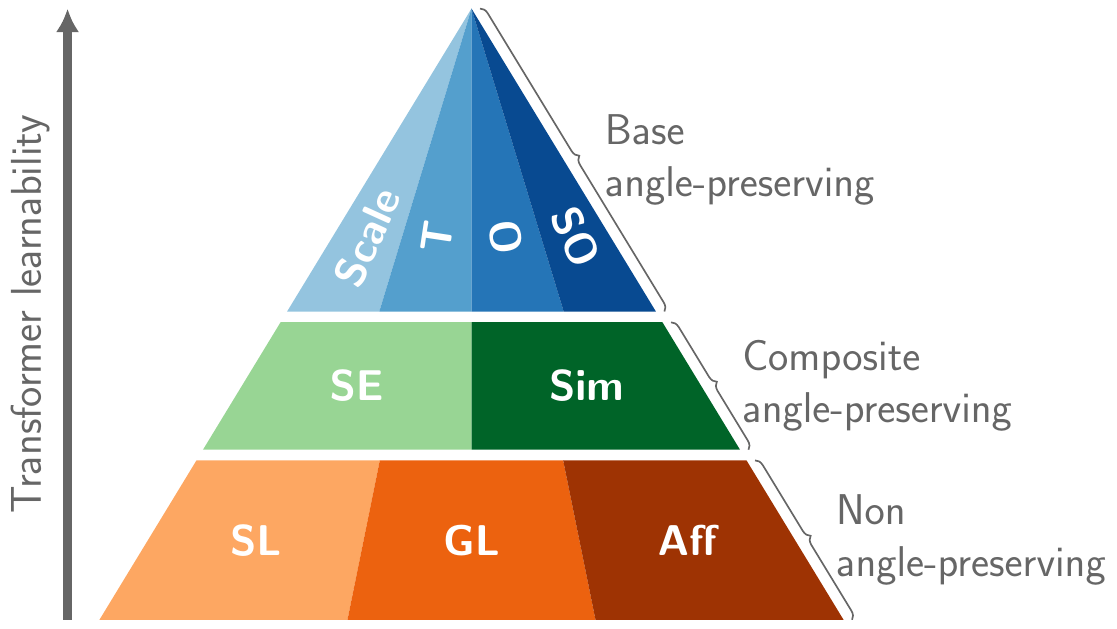}
    \end{minipage}
    \hfill
    \begin{minipage}[c]{0.5\textwidth}
        \begin{minipage}[b]{0.48\linewidth}
        \centering
        \includegraphics[height=4cm]{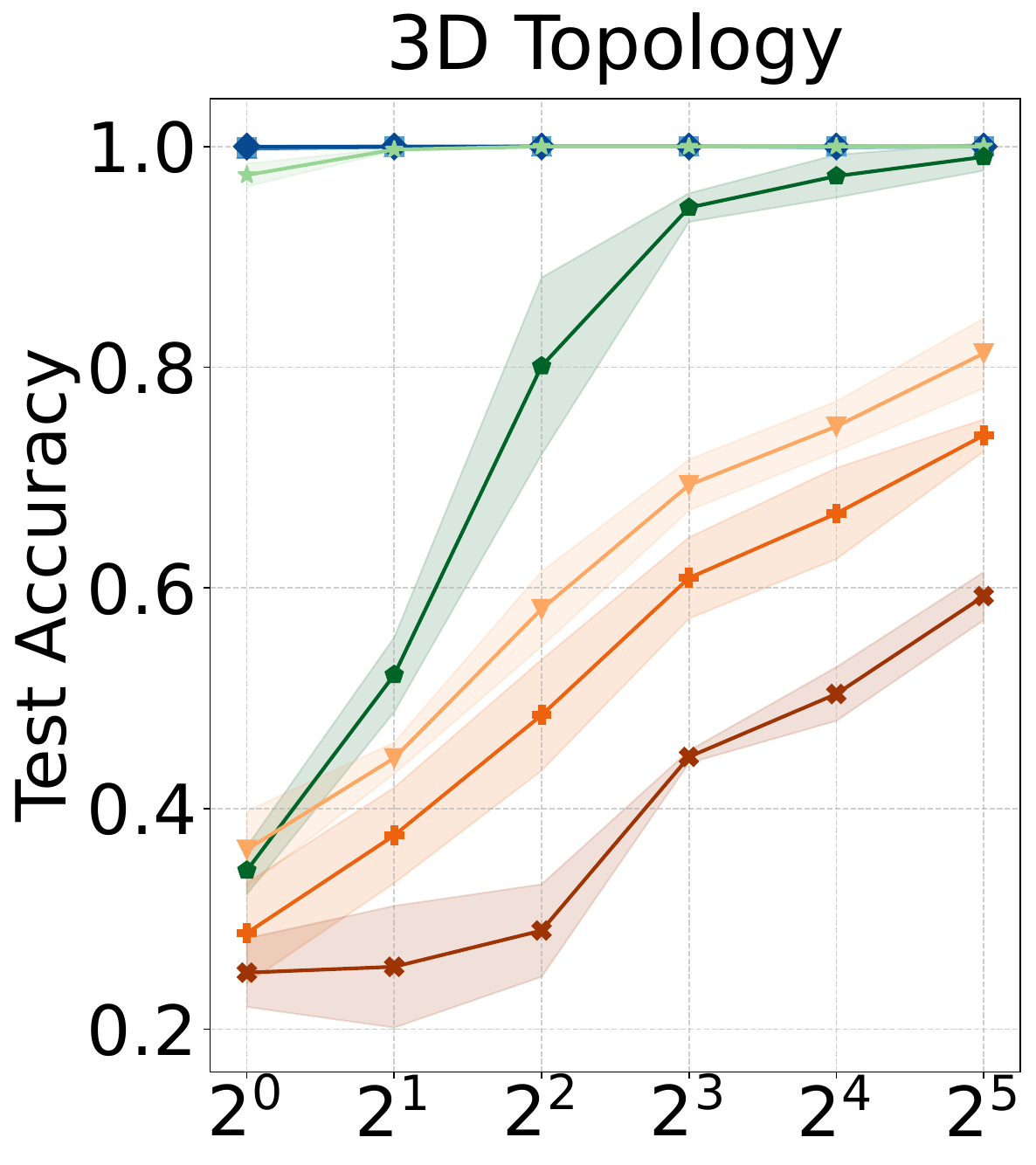}
    \end{minipage}
    \hfill
    \begin{minipage}[b]{0.5\linewidth}
        \centering
        \includegraphics[height=4cm]{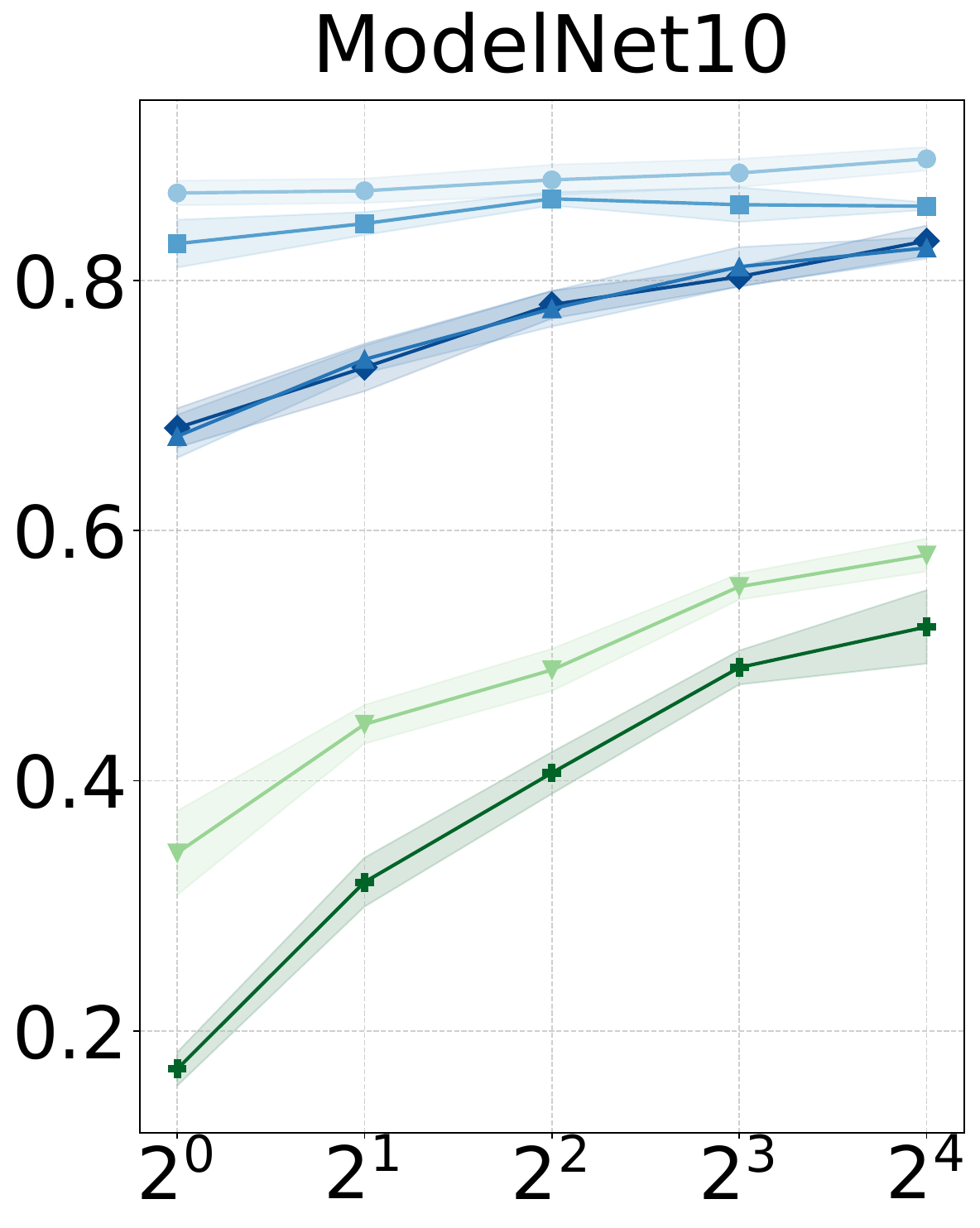}
    \end{minipage}
    \\
    
    \begin{minipage}[b] {0.48\linewidth}
        \centering
        
\includegraphics[height=4cm]{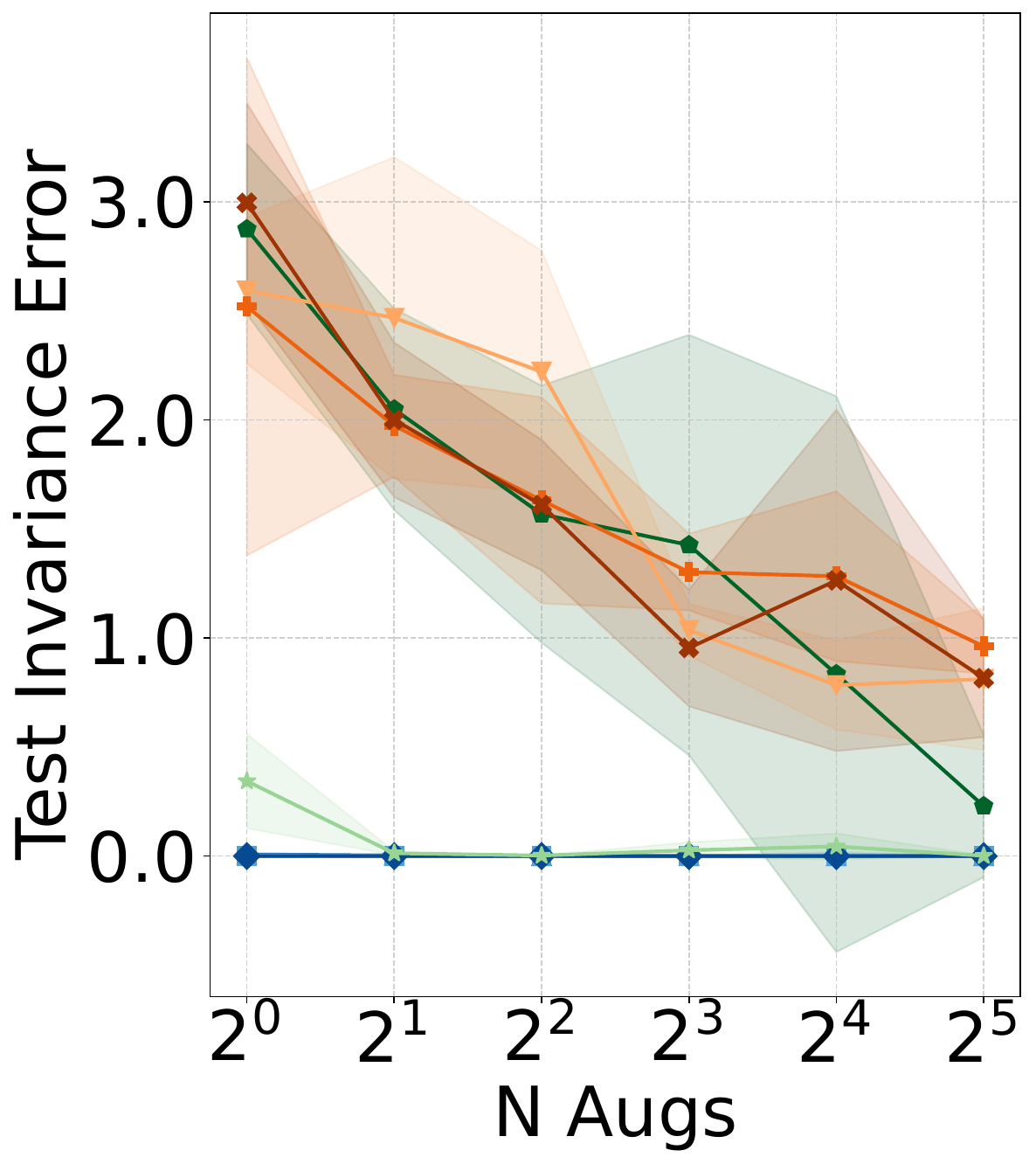}
    \end{minipage}
    \hfill
    \begin{minipage}[b]{0.48\linewidth}
        \centering
        \includegraphics[height=4cm]{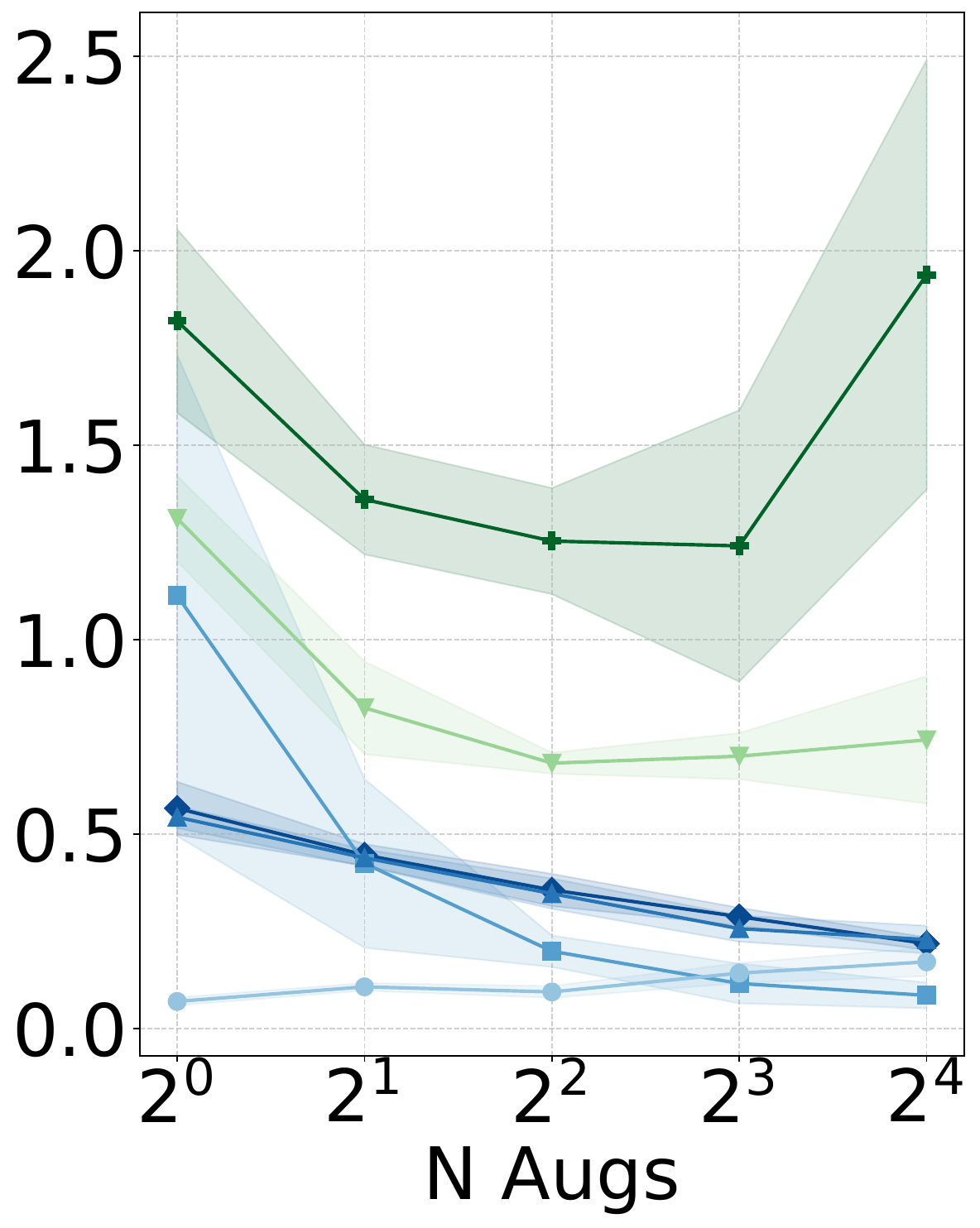}
    \end{minipage}
    \end{minipage}
    \vspace{-2mm}
     \caption{
    (\textit{Left}) Samples from different classes of the 3D Topology dataset (top), and legend for the plots on the right, summarizing our findings (bottom).
    (\textit{Right}) Mean $\pm$ std for test accuracies (top) and invariance errors (bottom) across 5 random seeds. Angle-preserving symmetries are \textit{efficiently} learned by the Transformer, while it struggles with non-angle-preserving ones. 
     }
    \label{fig:sym_inductive_bias}
        \vspace{-1em}
\end{figure}

\noindent\textbf{Datasets.} We study two complementary datasets. 
First, we build a synthetic topology-based dataset, called \textit{3D Topology}.
This dataset consists of five classes of point clouds exhibiting distinct topological properties, including genus, connectivity, and ambient entanglement (see bottom-left of Fig.~\ref{fig:modelnet} for an illustration). 
The labels are invariant to all the considered transformations. 
Thus, it serves to evaluate whether data augmentation can induce invariance to these larger symmetry groups. 
Second, we consider ModelNet10 \cite{wu20153d}, a widely used 3D shape classification dataset containing point clouds for 10 object categories. The task is to predict the category from point coordinates, and the label is invariant to global rigid motions and point permutations. 
Further details on the datasets are in \S\ref{app:experiment_details_motivation}.

\noindent\textbf{Metrics.} To evaluate whether a Transformer learns a given symmetry through data augmentation, we use two complementary criteria. First, we measure the downstream classification accuracy on a test set formed by randomly transforming the prototype samples according to the same symmetry group used during training. Second, we quantify the learned invariance using: 
$
    \frac{1}{N K} \sum_{i=1}^N \sum_{j=1}^K \mathrm{KL} \left( \Psi(x_i) \ \vert\vert \  \Psi( g_j x_i) \right)$, 
which contrasts the class-probability distributions for pairs of original and augmented samples.
We also analyze alternative invariance errors in \S\ref{app:experiment_details_invariance}.
The model is trained only with the classification objective on augmented samples and the KL-based invariance error is used solely as a diagnostic metric to track invariance and is not part of the training objective. High accuracy together with low invariance loss therefore provides evidence that the model has learned an approximately invariant classifier for the corresponding symmetry.

\noindent\textbf{Results.} 
Fig.~\ref{fig:sym_inductive_bias} reveals a clear separation between two families of symmetries.
On the one hand, transformations associated with angle-preserving symmetries and their subgroups, including $\operatorname{SO}(3)$, $\operatorname{O}(3)$, $\operatorname{T}(3)$, $\operatorname{SE}(3)$, $\operatorname{Scale}(3)$, and $\operatorname{Sim}(3)$, are learned \textit{efficiently}. For these groups, only a small number of augmentations is required to achieve high test accuracy and low invariance loss. 
On the other hand, non-angle-preserving transformation groups such as $\operatorname{Affine}(3)$, $\operatorname{SL}(3)$, and $\operatorname{GL}(3)$ require substantially larger augmentation budgets. Even with many augmentations, their accuracy remains lower and their invariance loss higher than what is observed for the angle-preserving groups.

The ModelNet10 results in Fig.~\ref{fig:sym_inductive_bias} show a more detailed stratification.
The base groups $\operatorname{T}(3)$ and $\operatorname{Scale}(3)$, achieve the strongest performance in both test accuracy and invariance loss. Their performance is also stable as the number of augmentations increases, suggesting that these transformations are learned with a small augmentation budget.
The other base groups $\operatorname{SO}(3)$ and $\operatorname{O}(3)$, also achieve strong performance, but require a larger augmentation budget. 
Finally, the Transformer struggles more with compositional groups formed by combining those base transformations, i.e., $\operatorname{SE}(3)$ and $\operatorname{Sim}(3)$. 
Thus, the fact that the model can efficiently learn individual base symmetries does not automatically imply an ability to learn their compositions with the same efficiency.

\begin{tcolorbox}[colback=mycolor!10, left=8pt,
right=8pt,
top=3pt,
bottom=3pt,
colframe=mycolor!70!black,, title=Finding 1: A hierarchy in the learnability of symmetries with Transformers]
Experimental evidence reveals the following ordering of symmetry families from easiest to hardest to learn for a Transformer: 
\begin{enumerate}
    \item base angle-preserving groups
    \item compositional angle-preserving symmetries
    \item non-angle-preserving symmetries
\end{enumerate}
\end{tcolorbox}
Based on these observations, below we focus on the base angle-preserving groups.
This allows us to better isolate how each symmetry is learned and which mechanisms may support it.

\section{Learning base angle-preserving symmetries}
\label{sec:ood}

So far, we have observed that a Transformer can learn base angle-preserving symmetries efficiently when trained with augmentations sampled from the same support as test augmentations. 
To more rigorously investigate its ability to learn to extrapolate within the symmetries, we conduct further experiments testing in- and out-of-distribution generalization on ModelNet10 for each of these groups. 

\begin{wrapfigure}[43]{r}{0.48\textwidth}
\vspace{-5mm}
    \centering
    \includegraphics[width=\linewidth]{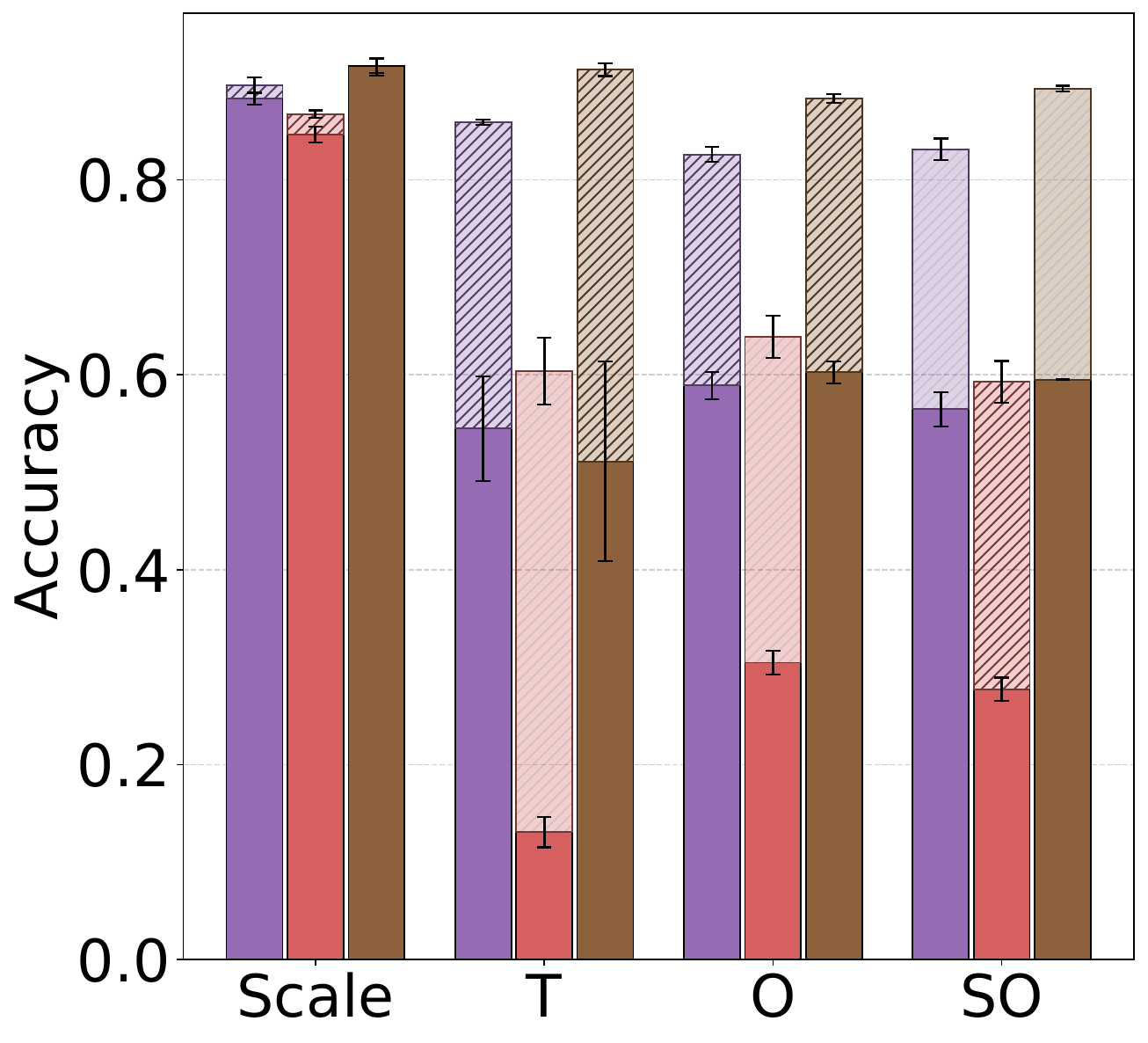}
    
    \hspace{1.2cm}\includegraphics[width=0.7\linewidth]{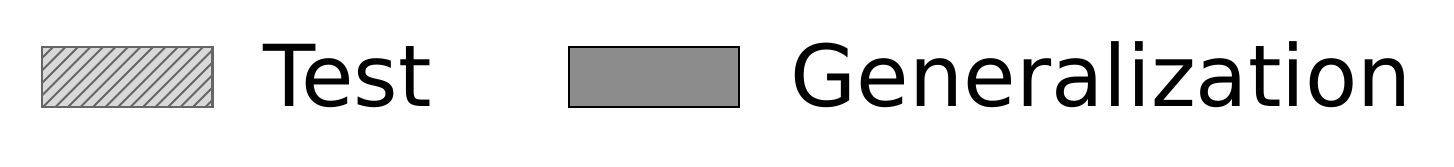}
    \includegraphics[width=\linewidth]{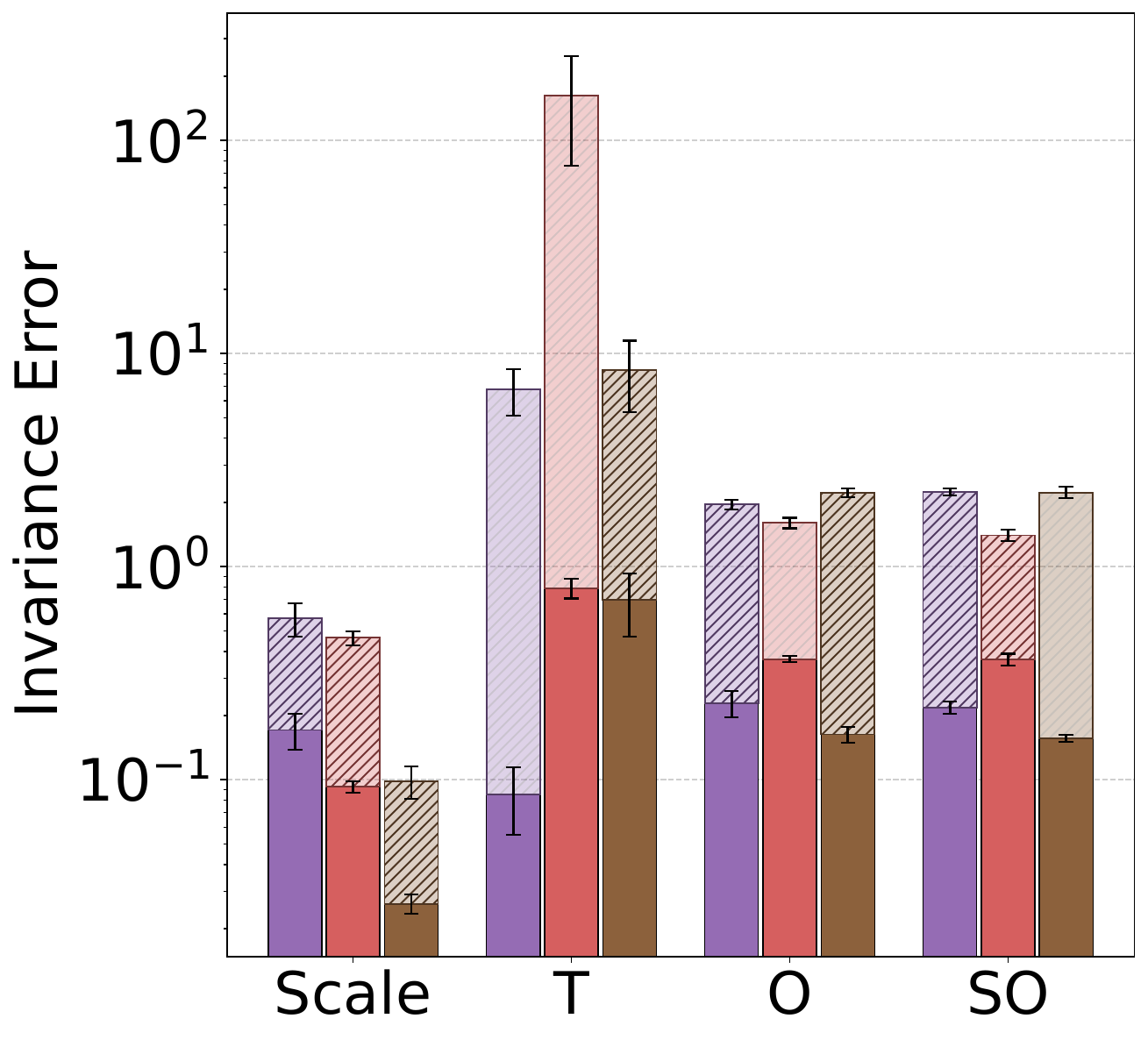}\\
    \hspace{1.2cm}\includegraphics[width=0.7\linewidth]{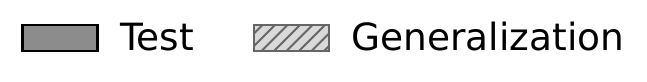}
    \includegraphics[width=\linewidth]{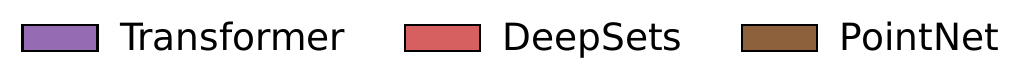}
    \vspace{-6mm}
    \caption{Accuracy (\textit{top}) and invariance error (\textit{bottom}) for different architectures on base angle-preserving symmetries under OOD evaluation. DeepSets exhibits the weakest performance, PointNet the strongest with Transformers as close second. 
    Across architectures, scaling invariance appears easiest to learn, whereas translation invariance appears most difficult.
    }
    \label{fig:modelnet}
\end{wrapfigure}

\noindent\textbf{OOD setting.}
In practice, the augmentation budget may not cover the full range of transformations that the model encounters at test time.
Thus, a learned invariance might need to hold on transformation ranges not seen at training time.
This is especially critical for non-compact groups, e.g., translation and scale, where the magnitude of the augmentation could change drastically at test time.

To explore this scenario, we define two disjoint augmentation ranges corresponding to transformations of different magnitudes.
During training, the model is exposed only to augmentations sampled from the smaller-magnitude range. We then consider two augmented versions of the test set: one using transformations from the same range as training, which we refer to as the \emph{test set}, and one using transformations from a disjoint, larger-magnitude range, which we refer to as the \emph{generalization set}. Further details are provided in \S\ref{app:ood}.

\noindent\textbf{Compared models: DeepSets and PointNet.}
To contextualize the performance of the considered vanilla Transformer architecture, we investigate whether other permutation invariant coordinate-token set architectures, i.e., DeepSets \citep{zaheerDeepSets2017} and PointNet \citep{qi2017pointnet}, can efficiently learn the base angle-preserving groups as well. PointNet is a more specialized point cloud model that extends DeepSets by so-called T-nets, which predict an affine matrix that is multiplied to the current representation matrix.
Thus they can learn to canonicalize orientation with respect to the group $\operatorname{O}$. For architectural details, see \S\ref{app:ood}.
Note that our goal is not to obtain SOTA performance, but to isolate how well each model can learn different invariances.

\noindent\textbf{Results.}
Fig.~\ref{fig:modelnet} shows that scaling poses no major challenge to any of the models in or out of distribution. Differences become more pronounced for the other groups. DeepSets struggles to express any of the other invariances, with by far the lowest test accuracies across groups and poor extrapolation
-- in the case of translation, prediction is almost reduced to random guessing. Introducing T-nets drastically improves the overall performance of PointNet over DeepSets as expected. Maybe more surprisingly, the Transformer falls only slightly short of the accuracy of PointNet for rotations and reflections. Hence the point-wise interaction of self-attention supports these transformations almost as reliably as PointNet's dedicated canonicalization modules.
We observe that, for translations, Transformers exhibit significantly smaller invariance errors and a little higher generalization accuracy than the competing PointNets. 

\begin{tcolorbox}[colback=mycolor!10, left=8pt,
right=8pt,
top=3pt,
bottom=3pt, colframe=mycolor!70!black,, title=Finding 2: Extrapolation on base angle-preserving symmetries]
For base angle-preserving symmetries, the Transformer learns to become invariant to the group structure beyond the limited augmentation ranges seen during training.
\end{tcolorbox}

\section{Mechanisms that can induce invariance to angle-preserving symmetries}
\label{sec:mechanisms}

Motivated by \textit{Finding 2},
we study how Transformers can learn features that are invariant to rotation and reflection, translation, and scaling. 
For each group, we train ablations with only the modules involved in the respective mechanism and verify that they can indeed become invariant. In addition, we identify the signature of each mechanism in the activations of full vanilla Transformers trained with augmentations from the corresponding group.
We refer to \S\ref{app:signature_mechanisms} for a comprehensive overview. 

The main components of the Transformer are central to our analysis.
We first define the self-attention module. 
For an input $\mX\in \RR^{n \times d}$,  the self-attention module is defined by 
\begin{equation}
    \begin{split}
        \SA(\mX) = \A(\mX)\mV, \quad\quad \A(\mX) = \operatorname{softmax} \left({\mQ\mK^\top}/{\sqrt{d}}\right),
    \end{split}
\end{equation}
where $
    \mQ = \mX\mW_Q, \mK = \mX\mW_K, \mV = \mX\mW_V
$, 
with parameters $\mW_Q, \mW_K,\mW_V \in \R^{d \times d}$. 
We next consider LayerNorm, which normalizes each token independently. For the $i$-th token, it computes
\begin{equation}
    \begin{split}
	\LN(\mX_i) = \vw \odot \frac{\mX_i - \mu_i}{\sqrt{\sigma_i^2 + \epsilon}} + \vb, \qquad
	\mu_i = \frac{1}{d}\sum\nolimits_{j=1}^{d} \mX_{ij}, \qquad
	\sigma_i^2 = \frac{1}{d}\sum\nolimits_{j=1}^{d}(\mX_{ij} - \mu_i)^2 
    \end{split}
\end{equation}
with learnable parameters $\vw,\vb \in \R^d$, where $\odot$ denotes elementwise multiplication and $\epsilon>0$ for numerical stability.
Finally, our analysis also relies on the residual connections. 
Let $\operatorname{MLP}$ denote a two-layer MLP with $\operatorname{ReLU}$ activations. 
A single Transformer block then computes
\begin{equation}
    \mY = \mX + \SA(\LN(\mX)), \qquad \mX' = \mY + \operatorname{MLP}(\LN(\mY)).
    \label{eq:Transformer-block}
\end{equation}
A full Transformer sequentially forwards through an in-projection, multiple such blocks, and a final out-projection. 
To isolate the mechanisms identified in our analysis, we use a single attention head and omit masks and positional encodings throughout this work. 
We present additional experiments with different positional encoding methods and multi-head attention setup  in \S\ref{app:additional_components}. 

By construction, this architecture is permutation equivariant. 
We next study how it can learn invariance to rotations, translations, and scaling.

\subsection{Rotation and reflection: stable attention over scalar relations}

We first  consider symmetries that fix the origin. Let $
    \operatorname{O}(d) = \{\mR \in \RR^{d \times d} : \mR^\top \mR = \mI_d\}$
act on a point cloud $\mX \in \RR^{n \times d}$ by right multiplication, $\mX \mapsto \mX \mR$.
This group contains both rotations
$\operatorname{SO}(d)$ and reflections.
A natural way to  build rotation- and reflection-invariant features is to use scalar quantities. By the first fundamental theorem of invariant theory for $\operatorname{O}(d)$, polynomial scalar functions of vector inputs that are invariant under $\operatorname{O}(d)$ can be written as functions of pairwise inner products. Equivalently, for point clouds, the Gram matrix $\mG(\mX) = \mX\mX^\top$ contains the canonical scalar information preserved by rotations and reflections. This view is in line with Villar et al. (2021) \cite{villar2021scalars}, where the invariant scalar products and scalar contractions provide a universal way to parameterize invariant and equivariant polynomial models for $\operatorname{O}(d)$ and related groups.

A Transformer can expose this type of scalar information through its attention scores. Indeed, the pre-softmax attention scores can be written as 
\begin{equation}
        \mQ\mK^\top/\sqrt{d}
    =
    \mX \mB \mX^\top /\sqrt{d},
    \qquad
    \mB = \mW_Q\mW_K^\top.
\end{equation}
Thus each score is a learned bilinear pairing $\vx_i^\top \mB \vx_j$. In the special case $\mB = \alpha \mI_d$, the scores are proportional to the Gram matrix and are invariant to every $\mR \in \operatorname{O}(d)$. More generally, exact isotropy of $\mB$ is not necessary on a finite augmentation orbit: because attention is row-wise softmax-normalized, only score differences within each row must be preserved.

\begin{lemma}
\label{lem:orbit-wise-attention-invariance}
    Let $\mX\in\RR^{n\times d}$ and let $\mB = \mW_Q\mW_K^\top\in\RR^{d\times d}$ and $\mS_{\mB}(\mX)
    =
    \frac{1}{\sqrt d}\mX\mB\mX^\top$. 
    Let $\mP_n
    =
    \mI_n - \frac{1}{n}\vone_n\vone_n^\top$ denote the centering matrix acting on the columns of a score matrix. Then, for
any $\mR\in\operatorname{O}(d)$, the following are equivalent: $
     \A(\mX\mR)=\A(\mX)$, and $
        \left(
        \mS_{\mB}(\mX\mR)-\mS_{\mB}(\mX)
    \right)\mP_n
    =
    \mZeros$.
Equivalently, we have
\begin{equation}
        \mX
    \left(
        \mR\mB\mR^\top-\mB
    \right)
    \mX^\top
    \mP_n
    =
    \mZeros .
\end{equation}
In row-wise notation, this means that for all query indices $i$ and all key indices $j,k$, $    \vx_i^\top\left(\mR\mB\mR^\top-\mB\right) (\vx_j-\vx_k)=0 .$
Thus attention-map invariance along an augmentation orbit only requires the row-wise score differences to be invariant. 
\end{lemma}

We next test whether trained models exhibit this scalar-score mechanism.
We train a vanilla Transformer on ModelNet10 with $16$ rotated augmentations per sample and evaluate its attention maps across the augmentation orbit of held-out inputs. Fig.~\ref{fig:rotation} shows representative attention maps for several inputs and their rotations. The maps remain nearly unchanged across rotations,  and the corresponding invariance errors are small.
 This behavior matches the activation-level signature predicted by Lemma~\ref{lem:orbit-wise-attention-invariance}: the learned score differences are approximately preserved under rotations of the input.

\begin{figure}[t]
    \centering
    
    \includegraphics[width=0.2\linewidth]{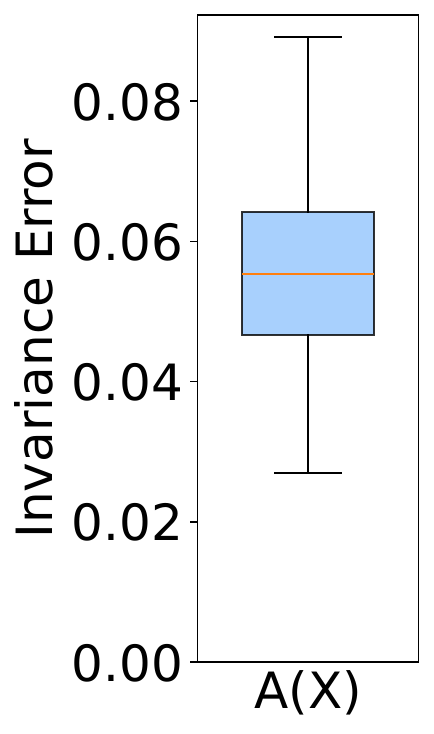}
    \hfill
    \includegraphics[width=0.75\linewidth]{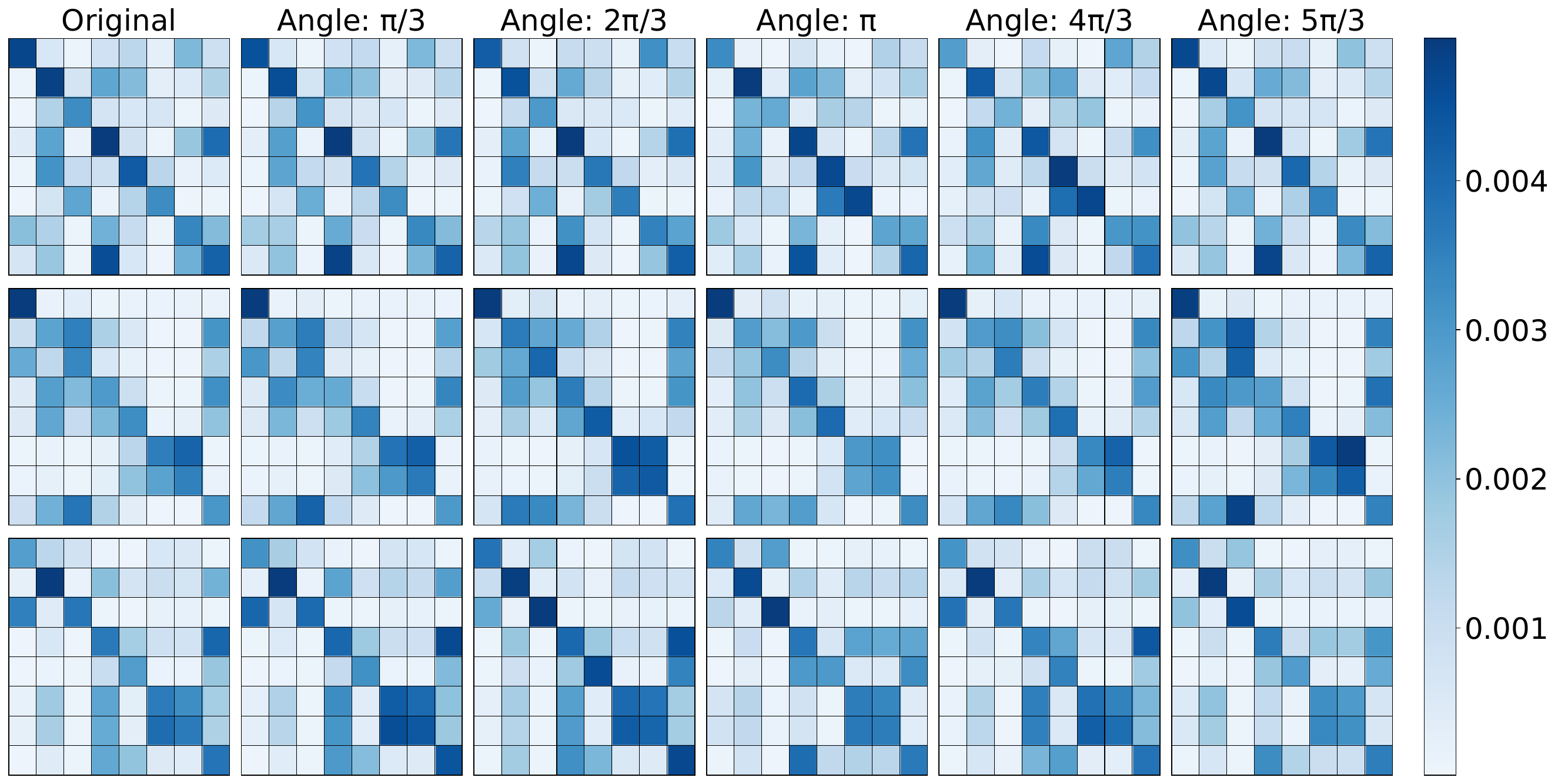}
        \vspace{-2mm}
     \caption{Attention maps of a full Transformer trained on ModelNet10 augmented with 16 rotational augmentations per sample remain highly stable across augmentations.
     (\textit{Left}) Invariance error of $\operatorname{A}$ across the entire test set. (\textit{Right}) Examples of the first $8\times8$  attention scores. Rows correspond to different samples, and columns correspond to rotations by the indicated angles.
     }
             \vspace{-1em}
    \label{fig:rotation}
\end{figure}

\subsection{Translation: constructing relative coordinates}

\begin{wrapfigure}[21]{r}{0.4\textwidth}
\vspace{-5 mm}
    \centering
        \includegraphics[width=\linewidth]{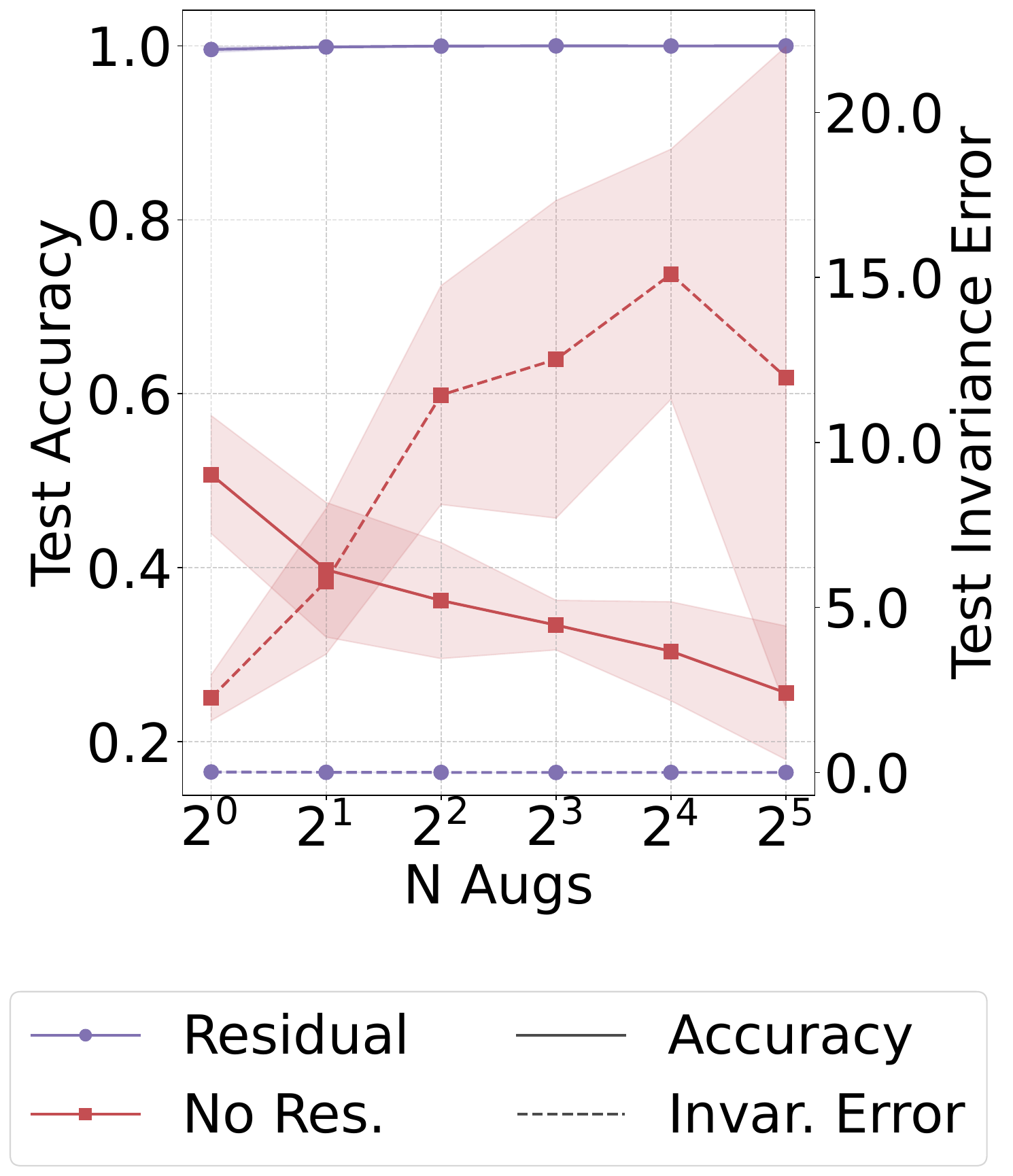}
            \vspace{-6mm}
     \caption{A model with 2 blocks of $\mX + \SA(\LN(\mX))$ (\emph{Residual}) can learn to be translation invariant (dashed) and solve the 3D Topology task (solid lines), but $\SA(\LN(\mX))$ (\emph{No Res.}) cannot.
     }
     
\vspace{-12 mm}
    \label{fig:translation-ablation}
\end{wrapfigure}

Several mechanisms can make a neural network invariant to translation. 
For example,  Villar et al. (2021) \cite{villar2021scalars} show that subtracting one of the points as a landmark to obtain relative coordinates, $f(\vx_1, \dots,\vx_n) = \Tilde{f}(\vx_2 -\vx_1, \dots,\vx_n -\vx_1)$, is a universal approximator when $\Tilde{f}$ is sufficiently expressive.
However, they do not test whether Transformers learn this mechanism.
Alternatively, Cordonnier et al. (2020) \cite{Cordonnier2020On} show that when using relative positional encodings, self-attention can become convolution-like by focusing attention on local neighborhoods.
 Since many forms of positional encoding are used in practice, we remain agnostic to a specific choice and omit positional encodings from our analysis.
A third option could be to center the input by setting $\mW_Q = \mW_K = \mZeros_{d \times d}$ and $\mW_V = - \mI_d$ in a self-attention block, which would deduct the mean from the residual stream.

\begin{figure}[t]
    \centering
    \begin{minipage}[c]{0.4\linewidth}
        \centering
        \includegraphics[width=\linewidth]{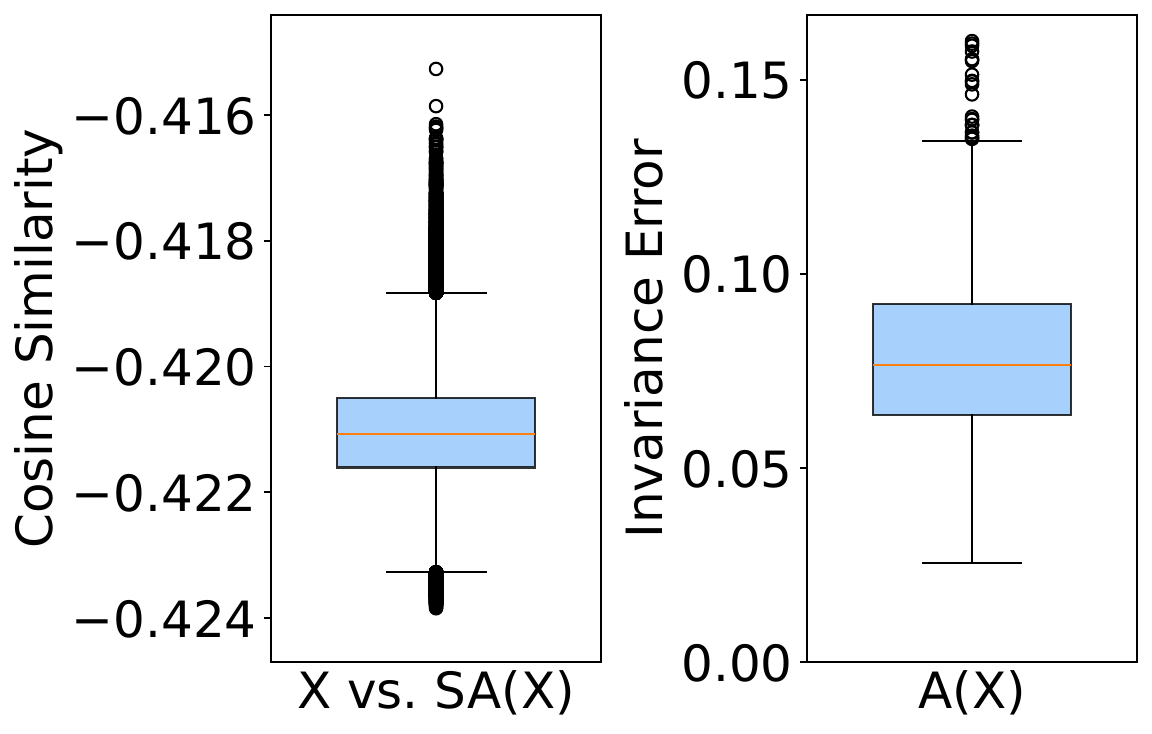}
    \end{minipage}
    \hfill
    \begin{minipage}[c]{0.55\linewidth}
        \centering
        \includegraphics[width=\linewidth]{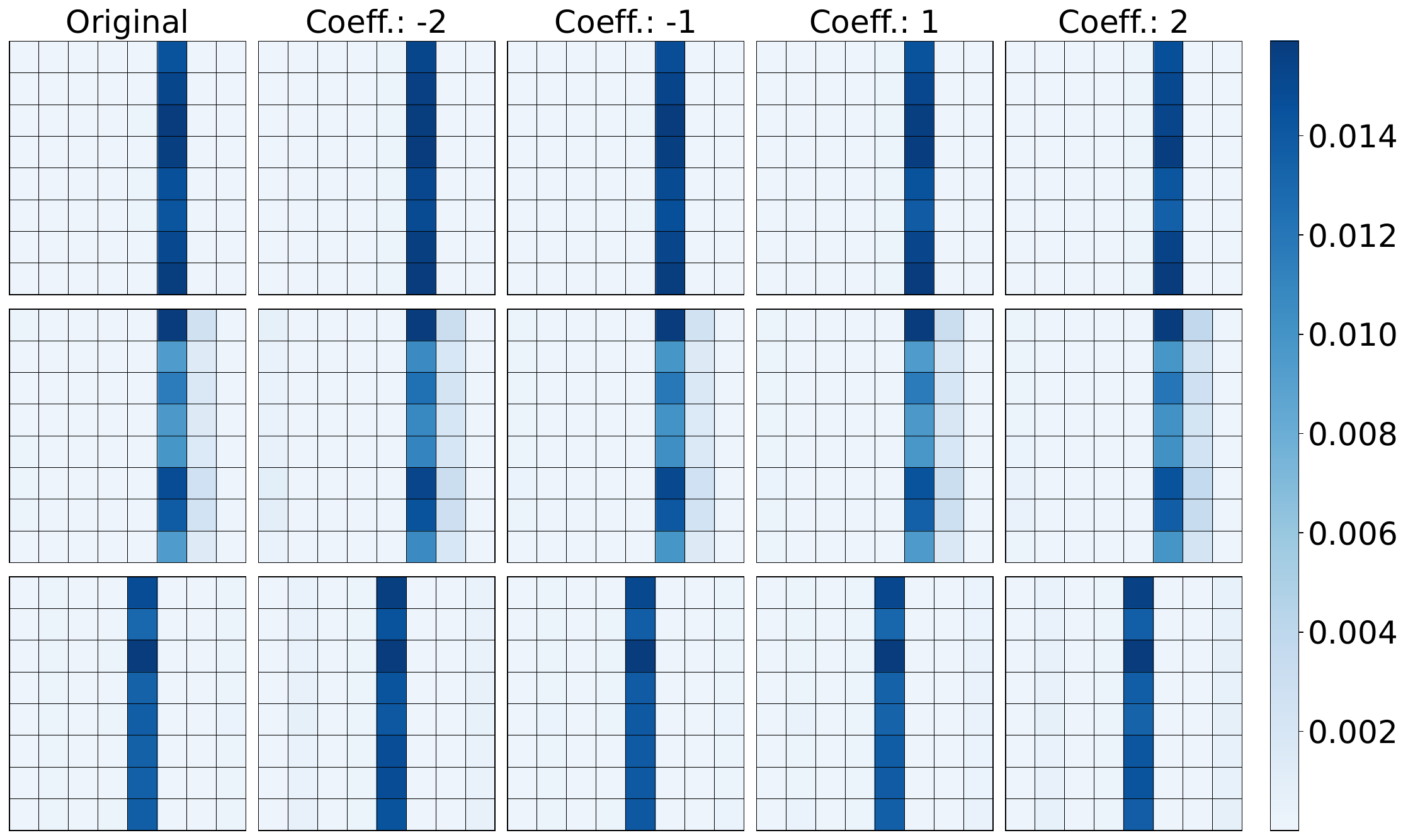}
    \end{minipage}
            \vspace{-2mm}
     \caption{Results for the first block of a full Transformer trained on ModelNet10 with 16 translation augmentations per sample. 
     (\textit{Left}) Cosine similarity between $\mX$ and $\SA(\mX)$ is strictly negative, indicating that the attention output has signs opposite to those of the residuals.
     (\textit{Middle}) The attention maps have low invariance error across the full test set.
     (\textit{Right}) Each row shows the first $8\times8$ attention scores of a sample and four of its translated augmentations spanning the training range. A translation coefficient of, e.g., 2 means the points are shifted by 2 in the positive direction for each axis. Samples from the same orbit often focus their attention on the same landmark point.
     }
    \label{fig:translation-mechanism}
\end{figure}

\begin{figure}[t]
    \centering
    \begin{minipage}[c]{0.45\linewidth}
        \centering
        \includegraphics[width=\linewidth]{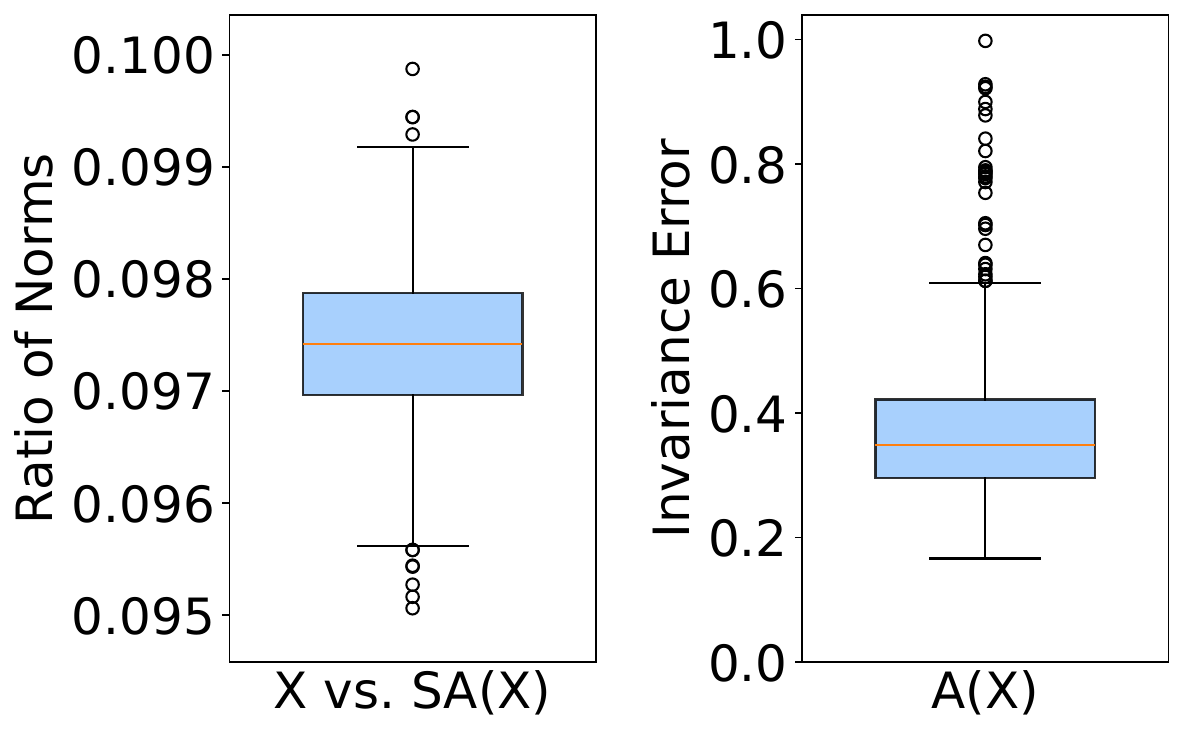}
    \end{minipage}
    \hfill
    \begin{minipage}[c]{0.5\linewidth}
        \centering
        \includegraphics[width=\linewidth]{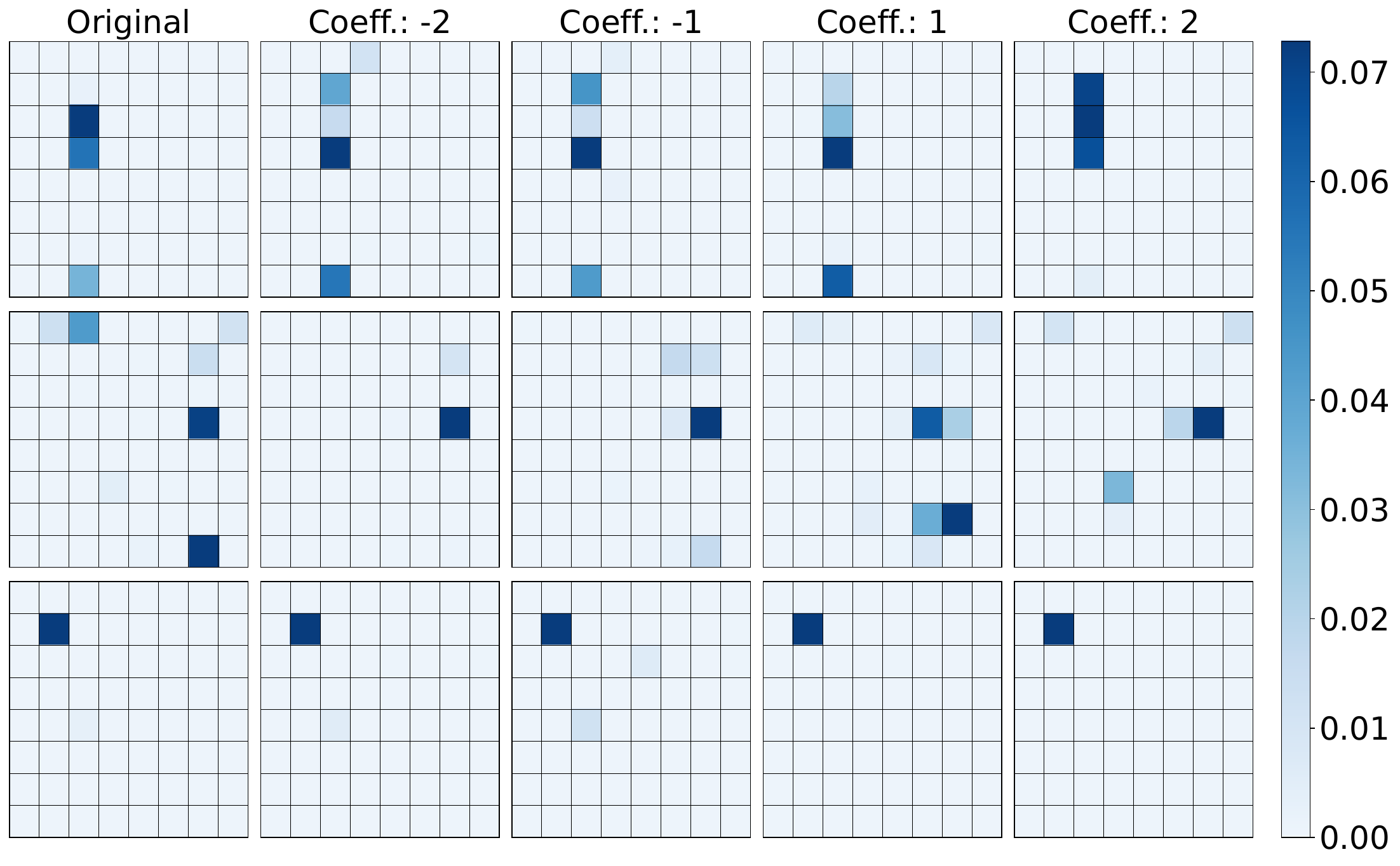}
    \end{minipage}
                \vspace{-2mm}
     \caption{Mechanism indicators for the second block of a Transformer trained on ModelNet10 with 16 scale augmentations. (\textit{Left}) 
     The feature vectors in the attention output have larger norms than those in the residual stream, reducing the residual non-invariance to approximately the noise level. 
     (\textit{Middle and Right}) Attention maps are mostly invariant under scalings.
     Rows correspond to different samples. 
     }
    \label{fig:scale-mechanism}
        \vspace{-1em}
\end{figure}

To identify whether any of the above suggestions matches the actual computation carried out by a Transformer trained with translations, we examine the attention maps $\A(\mX)$ for three different test samples and their augmentations in Fig.~\ref{fig:translation-mechanism} (\textit{right}). In each case, we see that all attention concentrates on a single point. This immediately rules out the dependence on local neighborhoods needed to implement a convolution, as well as the uniform attention for the mean computation. Instead, the pattern matches that of choosing a landmark.
We next proceed to investigate how the self-attention output interacts with the residual stream.
We compute the cosine similarity between $\mX$ and $\SA(\mX)$ and show its distribution over the entire test set in Fig.~\ref{fig:translation-mechanism} (\textit{left}).
 The cosine similarities are consistently negative, indicating that the self-attention output tends to oppose the residual stream. This behavior is consistent with the subtraction required by both the landmark- and mean-based mechanisms. In fact, both mechanisms are special cases of the following observation.

\begin{lemma}
\label{lemma:translation}
    Let $\operatorname{\Lambda}: \RR^{n \times d} \to \R^{n \times n}$ be translation invariant with row-stochastic outputs. Set $\mLambda = \operatorname{\Lambda}(\mX)$. Then $f(\mX) = g(\mX - \mLambda^\top \mX)$ is translation invariant for any function $g$.
    
\end{lemma}

So even when attention does not focus on a landmark, its invariance along with opposite directions of $\mX$ and $\operatorname{SA}(\mX)$ can ensure stable output under translation. 

One may note that the above formulation already requires attention maps to be invariant in the first place. However, it does not require them to be expressive. Meanwhile, the residual connection enables the landmark shift mechanism $f(\vx_1, \dots,\vx_n) = \Tilde{f}(\vx_2 -\vx_1, \dots,\vx_n -\vx_1)$, which is fully expressive among translation invariant functions \cite{villar2021scalars}. 
Hence, we probe with ablations whether the residual connections are needed to ensure both high test accuracy and low invariance loss.
Fig.~\ref{fig:translation-ablation} shows that, as Lemma \ref{lemma:translation} suggests, self-attention with residual connections achieves both of these properties simultaneously, while removing them breaks expressive invariance.

\subsection{Scale: LayerNorm and residual suppression}

\begin{wrapfigure}[18]{r}{0.42\textwidth}
     \vspace{-5mm}
    \centering
        \includegraphics[width=\linewidth]{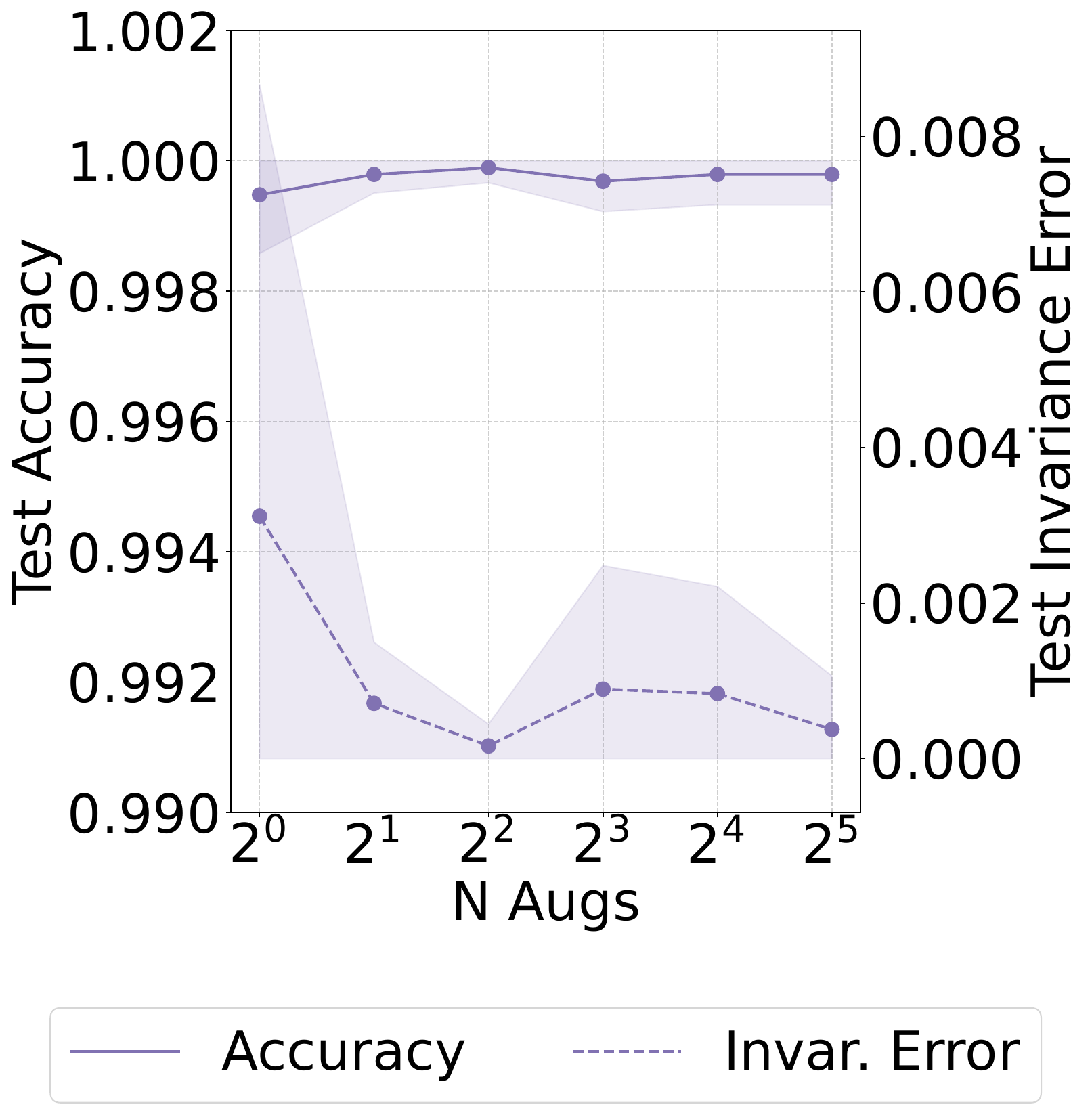}
            \vspace{-6mm}
     \caption{A model consisting only of $\SA(\LN(\mX))$ blocks can learn to be scale invariant and solve the 3D Topology task. 
     }
    \label{fig:scale-ablation}
\end{wrapfigure}

The LayerNorm $\LN$ is a natural candidate for learning scale invariance.
Without additive transformations like residual connections and bias terms, this extends to the entire computation for classification tasks.

\begin{lemma}
\label{lemma:scale}
    Let $\SA(\mX)=\A(\mX)\mX\mW_V$ be a single attention-only layer. 
    If for some $c>0$, $\A(c\mX)=\A(\mX)$, then $\SA(c\mX)=c\,\SA(\mX)$.
Hence, an attention-only stack whose attention maps are stable over the scale augmentation range is positively homogeneous $\Phi(c\mX)=c\,\Phi(\mX)$.
If the final prediction is made by taking an $\operatorname{argmax}$ over logits that are scaled by this same positive factor, then the predicted class is invariant to $c$.
\end{lemma}

Fig.~\ref{fig:scale-ablation} shows that a model with two layers of $\SA(\LN(\mX))$ easily becomes scale invariant. 
 In this setting, the bias in the input projection is the only source of invariance error. 
In a full Transformer, the residual stream introduces an additional source of non-invariance. 
Fig. \ref{fig:scale-mechanism} shows how the model compensates for this effect: the feature vectors in the attention output have larger norms than the corresponding vectors in the residual stream, which reduces the influence of the residual on the block output.

\section{From invariance to equivariance}
\label{sec:equivariance}

\begin{wrapfigure}[31]{r}{0.4\textwidth}
\vspace{-5mm}
    \centering
    \includegraphics[width=\linewidth]{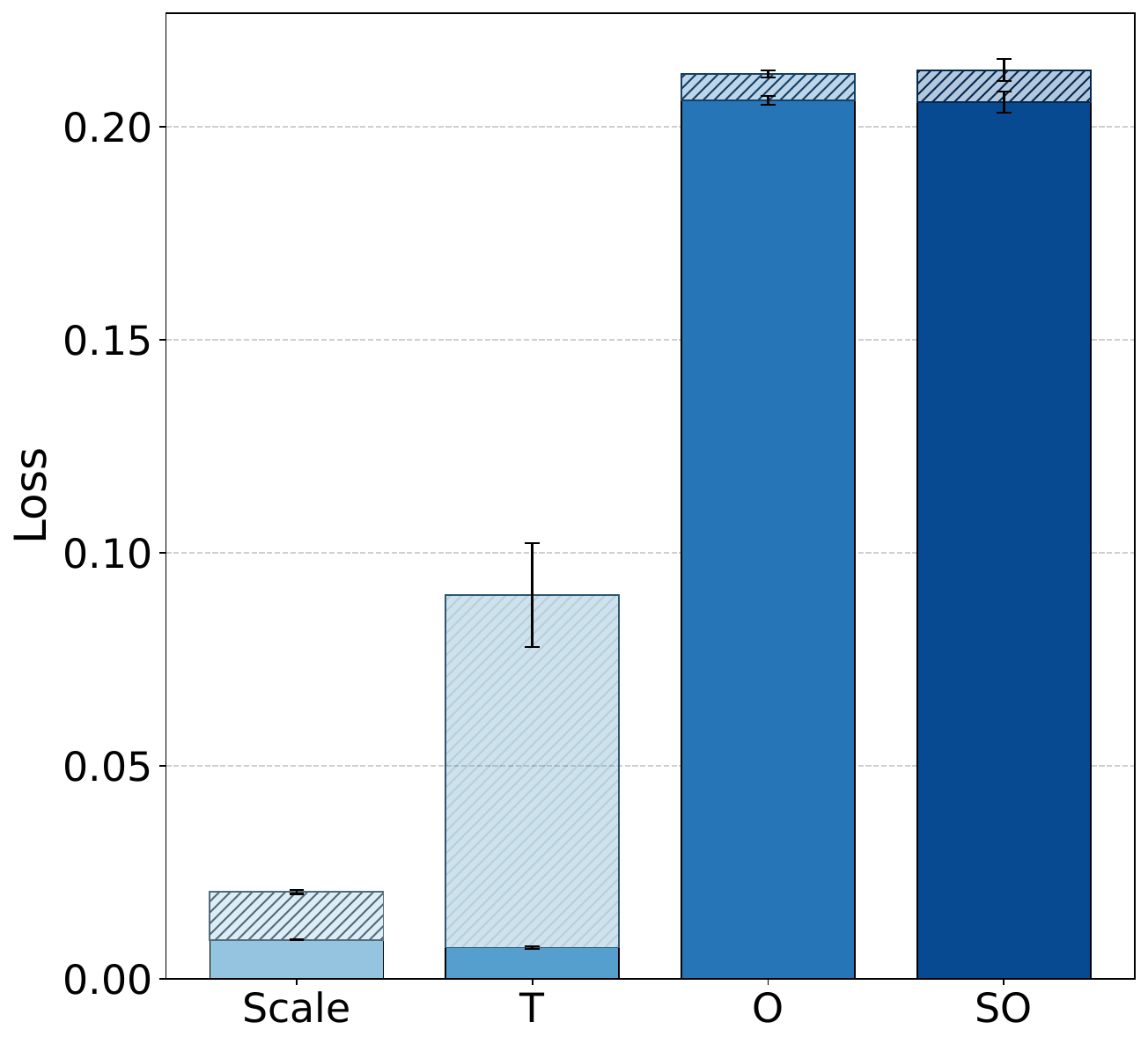}
    
    \includegraphics[width=\linewidth]{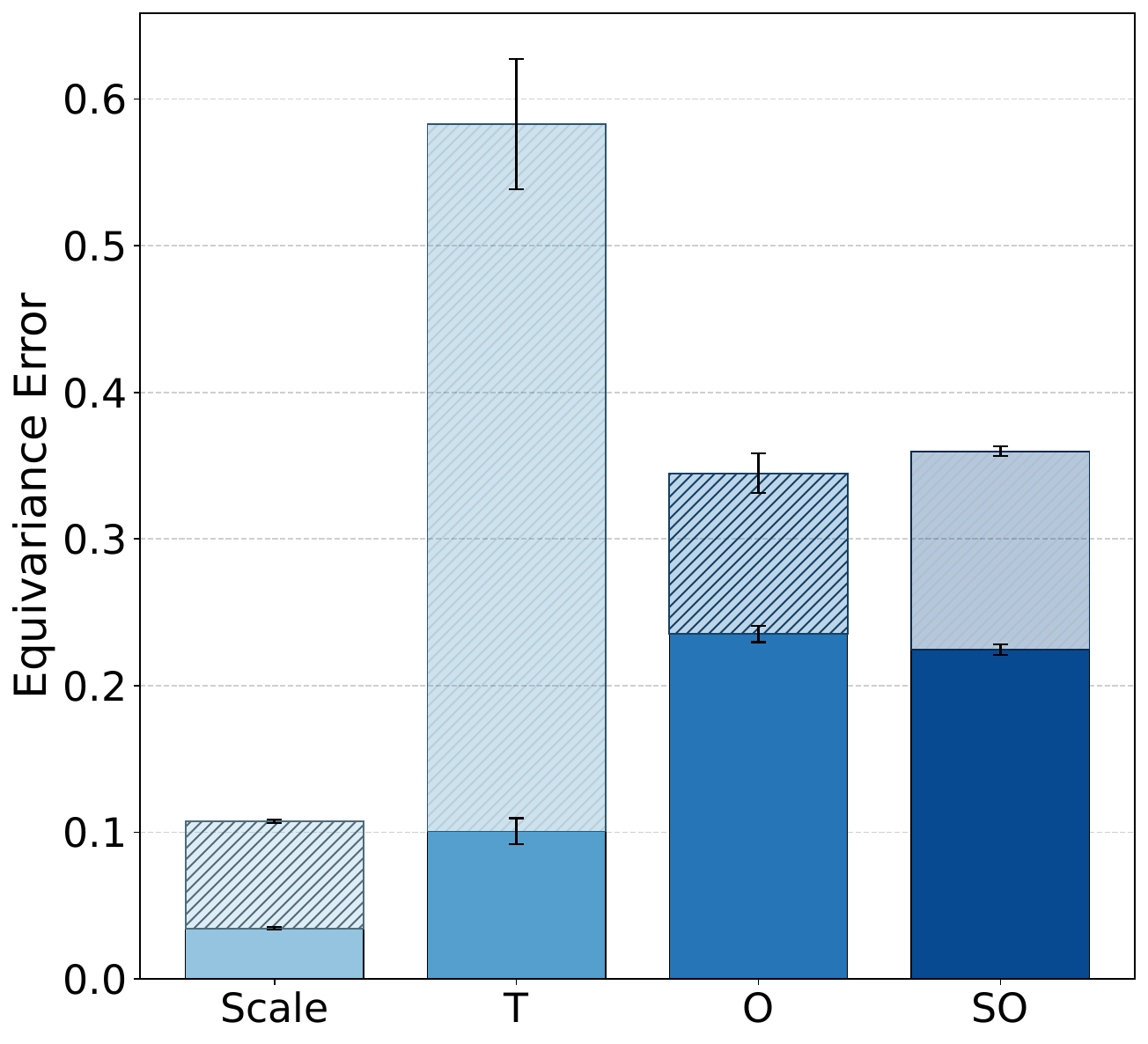}
    \includegraphics[width=0.75\linewidth]{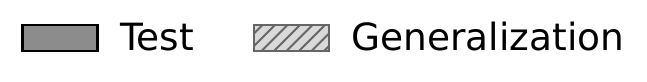}
                \vspace{-1.4 mm}
    \caption{
    Task losses and equivariance errors of Transformers, trained with 16 augmentations on ShapeNet tasks, are reasonably low. 
    Note that $\operatorname{Scale}$ and $\operatorname{T}$ consider a different task than $\operatorname{O}$ and $\operatorname{SO}$.}
    \label{fig:shapenet}
\end{wrapfigure}

Our previous analysis focused on how Transformers can learn approximately invariant predictions from limited data augmentation. 
However, many real-world tasks require the more general notion of \emph{equivariance}: when the input transforms under a group action, the output should transform in a corresponding way.
Formally, for all 
$g \in G, x \in \mathcal{X}$, we have $\Psi(\rho_\mathcal{X}(g) x) = \rho_\mathcal{Y}(g) \Psi(x)$, where $\rho_\mathcal{X}$ and $\rho_\mathcal{Y}$ are linear representations of the group on the input and output spaces, respectively.
As a result, we also investigate whether attention-based structures related to the previously identified invariance mechanisms can also emerge in equivariant vector-valued tasks.

\paragraph{Experimental setup.}
To this end, we conduct additional experiments using ShapeNet \cite{yi2016scalable}, another common point cloud dataset that is annotated by the surface-normal vector for each point of the input, as well as segmentation information. Predicting surface normals poses an $\operatorname{O}(3)$- and $\operatorname{SO}(3)$-equivariant task. To construct $\operatorname{T}(3)$ and $\operatorname{Scale}(3)$-equivariant targets, we let each point predict the centroid of the part it belongs to in the segmentation map.
Vanilla Transformer models are trained using data augmentation from each respective base angle-preserving group under the same OOD setting as in \S\ref{sec:ood}. In Fig.~\ref{fig:shapenet}, we report task loss and equivariance error (see \S\ref {app:experiment_details_invariance} for their definitions).
The models achieve low task loss and mostly low equivariance error both in and out of distribution, suggesting that they learn approximately equivariant representations.

\paragraph{Potential mechanism for $\operatorname{O}(3)$-equivariance.} We next investigate $\operatorname{O}(3)$-equivariance as a representative case. 
To identify a possible mechanism,  we build on Proposition 4 of  Villar et al. (2021)  \cite{villar2021scalars}, which establishes the following:
\begin{lemma}[\cite{villar2021scalars}]
    If $h$ is an $\operatorname{O}(d)$-equivariant vector function of $n$ vector inputs $\vv_1,\vv_2,..,\vv_n$, then $n$ $\operatorname{O}(d)$-invariant scalar functions $f_t(\cdot)$ exist such that $h(\vv_1,\vv_2,..,\vv_n) = \sum_{t=1}^1 f_t(\vv_1,\vv_2,..,\vv_n) \vv_t$. 
\end{lemma}
Self-attention has an analogous weighted-sum structure: $\operatorname{SA}(\mX) = \sum_j \operatorname{A}(\mX)_{ij} (\mX \mW_V)_j$. Therefore, we can see that if \textit{(i)} $\operatorname{A}(\mX \mR) = \operatorname{A}(\mX)$ and \textit{(ii)} $\mX\mR \mW_V = \mX \mW_V \mR$, then we have $\operatorname{SA}(\mX\mR)= \operatorname{A}(\mX\mR) \mX\mR \mW_V = \operatorname{A}(\mX\mR) \mX \mW_V \mR = \operatorname{SA}(\mX)\mR$. 
That is, self-attention is equivariant whenever its attention map is invariant and its value projection transforms equivariantly. 
Requiring \textit{(ii)} to hold for every $\mX$ and every $\mR \in \operatorname{O}(d)$ is equivalent to enforcing $\mW_V$ to commute with the standard action of $\operatorname{O}(d)$. This would restrict to an isotropic map, namely a scalar multiple of the identity.

\begin{wrapfigure}[17]{r}{0.4\textwidth}
     \vspace{-4.5mm}
    \centering
        \includegraphics[width=\linewidth]{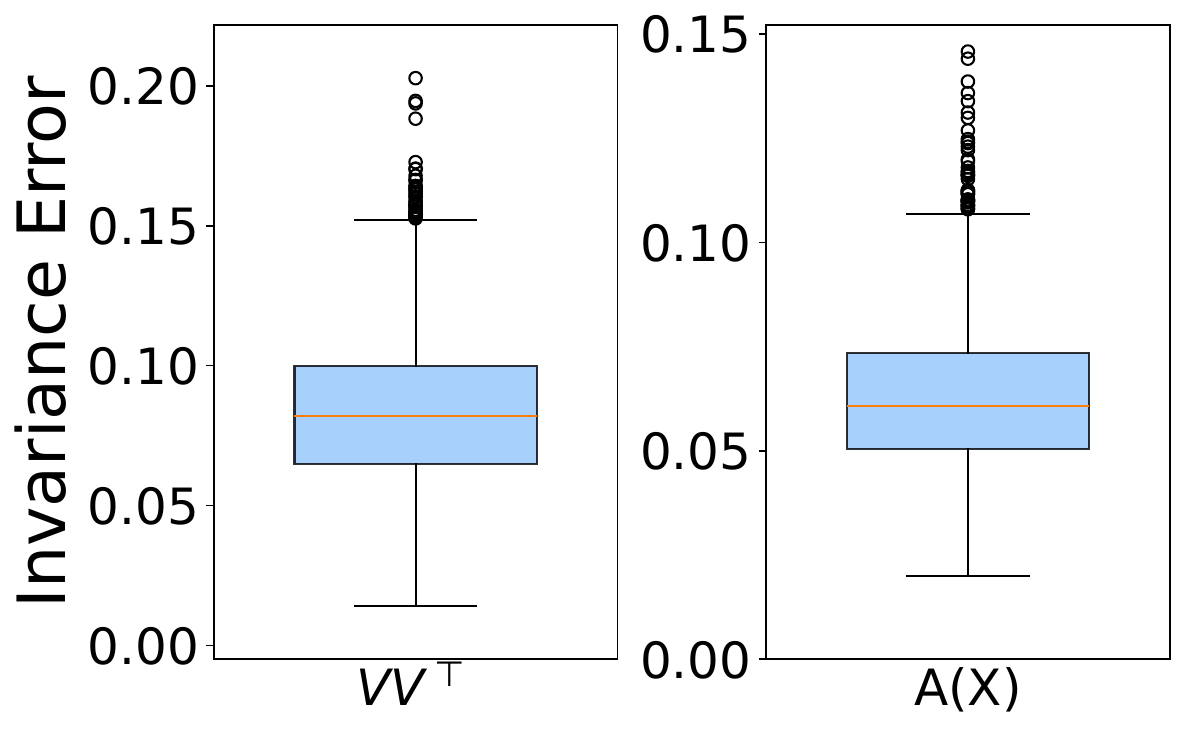}
                    \vspace{-4mm}
     \caption
     {
     Results for the first layer of a Transformer trained on ShapeNet with $16$ $\operatorname{O}(3)$ augmentations per sample.
     On the test set, the value representations are approximately $\operatorname{O}(3)$-equivariant (\textit{left}) and the attention maps are approximately $\operatorname{O}(3)$-invariant (\textit{right}).
     }
     \vspace{-10mm}
    \label{fig:mech_equiv}
\end{wrapfigure}

We first test condition \textit{(i)}. Fig.~\ref{fig:mech_equiv} (\textit{right}) reports the invariance error of the attention map $\operatorname{A}$ for a trained Transformer. The low error suggests that the attention map is approximately invariant and therefore approximately satisfies condition \textit{(i)}.
In contrast, the learned value projection is \textit{not} globally isotropic. The strongest form of condition \textit{(ii)}, i.e., exact commutation with every $\mR\in\operatorname{O}(d)$ over the full ambient input space, is \textit{not} reflected in the learned parameters. Instead, we hypothesize that the required equivariance may hold only approximately on the data support and along the augmentation orbits.
We test this hypothesis directly on the test set. Measuring the equivariance of $\mV=\mX\mW_V$ requires specifying how $\operatorname{O}(3)$ acts in the higher-dimensional latent space. We instead measure the invariance of the Gram matrix $\mV\mV^\top$ as a proxy. If $\mV$ transforms equivariantly under an orthogonal action, then $\mV\mV^\top$ remains invariant. Fig.~\ref{fig:mech_equiv} (\textit{left}) shows that the corresponding invariance error is small, which supports approximate equivariance of $\mV$ on the data support and along the augmentation orbits.
We provide further investigation of the remaining architectural components and symmetry groups in \S\ref{app:equivariance_results}. 
Overall, the Transformers trained for equivariant prediction exhibit mechanisms similar to those identified in our invariance analysis. This suggests that invariant components, e.g., those studied in \S\ref{sec:mechanisms}, can be used as building blocks for equivariance.

\section{Conclusions}
\label{sec:conclusions}

In this paper, we study how a vanilla Transformer can learn symmetries from limited data augmentation on point cloud data. 
Across the transformation groups we consider, Transformers learn angle-preserving symmetries more reliably than non-angle-preserving ones.
Among the angle-preserving groups, they learn the base symmetries (i.e., translation, rotation, and scale) more efficiently than compositions of these transformations.
Following those observations, we perform structural analyses of the trained Transformer models and identify mechanisms that implement invariance to each base angle-preserving symmetry.
We also extend this analysis from invariant to equivariant prediction, suggesting that the invariant mechanism identified in our analysis can also serve as building blocks for equivariant computation.
Our observations suggest that data augmentation should not be treated as a uniform substitute for architectural equivariance, and highlight the persisting need to understand when \textit{soft} equivariance is sufficient and when explicit equivariant structure remains necessary.
In addition, they also suggest that the recent success of Transformer-based soft-equivariant approaches could be partly attributed to the intrinsic alignment of that architecture with base angle-preserving symmetries, which happen to be among the most prevalent ones in scientific domains.

\noindent\textbf{Limitation discussion.} This work focuses on small vanilla Transformers and point cloud dataset to isolate the interplay of the basic architecture with the considered symmetries.
An important direction for future work is to test whether similar mechanisms arise in images, meshes, graphs, molecular structures, and other forms of geometric or scientific data.
In addition, studying larger datasets and more elaborate scientific tasks would help determine whether the same mechanisms persist when the data contain richer structure, noise, and interacting symmetries.
Lastly, the extension of our analysis to equivariance remains preliminary. Future work could identify mechanisms for additional equivariant tasks and symmetry groups, characterize how group actions are represented in latent spaces, and test the causal role of the observed mechanisms through targeted interventions.

\newpage

\section*{Acknowledgments}
We thank Leibniz Supercomputing Centre (LRZ) for providing computational resources and the anonymous reviewers for their constructive comments. 
We also thank Vincent Bürgin for feedback on an
earlier version of this manuscript. 
ESE acknowledges support from a PhD fellowship of the Munich Center for Machine Learning (MCML).
YEL acknowledges support from the Schmidt Futures Israeli Women’s Postdoctoral Award and the Viterbi Fellowship from the Faculty of Electrical and Computer Engineering at the Technion.
This research was supported by an Alexander von Humboldt Professorship.

\bibliographystyle{unsrt}
\bibliography{bib}


\newpage
\appendix

\begin{center}
    \LARGE  \textbf{Supplementary Material}    
\end{center}

\addcontentsline{toc}{section}{Appendices}

\startcontents[appendices]
\printcontents[appendices]{l}{1}{\setcounter{tocdepth}{2}}

\newpage

\section{Extended related work}
\label{app:extended_related_work}

In this section, we expand on the surrounding literature presented in \S\ref{sec:introduction} and \S\ref{sec:related_work} by discussing additional related works.

\subsection{The shift from hard equivariance towards data augmentation}

Recently, a trend has emerged that is moving away from the design and implementation of hard equivariant architectures especially when working with generative tasks in large-scale data regimes. The main motivation behind this development has been the empirical observation that data augmentation is able to provide a similar performance with a much simpler implementation and a reduced computational cost in such scenarios.
Most notable is the example of the departure of AlphaFold3 \cite{abramson2024accurate} from the constrained architectural equivariance of AlphaFold2 \cite{jumper2021highly}, but many other works have recently followed the same path, e.g., MCFs \cite{wang2024swallowing}, TransIP \cite{elhag2025learning}, Zatom-1 \cite{morehead2026zatom}, or TABASCO \cite{vonessen2025tabasco}.
All of these works have two main aspects in common: (1) they propose Transformer-based models, (2) they work with real-world scientific tasks (e.g., molecule and material generation, interatomic potential prediction) which are naturally restricted to a reduced set of possible symmetries, including the angle-preserving ones that we show the Transformer efficiently learns to be invariant to. 

\subsection{Approximate equivariance}

A limitation of equivariant architectures arises in their application to real-world data, which often deviates from strict mathematical symmetries due to, for instance, noisy or biased measurements and symmetry-breaking effects \cite{manolache2026learning, wang2022approximately}. In such cases, strictly enforcing equivariance can lead to a degradation of the training dynamics and predictive performance of the models. As a result, approximately equivariant networks emerge as a more fitting alternative for such scenarios, where they are still biased towards preserving the symmetry but are not strictly restricted to do so. 

Data augmentation has become the most predominant technique to encourage approximate equivariance in many different practical domains: natural language processing \cite{li2022data, pellicer2023data}, reinforcement learning \cite{ma2025comprehensive}, computer vision \cite{shorten2019survey}, or graph learning \cite{zhao2022graph}. 
A theoretical understanding of the role of data augmentation has been explored from multiple directions: its group-theoretic foundations \cite{chen2020group}, its interpretation as a regularizer \cite{lin2024good, soleymani2025robust}, or the effect it has on the training dynamics of neural networks \cite{shen2022data}.
There have also been multiple works that develop and investigate a kernel-based theory of augmentation \cite{dao2019kernel, patil2023generalized, mei2021learning}.

Other examples of approximate equivariance methods include constrained optimization approaches \cite{manolache2026learning,pertigkiozoglou2024improving}, the relaxed weight-sharing and weight-tying strategy from \cite{wang2022approximately}, the graph coarsening method for GNNs described in \cite{huang2023approximately}, the multi-task learning approach REMUL \cite{elhagRelaxedEquivarianceMultitask2024}, or the soft prior approach from \cite{finzi2021residual}.

\subsection{Learning symmetries with neural networks}

Related to our work, there have been some papers in the literature that explore the ability of different machine learning models to learn specific symmetries.
For instance, FASTColor \cite{villadamigo2025fastcolor} empirically shows that a Transformer learns to become invariant to the $\operatorname{SO}(2)$ symmetries in a particle physics task while avoiding invariance to other symmetries like $\operatorname{SL}(4)$, which aligns with our findings in this paper.
Also related is \cite{neyshabur2020towards}, which shows that fully-connected networks can learn to implement convolution structures when trained with an appropriate regularizer.
Different methods and metrics have been proposed to measure the learned equivariance in neural networks \cite{gruverlie, zhang2019making, karras2021alias, bouchacourt2021grounding}.
However, each of them comes with certain drawbacks, such as requiring various ad-hoc decisions or being designed with only certain transformations and types of data in mind.
Most related to our work is \cite{gruverlie}, which shows that Transformers can learn to be more equivariant than convolutional neural networks after training with data augmentation.

\section{Experimental details}
\label{app:experiment_details}

In this section, we present additional details about the experimental setup and empirical evaluations presented in \S\ref{sec:motivation} and \S\ref{sec:ood}.
All experiments were implemented in Pytorch and executed on NVIDIA A100 and H100 GPUs using CUDA version 12.2. We used the Adam optimizer with a learning rate of $0.001$, keeping the rest of its default hyperparameter values. The embedding dimension as well as the hidden dimension of the MLP in all experiments is 64. All models contain two blocks using single-head self-attention and a $2$-layer MLP.

\subsection{Generating symmetry augmentations}
\label{app:symmetry_augmentations}

We characterize the symmetry groups considered in this work in Table~\ref{tab:symmetry_groups}.
There are multiple ways of generating augmentations based on the action of these groups. In this paper, we follow an approach based on Lie group theory, which we describe in more detail below.

\subsubsection{Preliminaries}

First, we briefly review some relevant background regarding Matrix Lie groups and Lie algebras.
A matrix Lie group $G$ is a subgroup of the general linear group $\operatorname{GL}(d, \mathbb{R})$ that is also a smooth manifold such that multiplication and inversion are both smooth maps \citep{baker2003matrix, erdmann2006introduction}. 
The associated Lie algebra $\mathfrak{g} = T_{\mathrm{Id}} G$ is the tangent space of the group $G$ at the identity element. 
Each element $A\in \mathfrak{g}$ is an infinitesimal generator of one-parameter subgroups $t \mapsto \exp(tA)$ for $t\in \mathbb{R}$. 
Let $\{L_i\}_{i=1}^{\dim \mathfrak{g}}$ be a basis of $\mathfrak{g}$.
For matrix Lie groups, the exponential map $\exp:\mathfrak{g}\to G$ is given by the matrix exponential, which is a local diffeomorphism in a neighborhood of the identity.
If $G$ is connected and compact, $\exp$ is surjective \citep{kirillov2008introduction, hall2013lie}, hence every $g\in G$ can be written as $g=\exp(\sum_i w_i L_i)$ with $w_i\in\mathbb{R}$. 

Second, we recall that we are working with point cloud data, i.e., $x \in \mathcal{X} \subseteq \mathbb{R}^{n \times d}$. As a result, we want to generate augmentations that act on the entire point cloud. For computational simplicity and efficiency, we achieve this by sampling point-wise actions and mapping them across the entire corresponding point cloud. 
Thus, in the following we focus on explaining the generation of such point-wise augmentations.

\subsubsection{Parameterizing augmentations using Lie groups}

In order to represent the generators of the Lie Group that we want to sample from in a compact, consistent, and computationally efficient way, we adopt a reparameterization approach \citep{price2003useful, kingma2013auto} adapted to Lie groups \citep{falorsi2019reparameterizing}, which samples group elements through the associated Lie algebra. 

Let $\mathfrak{g}$ denote the Lie algebra associated with a connected matrix Lie group acting on the representation space of $\mathcal{X}$.
We parameterize the sampling of group elements through a standard Lie algebra and randomly sample coefficients $w=(w_1,\ldots,w_C)\in\mathbb{R}^C$ from a distribution $P(\gamma)$ with parameters $\gamma_0, \gamma_1$.
Specifically, we consider a uniform distribution $P(\gamma) = \mathcal{U}[-\gamma_0, \gamma_1]$.
A group element is obtained by mapping the resulting Lie algebra element to the group via the matrix exponential, given by 
\begin{equation}
    g = \exp \left[\sum\nolimits_{i=1}^C w_i L_i \right], \  w_i \sim P(\gamma).
    \label{eq:group_sampling}
\end{equation}

Given an input $x_i$, we draw $K$ independent group elements $g_{i,2},\ldots,g_{i,K+1}$ according to Eq.~\ref{eq:group_sampling} and form the corresponding augmentations 
\begin{equation}
    x_{i,j} := g_{i,j}\,x_i, \qquad  j \in [2, \ldots, K+1], 
\end{equation}
by multiplying with the original input $x_{i,1}:=x_i$.

\subsubsection{Examples}

In order to illustrate how such a representation of the group looks like for the symmetries that we consider, we present a standard basis for rotation and translation groups in Table~\ref{tab:ex_bases}.

\begin{table}
\centering
\begin{tabular}{lc}
    \toprule
    Symmetry group & Lie algebra basis \\ 
    \midrule
     $\operatorname{SO}(3)$ & $\left\{
        \begin{bmatrix}
        0 & 1 & 0\\
        -1 & 0 & 0\\
        0 & 0 & 0\\
        \end{bmatrix} 
        , \ 
        \begin{bmatrix}
        0 & 0 & 1\\
        0 & 0 & 0\\
        -1 & 0 & 0\\
        \end{bmatrix} 
        , \ 
        \begin{bmatrix}
        0 & 0 & 0\\
        0 & 0 & 1\\
        0 & -1 & 0\\
        \end{bmatrix} 
     \right\}$ \\[20pt]
     $\operatorname{T}(3)$ & $\left\{
        \begin{bmatrix}
        0 & 0 & 0 & 1\\
        0 & 0 & 0 & 0\\
        0 & 0 & 0 & 0\\
        0 & 0 & 0 & 0\\
        \end{bmatrix} 
        , \
        \begin{bmatrix}
        0 & 0 & 0 & 0\\
        0 & 0 & 0 & 1\\
        0 & 0 & 0 & 0\\
        0 & 0 & 0 & 0\\
        \end{bmatrix} 
        , \
        \begin{bmatrix}
        0 & 0 & 0 & 0\\
        0 & 0 & 0 & 0\\
        0 & 0 & 0 & 1\\
        0 & 0 & 0 & 0\\
        \end{bmatrix} 
    \right\}$
\\
    \bottomrule
\end{tabular}
\vspace{1em}
\caption{Example of Lie algebra bases for the rotation $\operatorname{SO}(3)$ and translation groups $\operatorname{T}(3)$.}
\label{tab:ex_bases}
\end{table}

Note that when working with affine transformations, such as translations, we first embed the point representations into homogeneous coordinates. After computing the augmentations, we project the points back to the original space.

\subsubsection{Additional remarks}

We should note that this symmetry augmentation implementation cannot guarantee full coverage of all the considered groups (e.g., $\operatorname{SL}(n), \operatorname{GL}(n)$), since it is restricted to sampling group elements from the connected component of the identity of the Lie group and the map is not always surjective for non-compact groups. 
However, it is still adequate for the purpose of our analysis due to a number of reasons.
First, it allows for a general, uniform and consistent implementation of the augmentation across all symmetry groups, which also improves the ease of their comparability.
In addition, this implementation is more computationally efficient than dedicated alternatives, which allows for a more extensive and comprehensive analysis.
Moreover, given that the goal of our work is to explore the ability of Transformers to learn different symmetries, focusing on local continuous transformations around the identity is a natural choice. It allows a grounded analysis that corresponds to a realistic setting in which the transformations create meaningful invariance targets while avoiding extreme transformations that could destroy the underlying semantic information of the inputs.
Lastly, the limitations of this implementation do not affect the main groups that we focus on in the paper, i.e., angle-preserving symmetries like $\operatorname{SO}(n)$, $\operatorname{T}(n)$, or $\operatorname{Scale}(n)$, all of which are fully covered by this parameterization, since the exponential maps are surjective onto their connected components.
In the case of $\operatorname{O}(n)$, we circumvent the restriction to the identity connected component by multiplying elements sampled from $\operatorname{SO}(n)$ by a randomly sampled reflection mask.

\subsubsection{Hardness of compositional groups}

The mechanistic analysis presented in \S\ref{sec:mechanisms} focuses on the base groups $\operatorname{SO}(3)/\operatorname{O}(3)$, $\operatorname{T}(3)$, and $\operatorname{Scale}(3)$. We selected these groups for our analysis because they showed the strongest finite-augmentation performance and clearest OOD generalization, making them the most suitable for searching for stable internal mechanisms.

When considering groups that compose them, there is an important distinction between representational compatibility and learnability from finite joint augmentations. At the level of exact functions, invariance to the generating base groups is sufficient for invariance to their composition. For example, if $f$ is exactly invariant to both translations and rotations, then for any $g=tR \in SE(3)$, $f(gx)=f(t(Rx))=f(Rx)=f(x)$. The same argument extends to $\operatorname{Sim}(3)$ when scale invariance is also present. The observed difficulty is more plausibly associated with finite sampling and optimization. We currently see three related explanations.

\begin{enumerate}
    \item \textit{Higher-dimensional augmentation coverage}

In three dimensions, translations and rotations each have three continuous degrees of freedom and uniform scaling has one, whereas $\operatorname{SE}(3)$ has six and $\operatorname{Sim}(3)$ has seven. For a fixed number of augmented views per sample, random samples provide substantially sparser coverage of the larger joint transformation space.

Moreover, when augmentations are sampled directly from a compositional group, almost every transformed example combines multiple factors. The model receives relatively little direct supervision corresponding to a pure translation, pure rotation, or pure scaling direction. By contrast, a model trained on a base group receives its entire augmentation budget along that factor.

    \item \textit{Simultaneous constraints on shared components}

The mechanisms identified for the base groups use overlapping architectural components but impose different functional requirements:

    \begin{itemize}
        \item rotation and reflection favor attention-score differences that are stable under orthogonal transformations
        \item translation favors an invariant attention rule followed by subtraction of a selected or weighted reference through the residual stream
        \item scale favors stable attention, positive homogeneity, LayerNorm, and suppression of non-homogeneous residual or bias contributions.
    \end{itemize}

Learning $\operatorname{SE}(3)$ or $\operatorname{Sim}(3)$ requires these conditions to hold simultaneously along joint transformation orbits. Although the mechanisms are compatible in principle, their simultaneous optimization within the same attention and residual pathways may be more difficult than learning each one separately.

    \item \textit{Approximation error under composition}

Exact invariance to each subgroup composes exactly, but finite-data training produces only approximate invariance. Errors associated with the constituent transformations can therefore compound when several transformations are applied jointly. A model that is moderately stable to rotation and moderately stable to translation need not be equally stable to arbitrary roto-translations unless both behaviors are learned with sufficient uniformity.
\end{enumerate}

Our new multi-head experiments in \S\ref{app:additional_components} provide preliminary evidence relevant to the optimization hypothesis. Increasing the number of heads does not materially improve performance for the compositional groups, and we do not observe a clear specialization in which distinct heads separately implement translation, rotation, and scale invariance. This suggests that simply providing additional parallel attention heads is not sufficient to resolve the joint learning problem. 

\subsection{Measuring learned equivariance}
\label{app:experiment_details_invariance}

We assess equivariance by measuring the consistency of model predictions under transformations sampled from the ground-truth symmetry group. Specifically, we compute
\begin{equation}
    \dfrac{1}{N K} \sum_{i=1}^N \sum_{j=1}^K \ell \left(\rho_\mathcal{Y}(g_j) (\Psi(x_i)), \Psi(\rho_\mathcal{X} (g_j)(x_i)) \right), 
\end{equation}
which contrasts the transformed predictions of the original samples and those of the augmented samples, averaged over $K$ sampled group elements and $N$ data points.
In the case of regression tasks and for attention maps, $\ell$ is the L1-norm of the difference, while for classification tasks it corresponds to the KL divergence applied to the probability distributions outputted by the model. 
The representations $\rho_{\mathcal{X}}:G\to\operatorname{GL}(n)$ and $\rho_{\mathcal{Y}}:G\to\operatorname{GL}(m)$ specify linear actions of the symmetry group $G$ on $\mathcal{X} \subseteq \mathbb{R}^n$ and $\mathcal{Y} \subseteq \mathbb{R}^m$, respectively. 
When measuring invariance, as is the case in the majority of our experiments, we set $\rho_\mathcal{Y} = Id$.

In the classification setting, the invariance condition of direct task relevance is equality of predictive distributions. We used KL divergence since it operates directly on the probability simplex, is zero exactly when the two predictive distributions coincide, and weights discrepancies according to the probability mass assigned by the reference prediction. It also ignores a common additive shift of all logits, which leaves the predictive distribution unchanged. By contrast, a raw-logit L2 distance is a stricter, representation-dependent criterion: it can be nonzero even when the two softmax distributions are identical, and it measures changes in logit parameterization in addition to changes in the model's prediction. This makes KL a natural primary measure of task-level predictive invariance.

Nevertheless, for completeness, we study three complementary metrics: symmetric KL divergence, Jensen-Shannon divergence, and the logit-invariance metric of \cite{moskalev2023genuine}. The last metric is an L2-type alternative based on the squared distance between the logits.

Across the base symmetries on Topology 3D, these metrics give consistent qualitative conclusions as can be observed in Figure~\ref{fig:topology_metrics}. 
Their correlations with the original KL measure range from moderate to near-perfect: agreement is strongest for Scale, and weakest for Translation. The weaker agreement for Translation is consistent with it being the most difficult OOD case in our experiments. Thus, the main ordering of the symmetries and the mechanistic conclusions are not artifacts of the choice of invariance metric.

\begin{figure}[h!]
    \centering
    \includegraphics[width=0.75\linewidth]{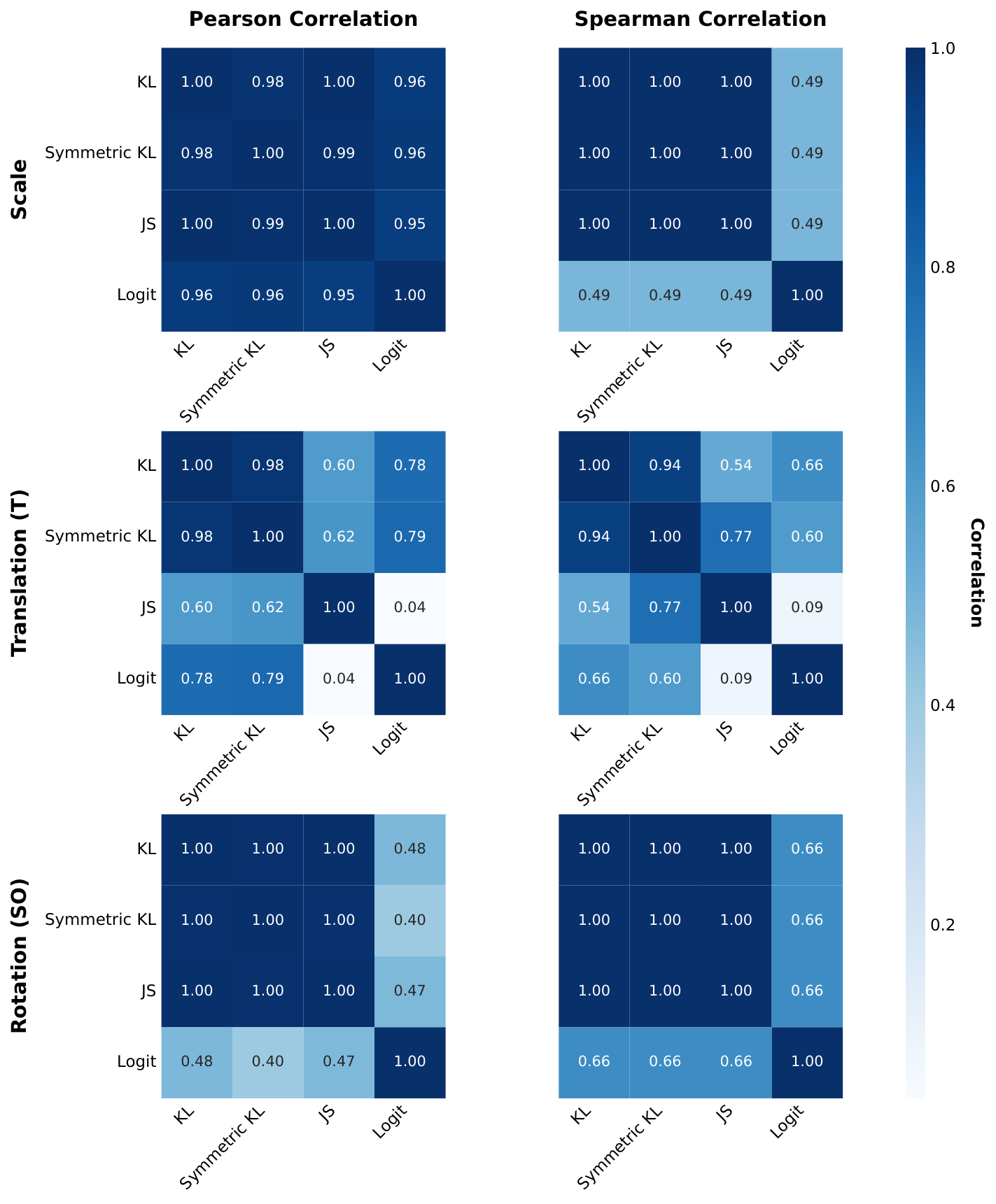}
    \caption{OOD Topology invariance metrics correlation for base angle-preserving symmetries.}
    \label{fig:topology_metrics}
\end{figure}

\subsection{Finite augmentation reveals a symmetry inductive bias in Transformers}
\label{app:experiment_details_motivation}

In this section, we provide additional details regarding the datasets and experiments presented in \S\ref{sec:motivation} of the paper.

\noindent\textbf{3D Topology.}
In order to demonstrate that our previous results are due to the alignment of the Transformer architecture with the symmetries instead of it being an artifact of the input-output mapping being lost through the non-rigid transformations, we define a synthetic benchmark. 
In particular, we construct a very simple dataset consisting of multiple randomly samplings of point clouds for five different shapes: a sphere, a torus, a double torus, two disjoint spheres, and two linked tori.
We randomly generate 100 samples per class for train and 50 for validation and testing.
The sphere, torus, and double torus represent increasing values of genus (or first Betti number, or number of holes), while the two disjoint spheres are differentiated through the zeroth Betti number (or number of connected components), and the linked tori distinguished themselves through the linking number (or ambient entanglement).
All of those properties remain preserved through not only the angle-preserving transformations, but also the the homeomorphism transformations represented by the non-angle-preserving groups we consider (i.e., $\operatorname{SL}(3), \operatorname{GL}(3), \operatorname{Affine}(3)$).

\noindent\textbf{ModelNet10.}
We also work with ModelNet10 \cite{wu20153d}, which following standard practice, we pre-process by centering and normalizing the point coordinates to the range $(-1, 1)$ and sampling $1024$ points per cloud. 
We use the official train-test splits.
Note that ModelNet10 objects are provided in an approximately canonical frame. Moreover, all of our comparisons use the same underlying dataset and differ only in the augmentation group. Thus, any invariance induced by the dataset distribution itself is shared across settings, while differences between models reflect the additional symmetry constraints imposed by augmentation.

\noindent\textbf{The effect of data augmentation.}
We illustrate the effect of the data augmentation process for the different symmetries we consider in our experiments in Figure~\ref{fig:topology_augmentations} for some sample point clouds of the 3D Topology dataset.

\begin{figure}
    \centering
    \includegraphics[width=0.75\linewidth]{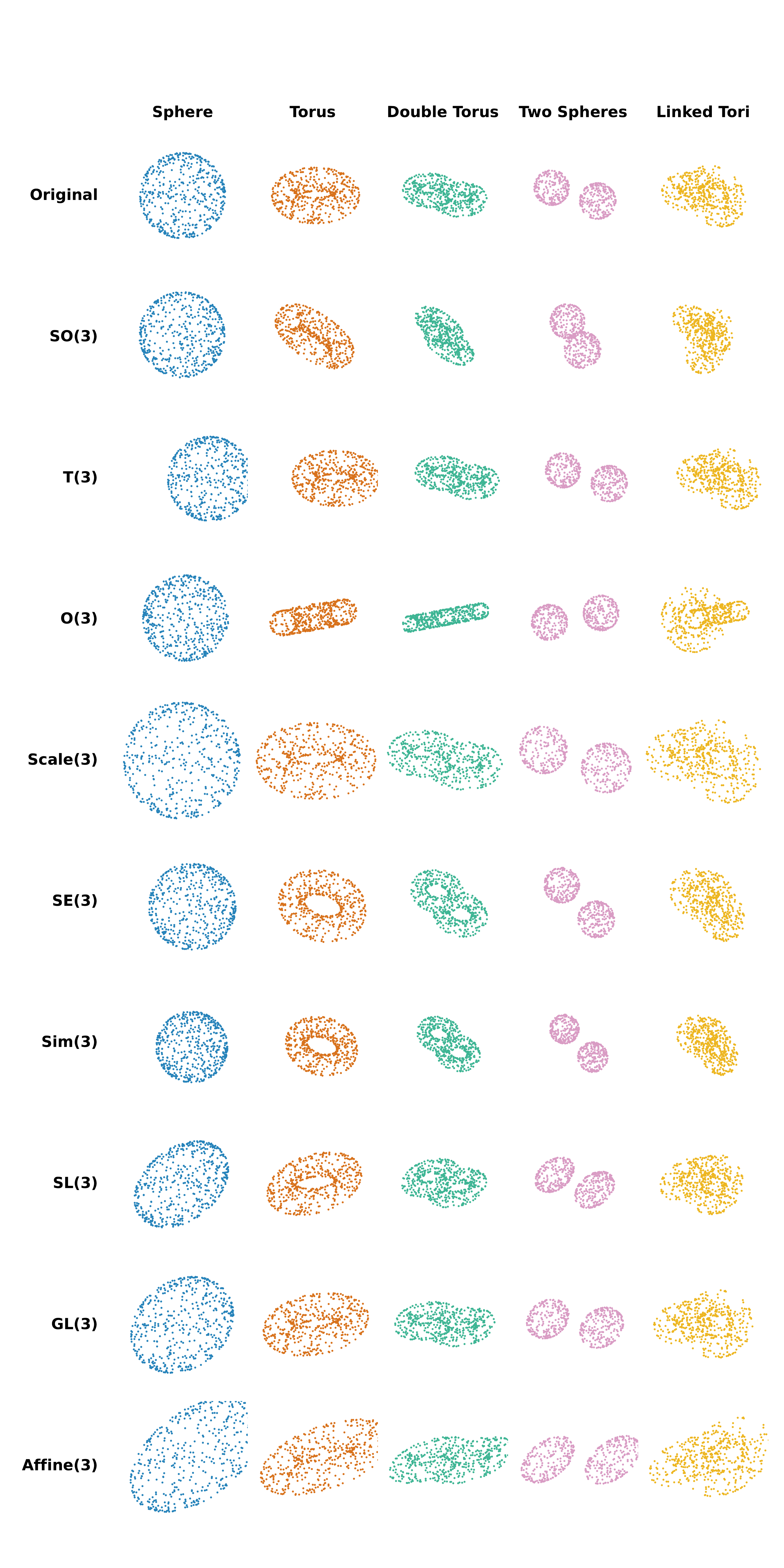}
    \caption{Illustration of augmentations from the different angle-preserving and non-angle-preserving symmetry augmentations that we consider in our experiments for the 3D Topology dataset.
    }
    \label{fig:topology_augmentations}
\end{figure}

\subsection{Learning base angle-preserving symmetries}
\label{app:ood}

In this section, we provide additional details regarding the architectures and experiments presented in \S\ref{sec:ood} of the paper.

\noindent\textbf{DeepSets \citep{zaheerDeepSets2017}.} 
Permutation-invariant architecture designed to work with set inputs, defined as $f(X) = \rho(\sum_{x \in X} \phi(x))$, where $X$ is a set, and $\rho, \phi$ are universal approximators. 
In practice, we implement $\phi$ as a shared MLP composed of point-wise linear layers, batch normalization, and ReLU activations applied independently to each element of the input. 
The global aggregation is performed using sum pooling across the input dimension. 
Finally, $\rho$ is defined as a 2-layer ReLU-MLP with dropout that processes the aggregated feature vector  and produces the final predictions. 

\noindent\textbf{PointNet \citep{qi2017pointnet}.}
A permutation-invariant architecture for point cloud tasks, e.g., object classification or part segmentation, defined similarly to DeepSets, but with the particularities of relying on max pooling and adding alignment networks known as transformation networks (T-nets) to both raw input coordinates and the intermediate feature space.
These T-nets are used to predict transformation matrices and apply said transformations to the coordinates of input points, seeking that the network is invariant to linear transformations, like rotation and and reflection.
In practice, both the overarching PointNet architecture and the T-nets are defined using shared MLPs with point-wise linear layers, batch normalization, and ReLU activations applied independently to each element of the input.
The aligned per-point features are aggregated via max pooling into a single global feature vector, which is then passed through a 3-layer MLP with dropout to produce the final outputs.

\begin{table}[t]
\centering
\caption{Parameter overview of compared architectures. Both operate on 3D point sets ($d_{\text{in}}=3$) and are parameter-matched to $\approx\!51$k. 
$C$ denotes the number of classes.
}
\label{tab:arch_comparison}
\small
\begin{tabular}{llll}
\toprule
Model & Block & Layer sizes & \#Params \\
\midrule
\multirow{2}{*}{DeepSets}
 & Encoder $\phi$          & $3\!\to\!156\!\to\!156$, sum-pool                                          & 25{,}740 \\
 & Decoder $\rho$          & $156\!\to\!156\!\to\!C$ (dropout $0.1$)                                    & 25{,}748 \\
\midrule
\multirow{5}{*}{PointNet}
 & Input T-Net ($k{=}3$)   & $3\!\to\!9\!\to\!18\!\to\!140$, maxpool, $140\!\to\!70\!\to\!35\!\to\!9$   & 16{,}099 \\
 & MLP 1            & $3\!\to\!9\!\to\!9$                                                        & 162 \\
 & Feature T-Net ($k{=}9$) & $9\!\to\!9\!\to\!18\!\to\!140$, maxpool, $140\!\to\!70\!\to\!35\!\to\!81$  & 18{,}745 \\
 & MLP 2            & $9\!\to\!18\!\to\!140$, maxpool                                            & 3{,}156 \\
 & Final MLP                   & $140\!\to\!70\!\to\!35\!\to\!C$ (dropout $0.1$)                            & 12{,}853 \\

\bottomrule
\end{tabular}
\end{table}

We try to roughly match the parameter count of the used Transformer, while approximately maintaining the relative module sizes, which results in the settings listed in Table \ref{tab:arch_comparison}. For the 10 classes of ModelNet10, the total counts come in at 50,954 (Transformer), 51,802 (DeepSets) and 51,087 (PointNet) trainable parameters.

\noindent\textbf{OOD setting.}
Seeking to explore whether the learned invariances also hold for transformations not seen at training time, we define the training, validation, test, and generalization ranges of augmentations as shown in Table~\ref{tab:ood_setup}.
During training, the model is exposed only to augmentations sampled from the smaller-magnitude range, and it is then tested on augmentations sampled from a disjoint larger-magnitude range.

\begin{table}[t]
    \centering
    \begin{tabular}{llcc}
        \toprule
        Transformation & Symmetry Group & Train / Validation / Test & Generalization \\
        \midrule
        Rotation & $\operatorname{SO}(3)$ & $[0^{\circ}, 90^{\circ})$ & $[90^{\circ}, 180^{\circ})$ \\
        Rotation and reflection & $\operatorname{O}(3)$ & $[0^{\circ}, 90^{\circ})$ & $[90^{\circ}, 180^{\circ})$ \\
        Translation & $\operatorname{T}(3)$ & $[0, 1.5)$ & $[1.5, 3)$ \\
        Scaling & $\operatorname{Scale}(3)$ & $[\times 1, \times 1.5)$ & $[\times 1.5, \times 2)$ \\
        \bottomrule
    \end{tabular}
    \vspace{1em}
    \caption{Ranges of augmentations used for training, validation, and testing, as well as for measuring the generalization of the model for the considered based symmetries in \S\ref{sec:ood}.}
    \label{tab:ood_setup}
\end{table}

\subsection{From invariance to equivariance}
\label{app:equivariance_details}

\noindent\textbf{ShapeNet.} ShapeNet \cite{yi2016scalable} consists of 16,881 coordinate point clouds representing 16 different shapes like e.g. \emph{Airplane} or \emph{Bag}. It is annotated by surface normal vectors for each point and two to six segmentation parts per category. 

\noindent\textbf{Equivariant targets.} For a smooth surface defined implicitly as $F(\mathbf{\vx}) = 0$ for $\vx \in \mathbb{R}^3$, the unit normal at a point $\vx$ is

\begin{equation}
    \operatorname{N}(\vx) = \frac{\nabla F(\vx)}{\|\nabla F(\vx)\|}.
\end{equation}

Normal vectors for each point of the input are the $\operatorname{O}(3)$- and $\operatorname{SO(3)}$-equivariant targets in \S\ref{sec:equivariance}. 

Given a segmentation map $\operatorname{S}:\sX \to \sP$ for a point cloud $\sX$ and part identities $\sP = \{1, \dots, p\}$, the centroid of part $k$ can be defined as

\begin{equation}
    \operatorname{C}(k) = \frac{1}{|\operatorname{S}^{-1}(k)|} \sum_{\vx \in \operatorname{S}^{-1}(k)} \vx.
\end{equation}

Mapping each point to the centroid of the part it belongs to requires a model to first determine the correct segment subset and then compute the mean over it. This poses the $\operatorname{T}(3)$- and $\operatorname{Scale}(3)$-equivariant task used in \S\ref{sec:equivariance}.

\noindent\textbf{Training.} We use the same small vanilla Transformer architecture as specified at the beginning of this section and \S\ref{sec:mechanisms}, but drop the final sum aggregation to obtain per-point outputs. The training objective is a mean squared error loss. 

\section{Additional experimental results}
\label{app:additional_experiments}

\subsection{3D Tetris}
\label{app:tetris}

3D Tetris \cite{thomas2018tensor} is a common synthetic benchmark for the study of point cloud invariance \cite{melnyk2021embed, melnyk2022steerable, winter2022unsupervised}, and contains eight shapes, each consisting of four three-dimensional points, representing the center of each Tetris block shape (see bottom-left of Fig.~\ref{fig:3d_tetris}).
We extend this dataset by generating 50 samples per shape by adding small amounts of jitter and noise to the points of each original shape in order to increase the complexity of the task.
Since the topology of Tetris shapes is identical, their only distinguishing features are geometric angles and distances, which are distorted under non-angle-preserving transformations ($\operatorname{SL}(3), \operatorname{GL}(3), \operatorname{Affine}(3)$), therefore the classification problem becomes ill-posed in such scenarios, and those symmetries are not expected to be learned too well.
Similarly, the reflections introduced by the orthogonal group $\operatorname{O}(3)$ would also result in a restricted upper-bound classification performance for the models, since they will collapse the representations of the two classes with opposite chirality together.

As previously observed in \S\ref{sec:motivation}, Figure~\ref{fig:3d_tetris} shows that there is again a stratification in the performance provided by the Transformer according to the type of symmetry groups. In particular, the results reinforce the previously identified increase in the ease of learnability of the symmetry when going from non-angle-preserving groups to compositional angle-preserving ones and then to base angle-preserving groups.
We should note, however, that the labels of this standard dataset are not necessarily exactly invariant to all non-angle-preserving transformations regardless of their strength, which may account for their low observed performance.

\begin{figure}
    \centering
    \begin{minipage}{0.37\textwidth}
    \includegraphics[width=\linewidth]{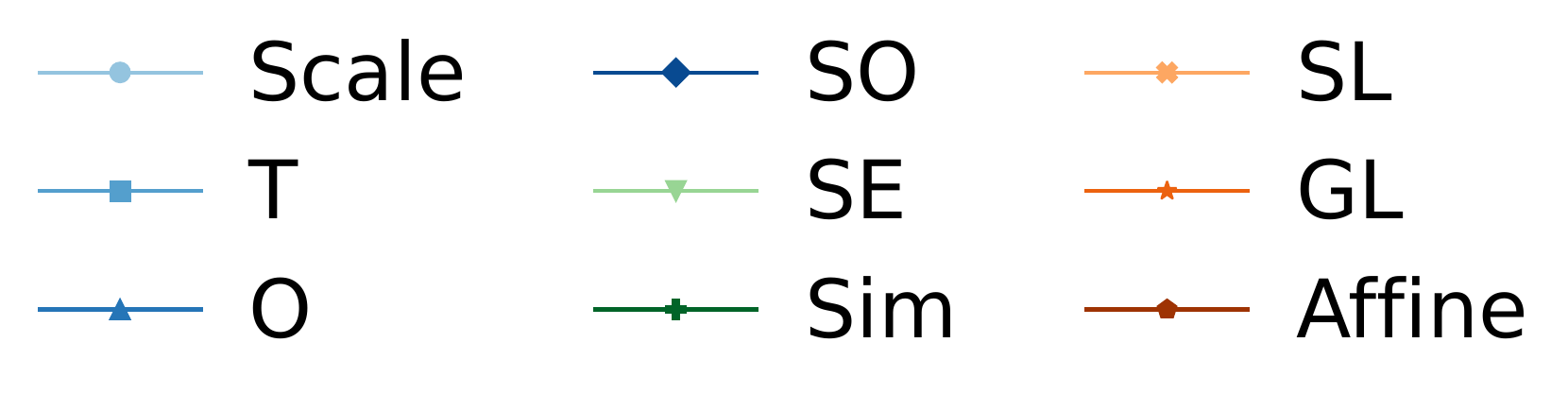} 
    \includegraphics[width=\linewidth]{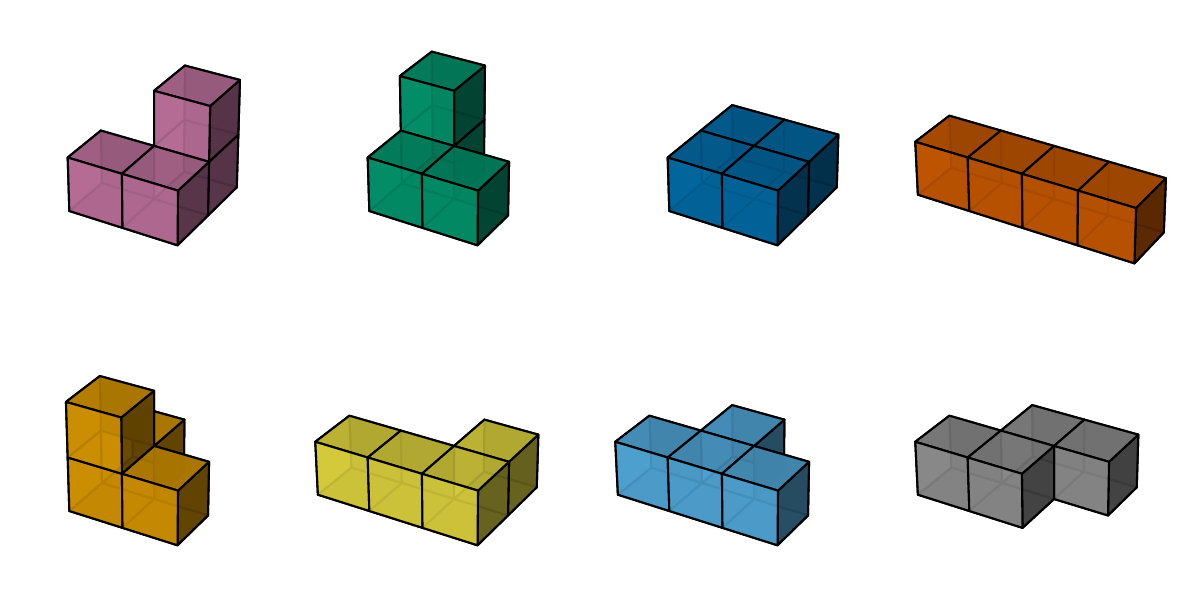}
    \end{minipage}
    \hfill
    \begin{minipage}{0.6\textwidth}
    \includegraphics[width=0.47\linewidth]{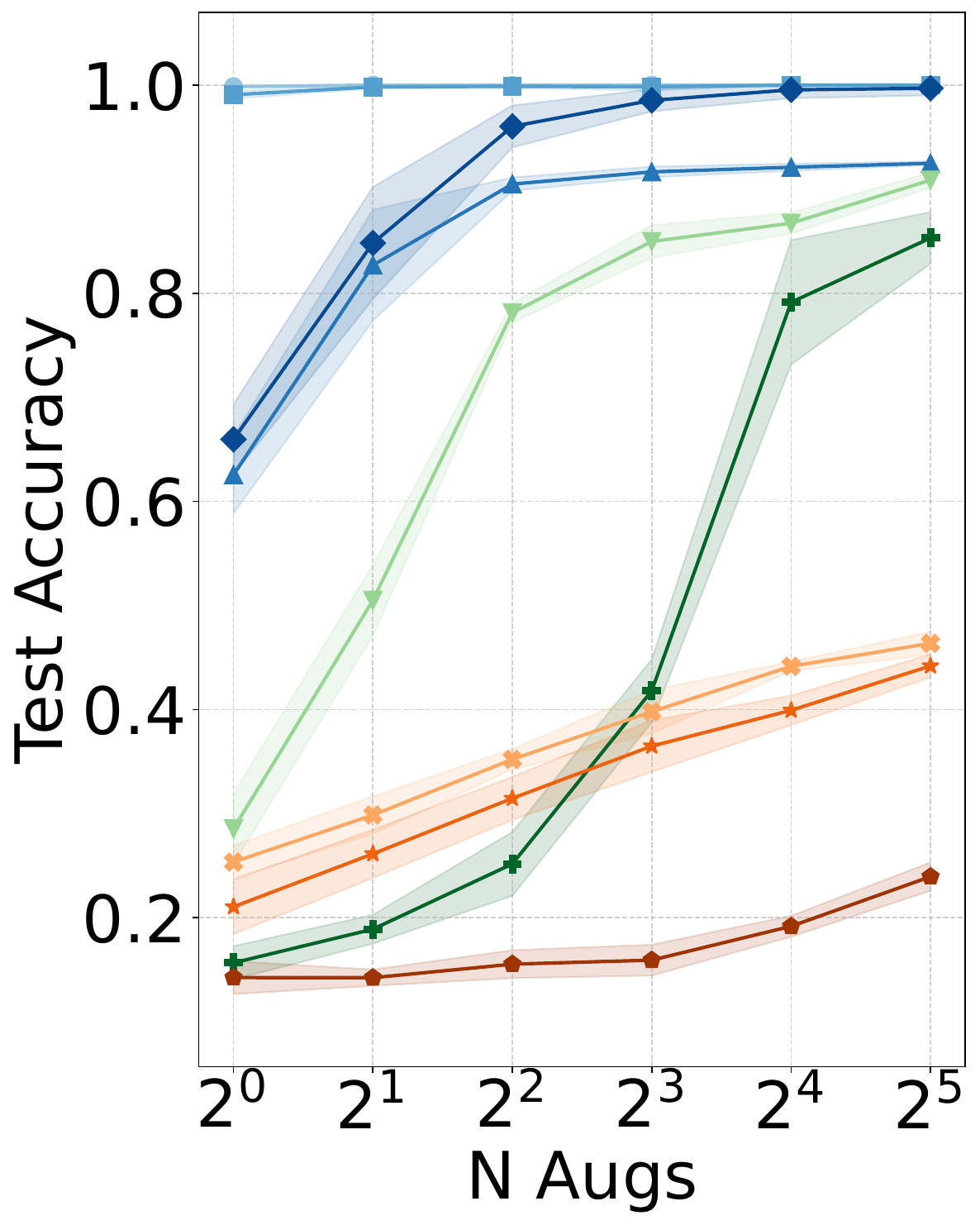}
    \hfill
    \includegraphics[width=0.47\linewidth]{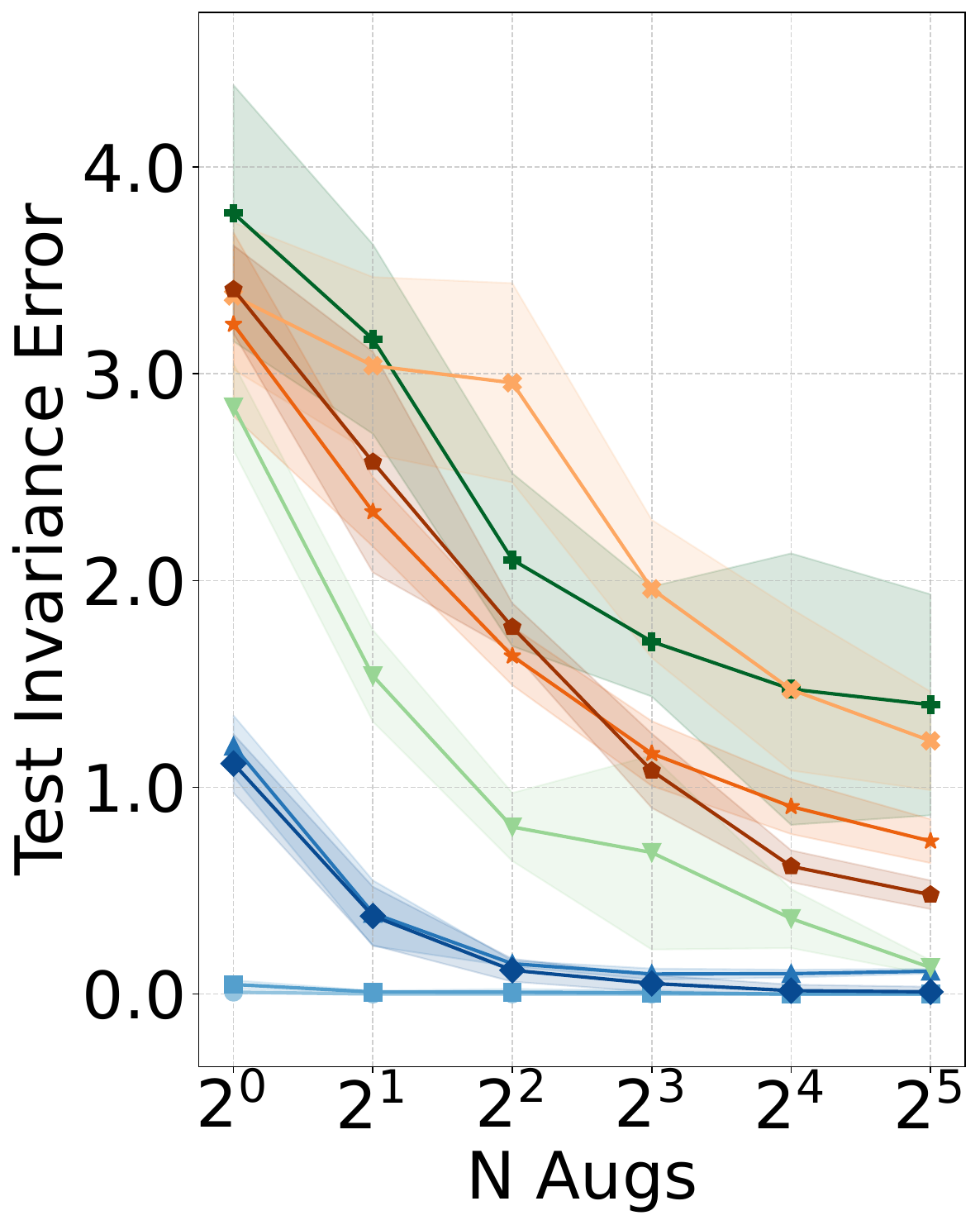}
    \end{minipage}
     \caption{Illustration of samples of 3D Tetris (bottom-left) and results measuring the mean and standard deviation for test accuracies (middle) and invariance losses (right) across 5 runs with different random seeds. Angle-preserving symmetries (blue and green tones) are efficiently learned by the Transformer, while it struggles with non-angle-preserving ones (orange tones).
     }
    \label{fig:3d_tetris}
\end{figure}

In addition, when analyzing the models trained with this dataset (see \Cref{fig:rotation-mechanism_tetris,fig:translation-mechanism_tetris,fig:ablations_tetris,fig:scale-mechanism_tetris}, we find further evidence of the invariance mechanisms identified in \S\ref{sec:mechanisms}.

\begin{figure}[t]
    \centering
    \begin{minipage}[c]{0.21\linewidth}
        \centering
        \includegraphics[width=0.9\linewidth]{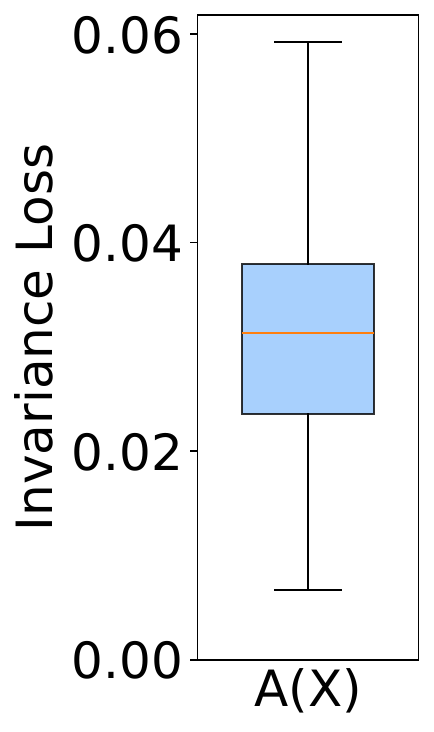}
    \end{minipage}
    \hfill
    \begin{minipage}[c]{0.71\linewidth}
        \centering
        \includegraphics[width=\linewidth]{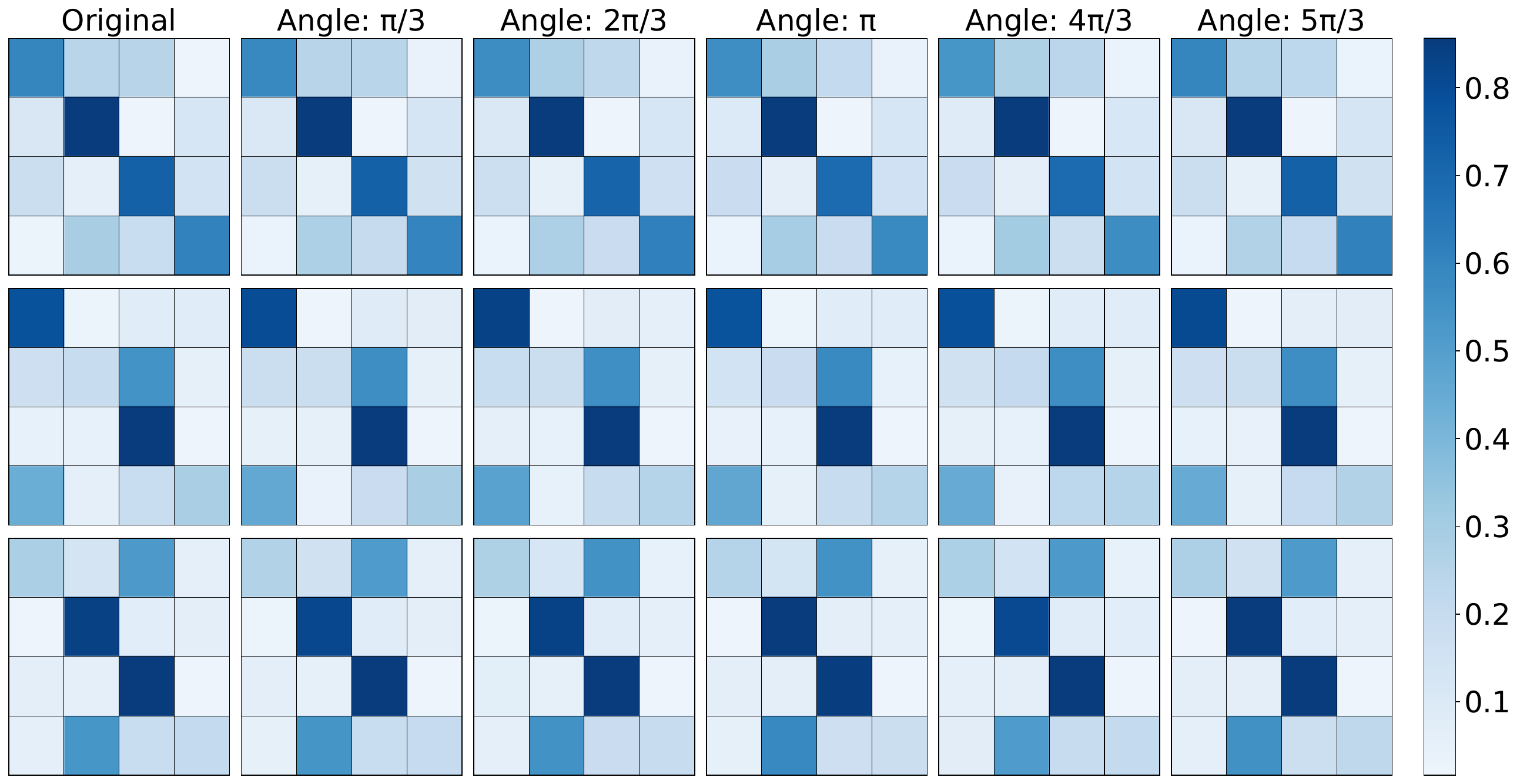}
    \end{minipage}
    
     \caption{Attention maps of a full Transformer trained on Tetris augmented by 32 rotations per sample are very stable across augmentations, as also indicated by their low invariance loss.  
     }
    \label{fig:rotation-mechanism_tetris}
\end{figure}

\begin{figure}[t]
    \centering
    \begin{minipage}[c]{0.45\linewidth}
        \centering
        \includegraphics[width=\linewidth]{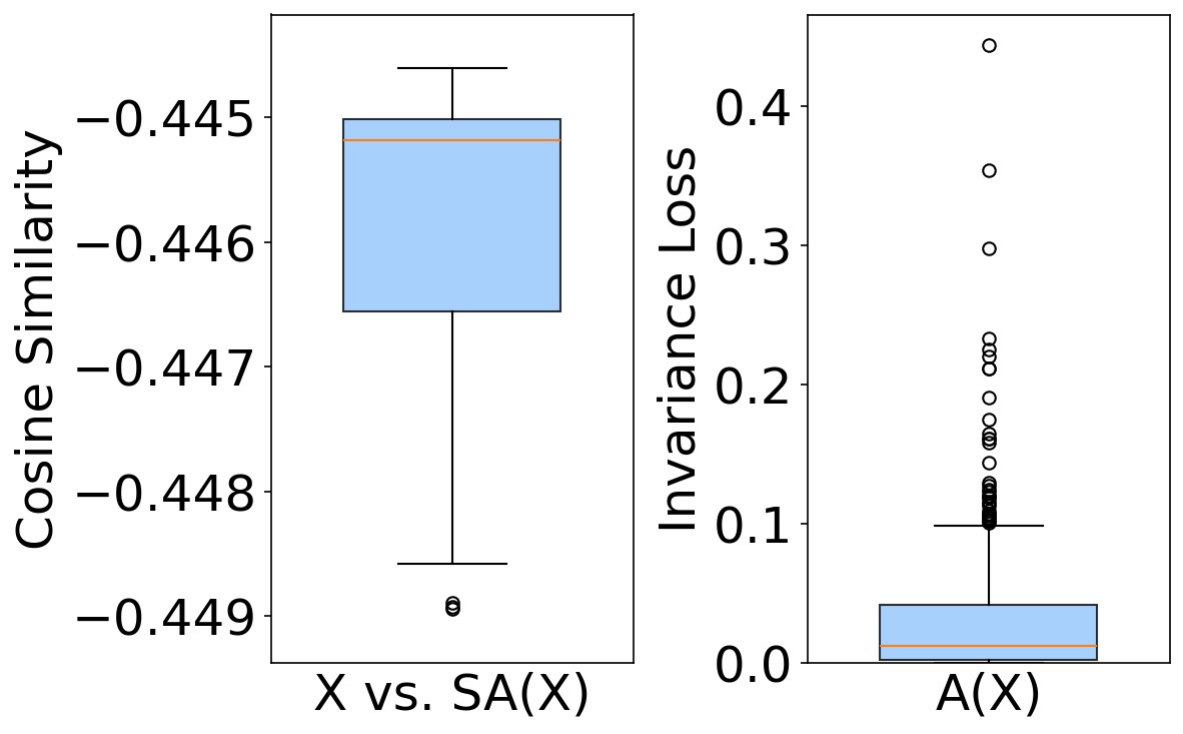}
    \end{minipage}
    \hfill
    \begin{minipage}[c]{0.5\linewidth}
        \centering
        \includegraphics[width=\linewidth]{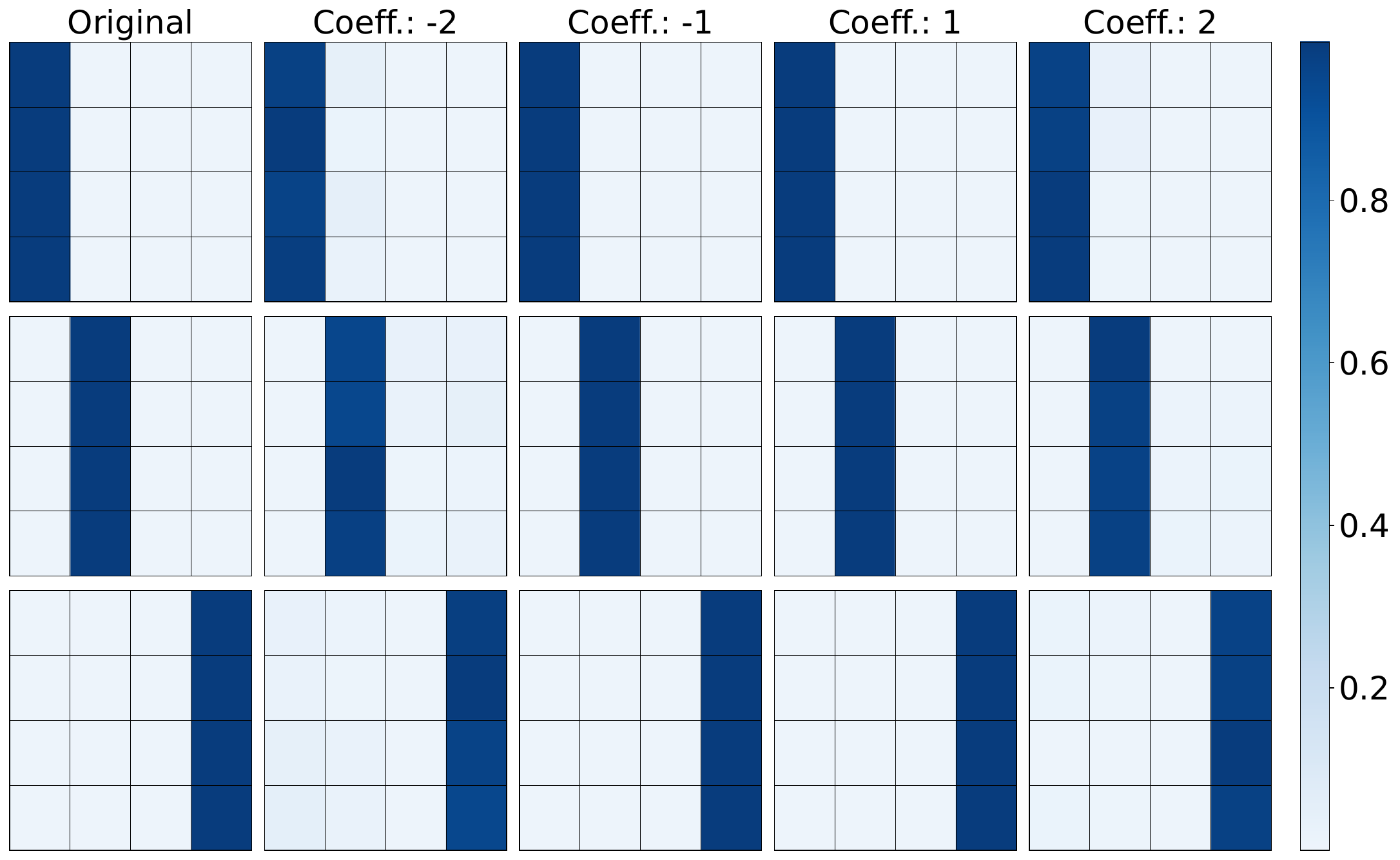}
    \end{minipage}
    
     \caption{Results for the first block a full Transformer trained on Tetris with 32 translations per sample. Right: Each row shows the attention maps of a sample (Original) and four of its augmentations over the full training range. A coefficient of, e.g., 2 means the points are shifted by 2 in the positive direction for each axis. Samples from the same orbit focus their attention on the same landmark point. Middle: Invariance loss of the attention maps is low for the entire test set. Left: Cosine similarity between $\mX$ and $\SA(\mX)$ is strictly negative, i.e. the attention output has signs opposite to those of the residuals.
     }
    \label{fig:translation-mechanism_tetris}
\end{figure}

\begin{figure}
    \centering
    \includegraphics[width=0.45\linewidth]{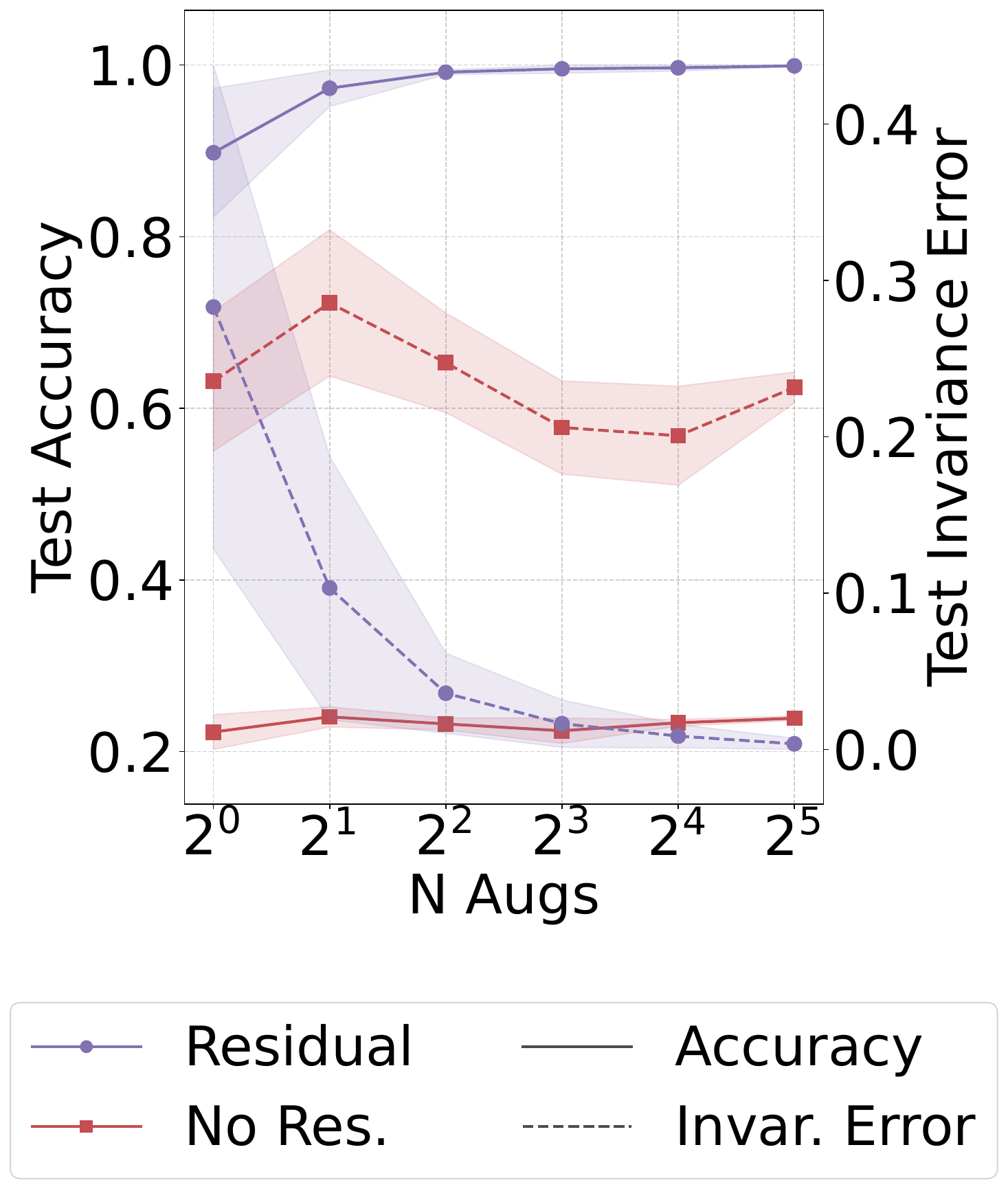}
    \hfill
    \includegraphics[width=0.46\linewidth]{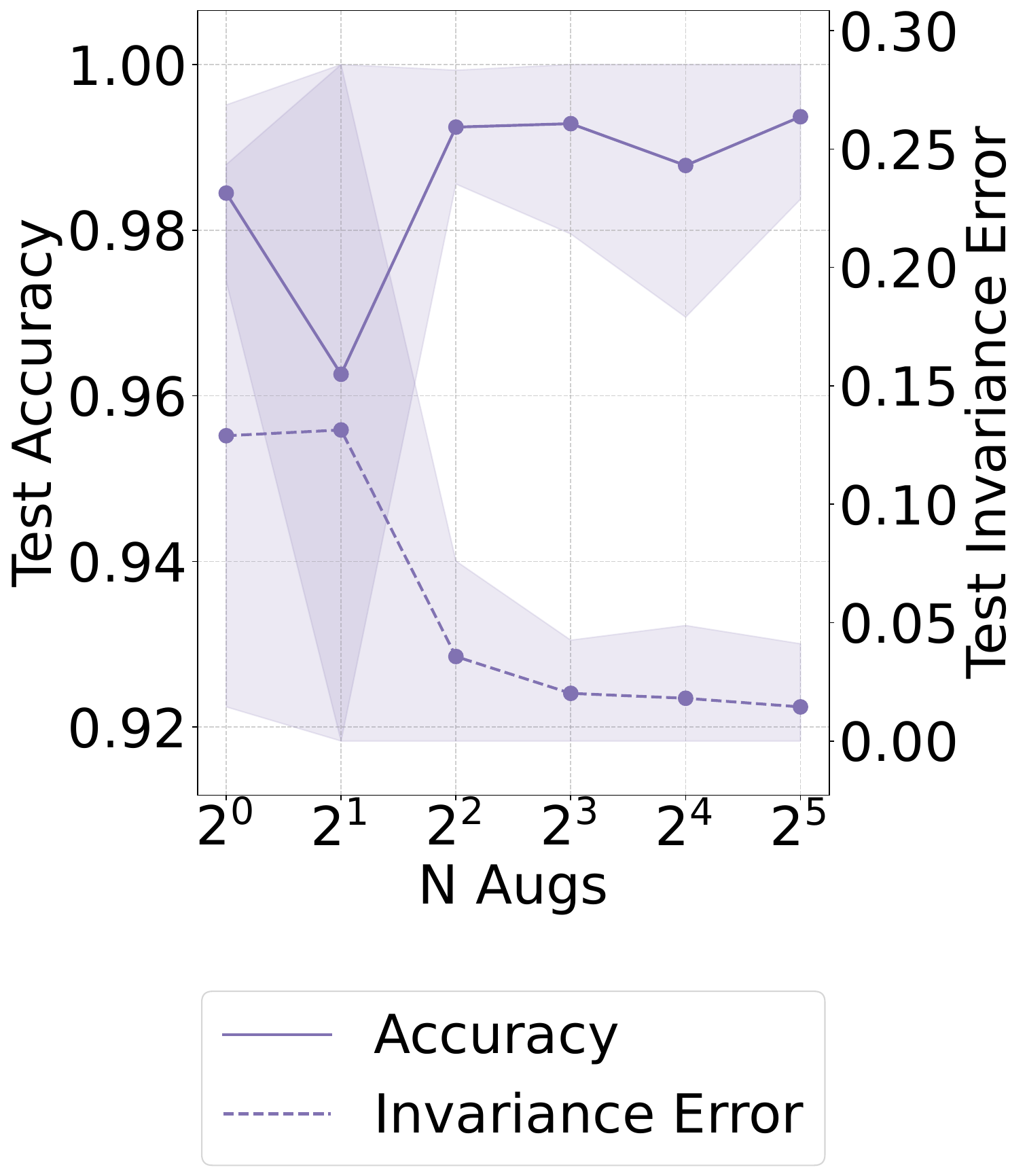}
     \caption{Mechanism ablations for Tetris dataset. Left: A model with 2 blocks of $\mX + \SA(\LN(\mX))$ (\emph{Residual}) can learn to be translation invariant and solve the Tetris task, but $\SA(\LN(\mX))$ (\emph{No Res.}) cannot. Right: A model consisting only of $\SA(\LN(\mX))$ blocks can learn to be scale invariant and solve the Tetris task. 
     }
    \label{fig:ablations_tetris}
\end{figure}

\begin{figure}[t]
    \centering
    \begin{minipage}[c]{0.45\linewidth}
        \centering
        \includegraphics[width=\linewidth]{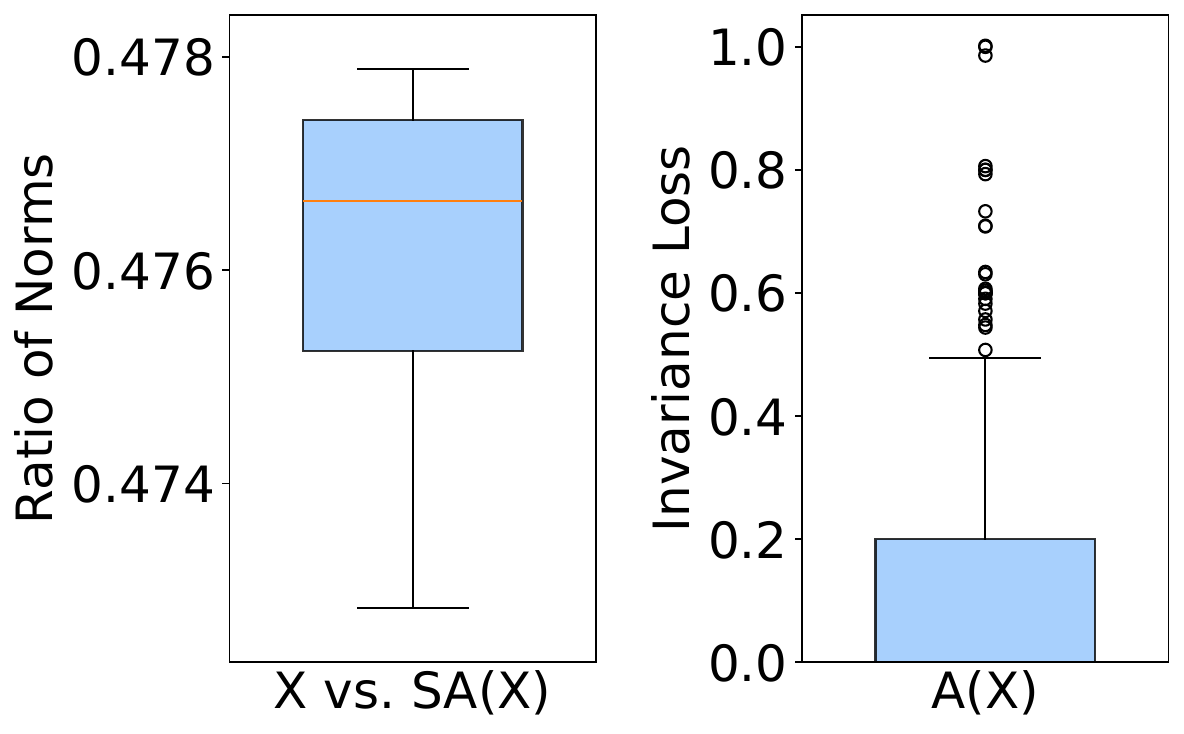}
    \end{minipage}
    \hfill
    \begin{minipage}[c]{0.5\linewidth}
        \centering
        \includegraphics[width=\linewidth]{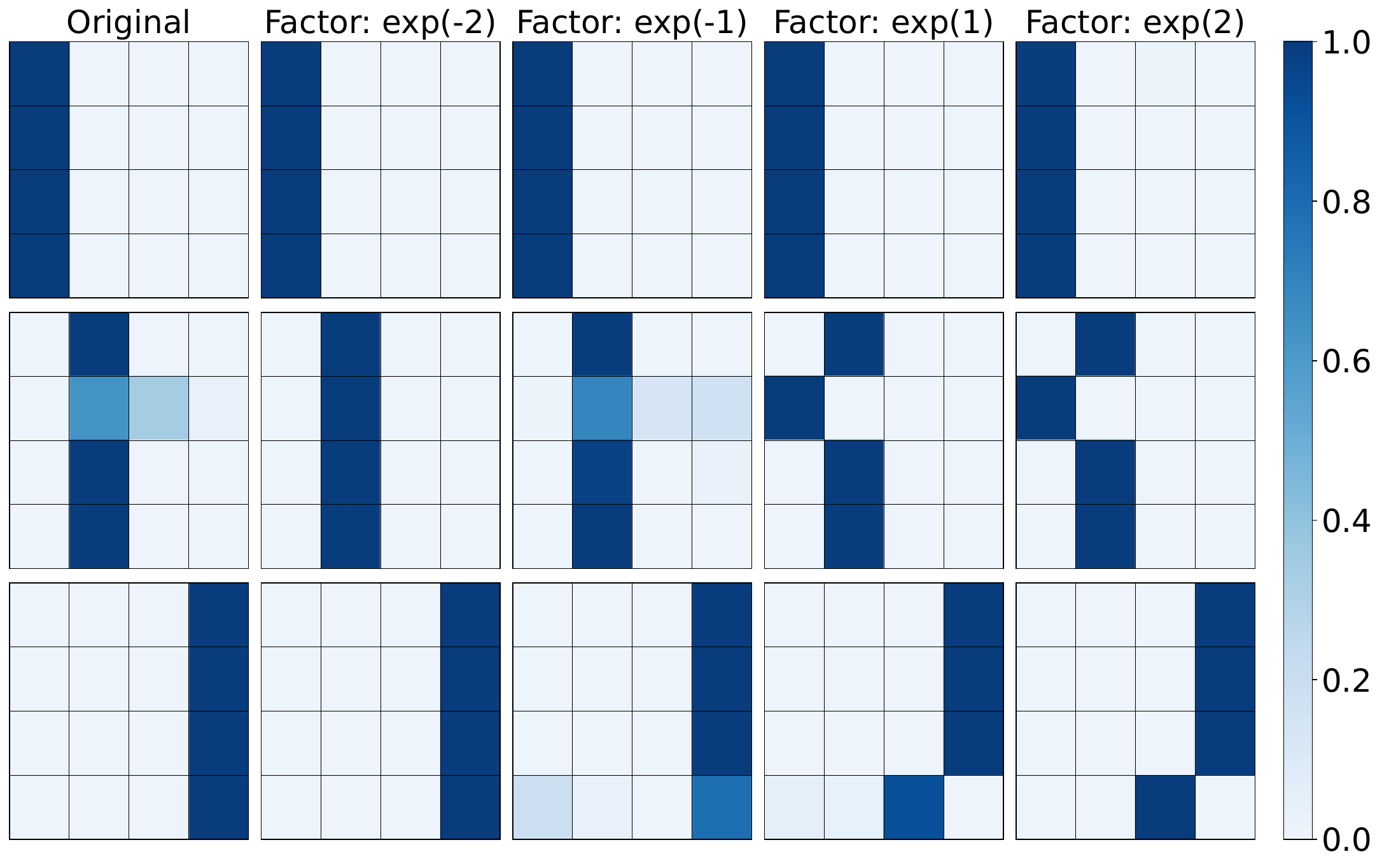}
    \end{minipage}
    
     \caption{Mechanism indicators for the second block of a Transformer trained on Tetris with 32 scale augmentations. Left: points in the attention output have larger norms than points in the residual. The non-invariance of the residuals is therefore reduced to the noise level. Middle and Right: Attention maps are mostly invariant under scalings. 
     Different rows correspond to different samples. 
     }
    \label{fig:scale-mechanism_tetris}
\end{figure}

\subsection{Molecular property prediction}
\label{app:qm9}

\begin{figure}[t]
    \centering
    
    \includegraphics[width=0.85\linewidth]{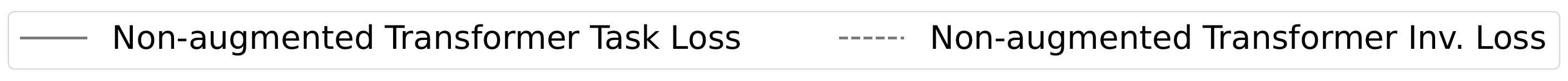} \\[0.5em] 
       
    \begin{minipage}[b]{0.24\linewidth}
        \centering
        \includegraphics[width=\linewidth]{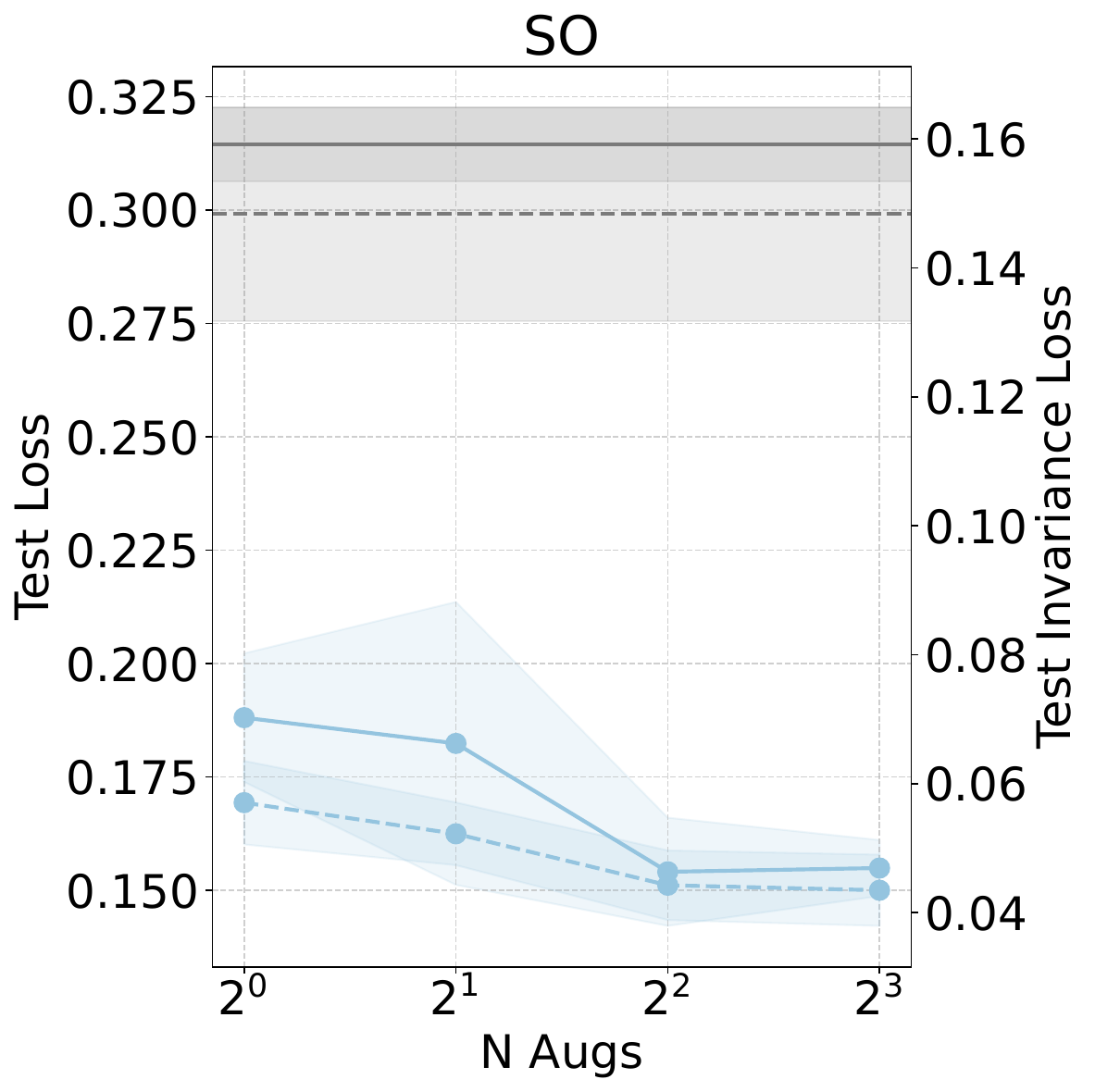}
    \end{minipage}
    \hfill
    \begin{minipage}[b]{0.24\linewidth}
        \centering
        \includegraphics[width=\linewidth]{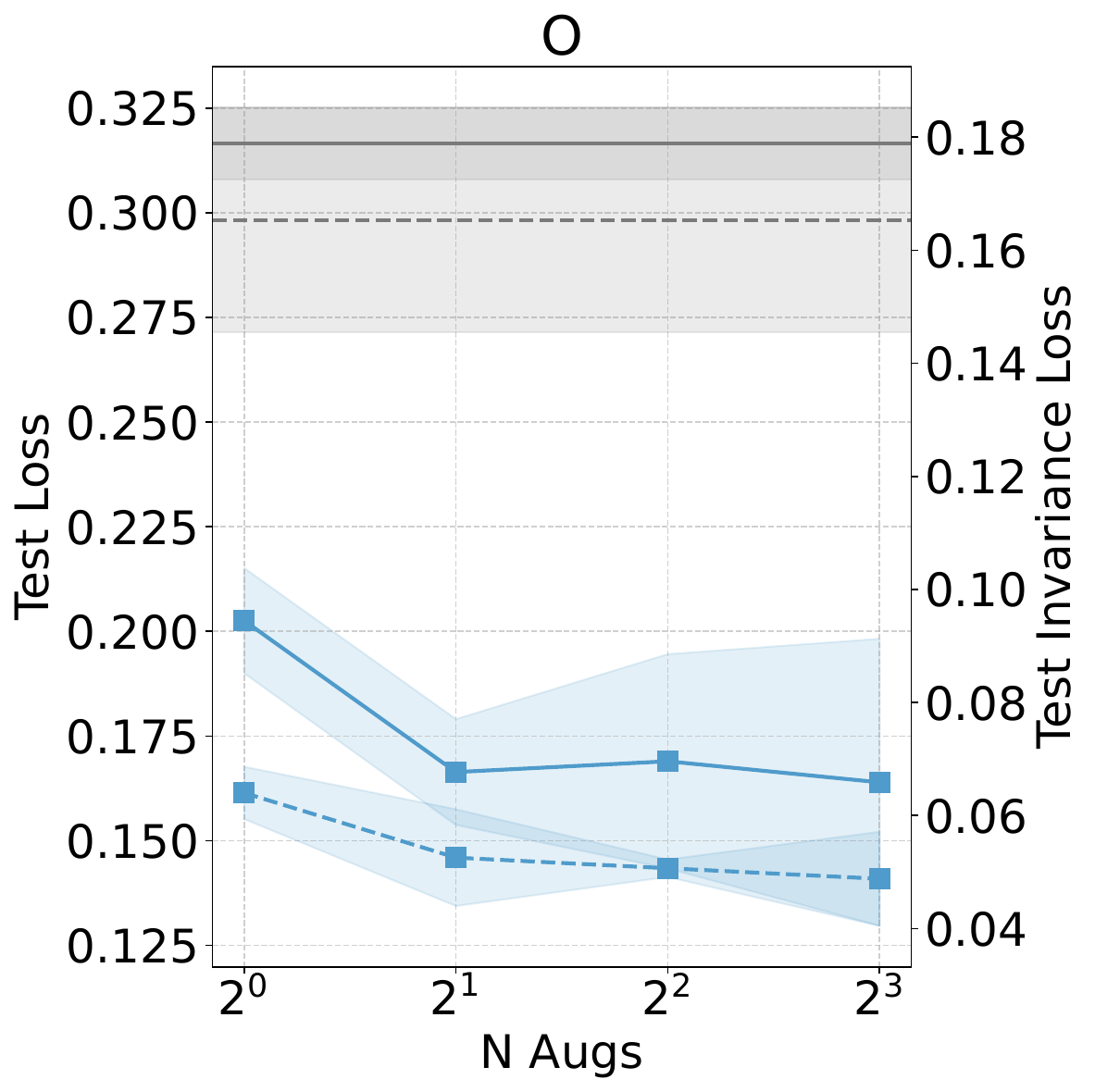}
    \end{minipage}
    \hfill
    \begin{minipage}[b]{0.24\linewidth}
        \centering
        \includegraphics[width=\linewidth]{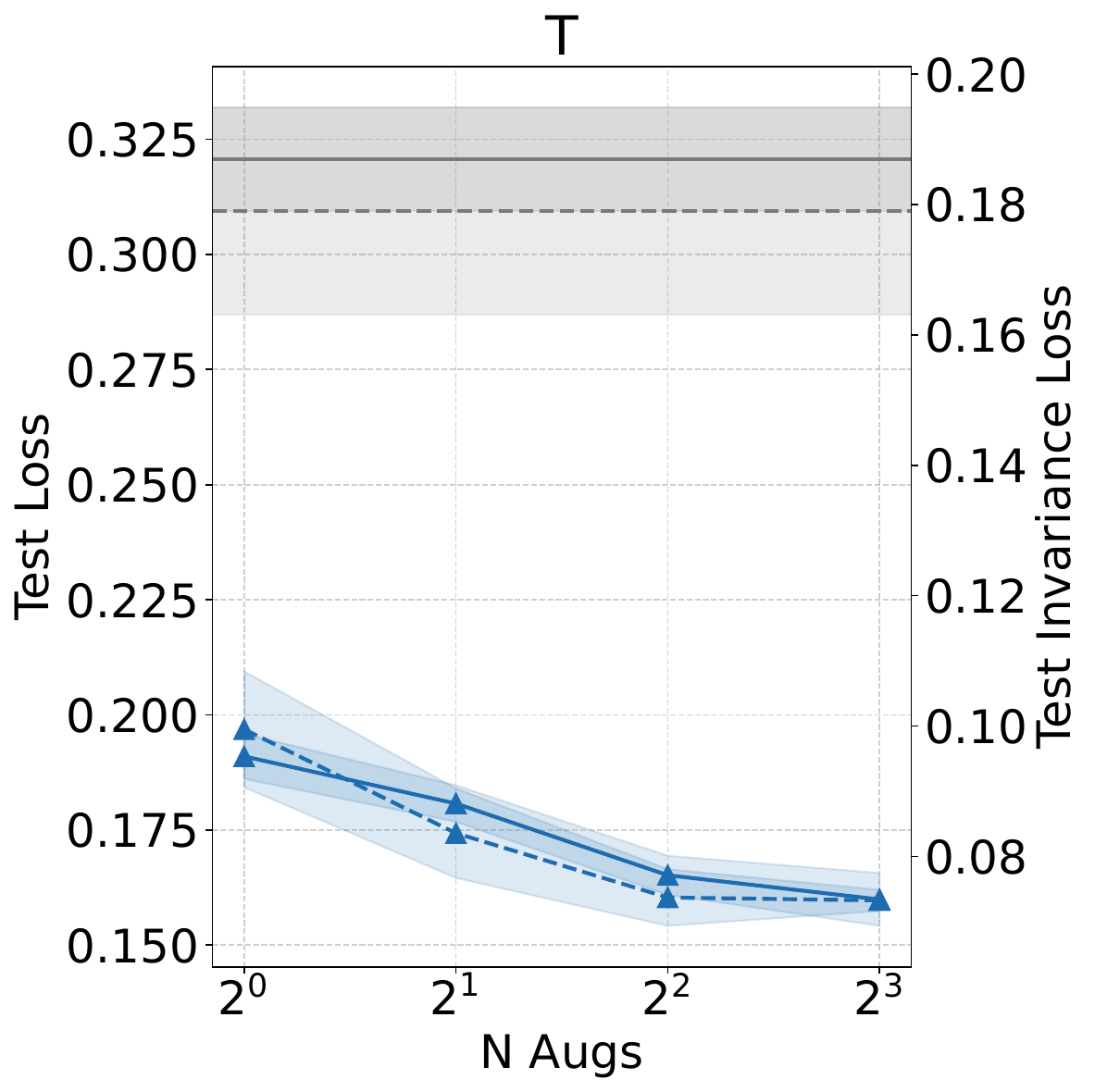}
    \end{minipage}
    \hfill
    \begin{minipage}[b]{0.24\linewidth}
        \centering
        \includegraphics[width=\linewidth]{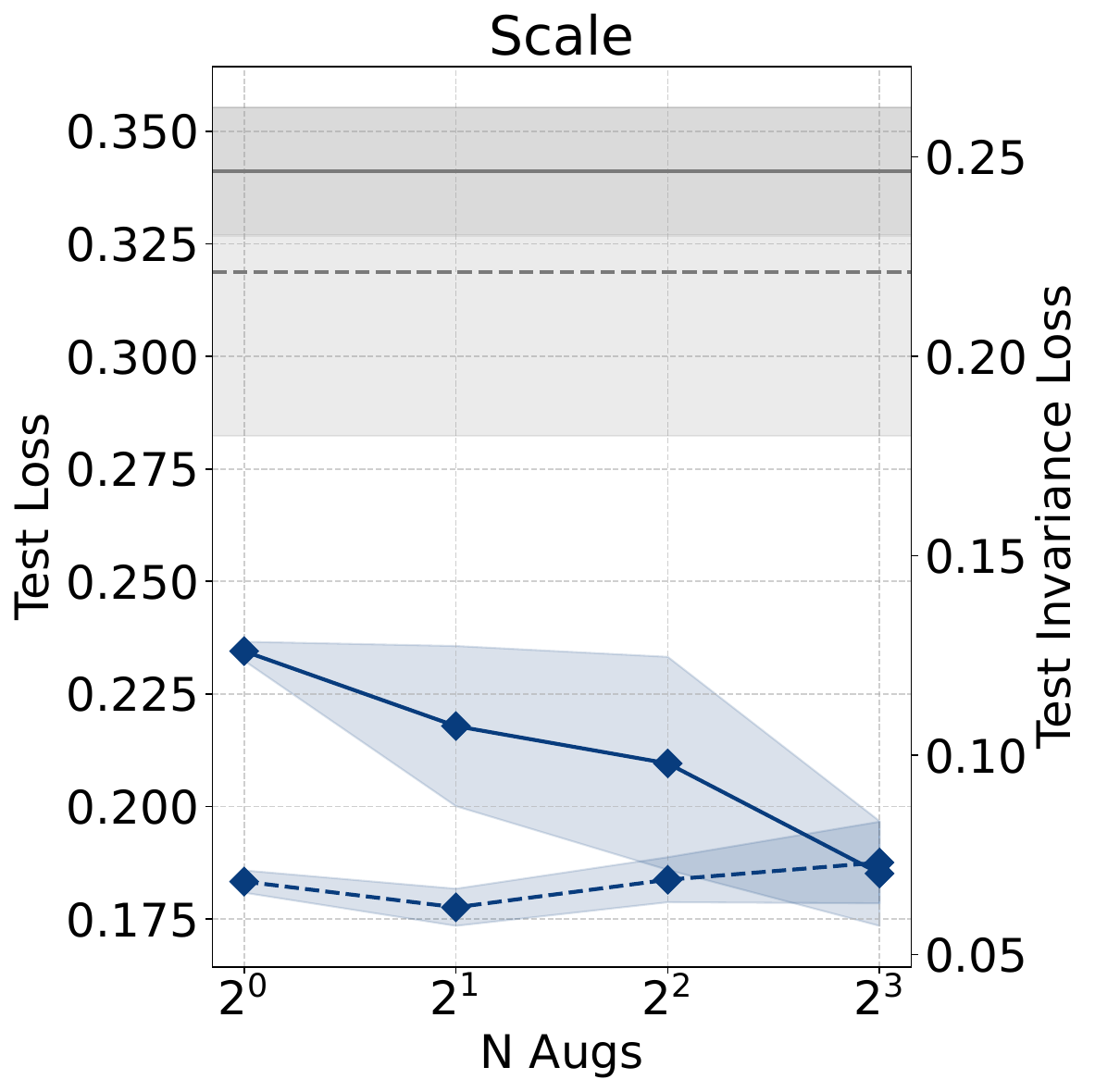}
    \end{minipage}
    \\[0.5em] 
    
    \begin{minipage}[b]{0.24\linewidth}
        \centering
        \includegraphics[width=\linewidth]{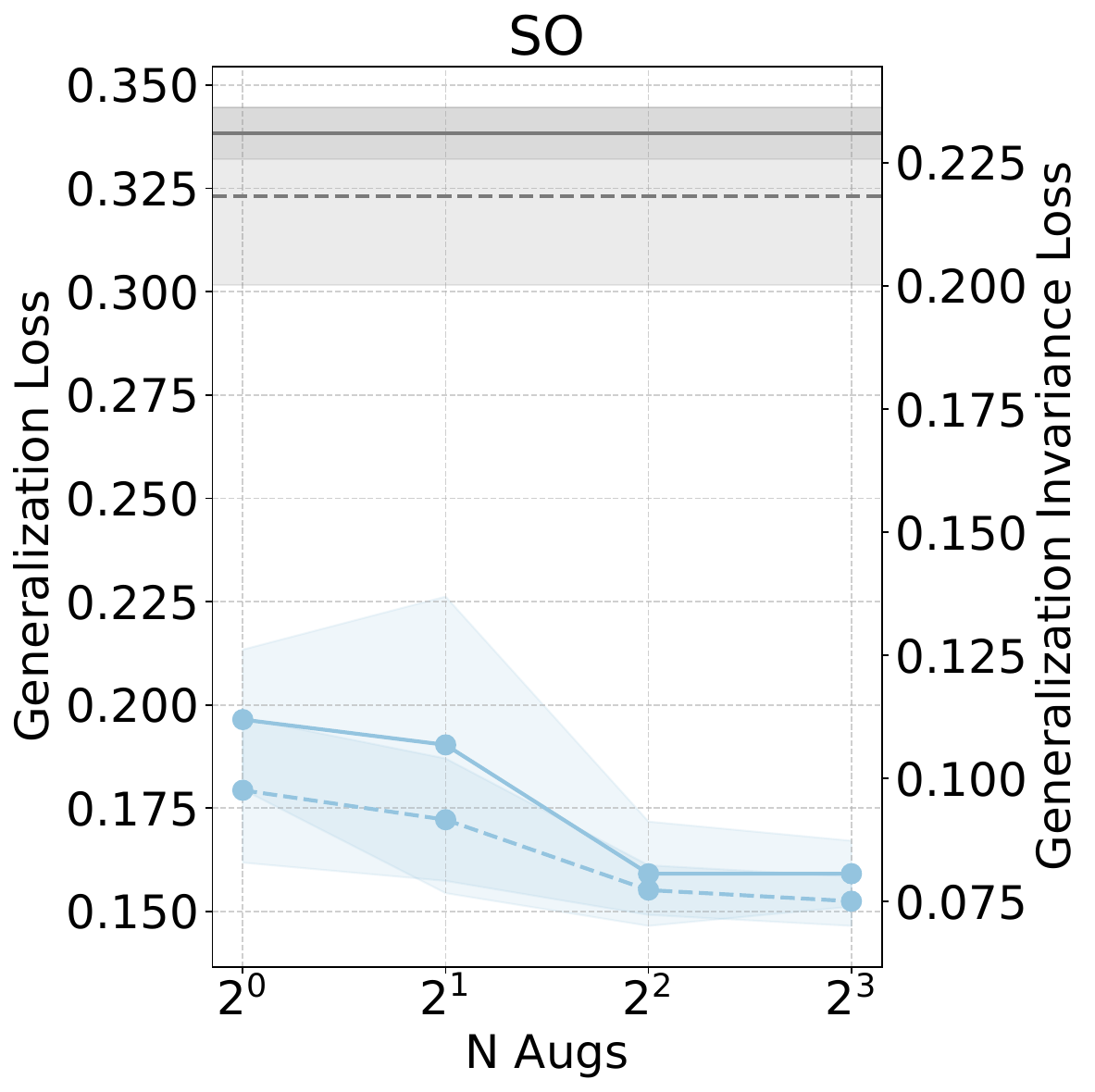}
    \end{minipage}
    \hfill
    \begin{minipage}[b]{0.24\linewidth}
        \centering
        \includegraphics[width=\linewidth]{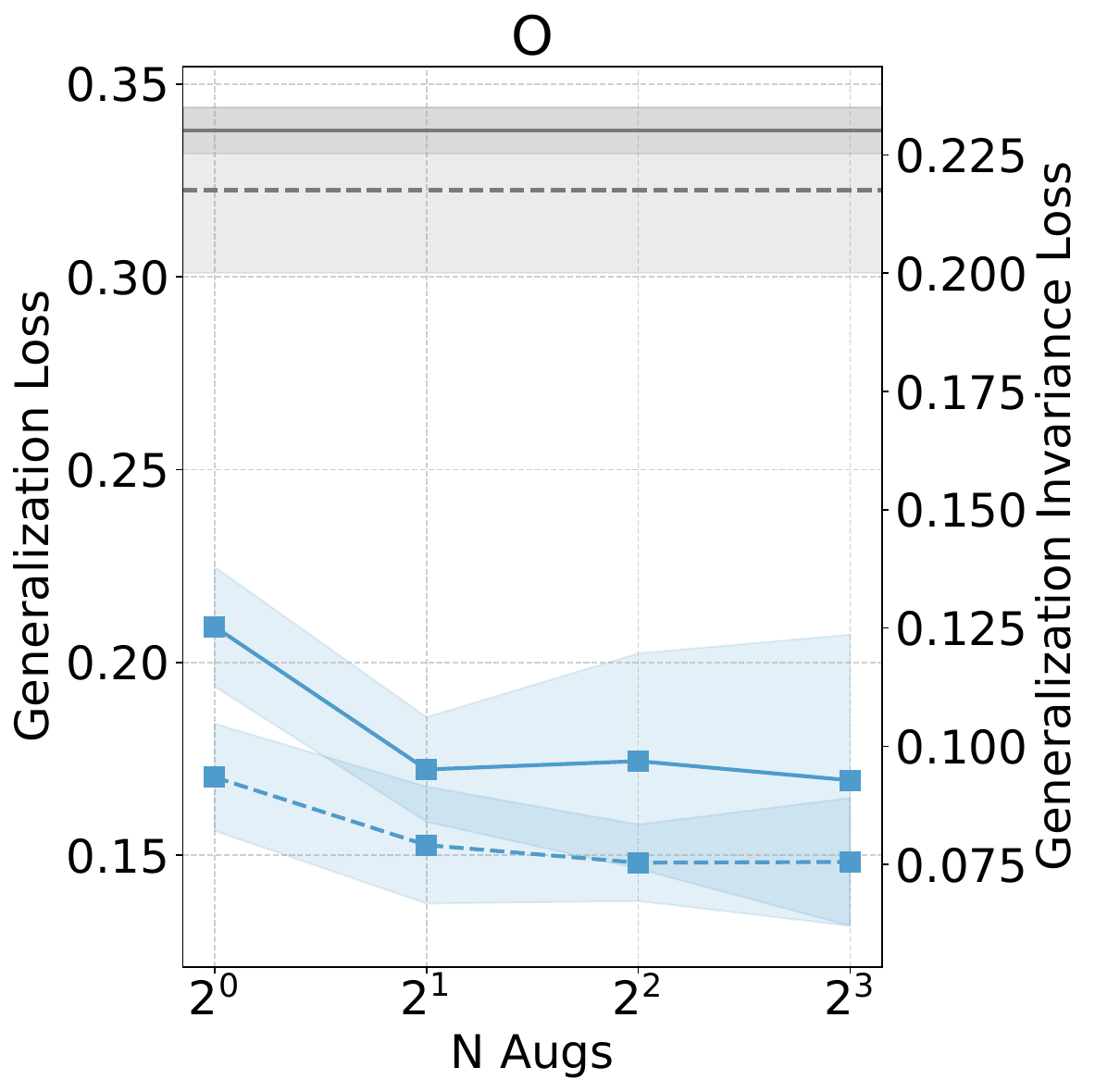}
    \end{minipage}
    \hfill
    \begin{minipage}[b]{0.24\linewidth}
        \centering
        \includegraphics[width=\linewidth]{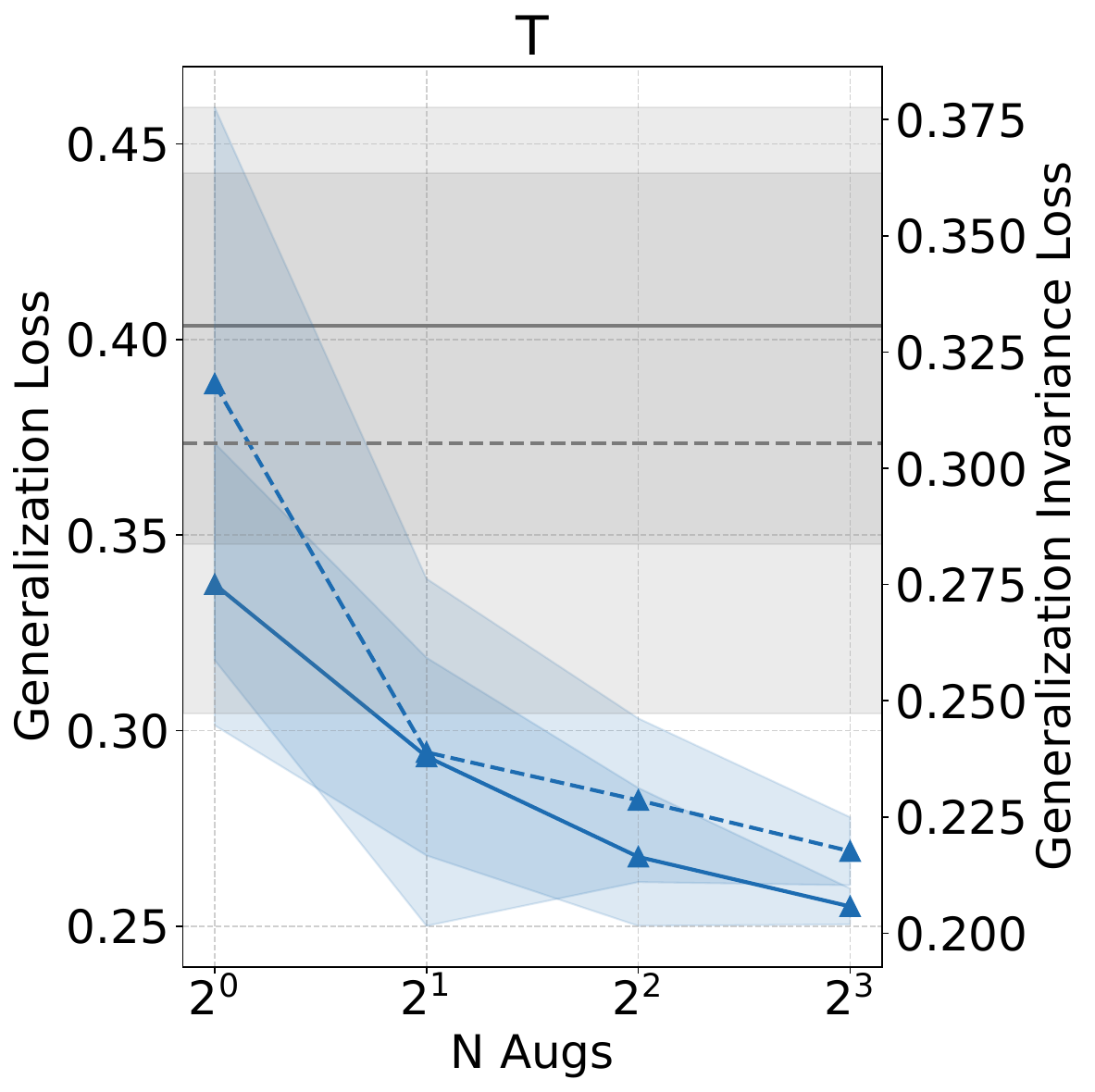}
    \end{minipage}
    \hfill
    \begin{minipage}[b]{0.24\linewidth}
        \centering
        \includegraphics[width=\linewidth]{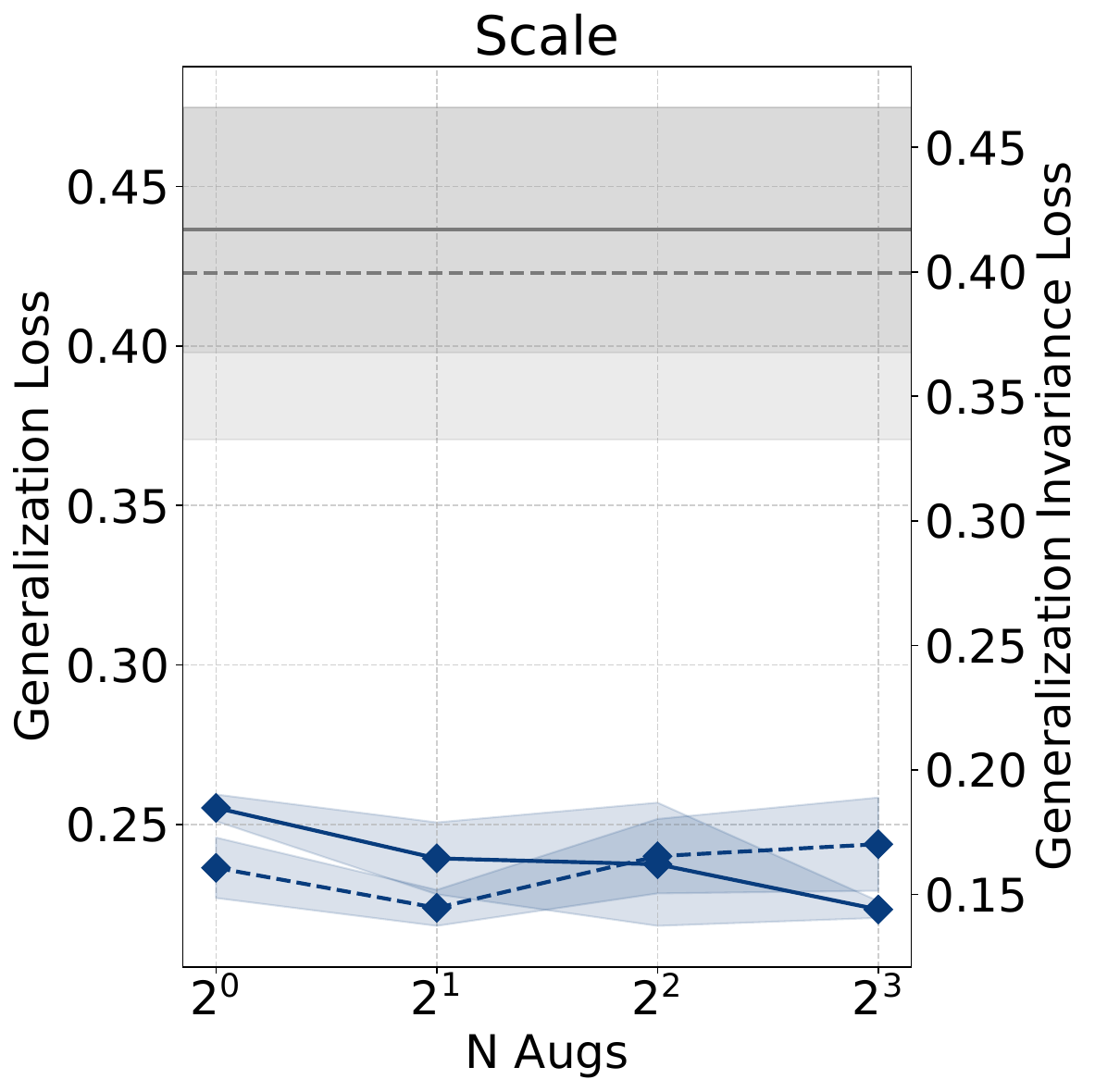}
    \end{minipage}
    
     \caption{Test (in-distribution) and OOD generalization (disjoint distribution) task (solid line) and invariance losses (dashed line) for base angle-preserving symmetries in the QM9 dataset across 3 random seeds reported as MAE in eVs (lower is better). 
     The gray lines represent the performance of a Transformer trained without data augmentation and tested in each augmented testing set as a reference to evaluate the improvement derived from data augmentation.
     All symmetries are learned robustly with enough augmentations, with translations being the most sensitive to distribution shifts.
     }
    \label{fig:qm9_ood_singles2}
\end{figure}

In order to complement the datasets obtained with the previous datasets, we also consider the QM9 dataset \citep{blum2009970, rupp2012fast}, which differs in that it does not consists of point clouds with the same number of points.
In order to overcome this difference, we implement a sequence length bucketing approach \cite{khomenko2016accelerating}, which allows for more efficient training than padding all molecules to the same length. 
We split the dataset into 80\% of training samples, 10\% validation samples, and 10\% test samples. 
Just like for ModelNet10, we carry out an out-of-distribution analysis to measure the robustness with which base angle-preserving symmetries are learned in this dataset. In that sense, we again define the training, validation, test, and generalization ranges of augmentations as shown in Table~\ref{tab:ood_setup}.

Fig.~\ref{fig:qm9_ood_singles2} shows that the Transformer largely preserves its task performance and invariance behavior on the OOD generalization set. Thus, although the model is trained using only a limited number of augmentations sampled from a narrow range of transformations, it is still able to extrapolate approximately invariant behavior to a disjoint range of unseen transformations at test time.
Moreover, the trends in both task and invariance losses are similar on both test and OOD generalization.
Increasing the number of training augmentations within the restricted range also improves performance under larger, unseen transformations. 
This observation provides a strong indication that the model is learning an underlying mechanism that encodes invariance to the symmetry in a robust and generalizable manner. 
The main exception is translation, for which OOD generalization performance lags behind rotations and scaling. 
This indicates that the robustness and generality of the learned mechanism for translation invariance is more sensitive to the number of augmentations observed during training. 

\subsection{Additional architectural components}
\label{app:additional_components}

In order to extend our analysis from the vanilla Transformer architecture towards more realistic models, we explore the effect of additional architectural components, namely positional encodings and multi-head attention. 
These experiments provide preliminary evidence that the principal trends are not solely consequences of single-head attention or the complete absence of positional information.

\subsubsection{Positional encodings}

Positional encodings are a common component of modern Transformer-based architectures. 
Self-attention treats the inputs as an unordered collection of elements and is equivariant to their permutations. Positional encodings break or constrain this symmetry by associating tokens with information about, e.g., their absolute coordinates, their relative displacements, or other geometrical relationships \citep{vaswani2017attention,shaw2018self}. 
For spatial data, encodings can similarly be designed to provide information about distances or directions, thereby introducing an inductive bias adapted to the geometry of the domain. 
There are also methods that have explored the definition of equivariant self-attention layers through the design of positional encodings that are invariant under selected group transformations \citep{romero2021group}.

In the following experiments, we explore how incorporating geometrical information into the model through positional encodings affects the performance observed for the synthetic 3D Tetris dataset.
In particular, for a point cloud $P=\{p_i\}_{i=1}^{T}, \ p_i\in\mathbb{R}^{d}$,
we denote the geometric displacement and distance for a pair of points by
$\Delta_{ij}=p_j-p_i, \ r_{ij}^{(p)}=\|\Delta_{ij}\|_p$, and define the following geometric positional encodings.

\noindent\textbf{Geometric sinusoidal absolute positional encoding.} 
We adapt the sinusoidal absolute positional encoding of \cite{vaswani2017attention} to spatial coordinates by applying sinusoidal features independently to each coordinate axis.
In particular, for each coordinate axis \(a\in\{1,\dots,d\}\) and frequency index
\(m\in\{1,\dots,F\}\), define
$
    \omega_m = B^{-m/F},
$
where $B$ is usually $10,000$. The geometric sinusoidal features are
\[
    \gamma(p_i)
    =
    \Big[
        \sin(\omega_m p_{i,a}),
        \cos(\omega_m p_{i,a})
    \Big]_{a=1,\dots,d;\,m=1,\dots,F}.
\]
These features are then projected to the model dimension and added to the token:
$    \tilde{x}_i
    =
    x_i + W_{\gamma}\gamma(p_i).
$
This encoding provides the model with information about the absolute coordinates of each point.

\noindent\textbf{Geometric ALiBi.}
We adapt Attention with Linear Biases (ALiBi) \citep{press2022train} to geometric distances. Standard ALiBi applies a linear penalty based on separation in the sequence. Here, we instead penalize attention between points according to their geometric distance:
\[ \tilde{a}_{ij} = a_{ij} - m \frac{\|p_j-p_i\|_p}{s}, \] 
where \(a_{ij}\) is the original attention logit, \(m>0\) is the slope associated with the attention head, \(p\) specifies the distance metric, and \(s>0\) is a scale factor. The bias therefore favors interactions between points that are close under the chosen metric.

\noindent\textbf{Geometric learned relative position bias.} 
Following learned relative-position representations and spatial relative-position bias tables \citep{shaw2018self,liu2021swin}, we discretize the pairwise displacement between points and associate each displacement bucket with a learned scalar bias. 

We define the integer displacement as
$
    \delta_{ij}
    =
    \operatorname{clip}
    \left(
        \operatorname{round}(p_j-p_i),
        -K,
        K
    \right)
    \in \{-K,\dots,K\}^{d}.
$
A learned lookup table maps each displacement bucket to a scalar:
$
    b:\{-K,\dots,K\}^{d}\to\mathbb{R}.
$
The attention logit is then modified as
$
    \tilde{a}_{ij}
    =
    a_{ij} + b(\delta_{ij}).
$

\noindent\textbf{Geometric RoPE.}
We also consider a multidimensional, coordinate-based version of Rotary Position Embedding (RoPE) \citep{su2024roformer}. In standard RoPE, query and key vectors are rotated by phases determined by their sequence positions. In the geometric setting, these phases can instead be determined by the associated spatial coordinates.

For each rotary pair \(r\), we define a frequency vector
$
    \omega_r\in\mathbb{R}^{d},
$
and the phase assoicated with point $p_i$
$
    \theta_r(p_i)=\langle p_i,\omega_r\rangle.
$
Within each two-dimensional rotary subspace, we apply the rotation matrix 
\[
    R(\theta)
    =
    \begin{pmatrix}
        \cos\theta & -\sin\theta\\
        \sin\theta & \cos\theta
    \end{pmatrix}.
\]
As a result, we obtain the modified queries and keys
$
    \tilde{q}_i = R(\theta(p_i))q_i, \ 
    \tilde{k}_j = R(\theta(p_j))k_j,
$ where \(R(\theta(p_i))\) denotes the block-diagonal operator formed from the pairwise rotations, 
and the attention logits become 
$
    \tilde{a}_{ij}
    =
    \frac{\langle \tilde{q}_i,\tilde{k}_j\rangle}{\sqrt{d}}.
$
Because planar rotations compose through differences of their angles, each rotary pair depends on the relative phase
$ 
\theta_r(p_j)-\theta_r(p_i)
=
\langle p_j-p_i,\omega_r\rangle.
$ Thus, although the rotations themselves are defined using absolute coordinates, the contribution to the query-key inner product depends on the relative displacement between the pair of points.

\noindent\textbf{Results.} 
We carry out experiments using the synthetic 3D Tetris dataset discussed in \S\ref{app:tetris}.
As we can see in Figures~\ref{fig:pes_geom_accuracy} and \ref{fig:pes_geom_invariance}, the test performance of the Transformer remains stable across each of the previously defined geometric positional encodings, and is itself similar to the original results without positional encodings (\texttt{no\_pe}). 
In addition, repeating the mechanistic diagnostics for the base groups (Fig.~\ref{fig:pe_attention_maps}) also shows that orbit-stable attention maps and the principal activation signatures remain qualitatively present.
As a result, these experiments suggest that, within the coordinate-token setting, adding auxiliary geometric positional information does not eliminate the tendency to learn the reported invariant computations. 

\begin{figure}
    \centering
    \includegraphics[width=0.8\linewidth]{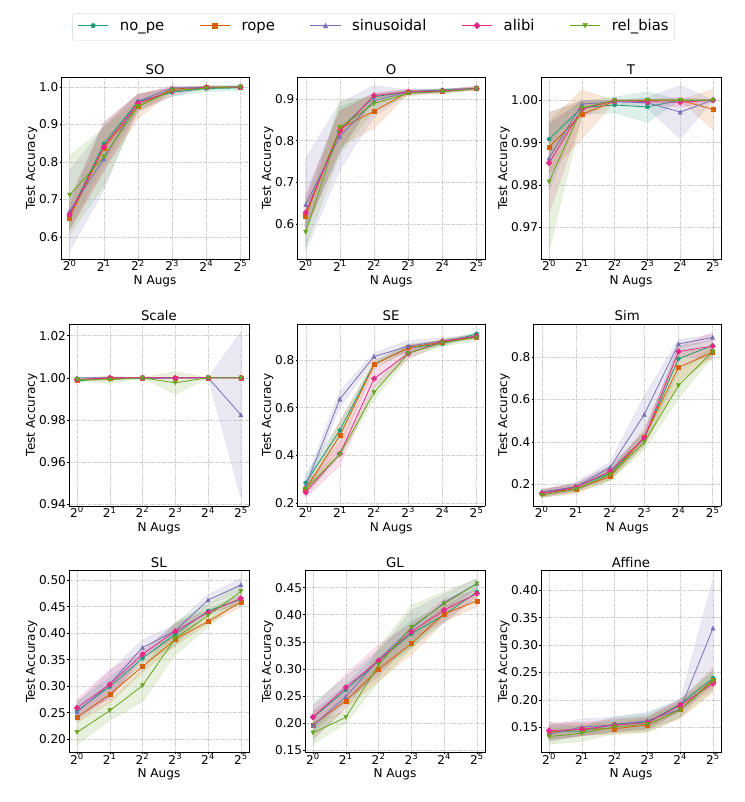}
    \caption{Test accuracy for different positional encoding approaches as measured for each symmetry group considered in the 3D Tetris dataset.}
    \label{fig:pes_geom_accuracy}
\end{figure}

\begin{figure}
    \centering
    \includegraphics[width=0.8\linewidth]{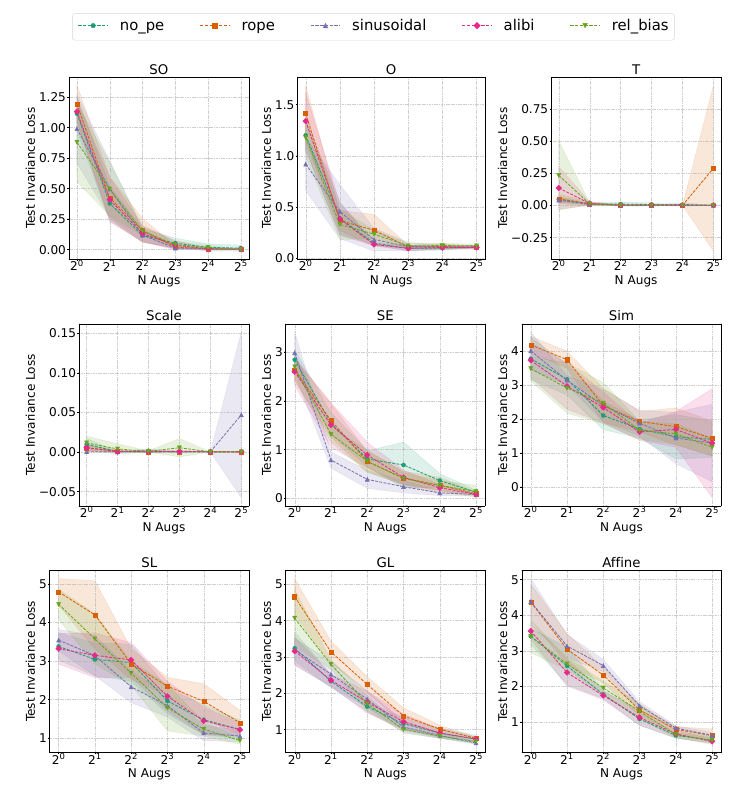}
    \caption{Test invariant error for different positional encoding approaches as measured for each symmetry group considered in the 3D Tetris dataset.}
    \label{fig:pes_geom_invariance}
\end{figure}

\begin{figure}[h!]
    \centering
    \includegraphics[width=\linewidth]{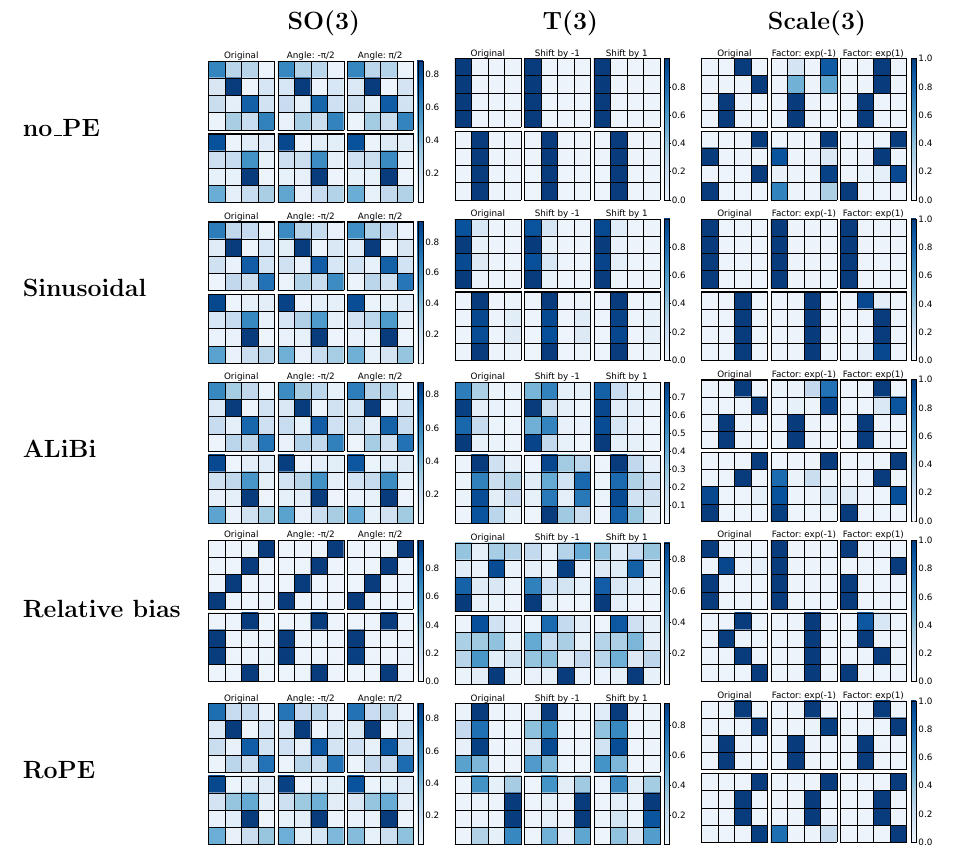}
    \caption{Attention softmax for first layer of a Transformer trained on $32$ augmentations using each considered geometric positional encoding and for each base angle-preserving symmetry. 
    The attention maps remain approximately invariant across augmentations for all positional encodings.}
    \label{fig:pe_attention_maps}
\end{figure}

\clearpage
\subsubsection{Multi-head attention}

\begin{wrapfigure}[19]{r}{0.45\textwidth}
\vspace{-5mm}
    \includegraphics[width=0.48\linewidth]{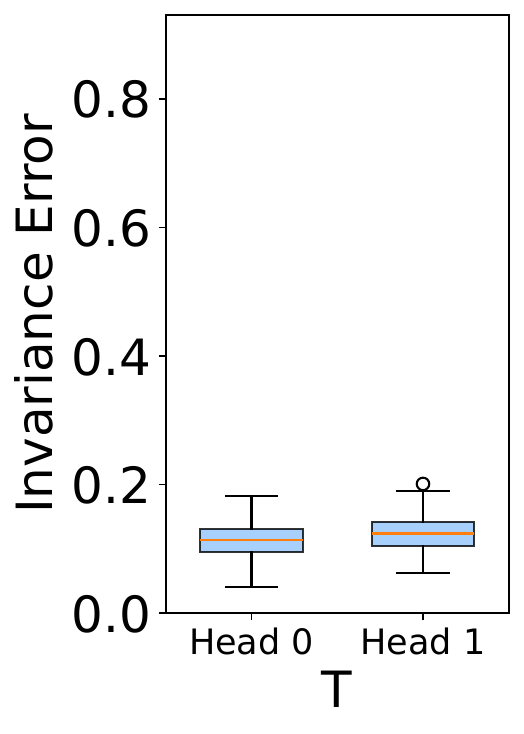}\hfill
    \includegraphics[width=0.48\linewidth]{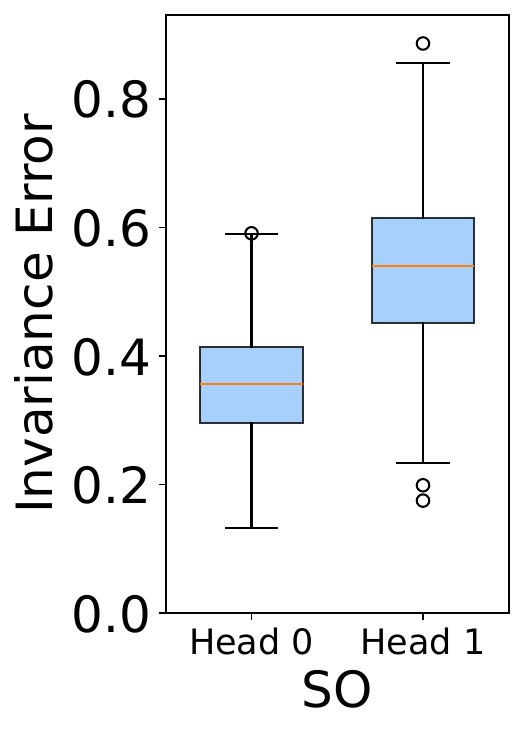}
    \caption{Invariance of attention heads to $\operatorname{T}(3)$ (left) and $\operatorname{SO}(3)$ (right) for first layer of a Transformer trained with $32$ $\operatorname{SE}(3)$-augmentations per sample. Head 0 is more invariant than head 1 to both base subgroups, so they do not seem to specialize in either of these groups. }
    \label{fig:mha_se}
\end{wrapfigure}

We carry out further experiments on the 3D Tetris dataset using 2 and 4 attention heads, respectively. To observe extrapolation behavior, we use the OOD Setting defined in \S\ref{sec:ood} and \S\ref{app:ood} for base angle-preserving groups. As can be seen in Figure \ref{fig:mha_tetris}, the qualitative behavior is very consistent across head counts. The only visible exception is extrapolation on $\operatorname{T}(3)$, which seems sensitive to hyperparameter settings in general with no clear trend regarding number of attention heads. Quantitatively, the performance on compositional groups both in terms off accuracy and invariance shows small improvements when using more heads. This suggests that heads might be specializing in base subgroups, easing the pressure to become invariant to multiple transformation types at once. However, we see no empirical evidence for this hypothesis. As shown in Figure \ref{fig:mha_se} by way of example, it is often one head that is more invariant than the other on all present base subgroups. Either the attention heads work together by a different mechanism or since performance gains are small in the first place, they might be attributable to increased parameter counts. Overall, the number of attention heads seems to have no major impact on the results presented in this paper.

\begin{figure}[t]
\centering

\newlength{\mhaplotwidth}
\newlength{\mhalabelwidth}
\setlength{\mhaplotwidth}{0.22\textwidth}
\setlength{\mhalabelwidth}{0.04\textwidth}

\begin{tabular}{
  >{\centering\arraybackslash}m{\mhalabelwidth}
  @{\hspace{2pt}}>{\centering\arraybackslash}m{\mhaplotwidth}
  @{\hspace{2pt}}>{\centering\arraybackslash}m{\mhaplotwidth}
  @{\hspace{6pt}}>{\centering\arraybackslash}m{\mhaplotwidth}
  @{\hspace{2pt}}>{\centering\arraybackslash}m{\mhaplotwidth}
}

  & \multicolumn{2}{c}{\textbf{In Distribution}} & \multicolumn{2}{c}{\textbf{Out of Distribution}} \\[2pt]

  \rotatebox[origin=c]{90}{\small \textbf{1 Head}}
  & \includegraphics[width=\mhaplotwidth]{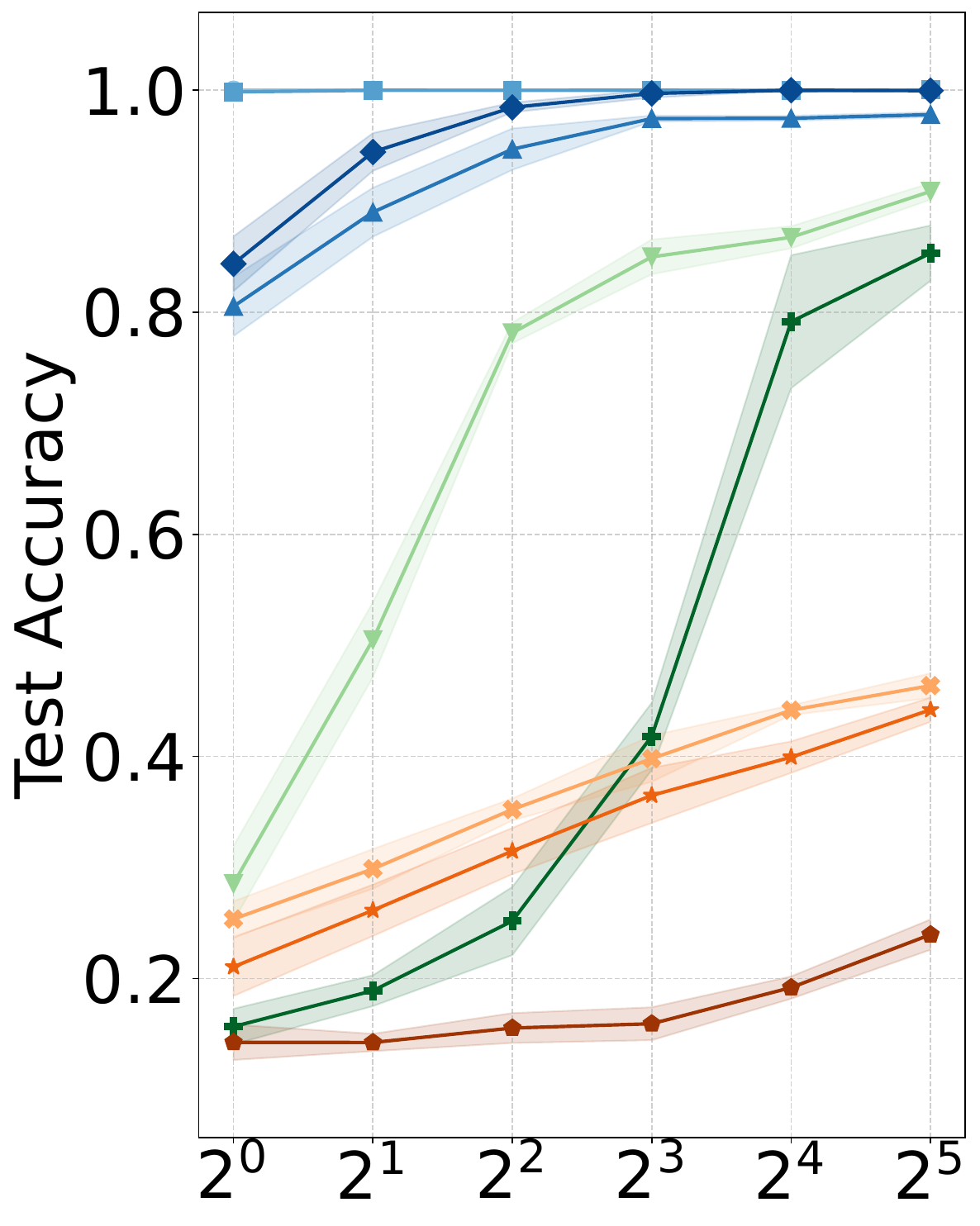}
  & \includegraphics[width=\mhaplotwidth]{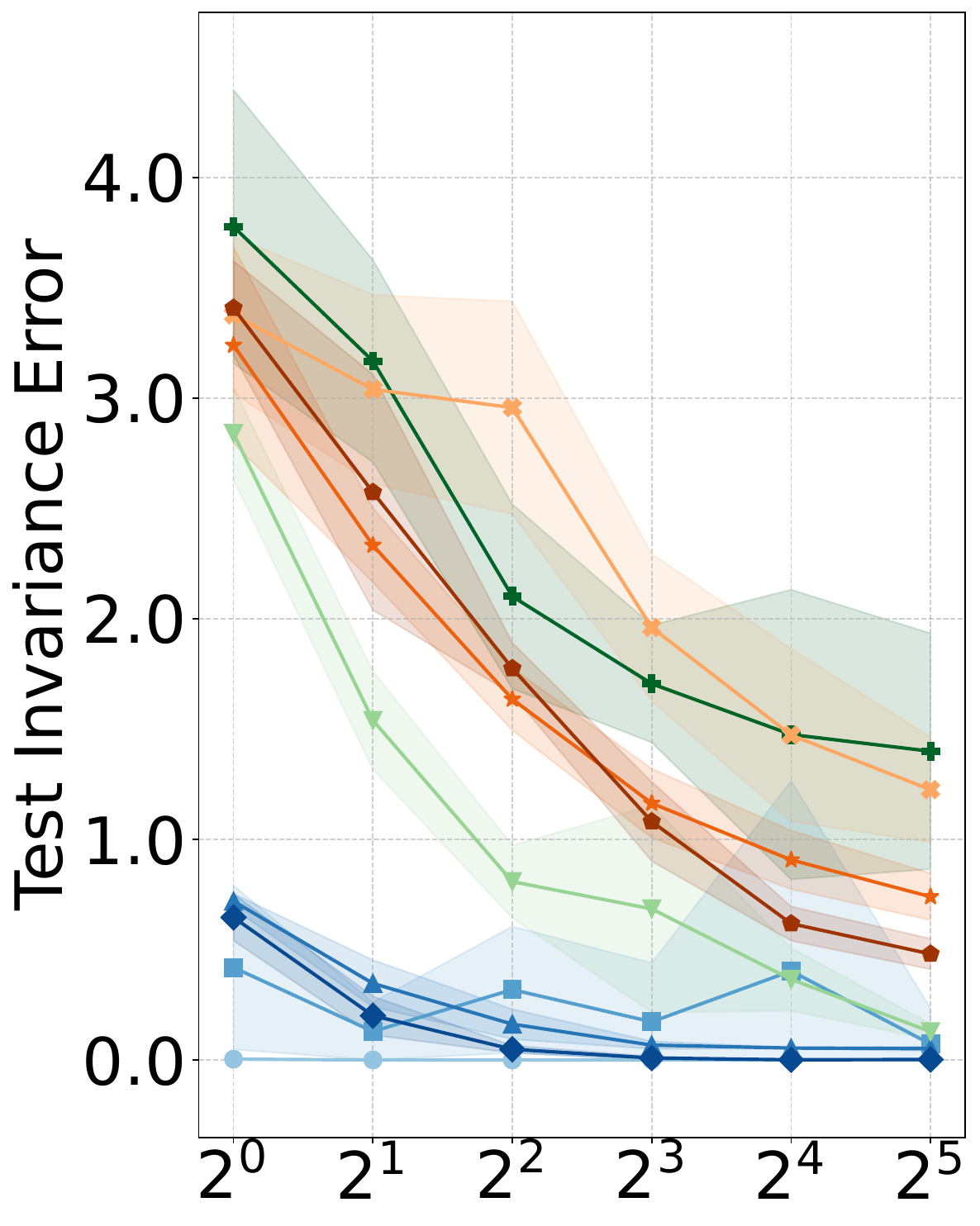}
  & \includegraphics[width=\mhaplotwidth]{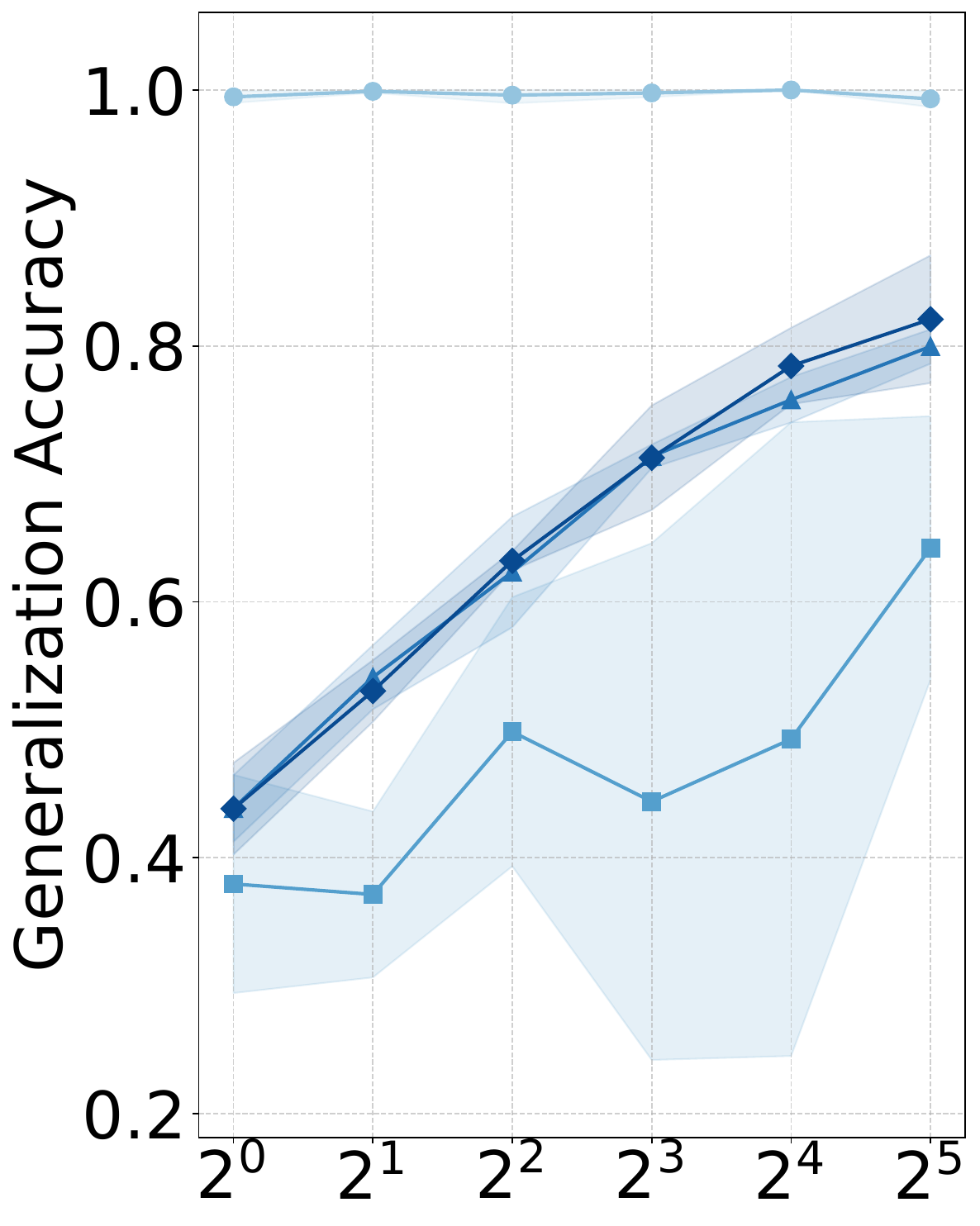}
  & \includegraphics[width=\mhaplotwidth]{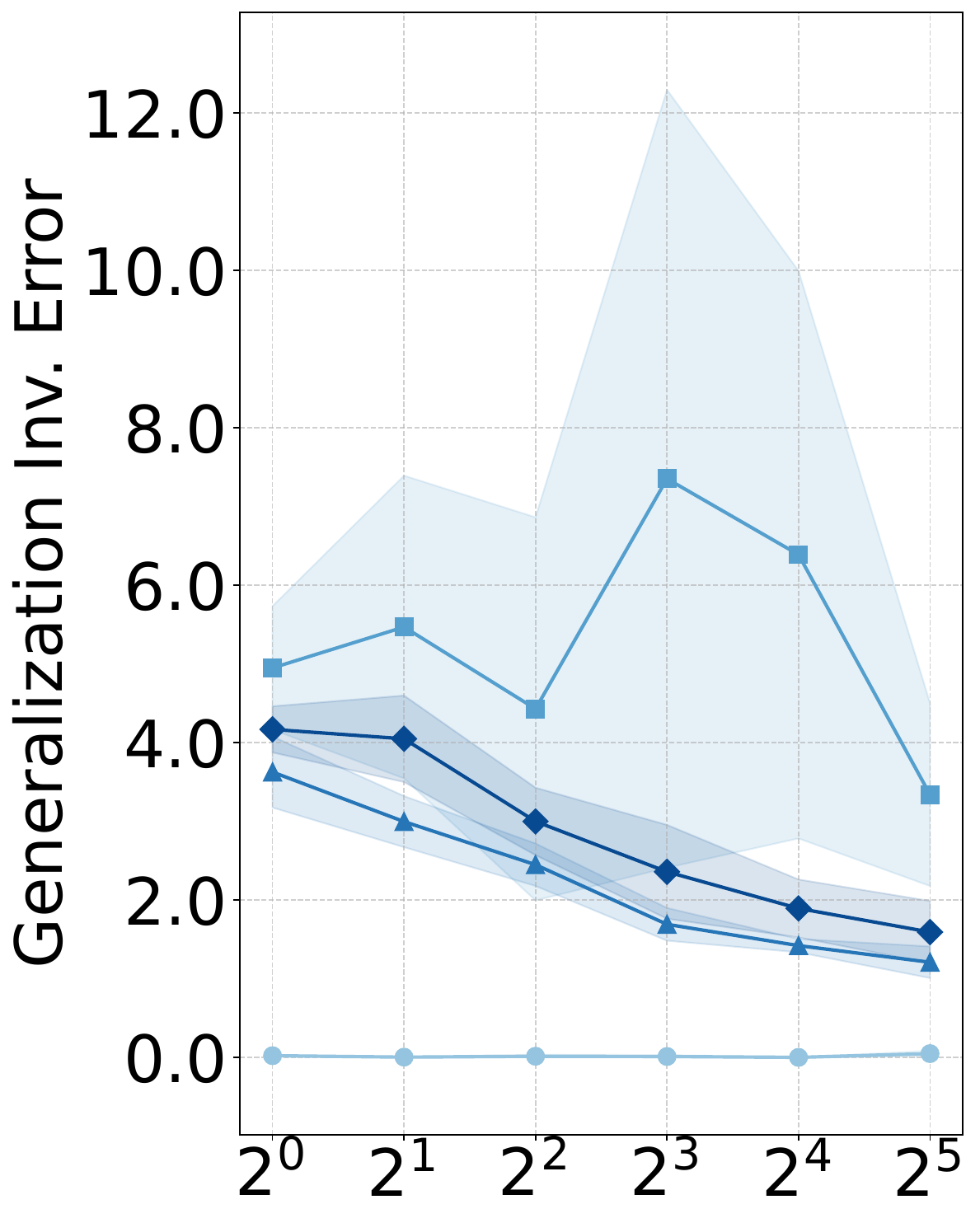} \\[4pt]

  \rotatebox[origin=c]{90}{\small \textbf{2 Heads}}
  & \includegraphics[width=\mhaplotwidth]{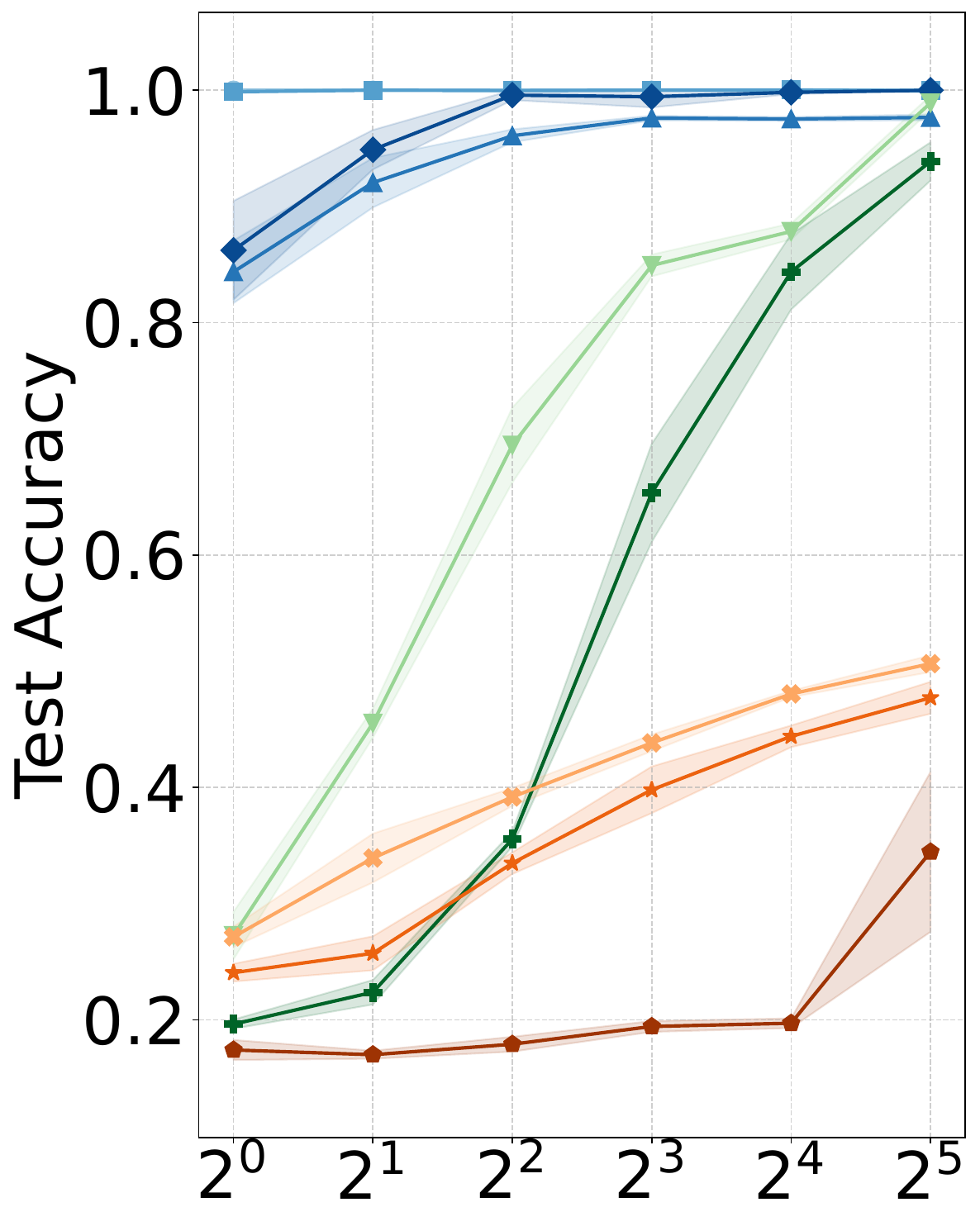}
  & \includegraphics[width=\mhaplotwidth]{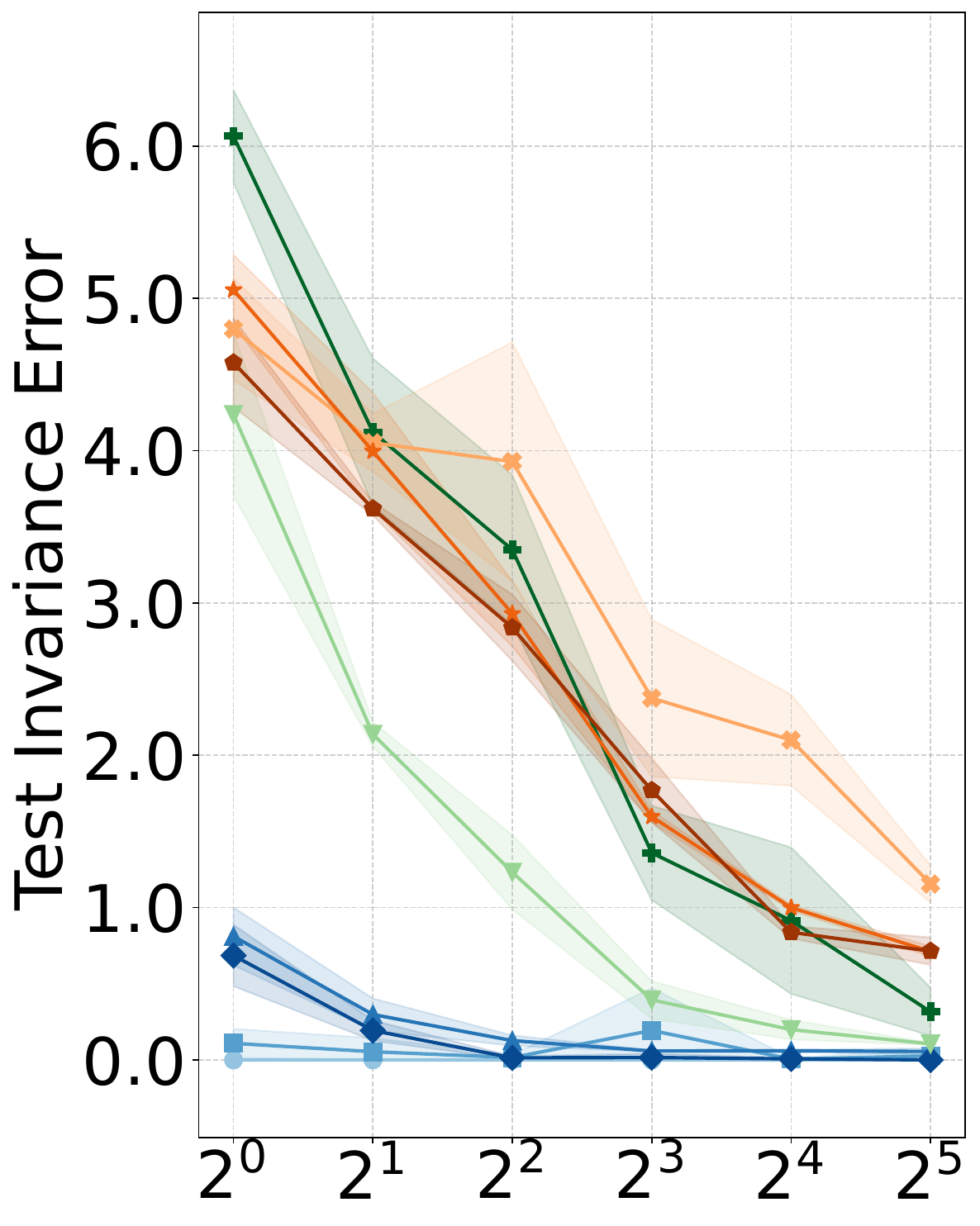}
  & \includegraphics[width=\mhaplotwidth]{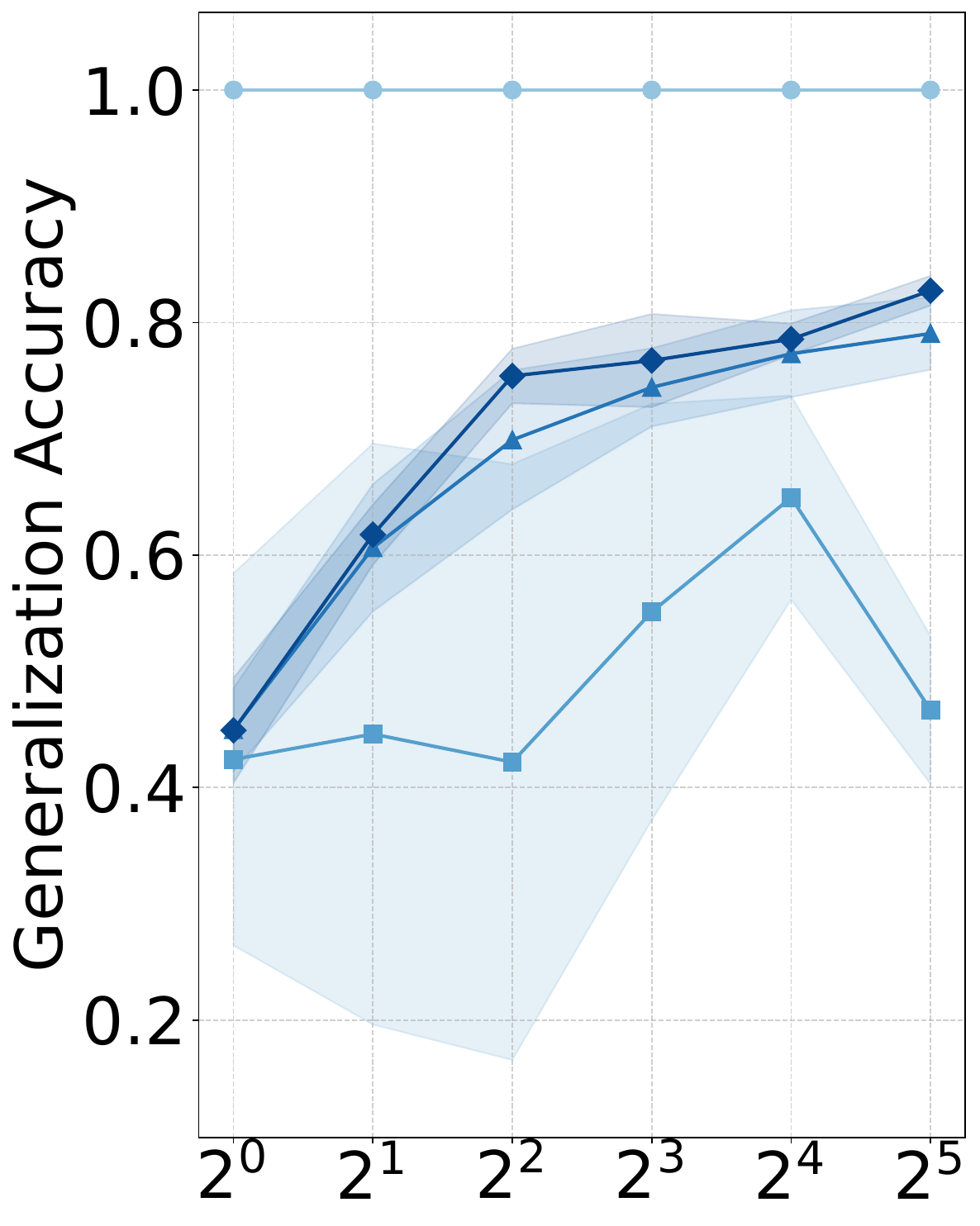}
  & \includegraphics[width=\mhaplotwidth]{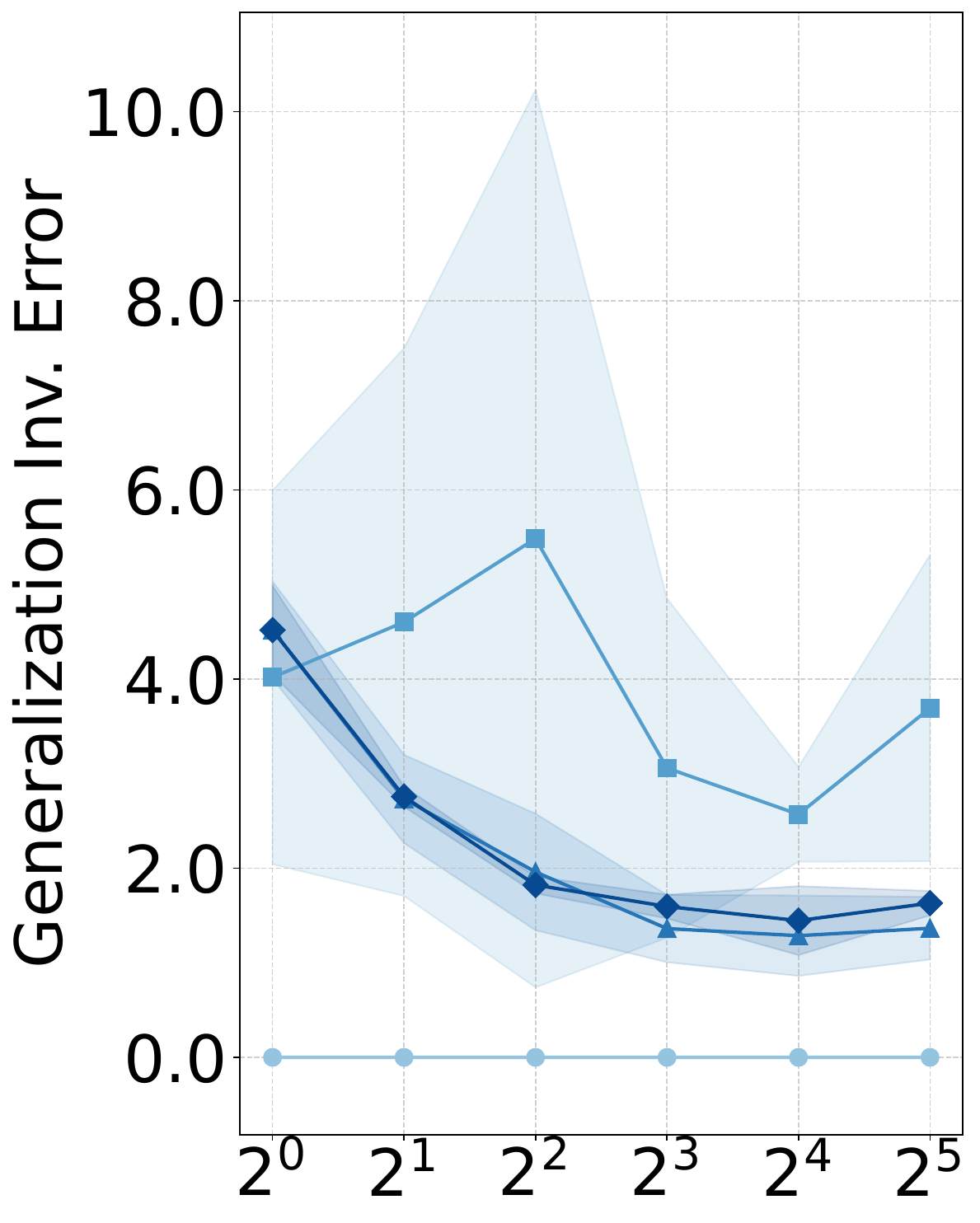} \\[4pt]

  \rotatebox[origin=c]{90}{\small \textbf{4 Heads}}
  & \includegraphics[width=\mhaplotwidth]{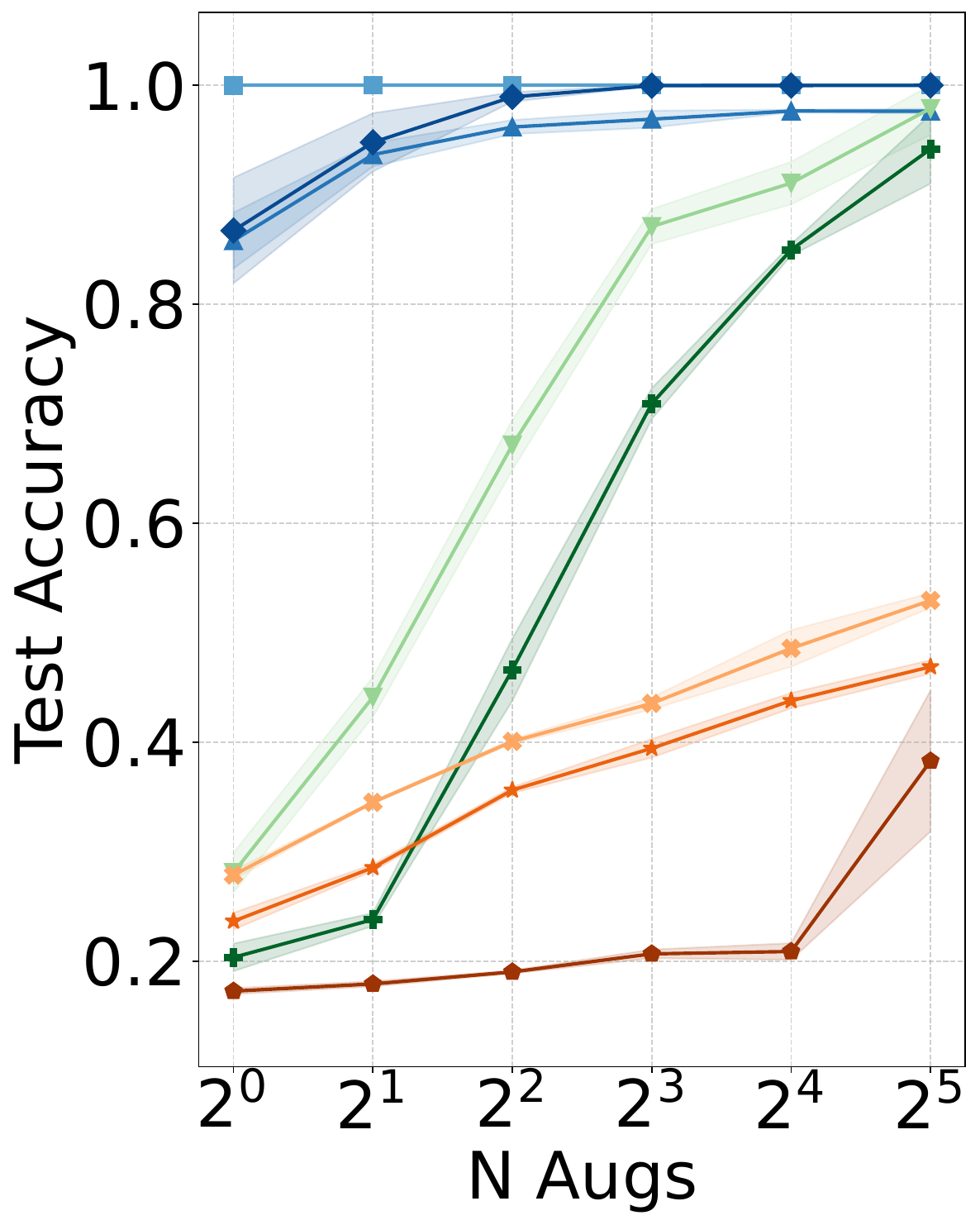}
  & \includegraphics[width=\mhaplotwidth]{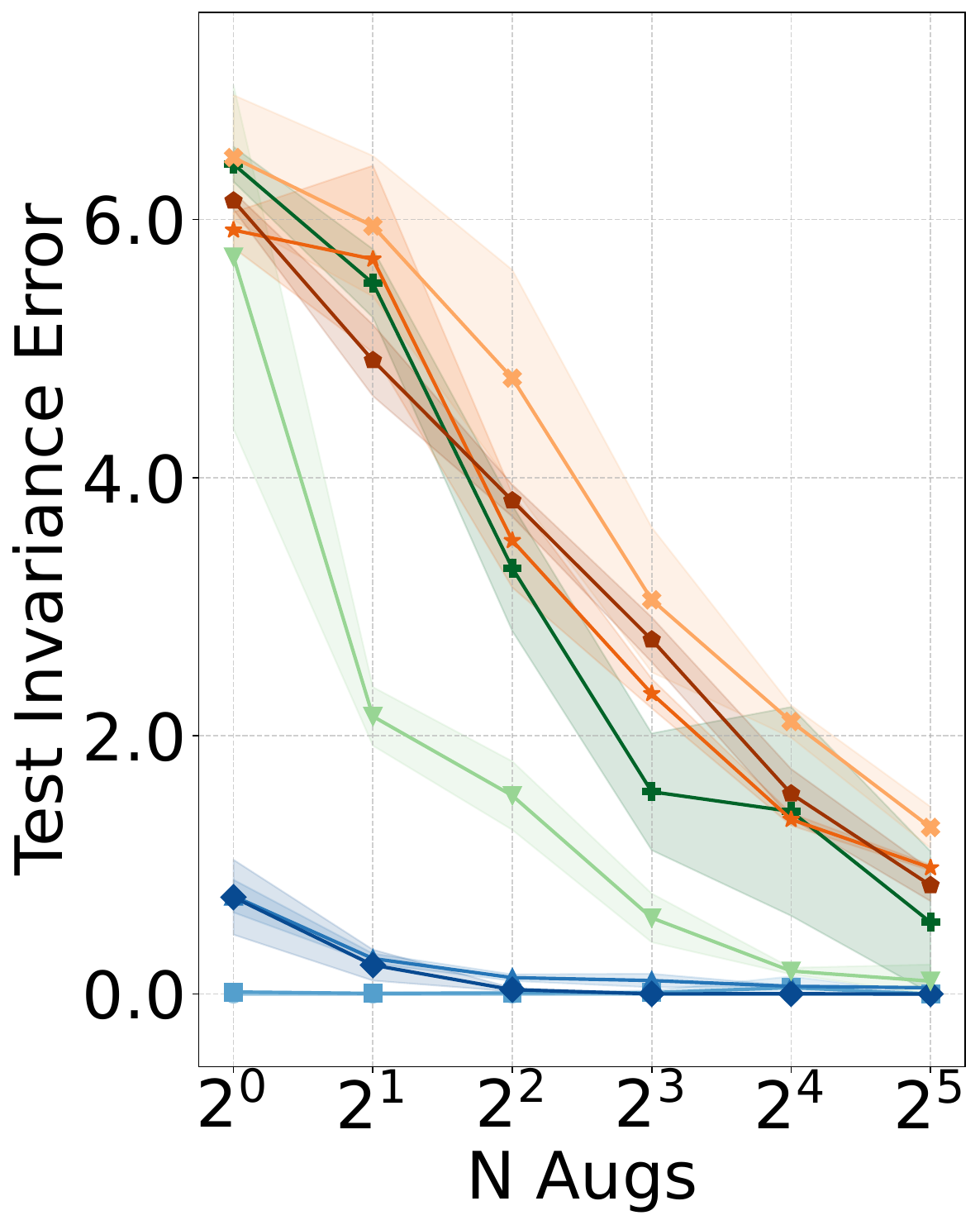}
  & \includegraphics[width=\mhaplotwidth]{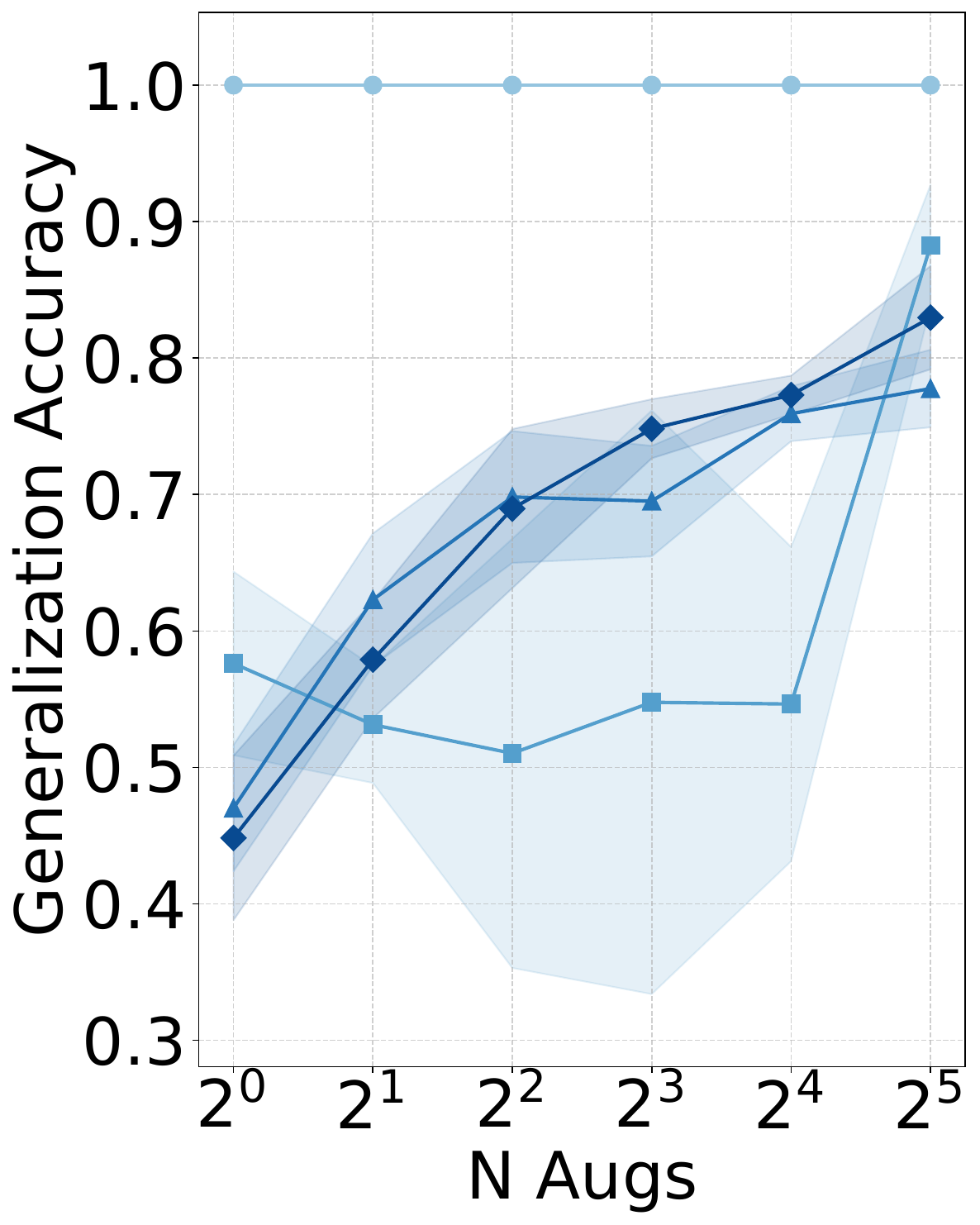}
  & \includegraphics[width=\mhaplotwidth]{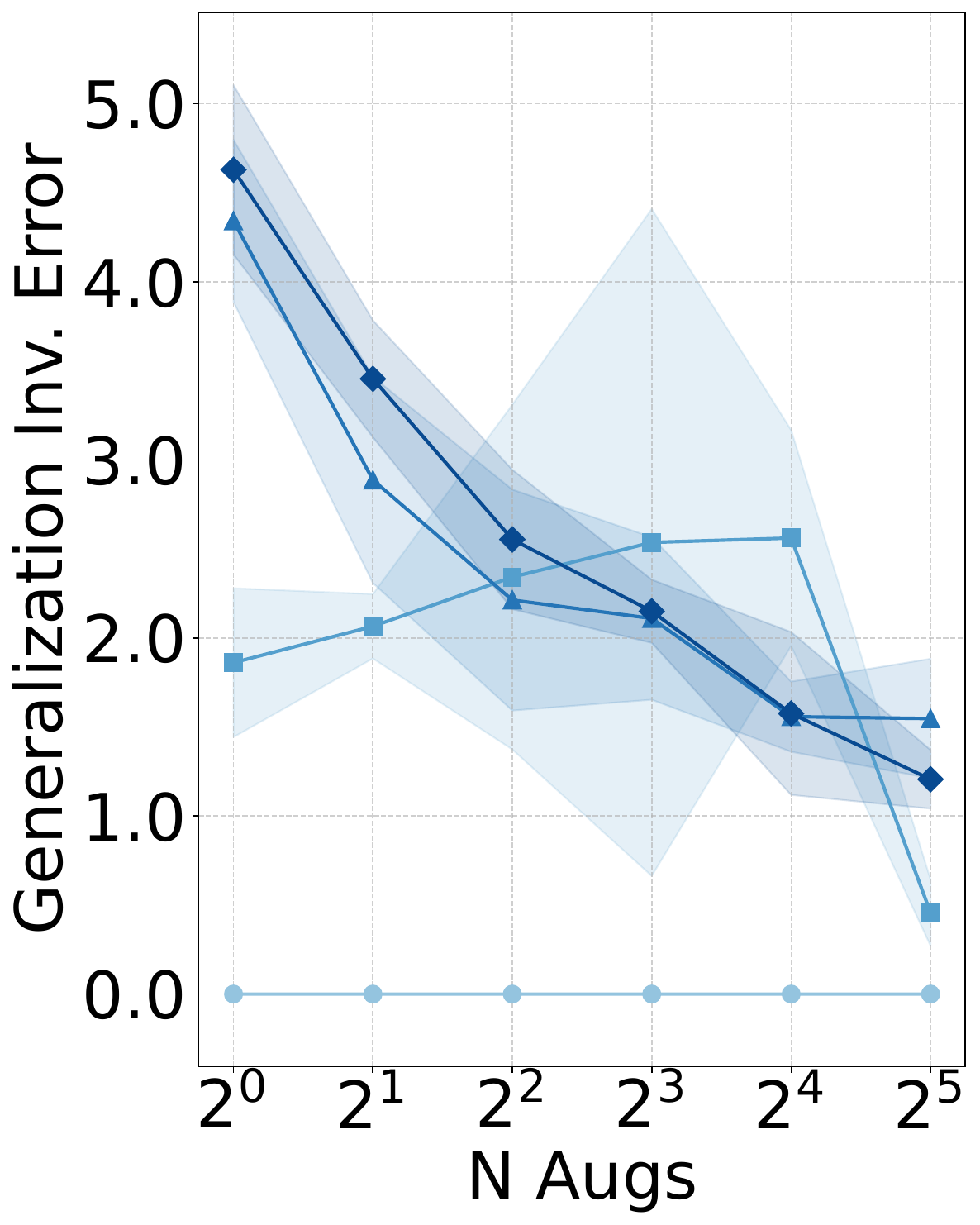} \\
\end{tabular}

\includegraphics[width=\linewidth]{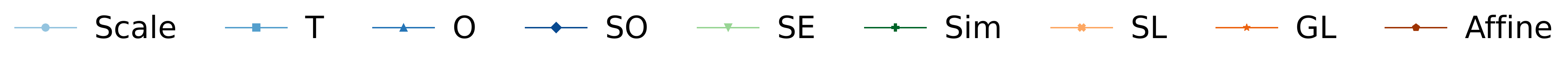}

\caption{In- and out-of-distribution results for 3D Tetris comparing 1, 2, and 4 attention heads. Increasing the number of heads slightly increases performance on compositional groups. The only visible qualitative difference seems to be extrapolation on $\operatorname{T}(3)$, which seems sensitive to hyperparameter settings with no clear trend. }
\label{fig:mha_tetris}
\end{figure}

\subsection{Signature of detected mechanisms for additional models}
\label{app:signature_mechanisms}

\subsubsection{Translation}

As we can see in Figure~\ref{fig:modelnet_translation_grid_layer0}, the detected mechanism prevails across seeds and number of augmentations for the first block of Transformers trained on ModelNet10 with translations. That is, in the first block, self-attention subtracts approximately equivariant features from the residuals. The second block then displays highly invariant attention maps, as shown in Figure~\ref{fig:modelnet_translation_grid_layer1}.

\begin{figure}[h]
\centering
\renewcommand{\arraystretch}{1.5}
\begin{tabular}{ccc}
  Pointwise Cosine Similarity of $\mX$ and $\SA(\mX)$ & Invariance Loss of $\A(\mX)$ & N Augs \\[6pt]
  \includegraphics[width=0.45\linewidth]{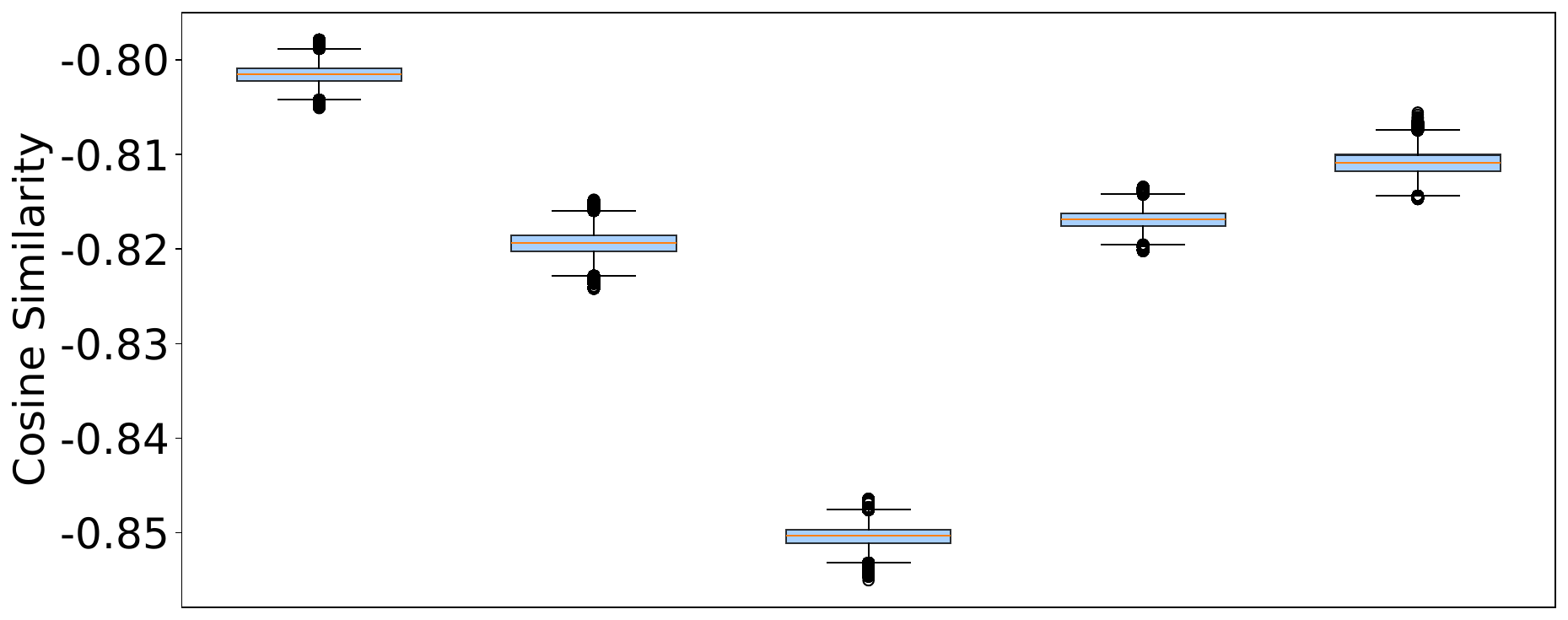} &
  \includegraphics[width=0.45\linewidth]{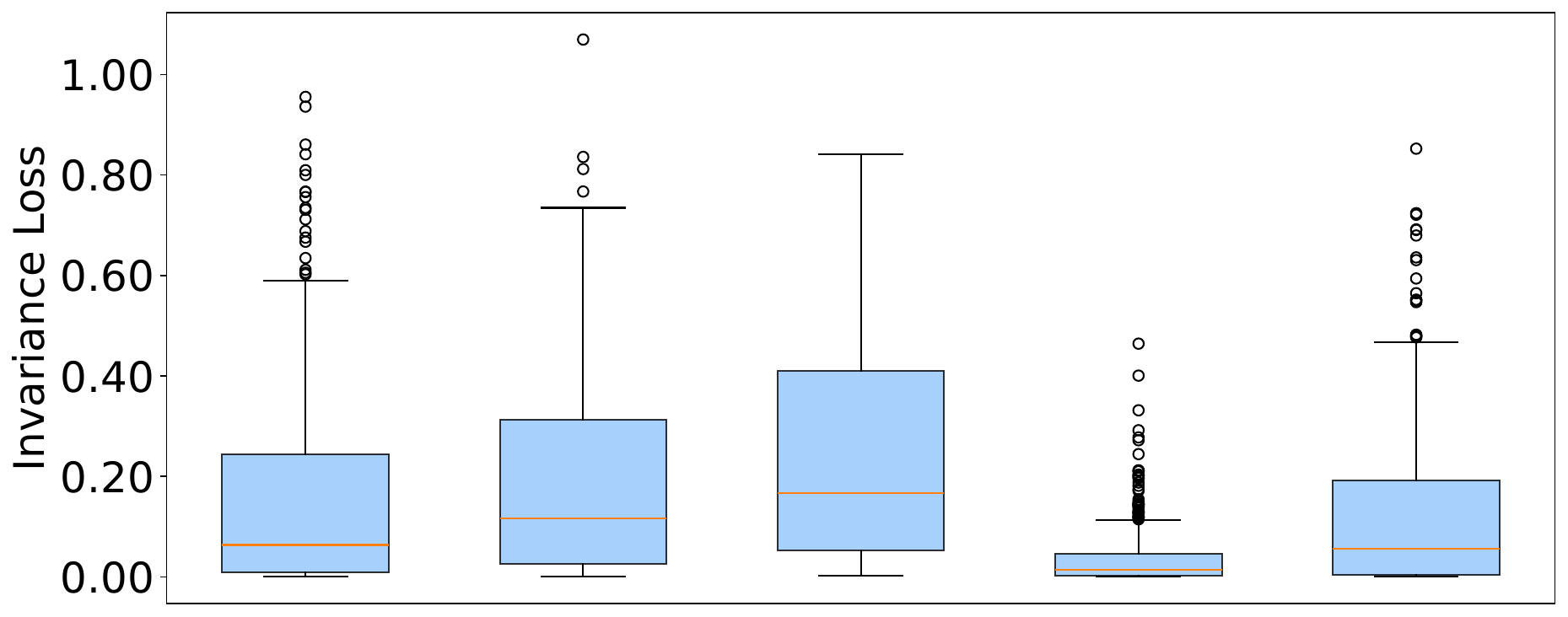} &
  \raisebox{1.2cm}{1} \\[4pt]
  \includegraphics[width=0.45\linewidth]{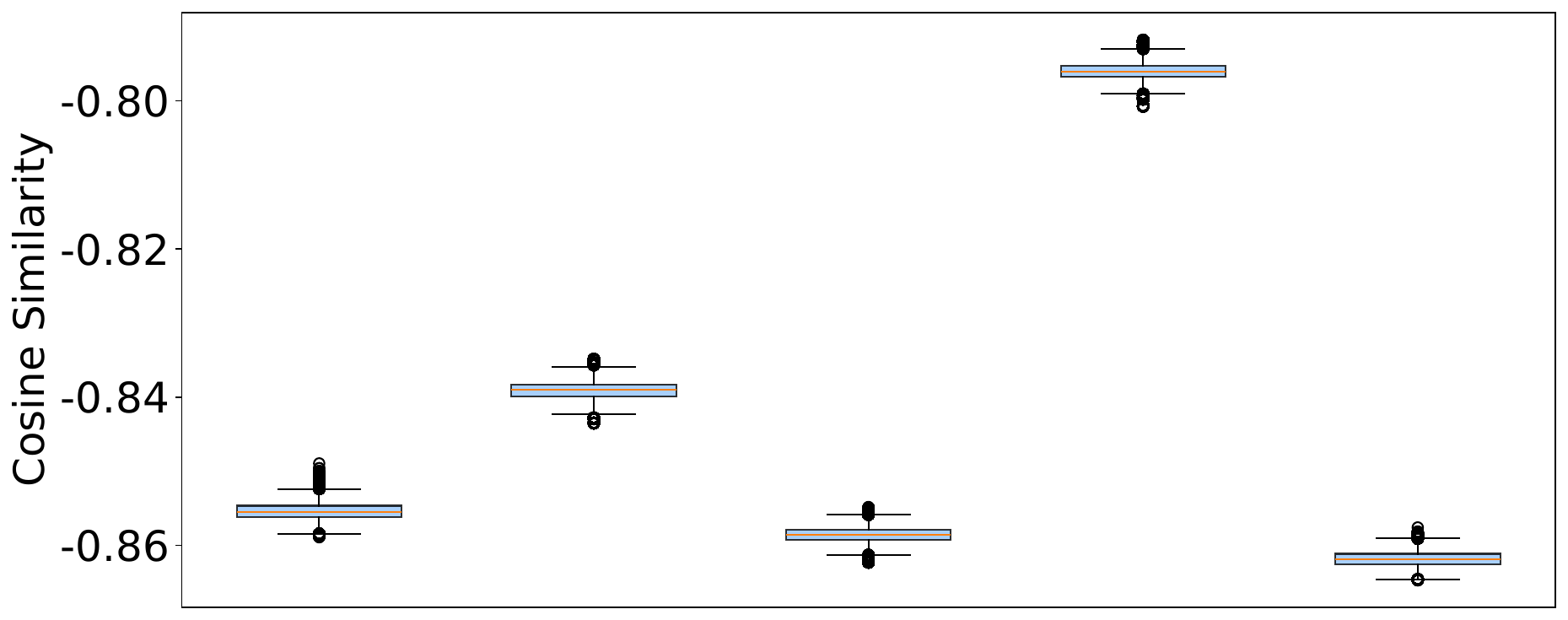} &
  \includegraphics[width=0.45\linewidth]{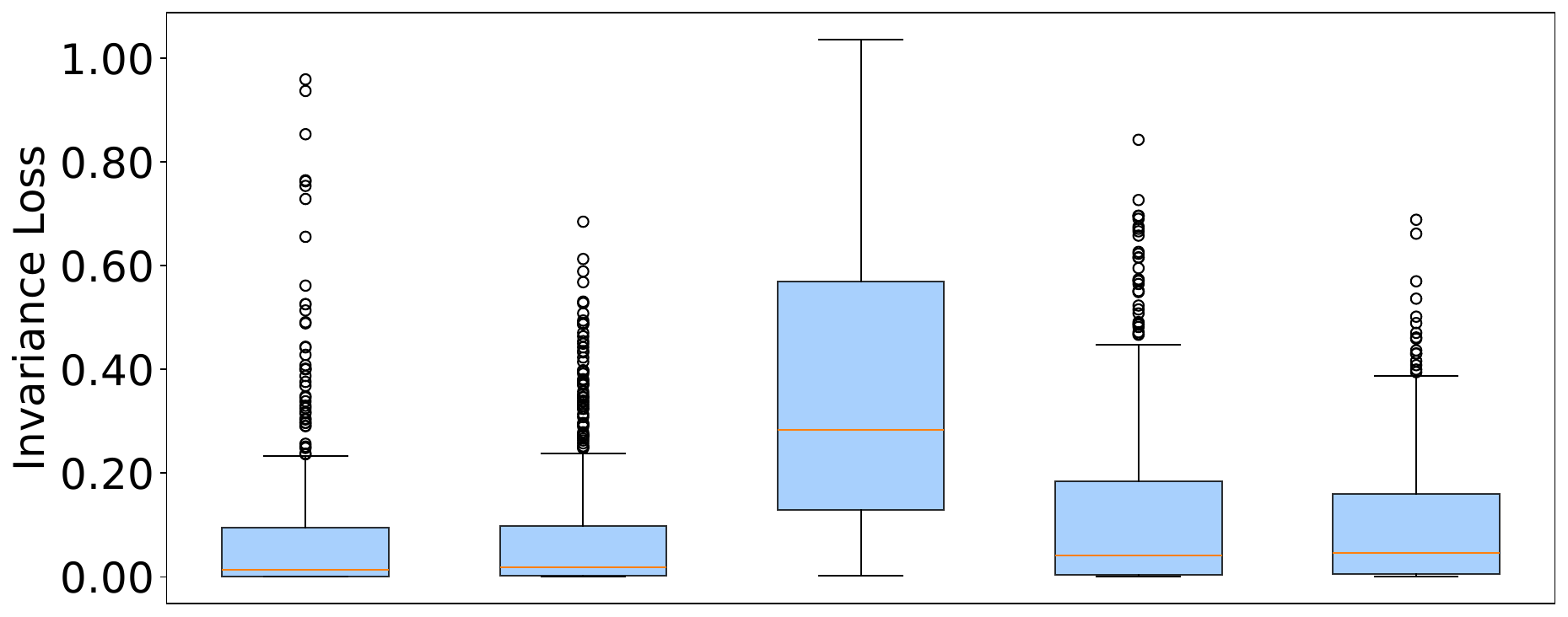} &
  \raisebox{1.2cm}{2} \\[4pt]
  \includegraphics[width=0.45\linewidth]{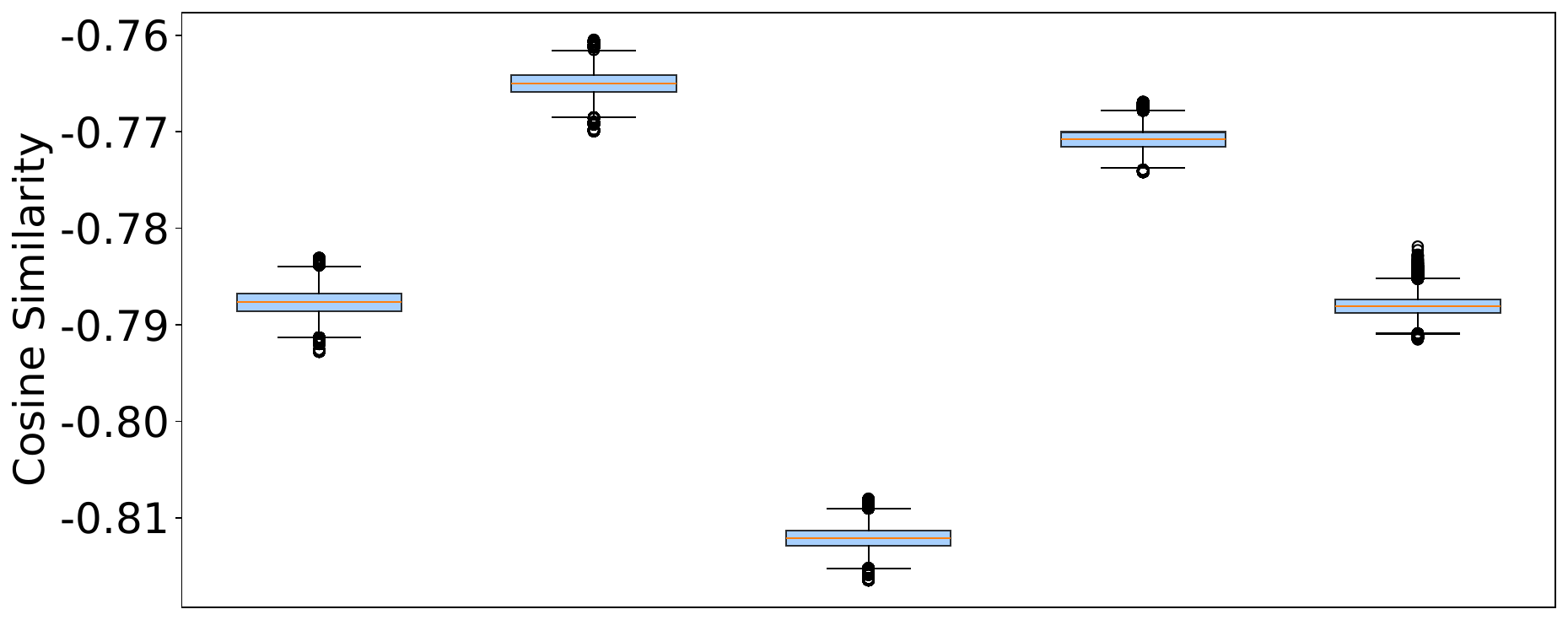} &
  \includegraphics[width=0.45\linewidth]{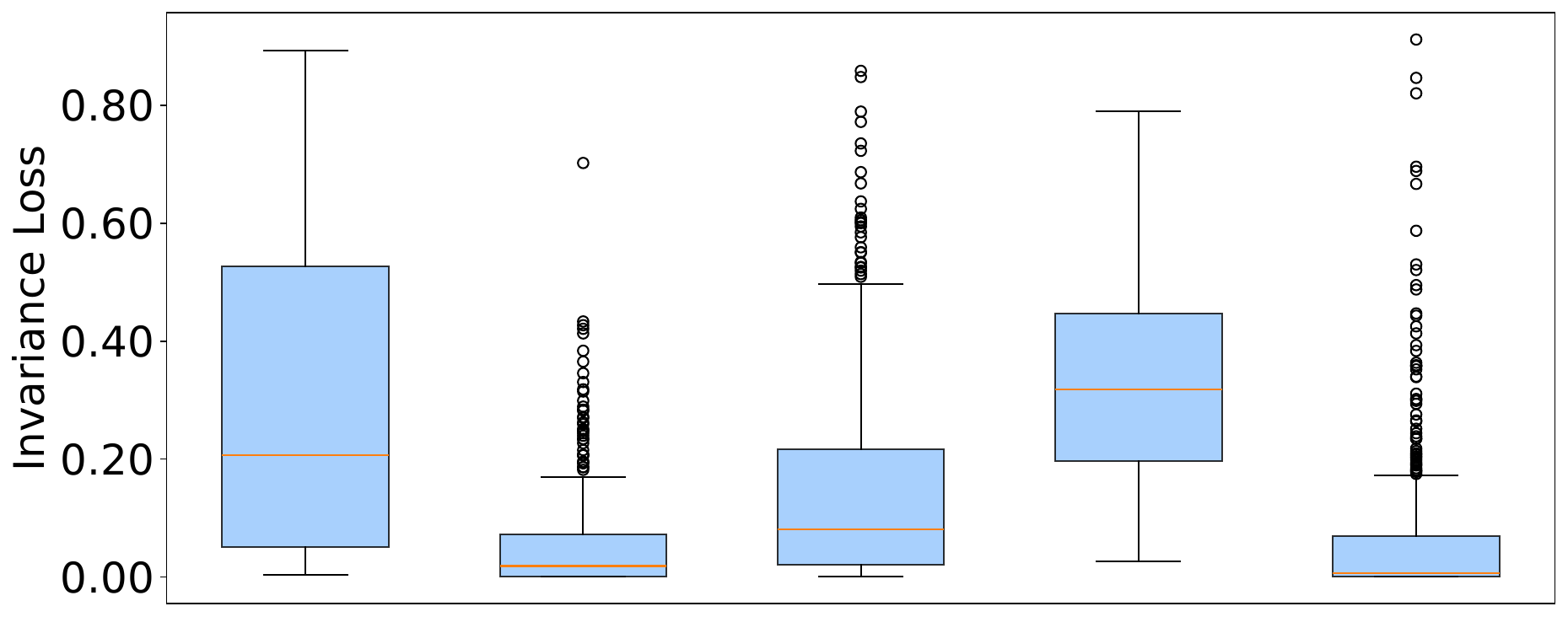} &
  \raisebox{1.2cm}{4} \\[4pt]
  \includegraphics[width=0.45\linewidth]{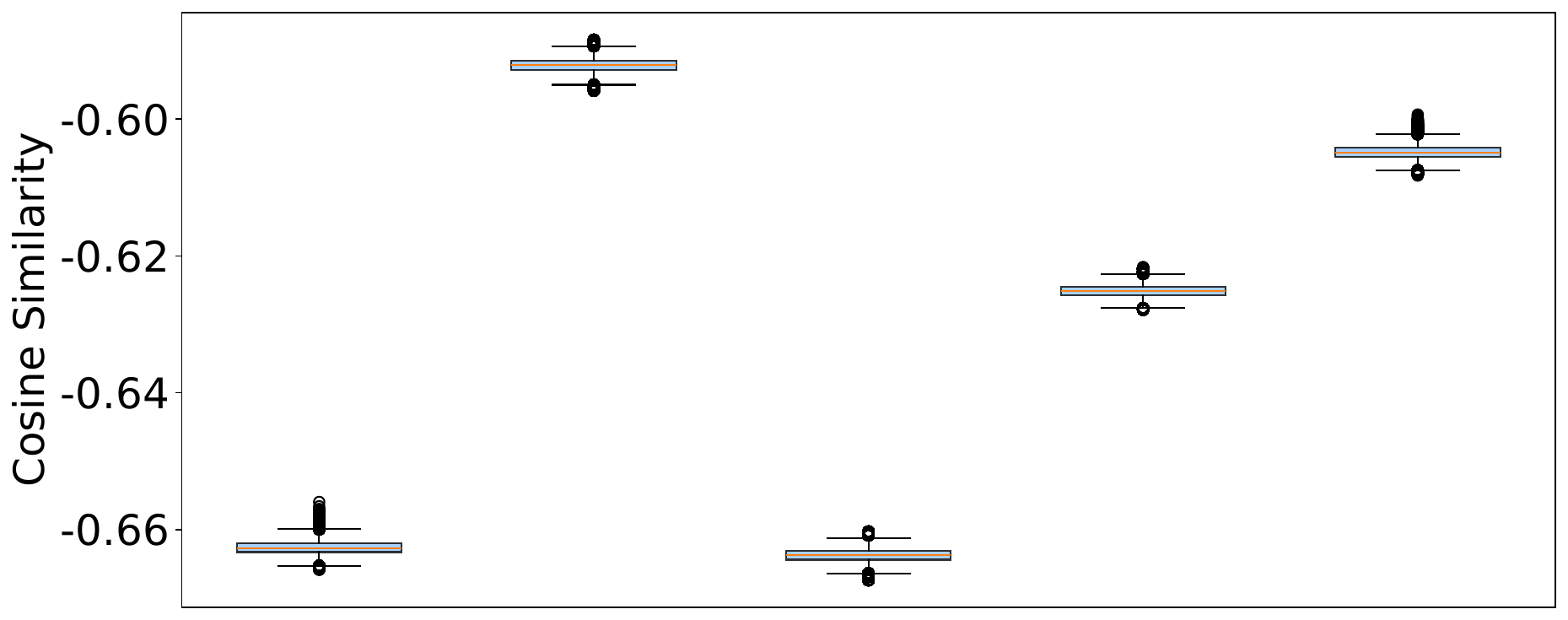} &
  \includegraphics[width=0.45\linewidth]{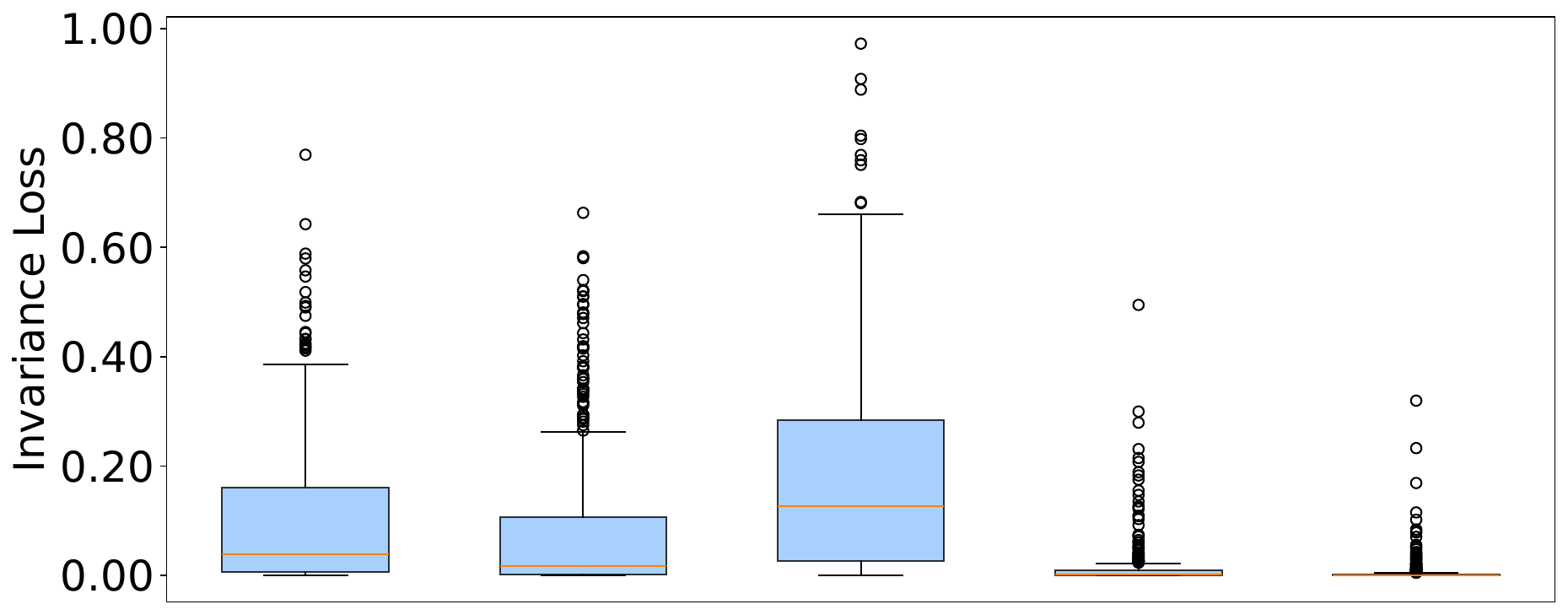} &
  \raisebox{1.2cm}{8} \\[4pt]
  \includegraphics[width=0.45\linewidth]{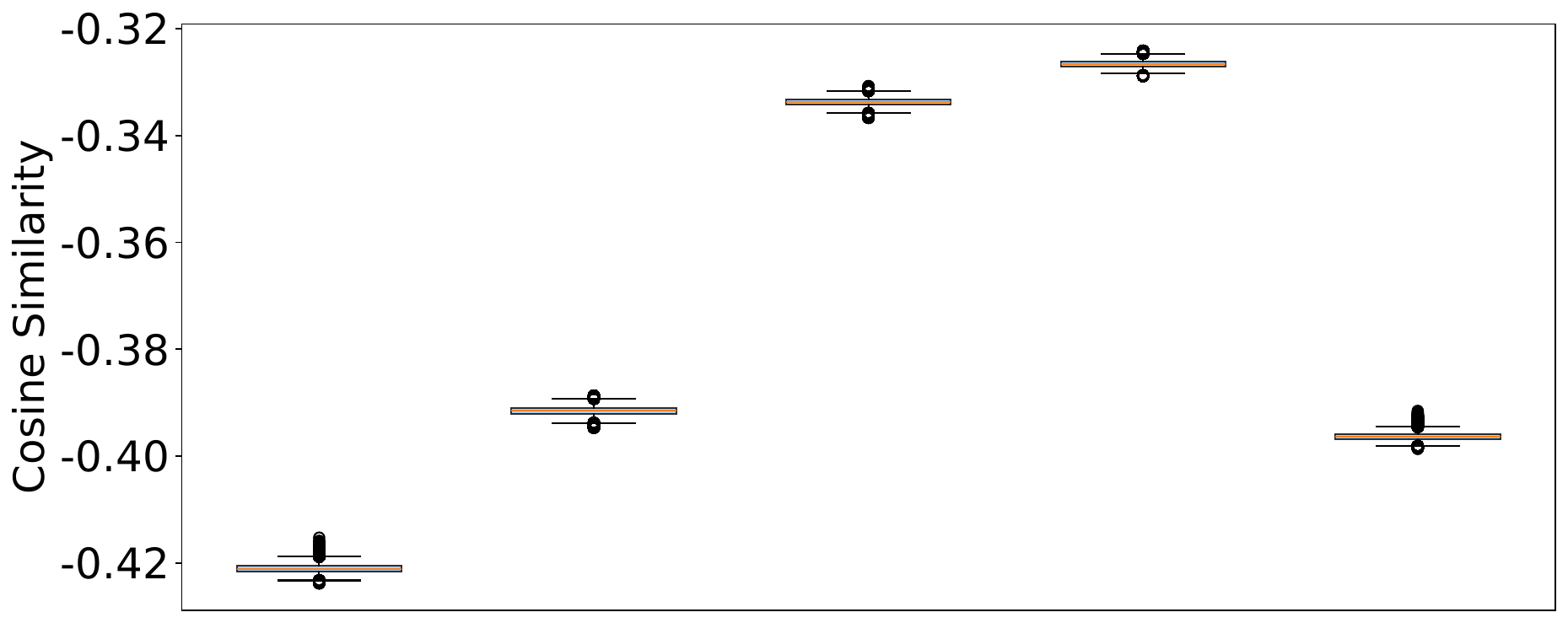} &
  \includegraphics[width=0.45\linewidth]{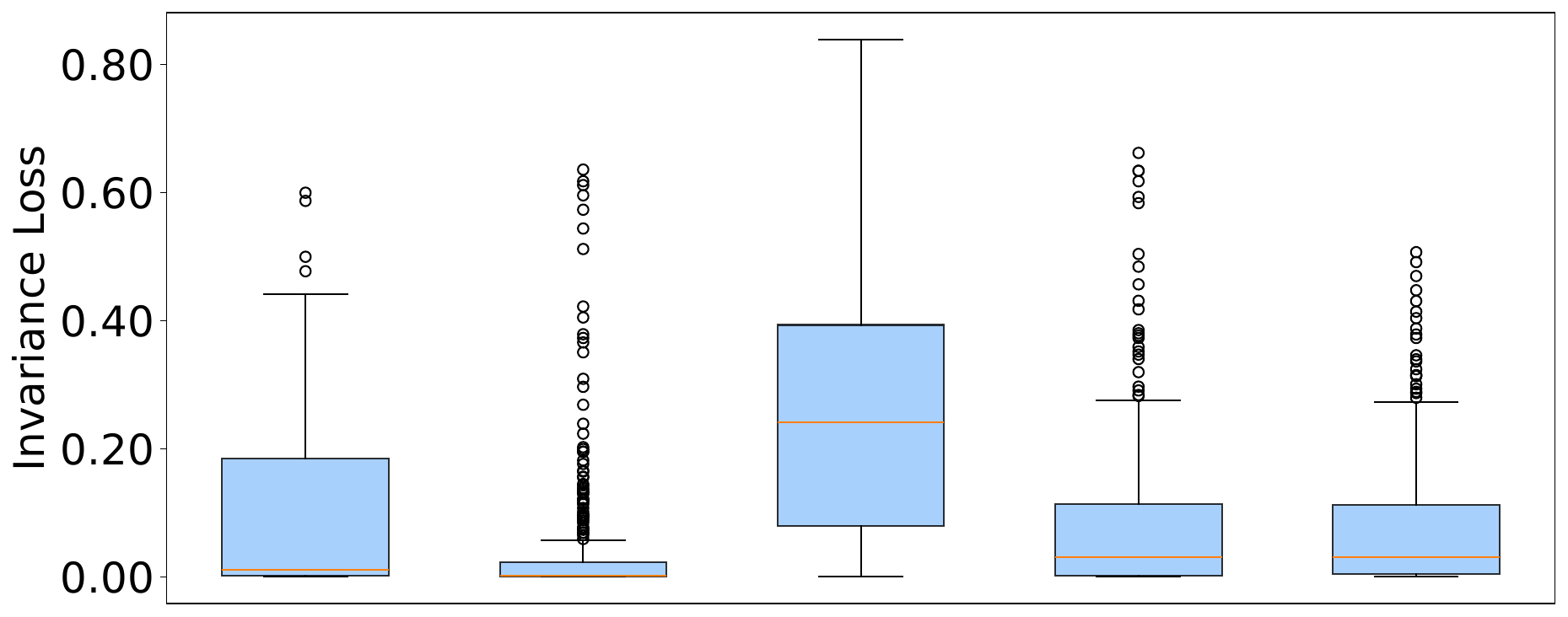} &
  \raisebox{1.2cm}{16} \\[4pt]
\end{tabular}
\caption{Mechanism indicators for first layer of all models trained on ModelNet10 with translations.}
\label{fig:modelnet_translation_grid_layer0}
\end{figure}

\begin{figure}[h]
\centering
\renewcommand{\arraystretch}{1.5}
\begin{tabular}{ccc}
  Pointwise Cosine Similarity of $\mX$ and $\SA(\mX)$ & Invariance Loss of $\A(\mX)$ & N Augs \\[6pt]
  \includegraphics[width=0.45\linewidth]{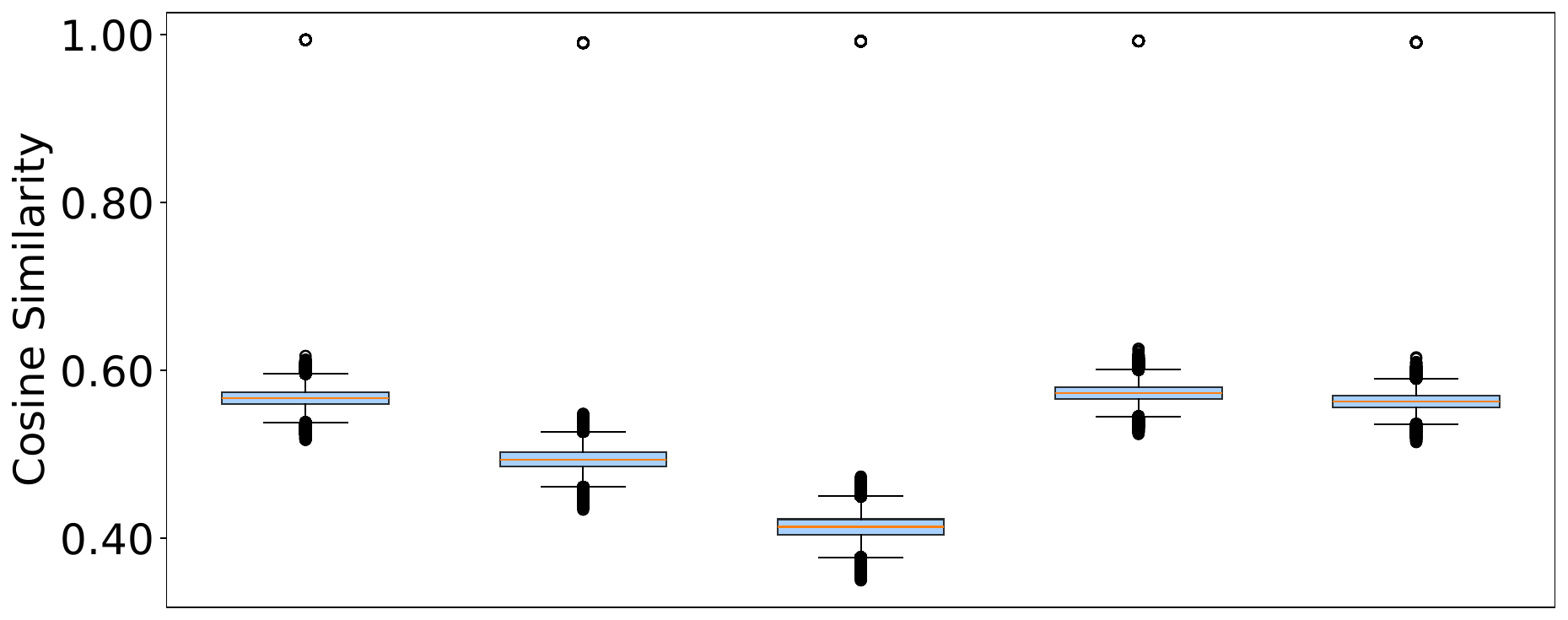} &
  \includegraphics[width=0.45\linewidth]{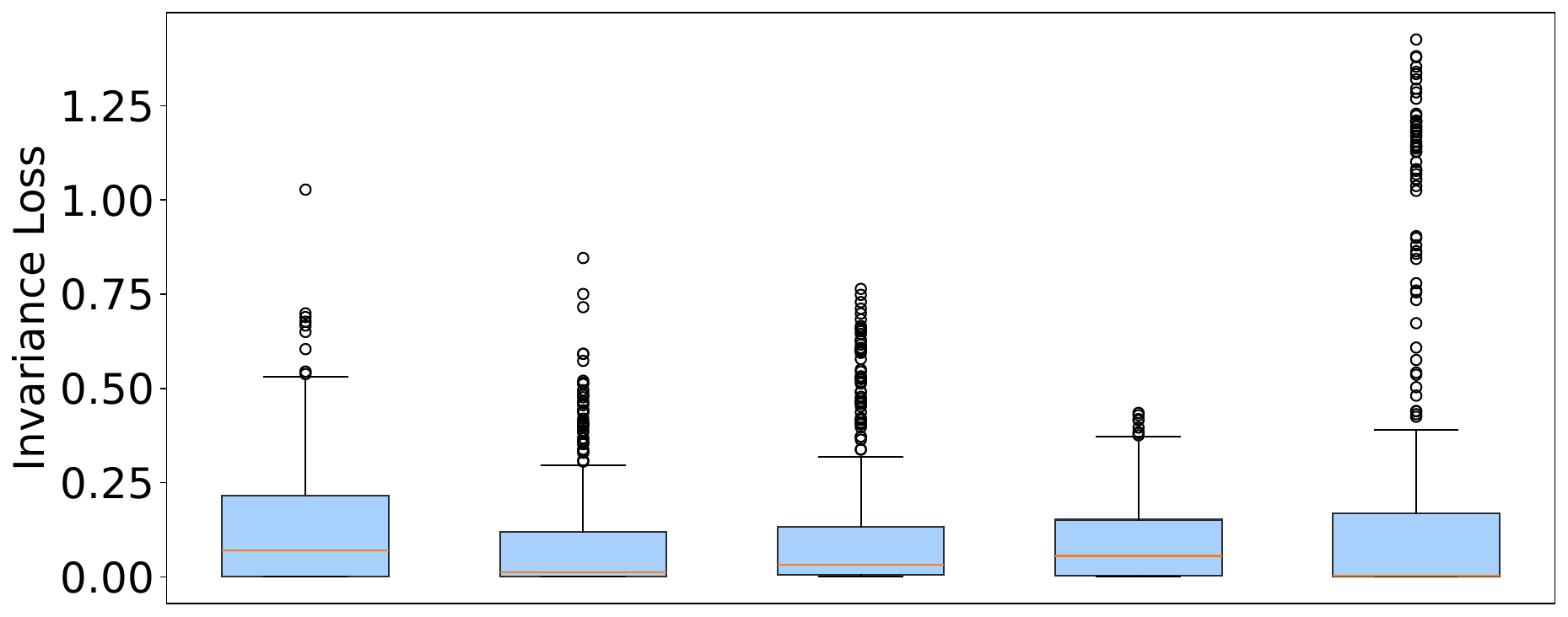} &
  \raisebox{1.2cm}{1} \\[4pt]
  \includegraphics[width=0.45\linewidth]{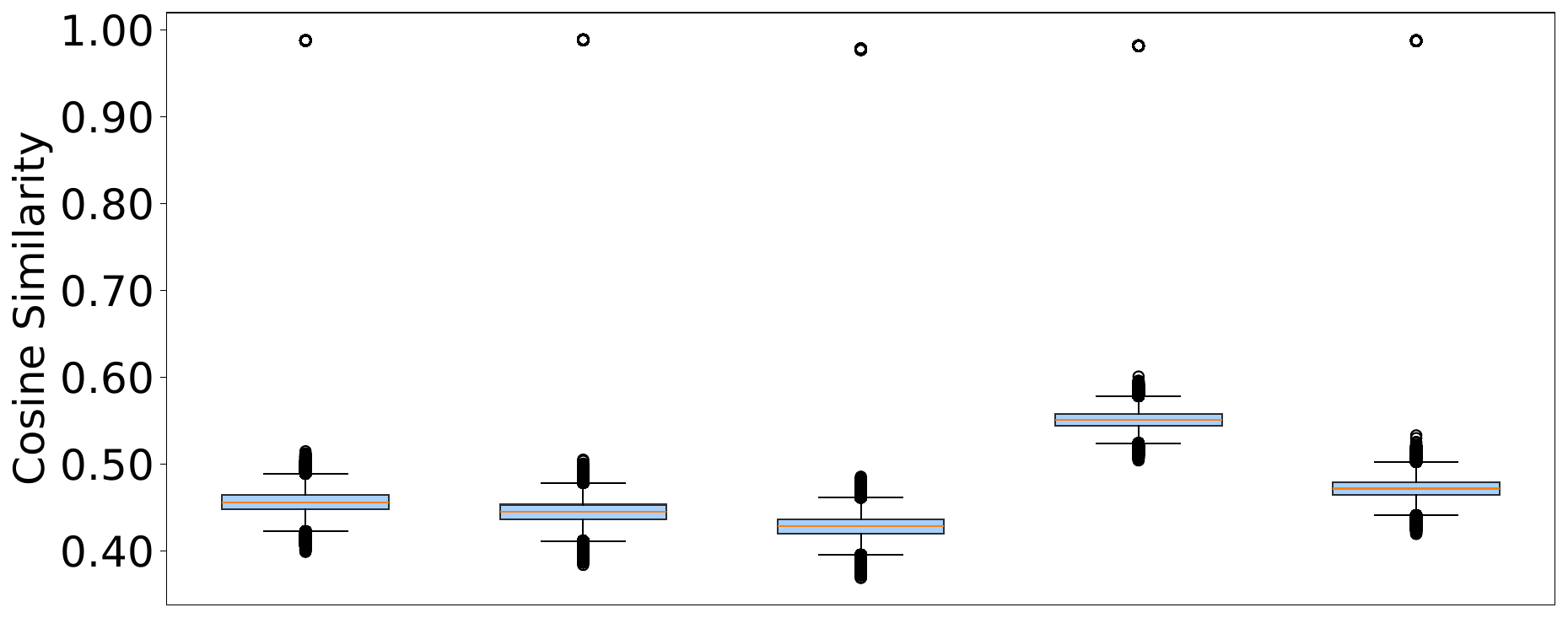} &
  \includegraphics[width=0.45\linewidth]{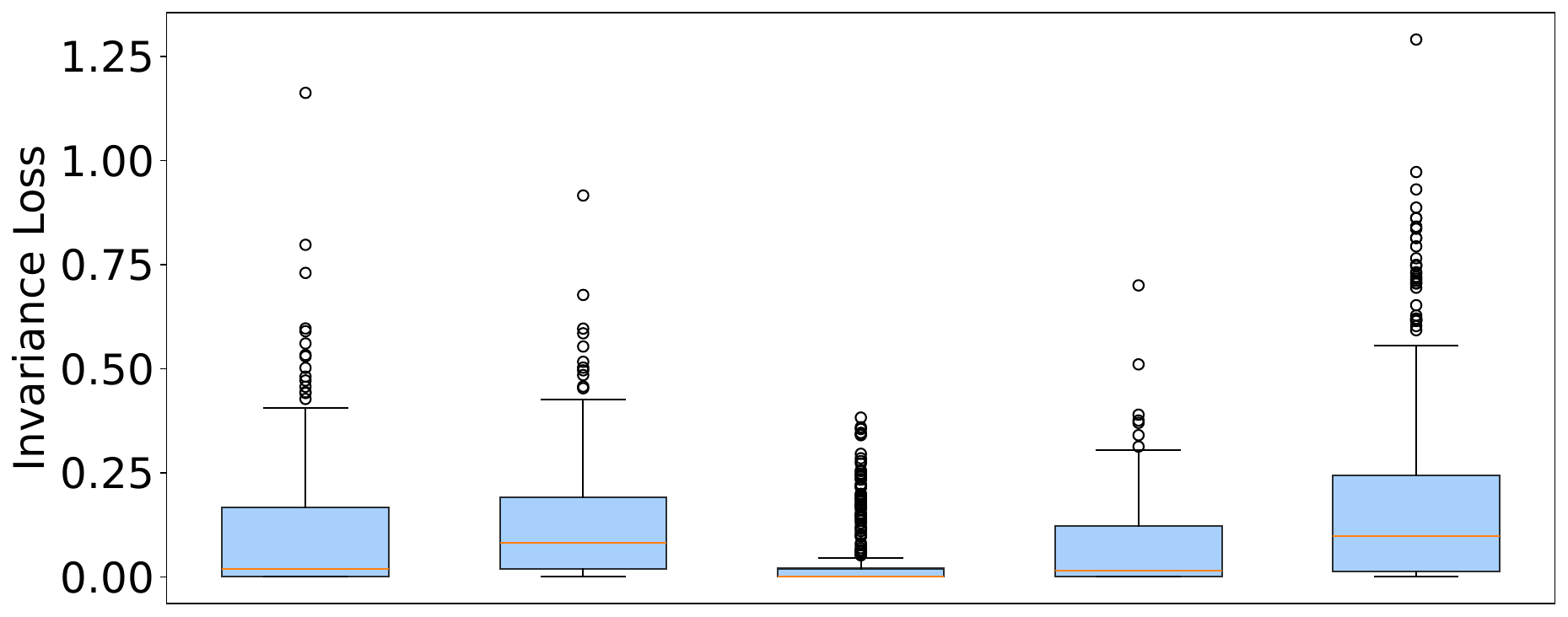} &
  \raisebox{1.2cm}{2} \\[4pt]
  \includegraphics[width=0.45\linewidth]{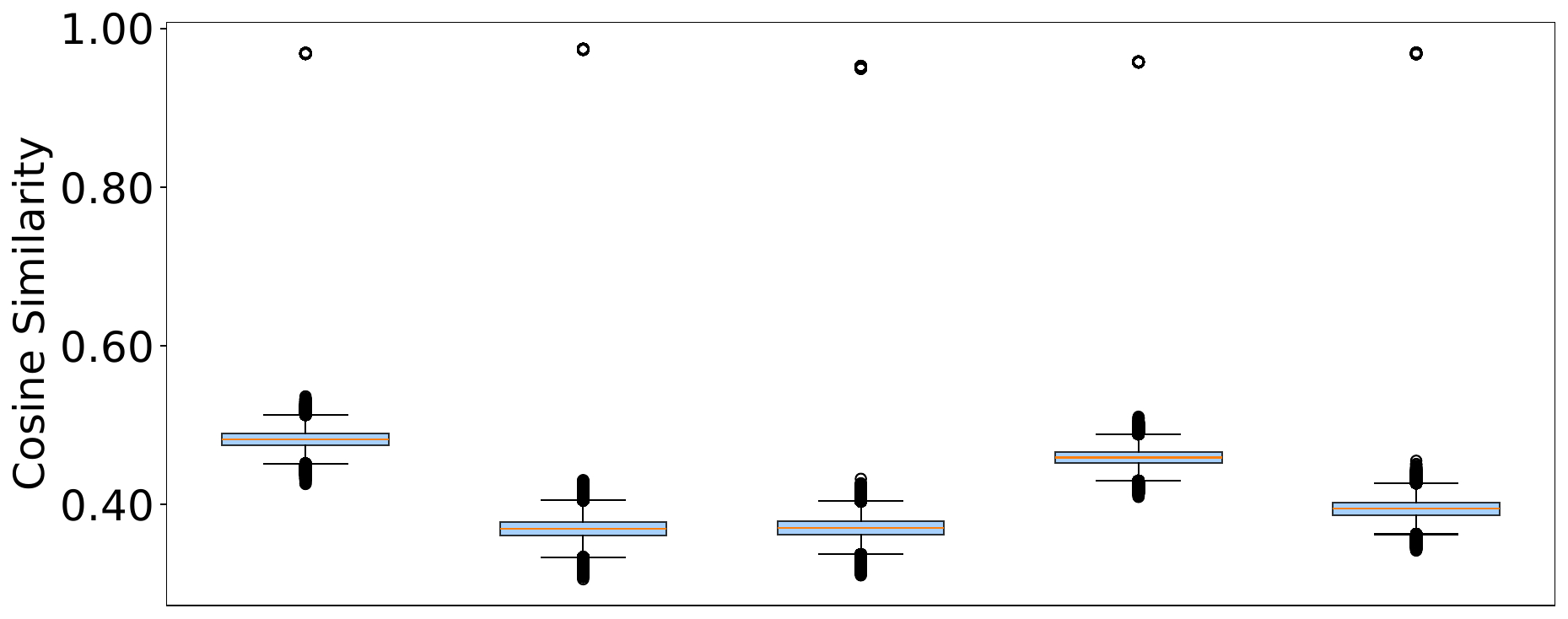} &
  \includegraphics[width=0.45\linewidth]{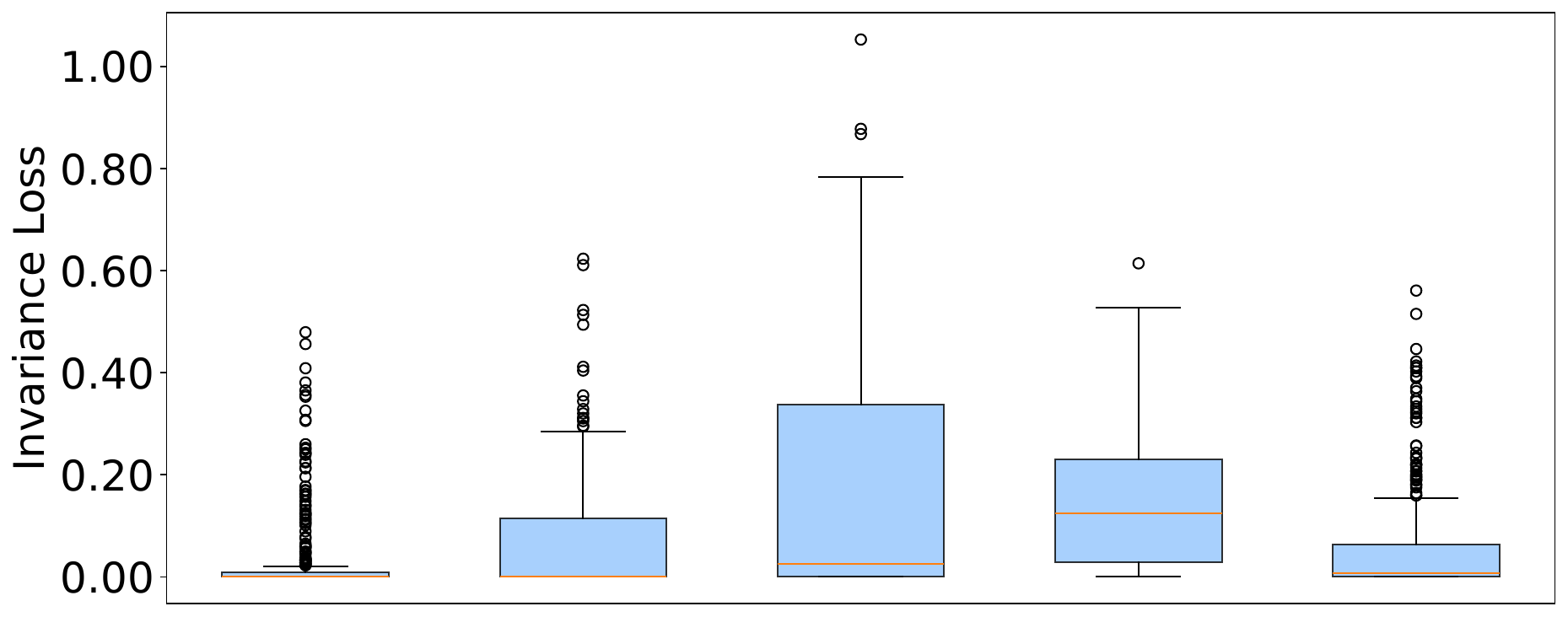} &
  \raisebox{1.2cm}{4} \\[4pt]
  \includegraphics[width=0.45\linewidth]{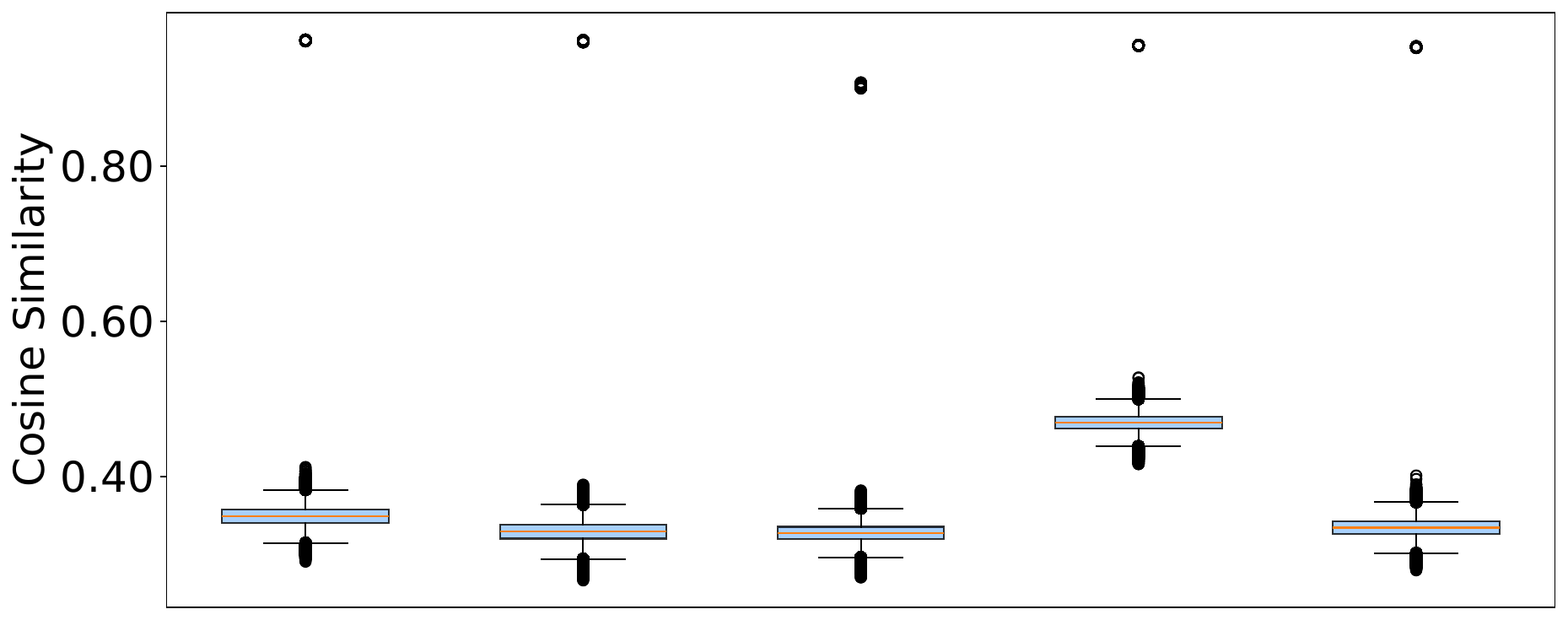} &
  \includegraphics[width=0.45\linewidth]{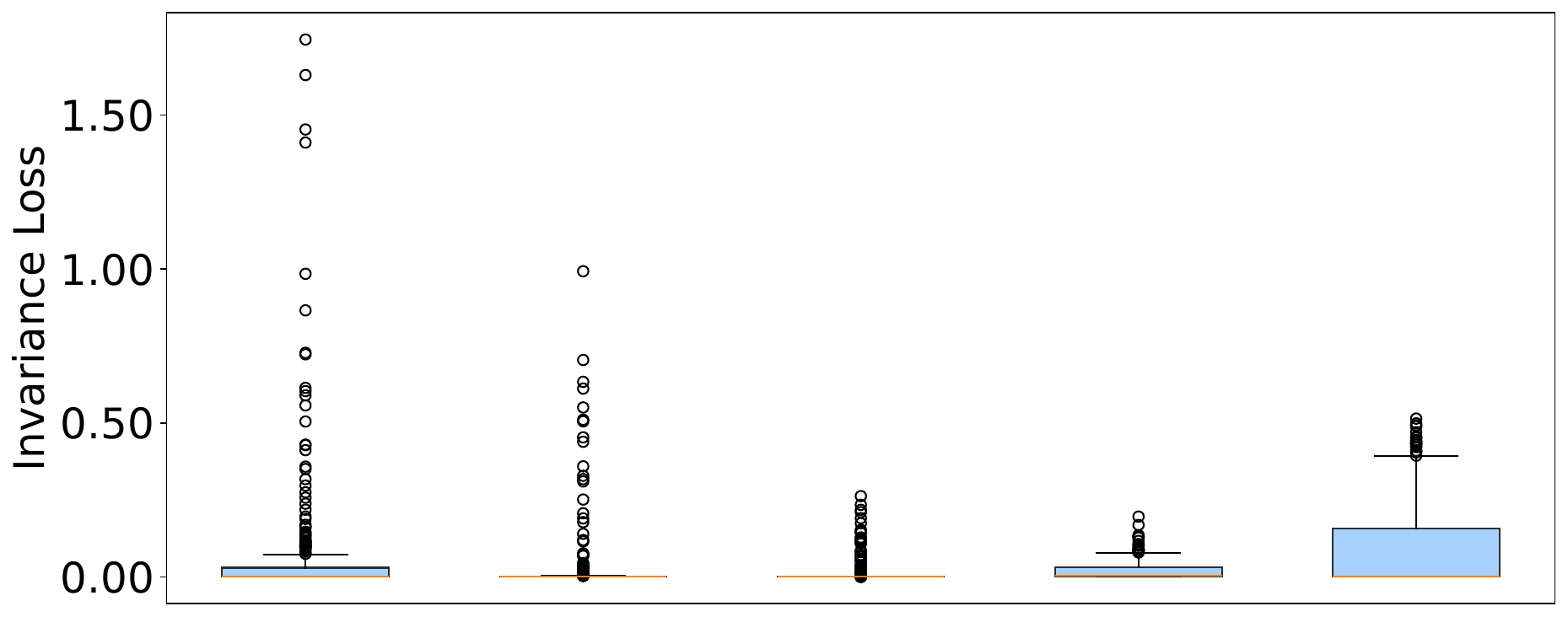} &
  \raisebox{1.2cm}{8} \\[4pt]
  \includegraphics[width=0.45\linewidth]{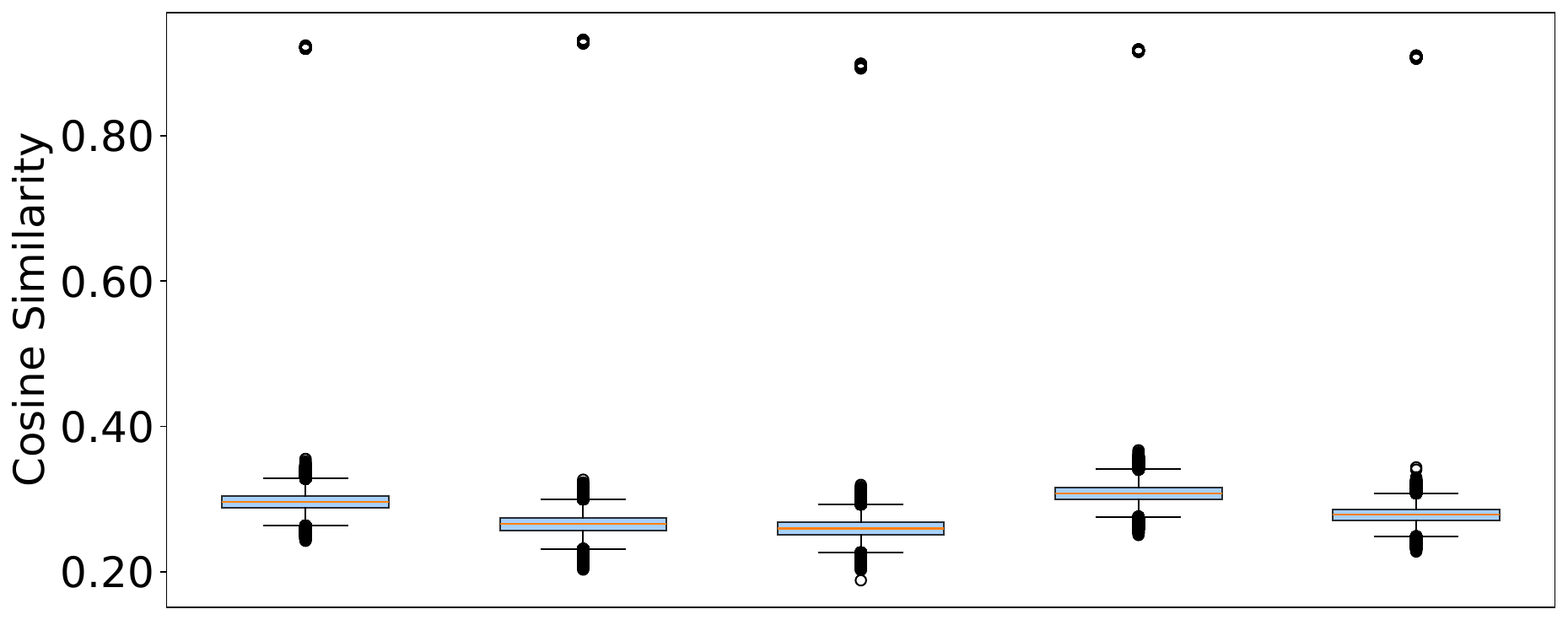} &
  \includegraphics[width=0.45\linewidth]{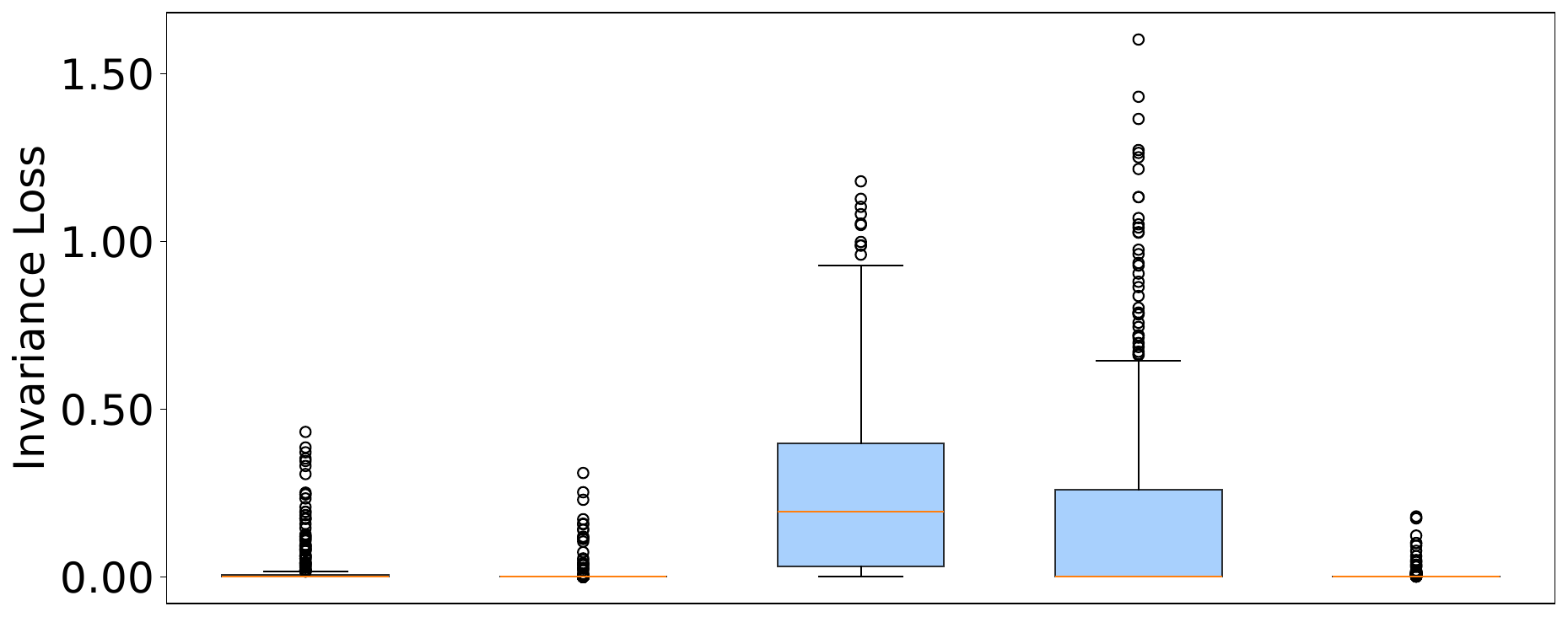} &
  \raisebox{1.2cm}{16} \\[4pt]
\end{tabular}
\caption{Mechanism indicators for second layer of all models trained on ModelNet10 with translations.}
\label{fig:modelnet_translation_grid_layer1}
\end{figure}

\subsubsection{Rotation}

\begin{figure}[h]
\centering
\renewcommand{\arraystretch}{1.5}
\begin{tabular}{cccc}
  Block 1 & Block 2 & N Augs \\[6pt]
  \includegraphics[width=0.45\linewidth]{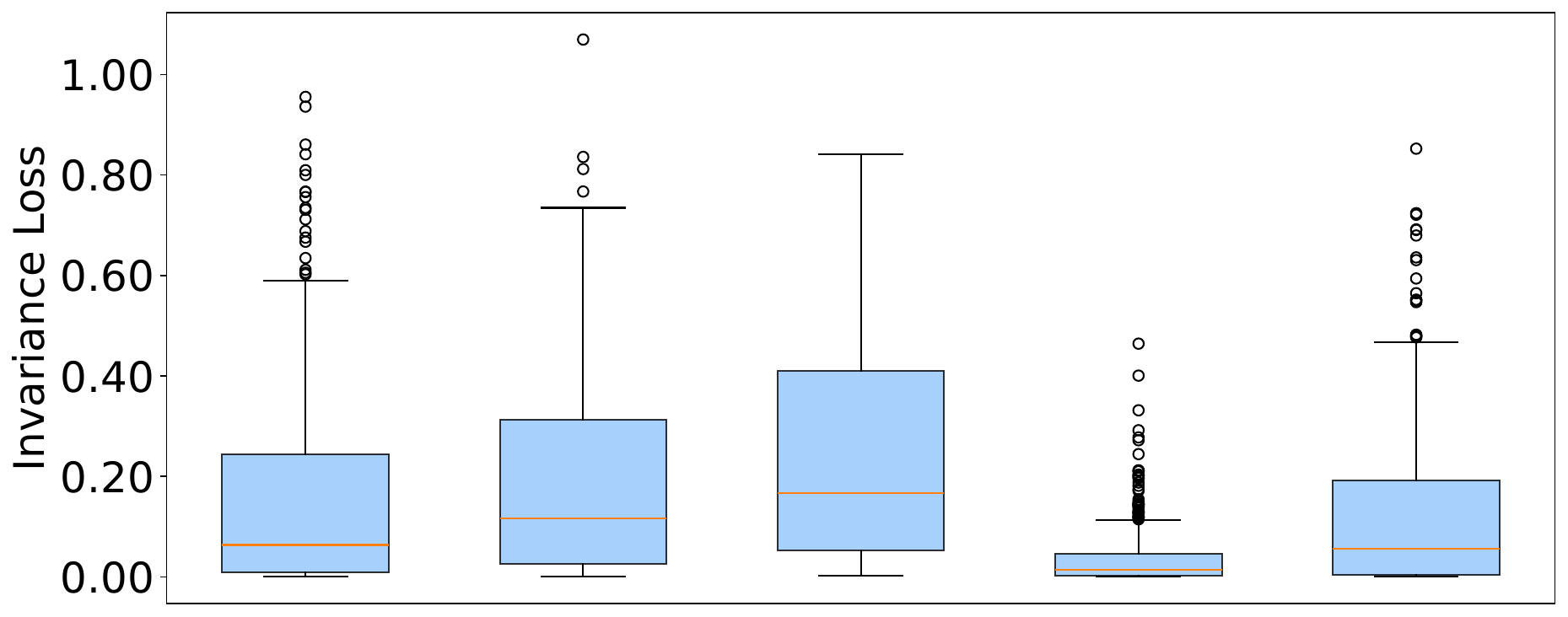} &
  \includegraphics[width=0.45\linewidth]{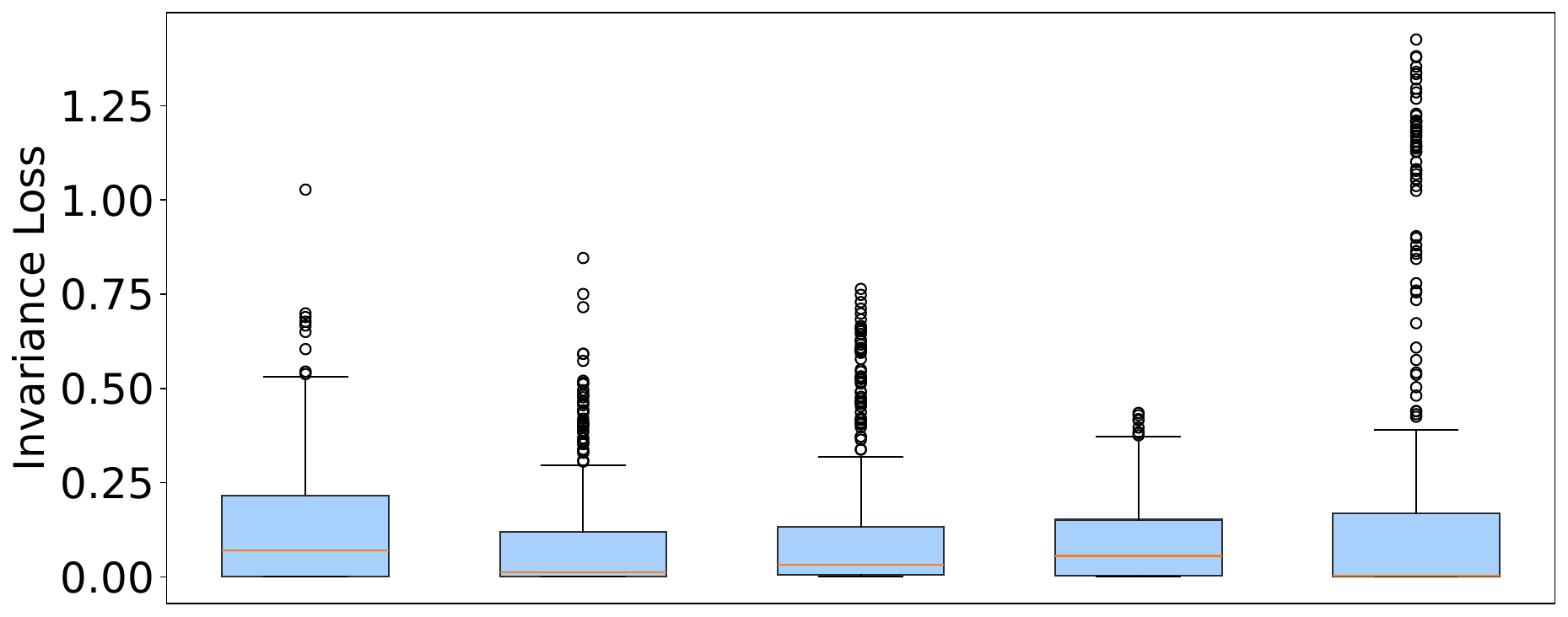} &
  \raisebox{1.2cm}{1} \\[4pt]
  \includegraphics[width=0.45\linewidth]{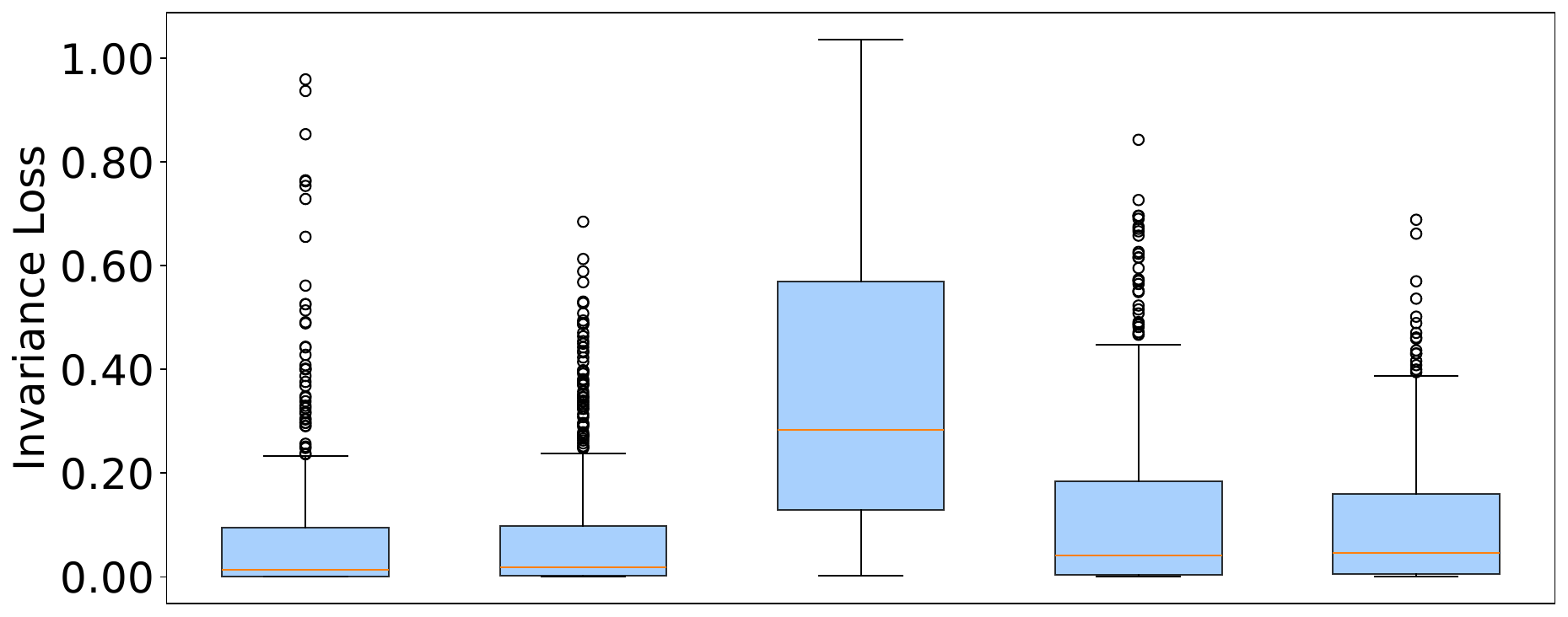} &
  \includegraphics[width=0.45\linewidth]{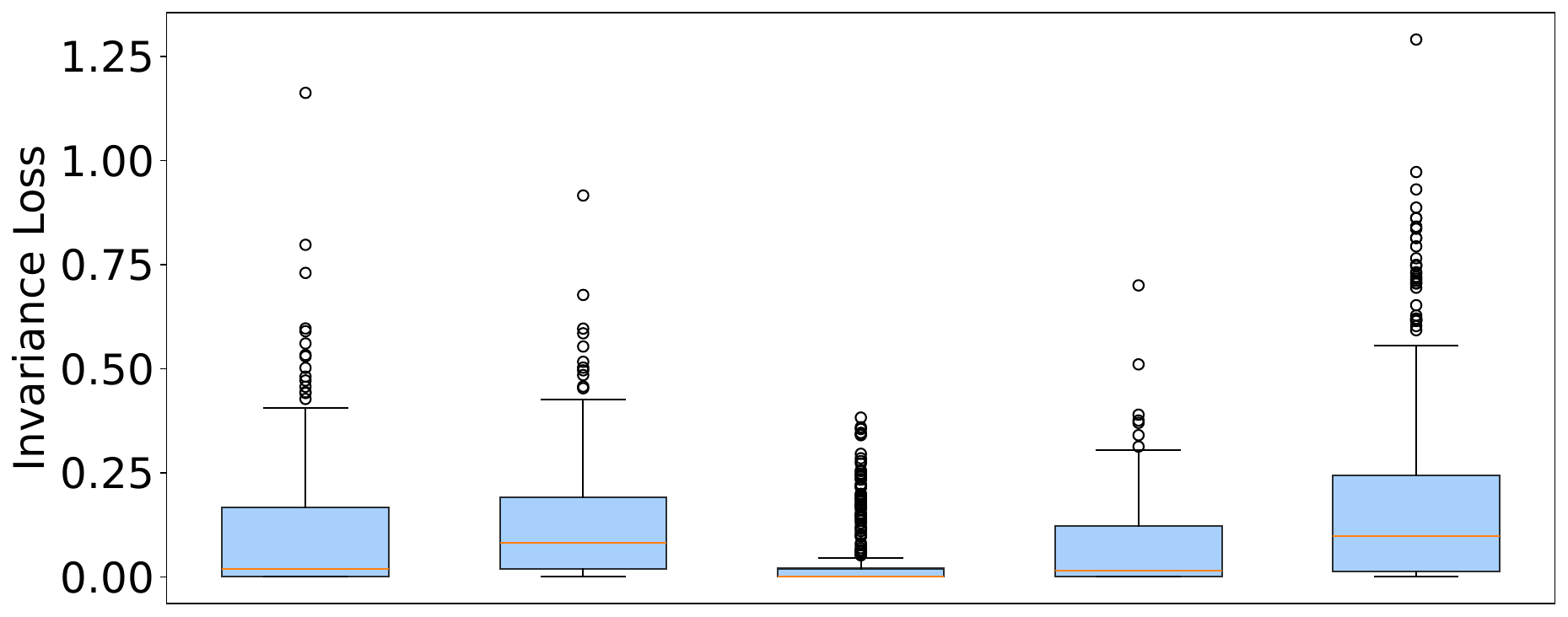} &
  \raisebox{1.2cm}{2} \\[4pt]
  \includegraphics[width=0.45\linewidth]{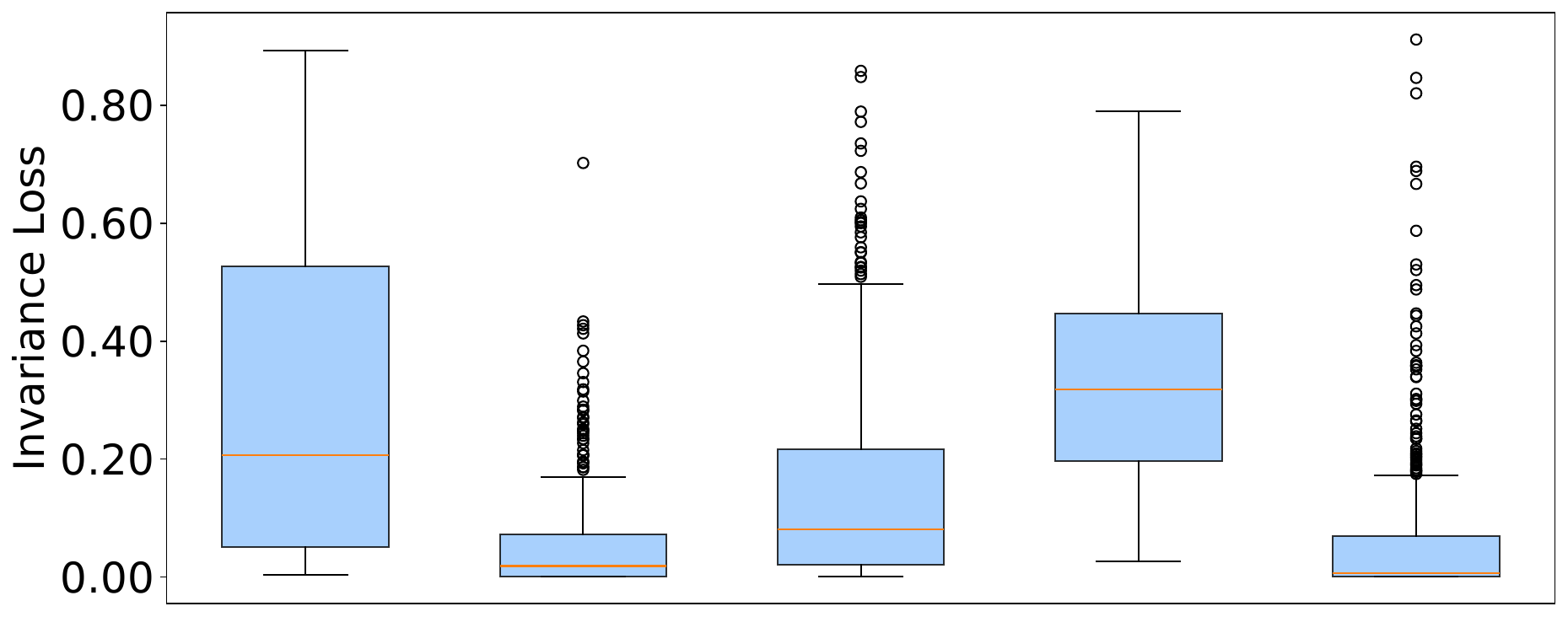} &
  \includegraphics[width=0.45\linewidth]{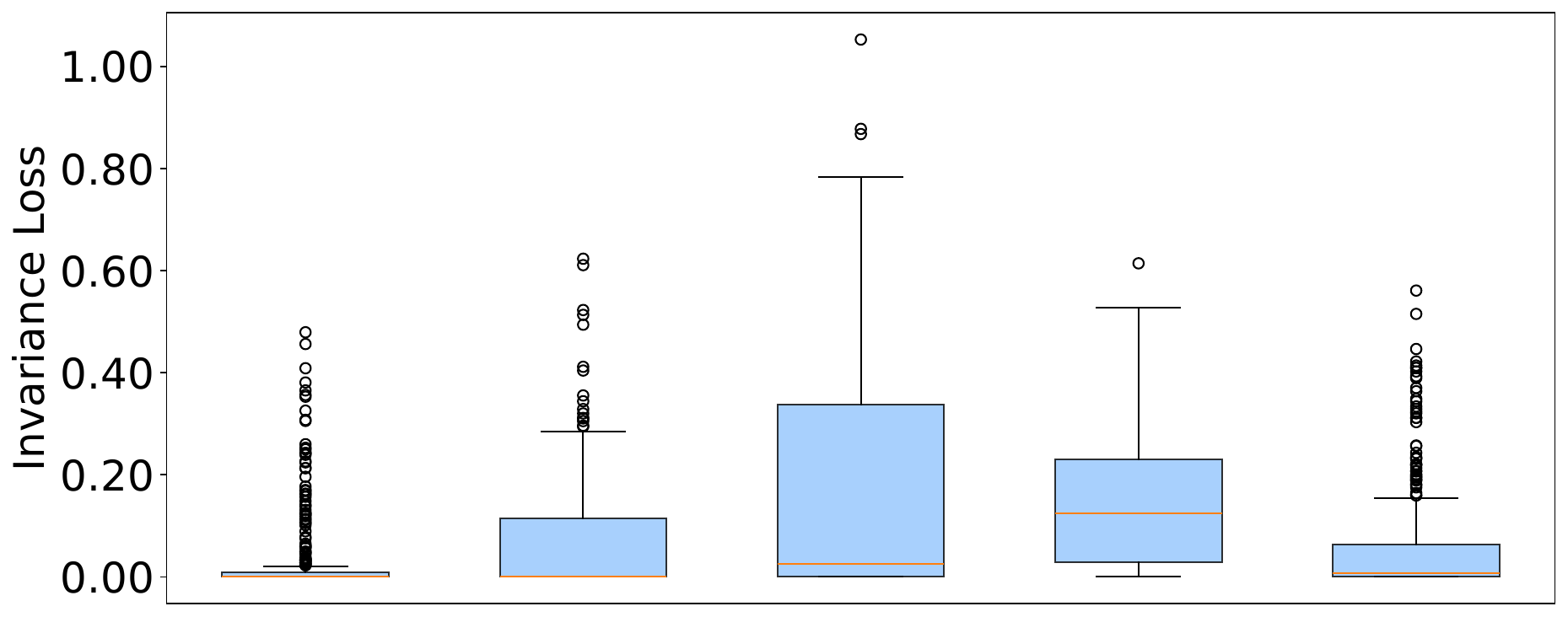} &
  \raisebox{1.2cm}{4} \\[4pt]
  \includegraphics[width=0.45\linewidth]{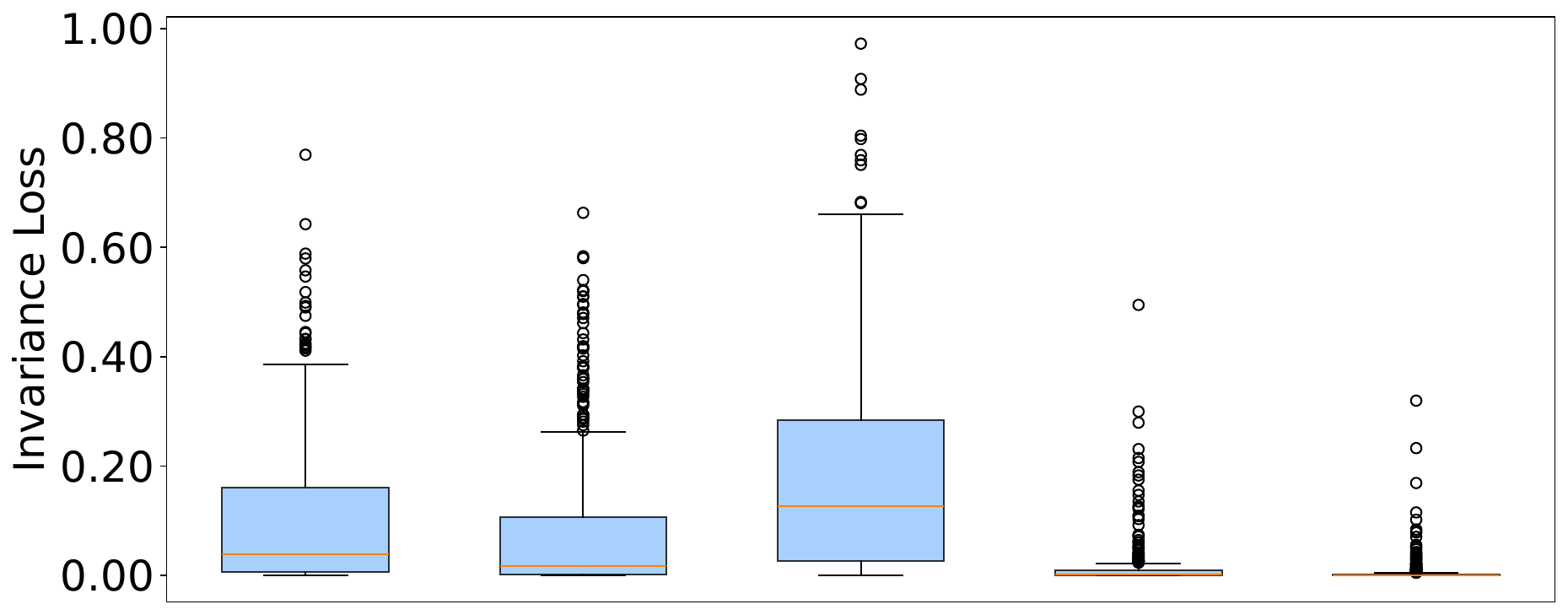} &
  \includegraphics[width=0.45\linewidth]{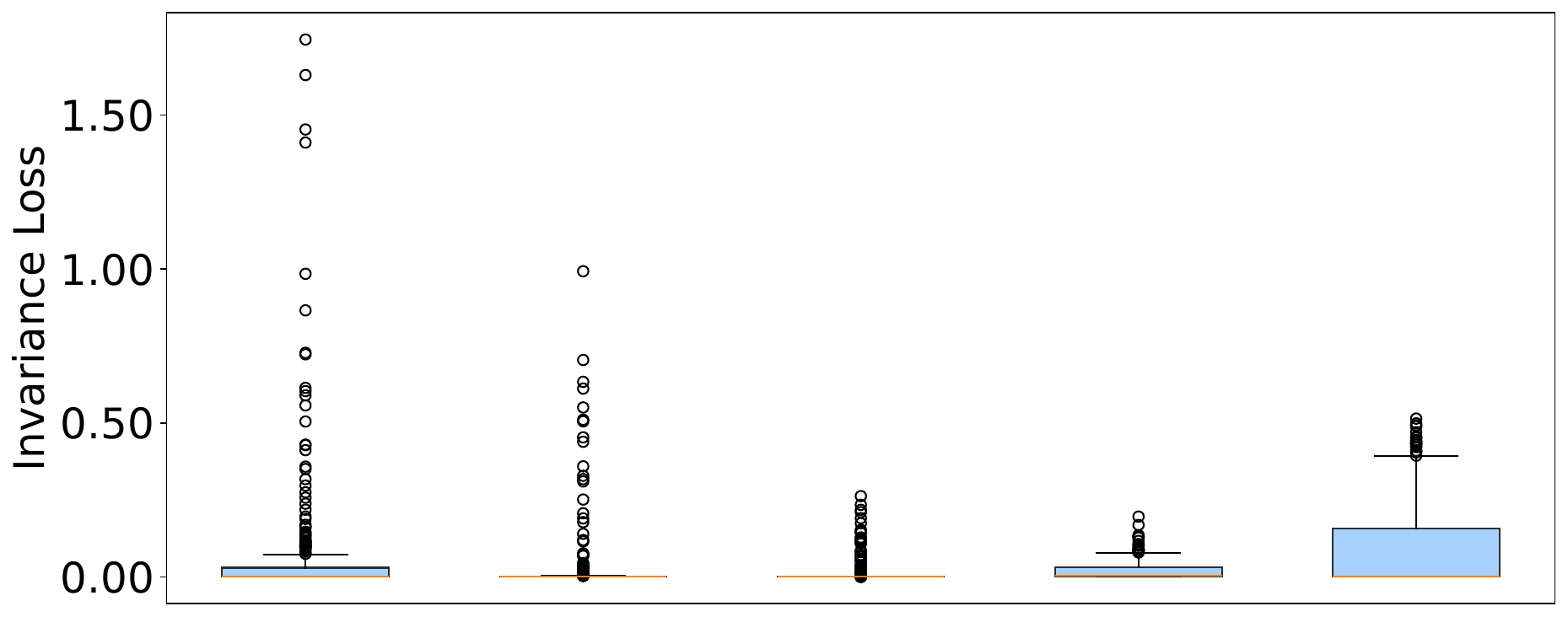} &
  \raisebox{1.2cm}{8} \\[4pt]
  \includegraphics[width=0.45\linewidth]{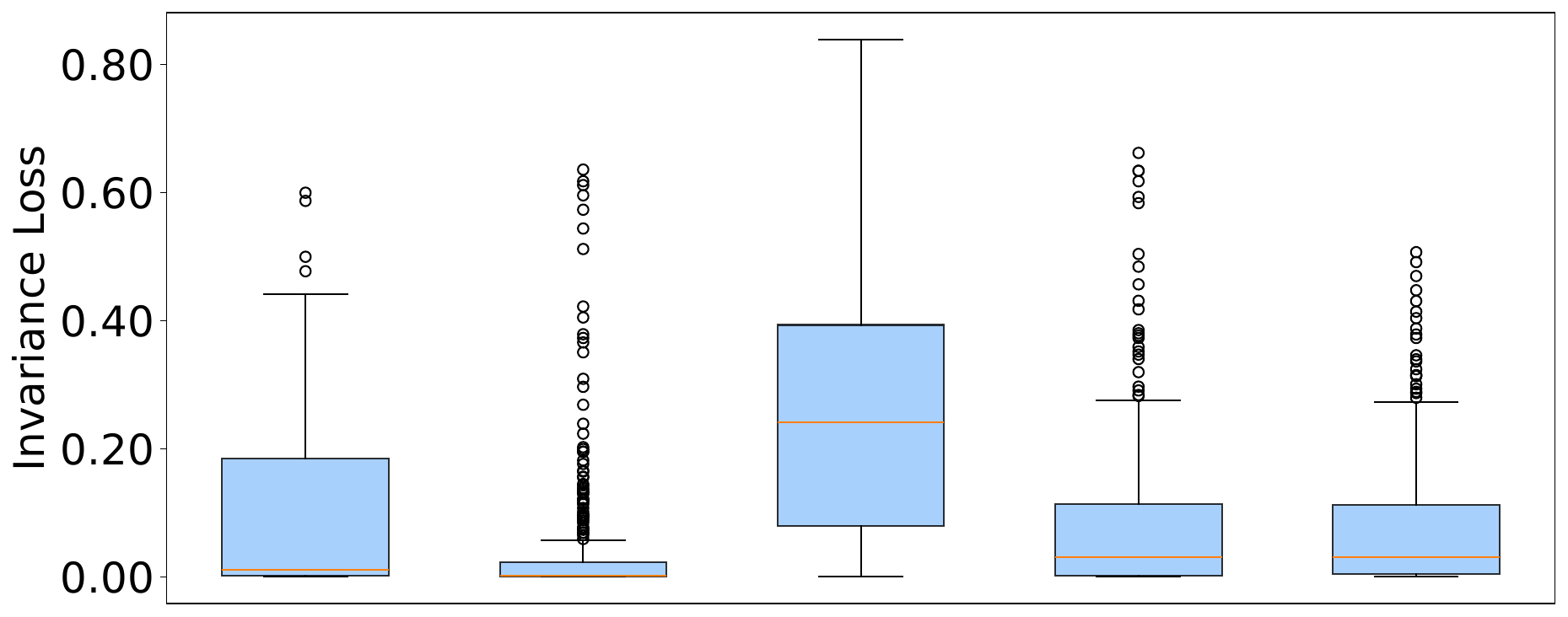} &
  \includegraphics[width=0.45\linewidth]{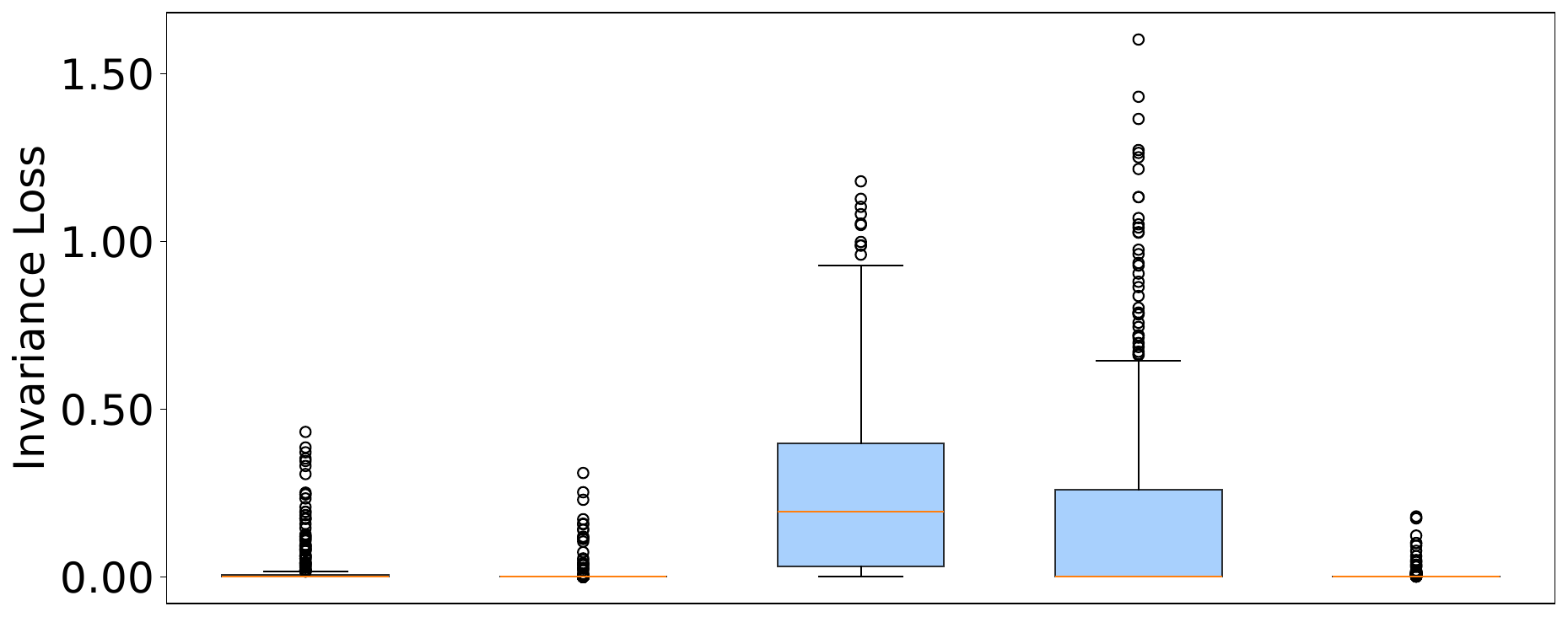} &
  \raisebox{1.2cm}{16} \\
\end{tabular}
\caption{Invariance loss per block of the attention maps of all models trained with rotations on ModelNet10.}
\label{fig:rotation_grid_modelnet}
\end{figure}

Figure~\ref{fig:rotation_grid_modelnet} shows that both blocks of Transformers trained on ModelNet10 with rotation augmentations exhibit strong rotation invariance. In both cases, increasing the number of augmentations induces higher invariance, as can be expected.

\subsubsection{Scale}

When trained with scale augmentations, Transformers only exhibit the signature of scaling up self-attention outputs with respect to the residuals in the second block. This can be observed by comparing the measurements on the first block shown in Figure~\ref{fig:scale_modelnet_grid_layer0} with those on the second block shown in Figure~\ref{fig:scale_modelnet_grid_layer1}. 

\begin{figure}[h]
\centering
\renewcommand{\arraystretch}{1.5}
\begin{tabular}{ccc}
  Invariance Loss of $\A(\mX)$ & $\| \mX_{i,:}\|/ \| \SA(\mX)_{i,:}\|$ & N Augs \\[6pt]
  \includegraphics[width=0.45\linewidth]{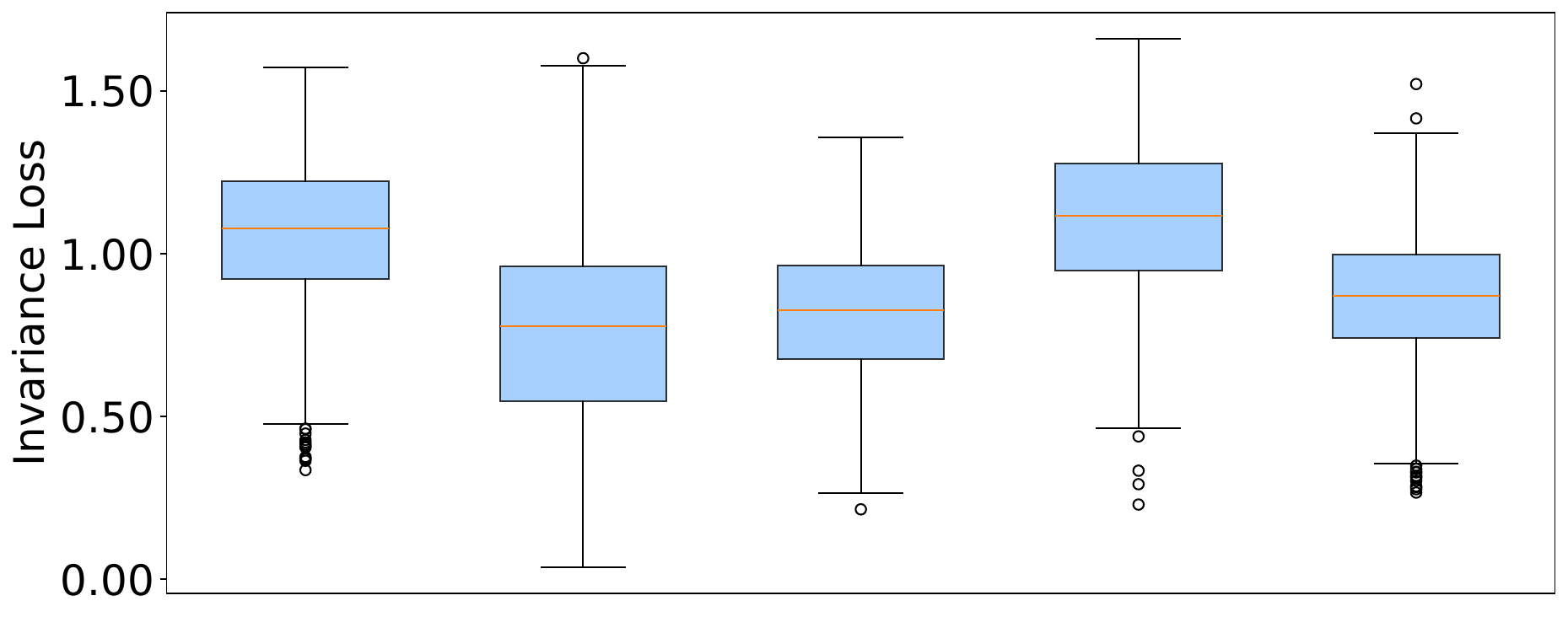} &
  \includegraphics[width=0.45\linewidth]{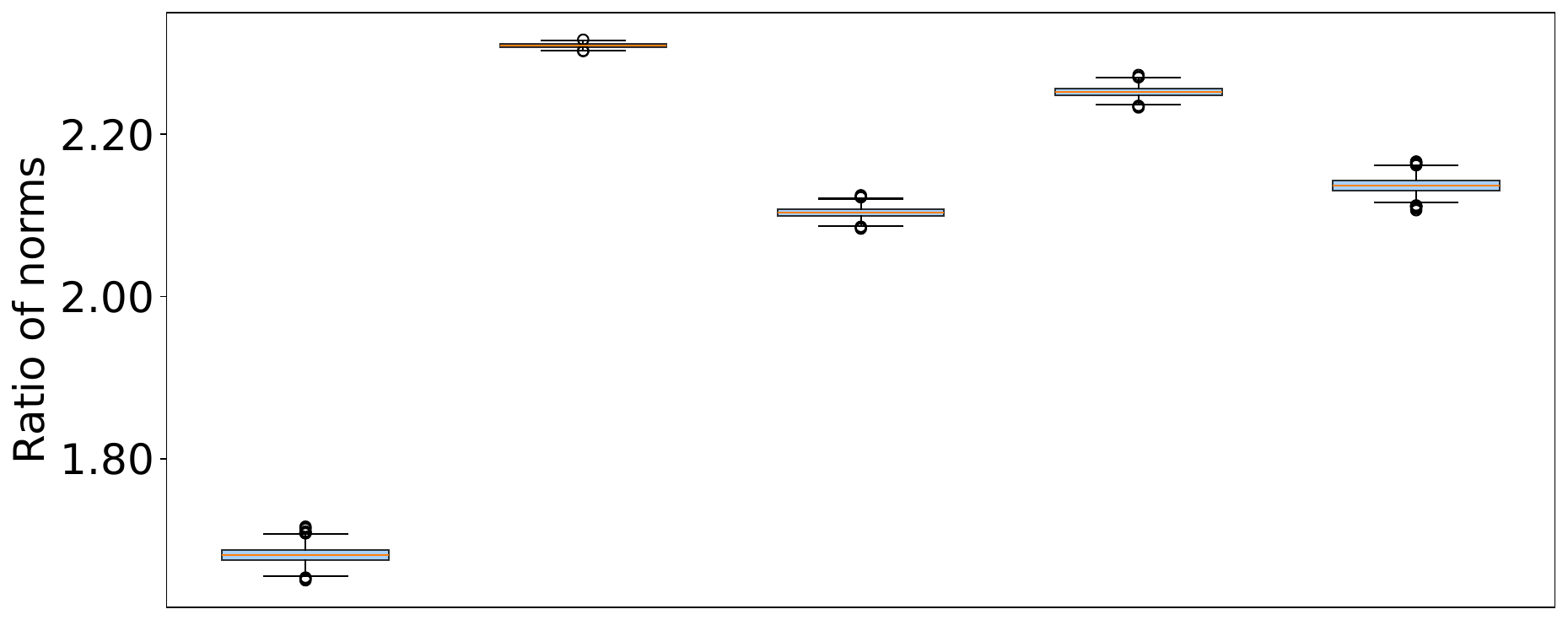} &
  \raisebox{1.2cm}{1} \\[4pt]
  \includegraphics[width=0.45\linewidth]{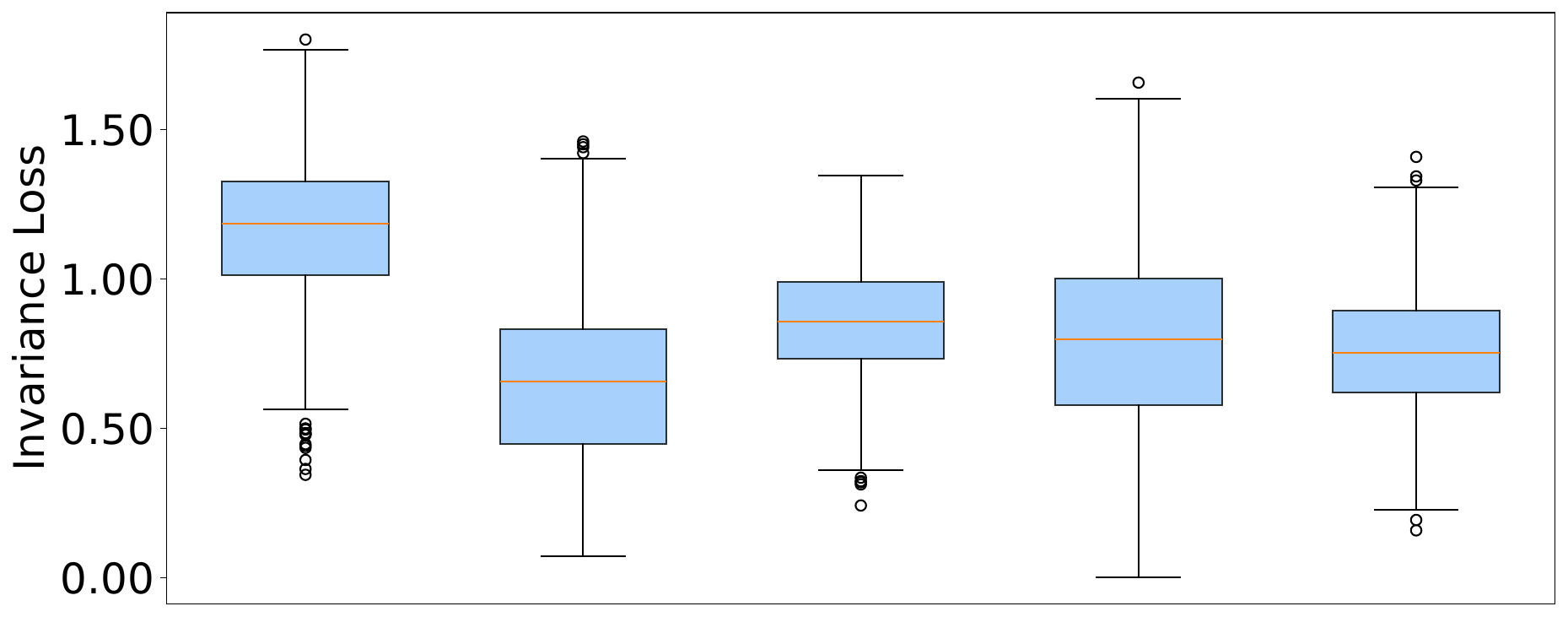} &
  \includegraphics[width=0.45\linewidth]{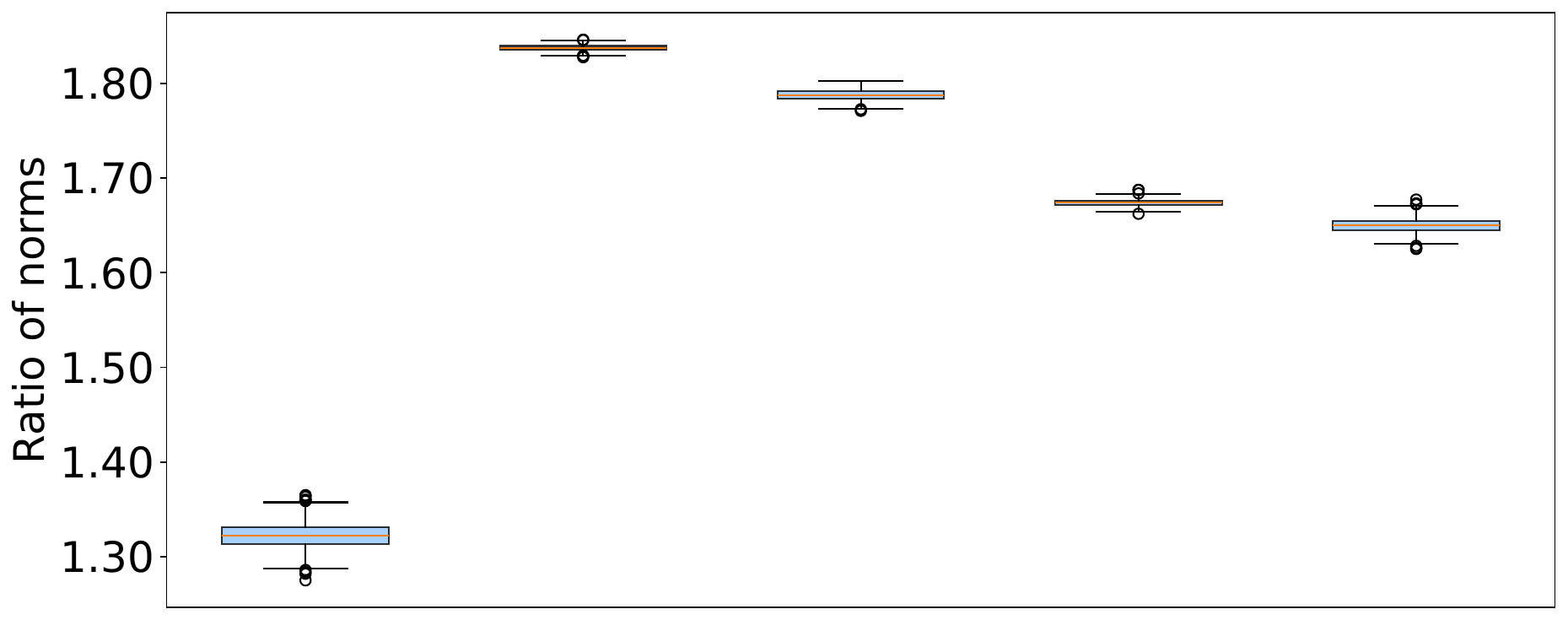} &
  \raisebox{1.2cm}{2} \\[4pt]
  \includegraphics[width=0.45\linewidth]{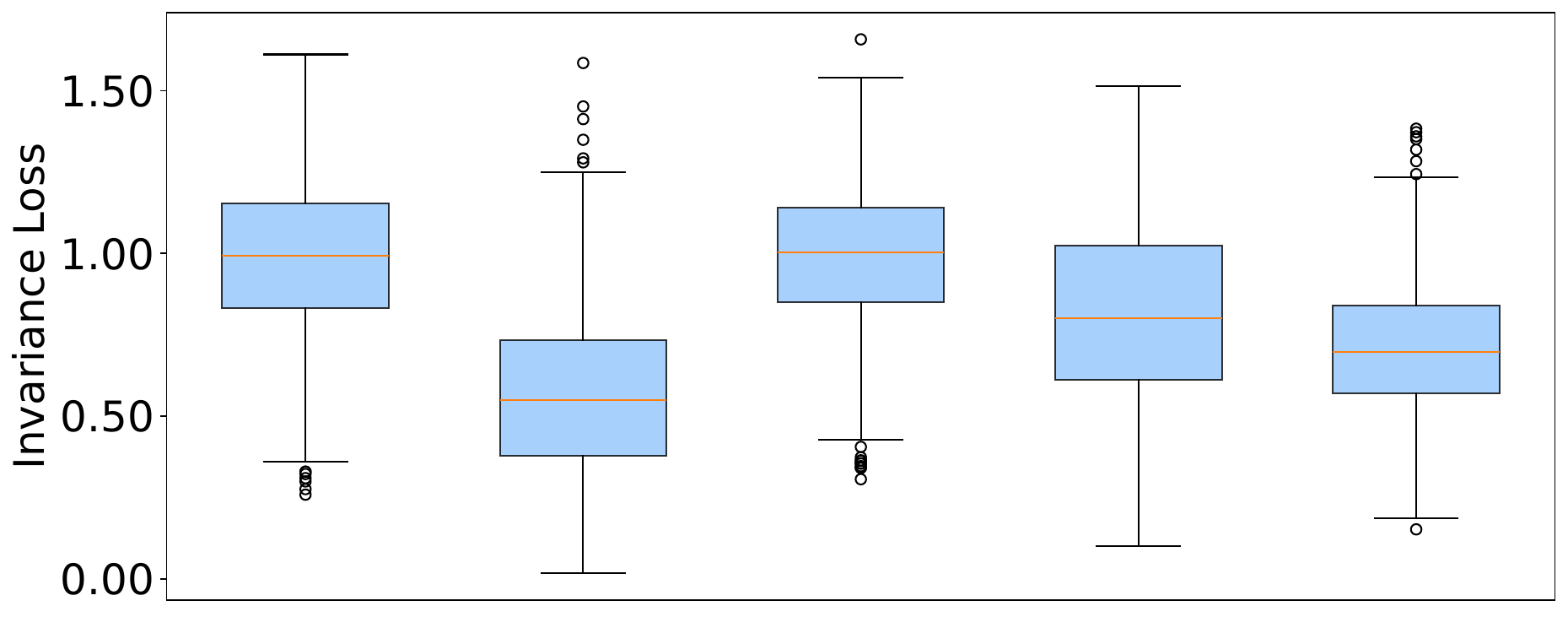} &
  \includegraphics[width=0.45\linewidth]{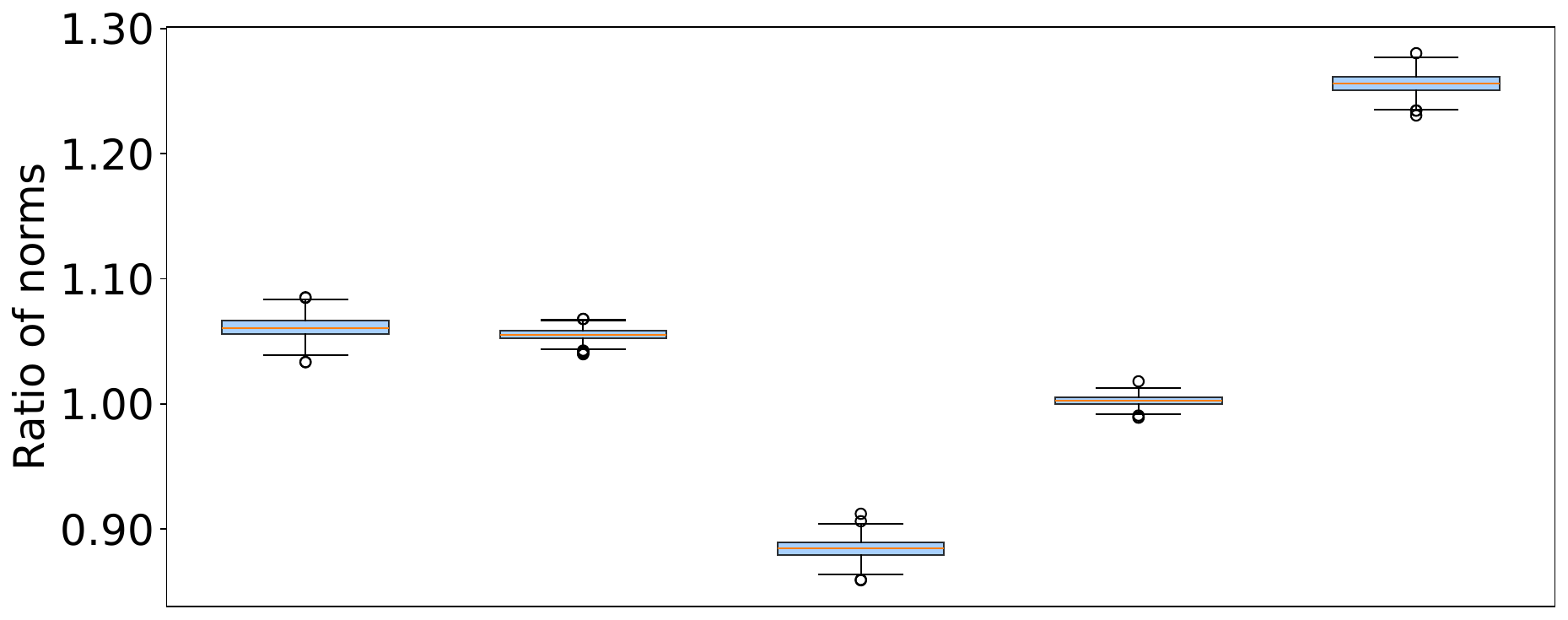} &
  \raisebox{1.2cm}{4} \\[4pt]
  \includegraphics[width=0.45\linewidth]{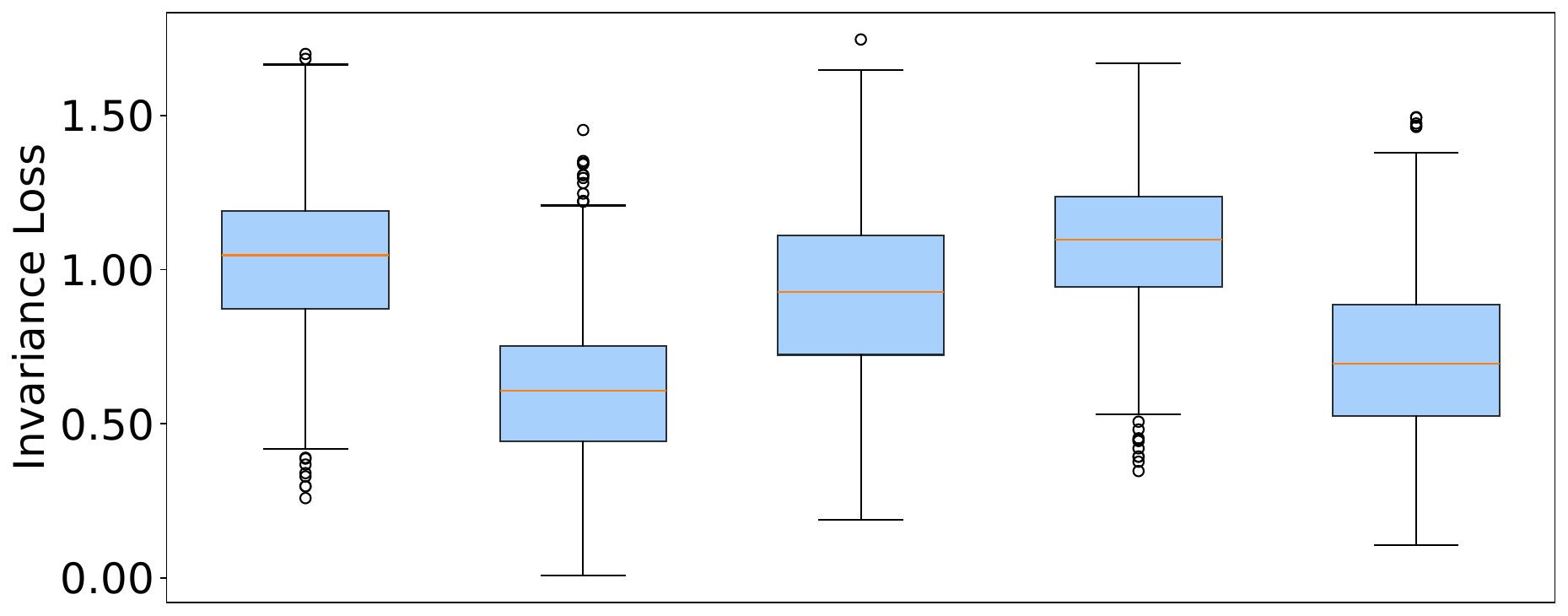} &
  \includegraphics[width=0.45\linewidth]{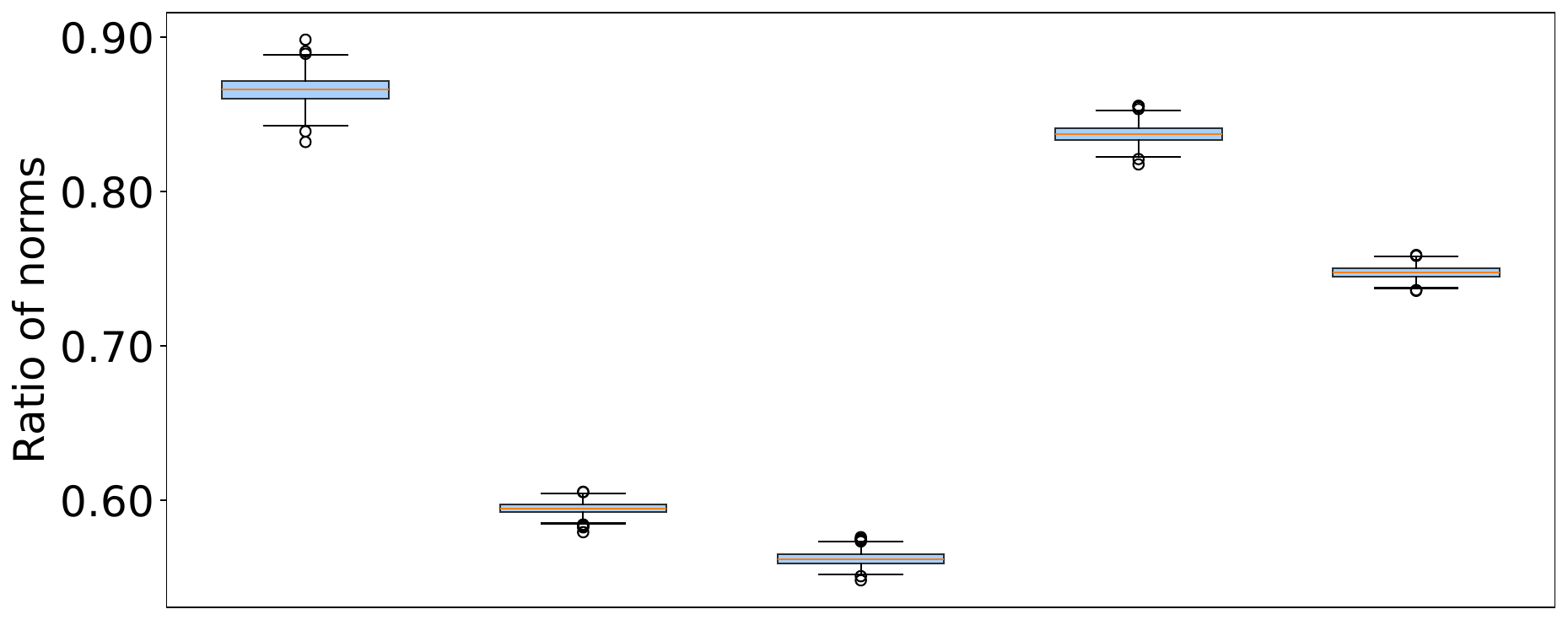} &
  \raisebox{1.2cm}{8} \\[4pt]
  \includegraphics[width=0.45\linewidth]{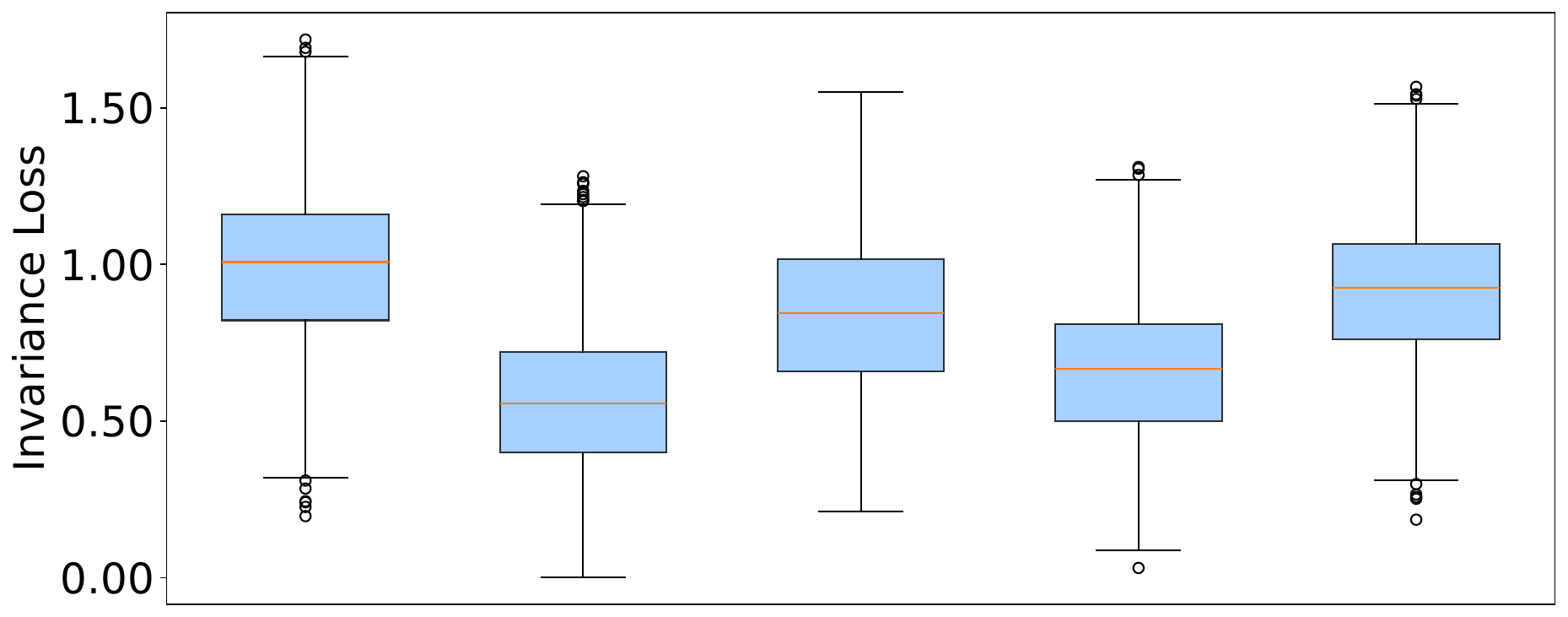} &
  \includegraphics[width=0.45\linewidth]{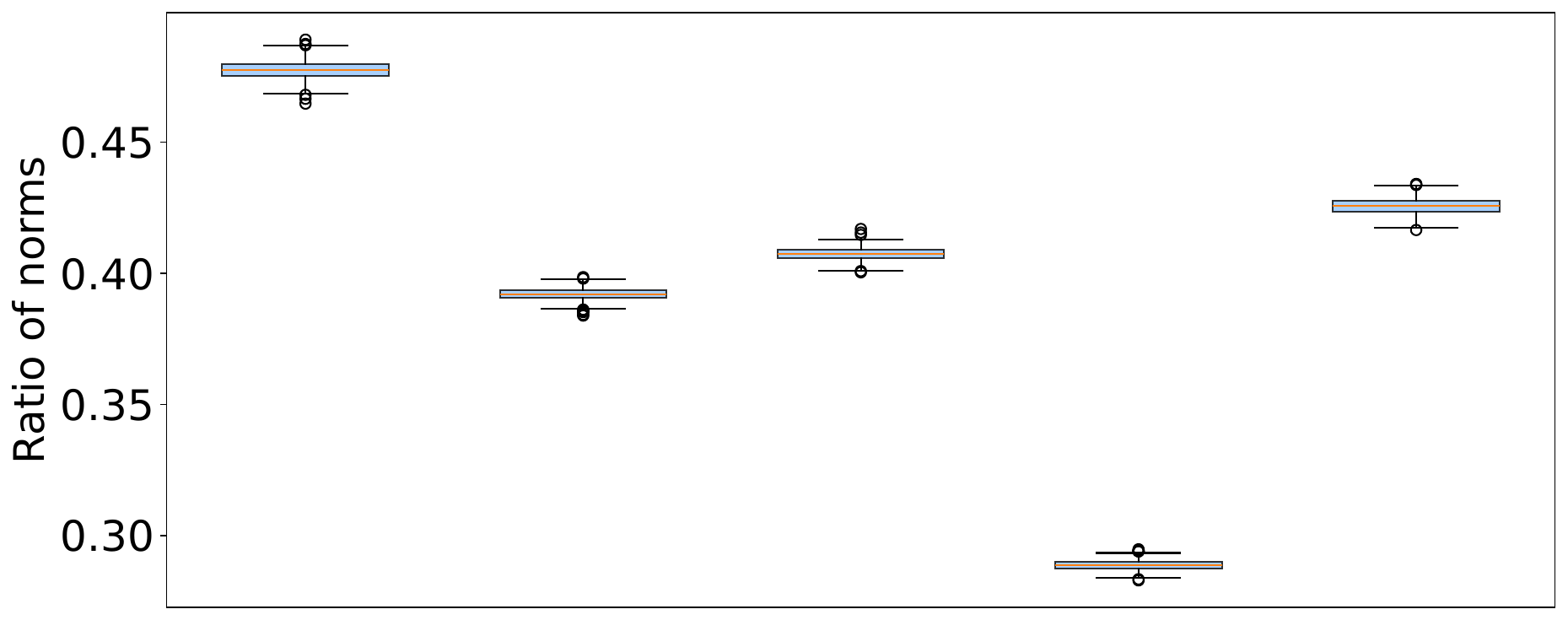} &
  \raisebox{1.2cm}{16} \\
\end{tabular}
\caption{Mechanism indicators of first layer (layer 0) of all models trained on ModelNet10 with scale augmentations.}
\label{fig:scale_modelnet_grid_layer0}
\end{figure}

\begin{figure}[h]
\centering
\renewcommand{\arraystretch}{1.5}
\begin{tabular}{ccc}
  Invariance Loss of $\A(\mX)$ & $\| \mX_{i,:}\|/ \| \SA(\mX)_{i,:}\|$ & N Augs \\[6pt]
  \includegraphics[width=0.45\linewidth]{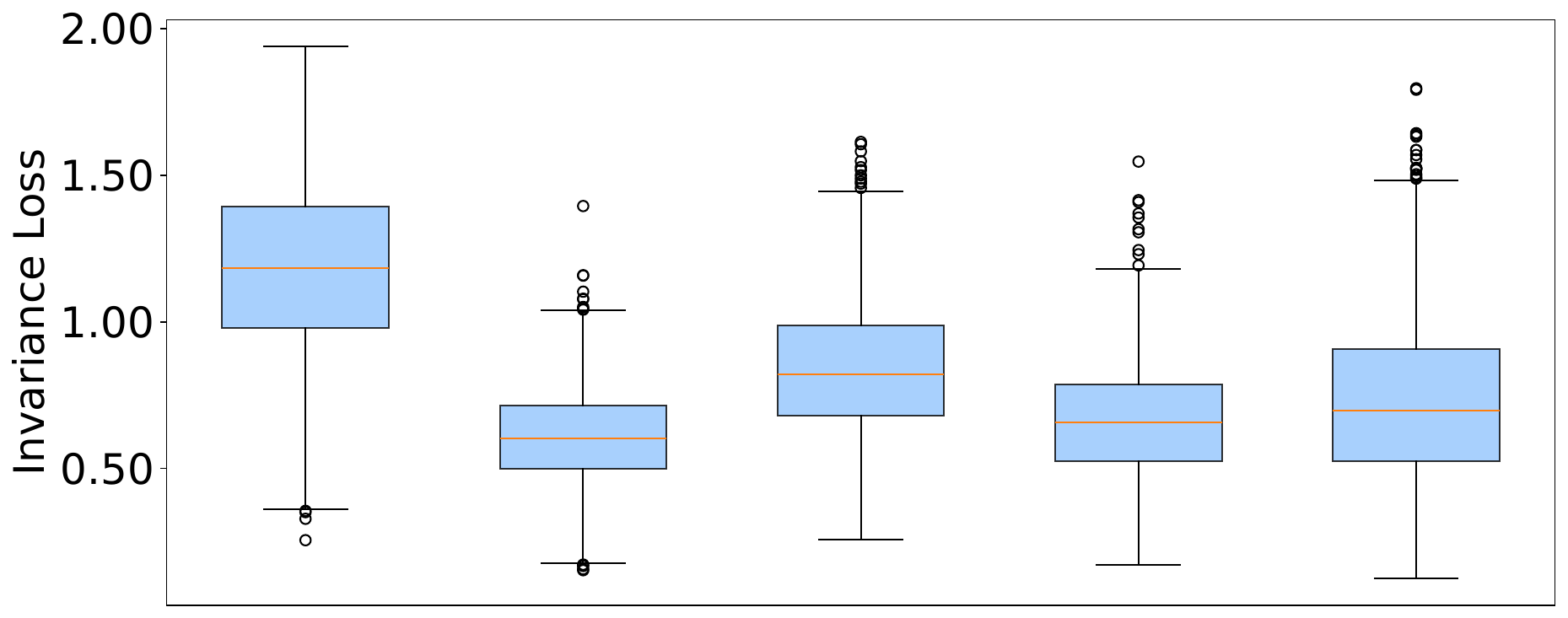} &
  \includegraphics[width=0.45\linewidth]{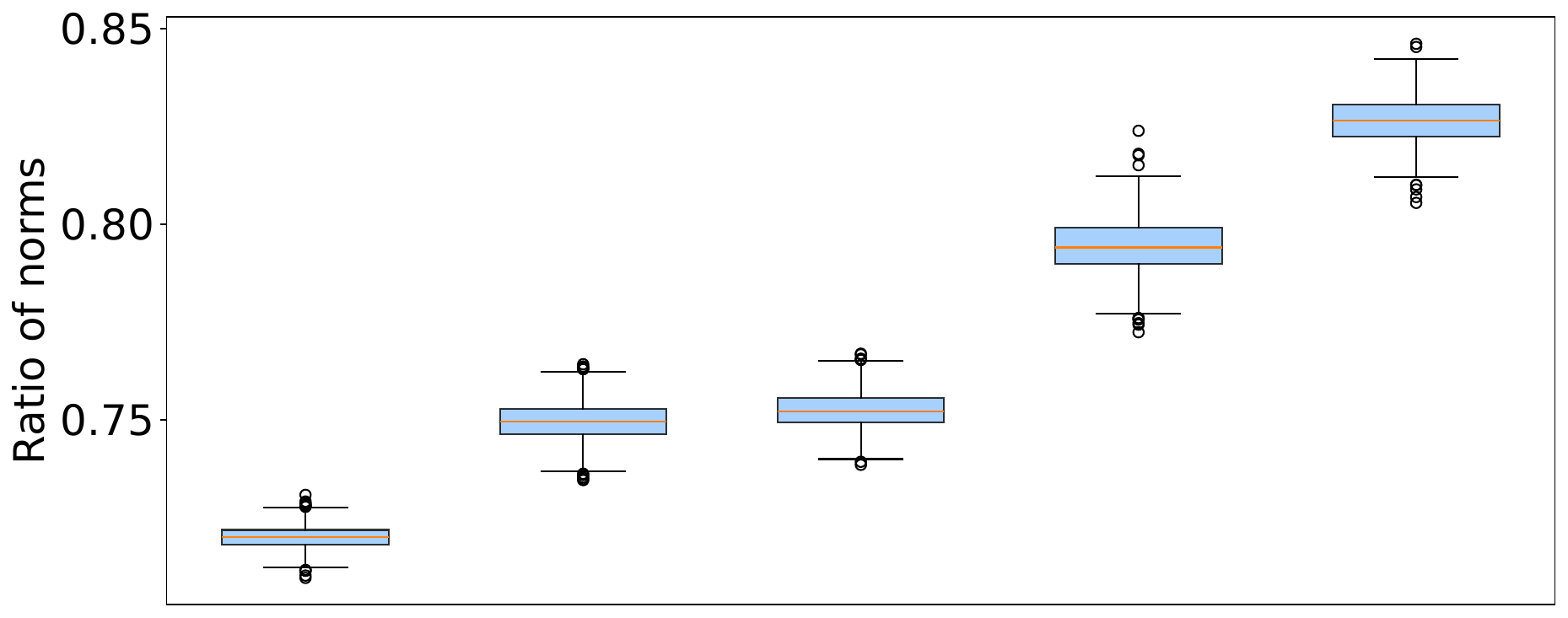} &
  \raisebox{1.2cm}{1} \\[4pt]
  \includegraphics[width=0.45\linewidth]{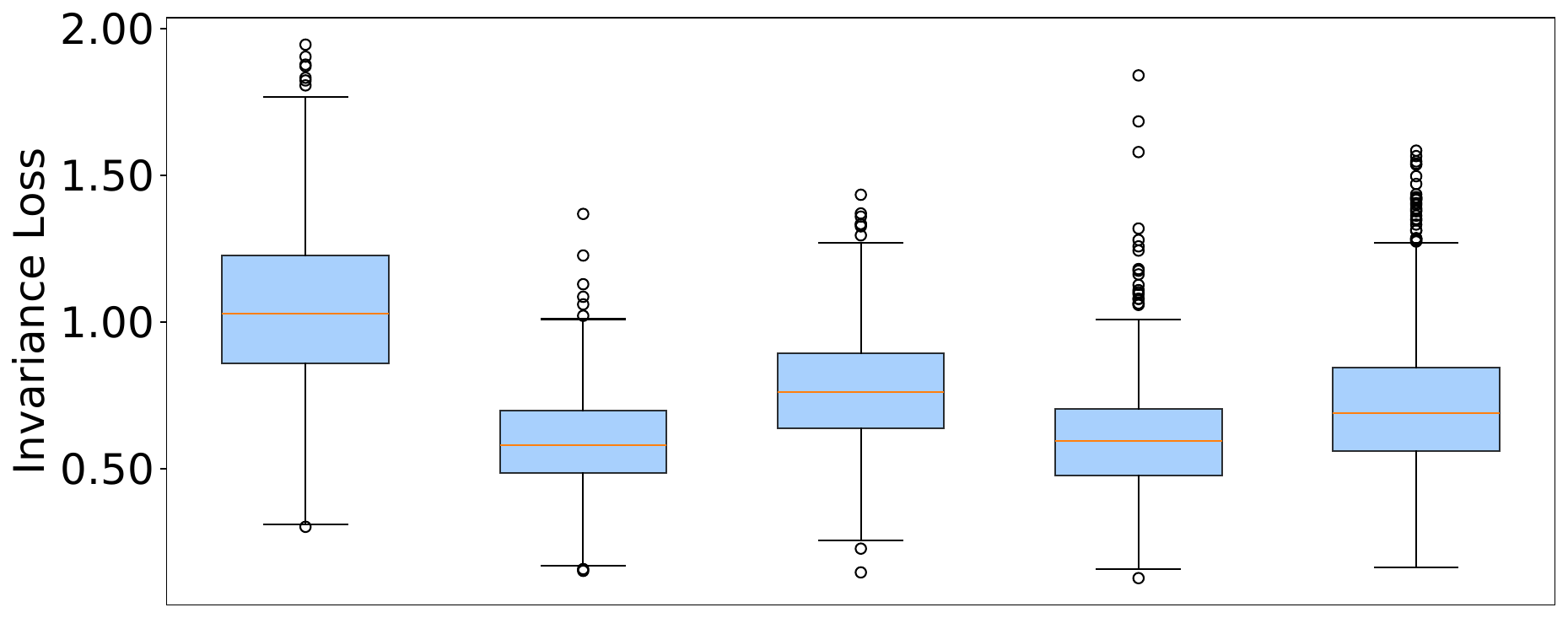} &
  \includegraphics[width=0.45\linewidth]{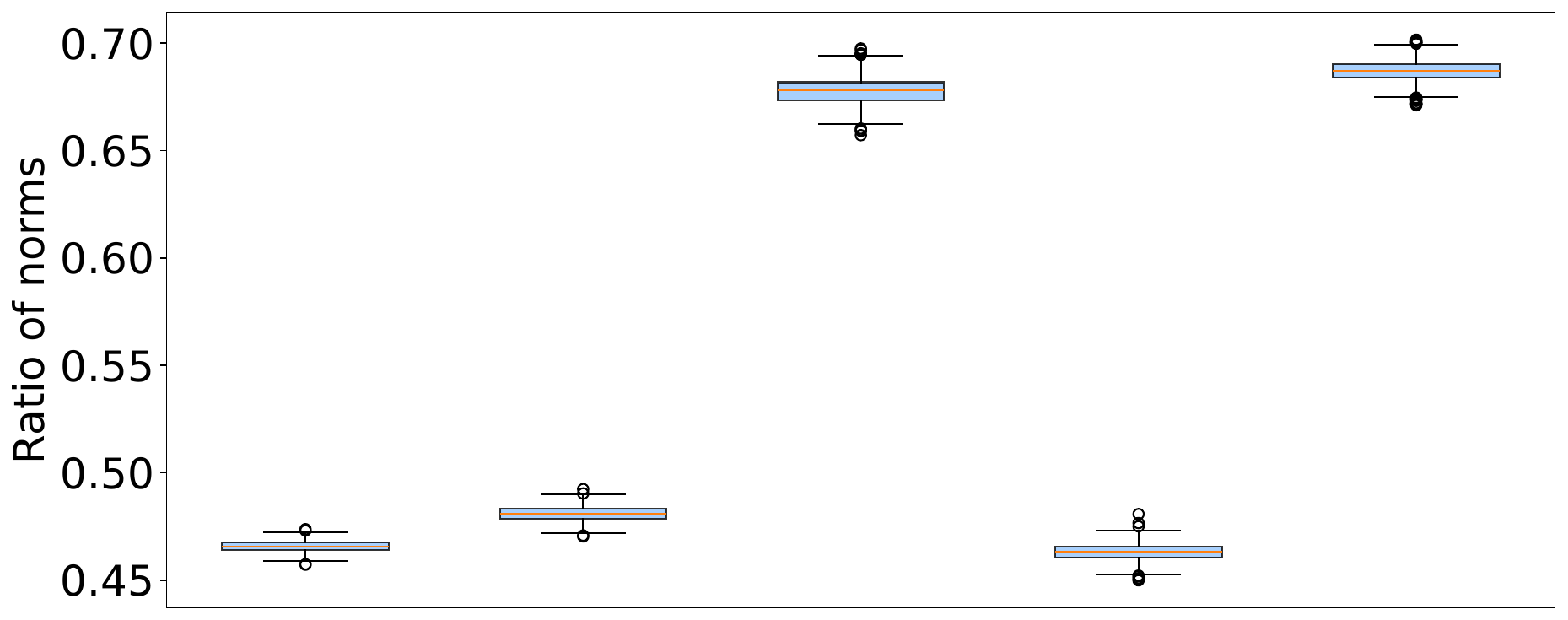} &
  \raisebox{1.2cm}{2} \\[4pt]
  \includegraphics[width=0.45\linewidth]{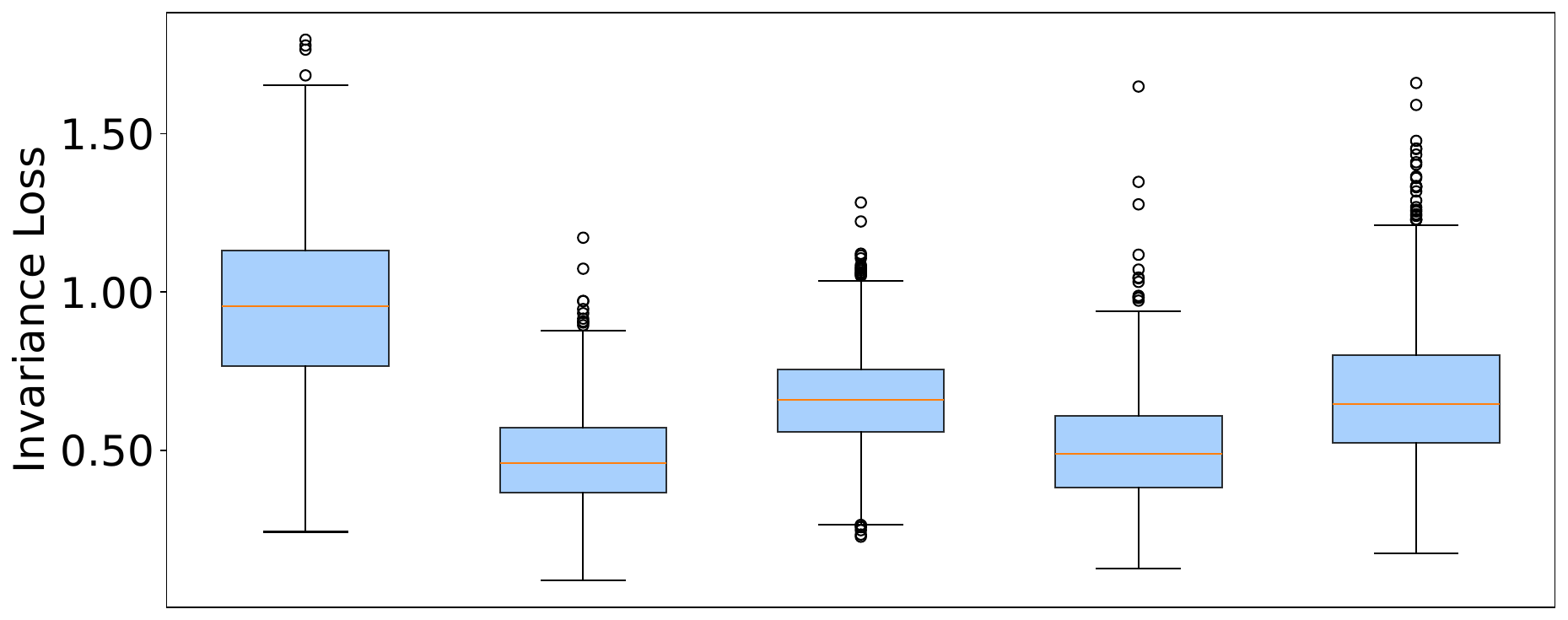} &
  \includegraphics[width=0.45\linewidth]{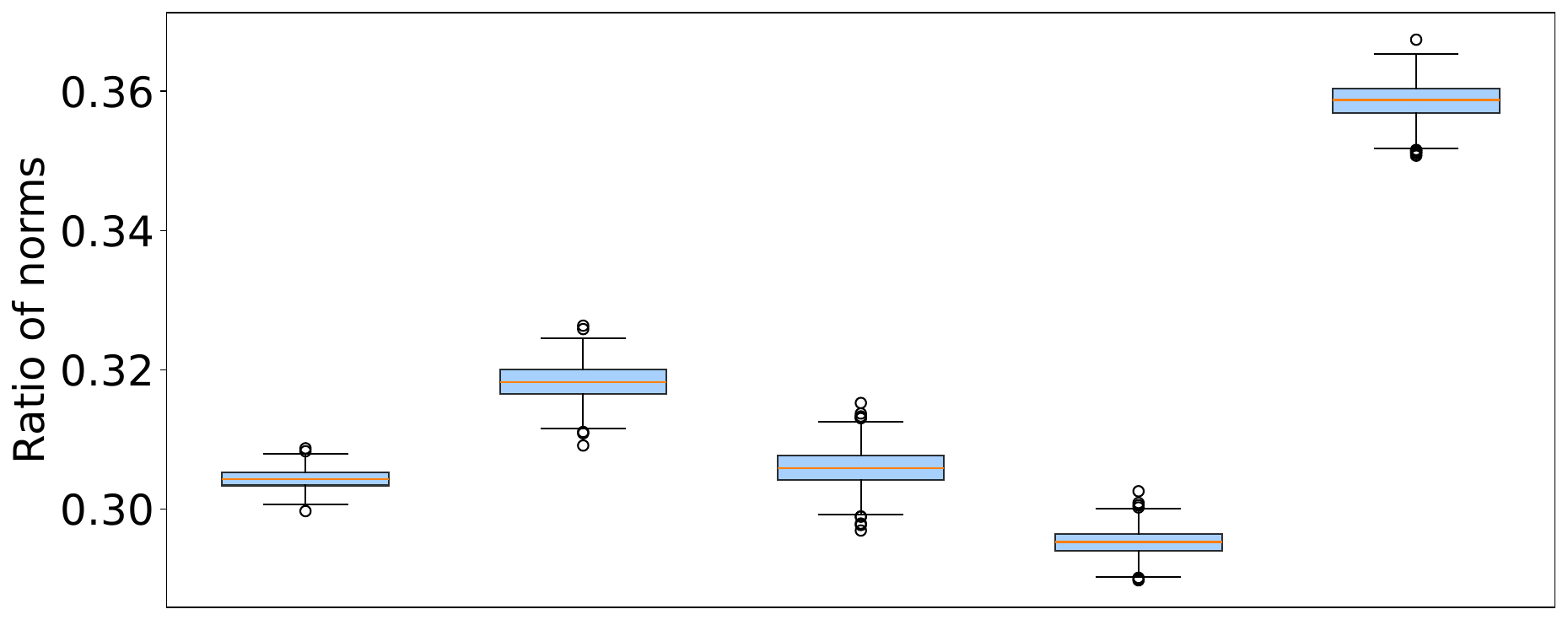} &
  \raisebox{1.2cm}{4} \\[4pt]
  \includegraphics[width=0.45\linewidth]{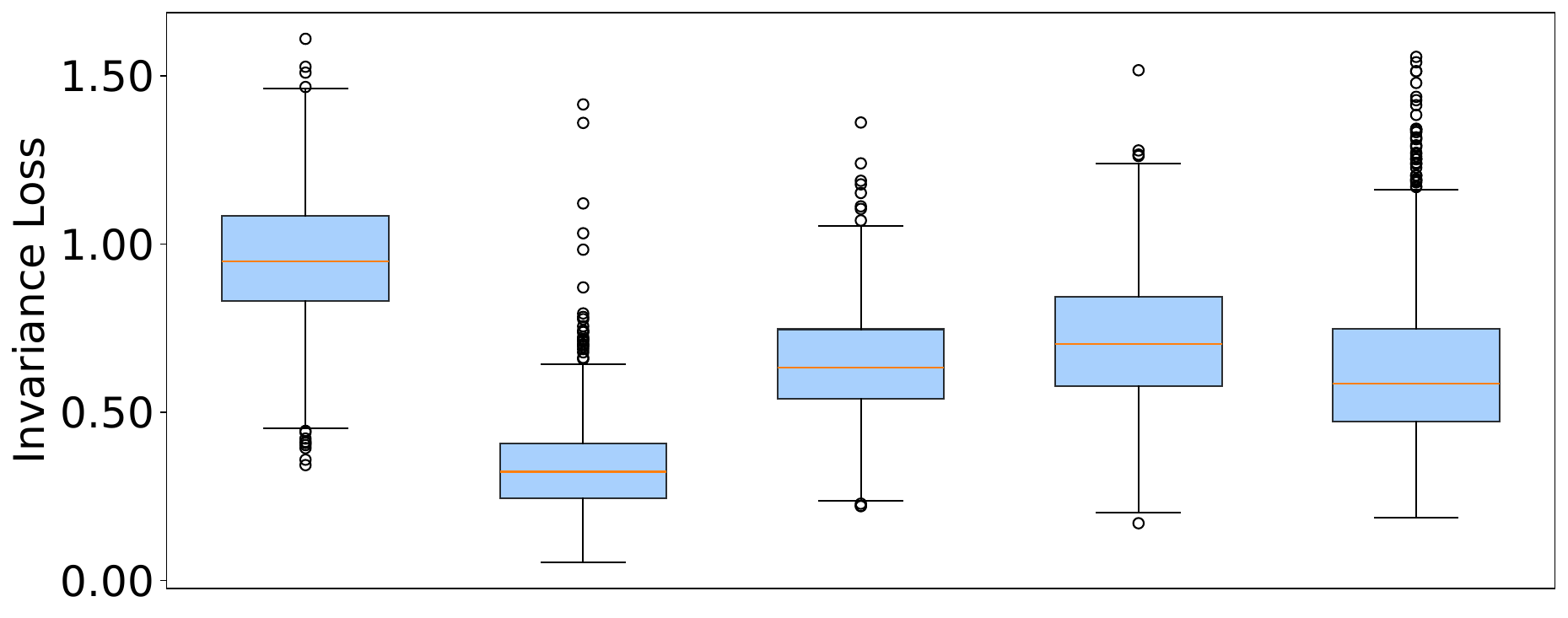} &
  \includegraphics[width=0.45\linewidth]{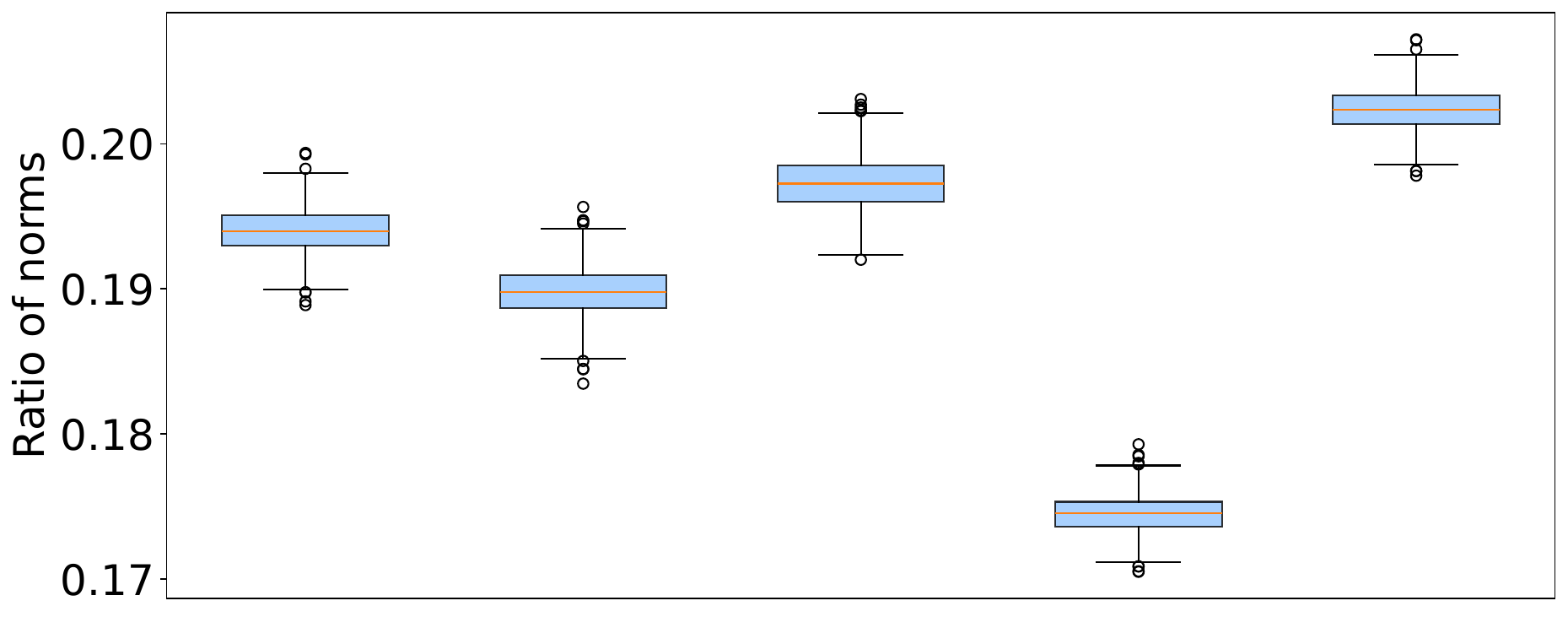} &
  \raisebox{1.2cm}{8} \\[4pt]
  \includegraphics[width=0.45\linewidth]{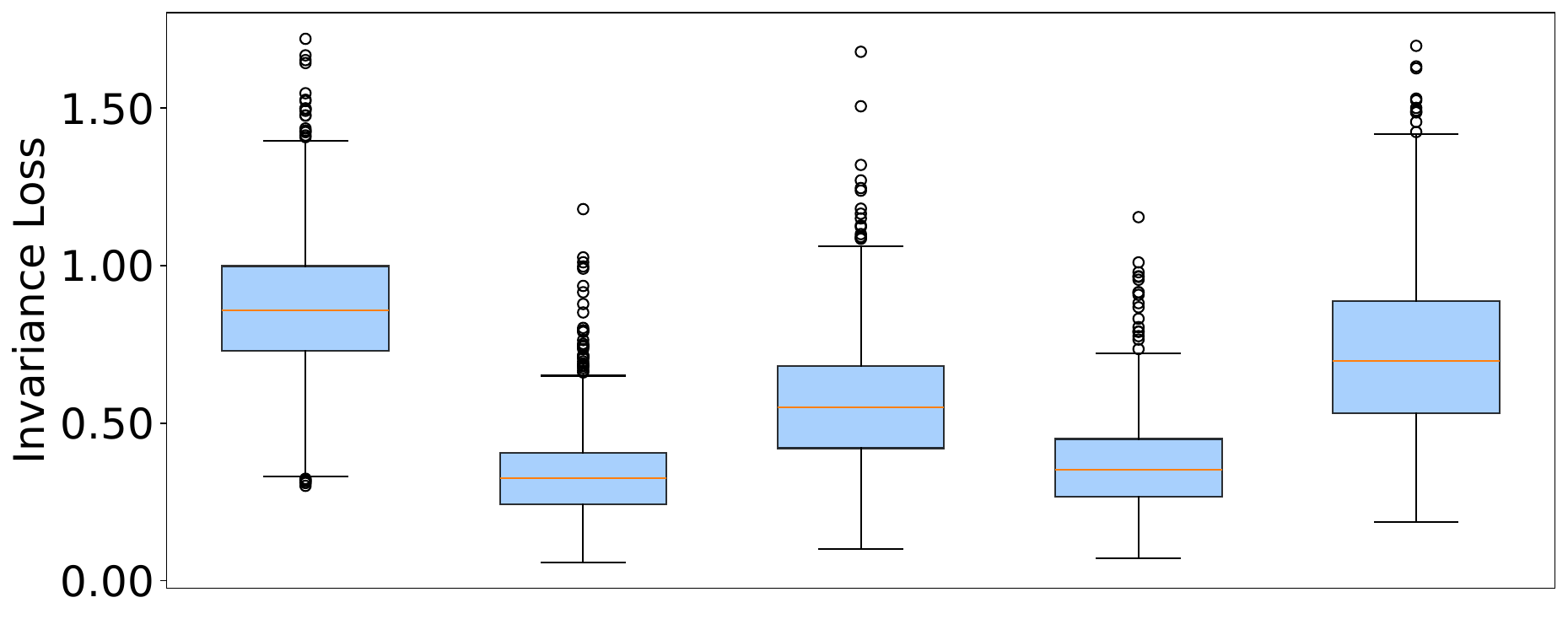} &
  \includegraphics[width=0.45\linewidth]{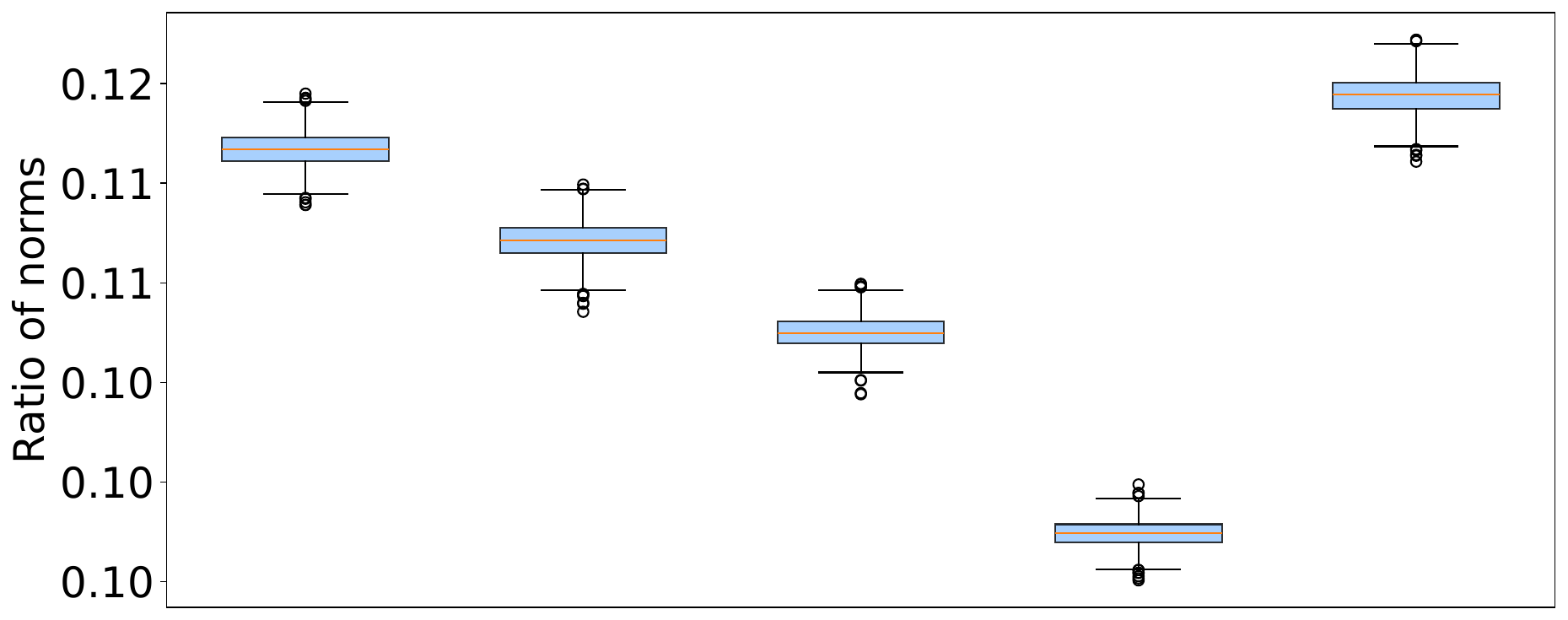} &
  \raisebox{1.2cm}{16} \\
\end{tabular}
\caption{Mechanism indicators of second layer (layer 1) of all models trained on ModelNet10 with scale augmentations.}
\label{fig:scale_modelnet_grid_layer1}
\end{figure}

\clearpage
\subsection{Equivariance}
\label{app:equivariance_results}

We provide further insight into how approximate $\operatorname{O}(3)$-equivariance is maintained throughout a forward pass and isolate first components that may be relevant to the learned equivariant mechanisms for other base angle preserving groups.

\subsubsection{$\operatorname{O}(3)$-equivariance of remaining architecture}

Once invariance is induced in a model, it cannot be destroyed by subsequent modules. The same does not hold true for equivariance: Once established, it needs to be maintained by the entire following computation to hold for the output. 

So while we have proposed a mechanism to ensure $\operatorname{O}$-equivariance in the self-attention module $\operatorname{SA}$ in \S\ref{sec:equivariance}, this property needs to be supported by the remaining architectural components as well.

The residual connections do not pose a problem here, since for any module $\operatorname{M}$ and $\mR \in \operatorname{O}(d)$, if 
$\operatorname{M}(\mX\mR)= \operatorname{M}(\mX)\mR$, then $\operatorname{M}(\mX\mR) + \mX\mR = \big(\operatorname{M}(\mX) +\mX \big) \mR$. 

The MLP on the other hand is more problematic in this context. Strict equivariance to $\operatorname{O}$ would require zero bias terms and isotropic projection matrices, which again we do not observe in trained models. So as for the value projection, we resort to measuring a proxy for the empirical equivariance error of the MLP's output by computing the invariance error of its inner product. To make the metric more comparable between modules, we normalize by a random baseline. Concretely, we divide by the error obtained when matching original samples to augmentations of random other samples. This should compensate for varying overall spread in the latent space at different layers. As we can see in Figure \ref{fig:comparison_V_MLP}, forwarding through $\operatorname{MLP}$ increases the relative invariance error of inner products, while still remaining significantly below one. 

\begin{figure}
    \centering
    \includegraphics[width=0.48\linewidth]{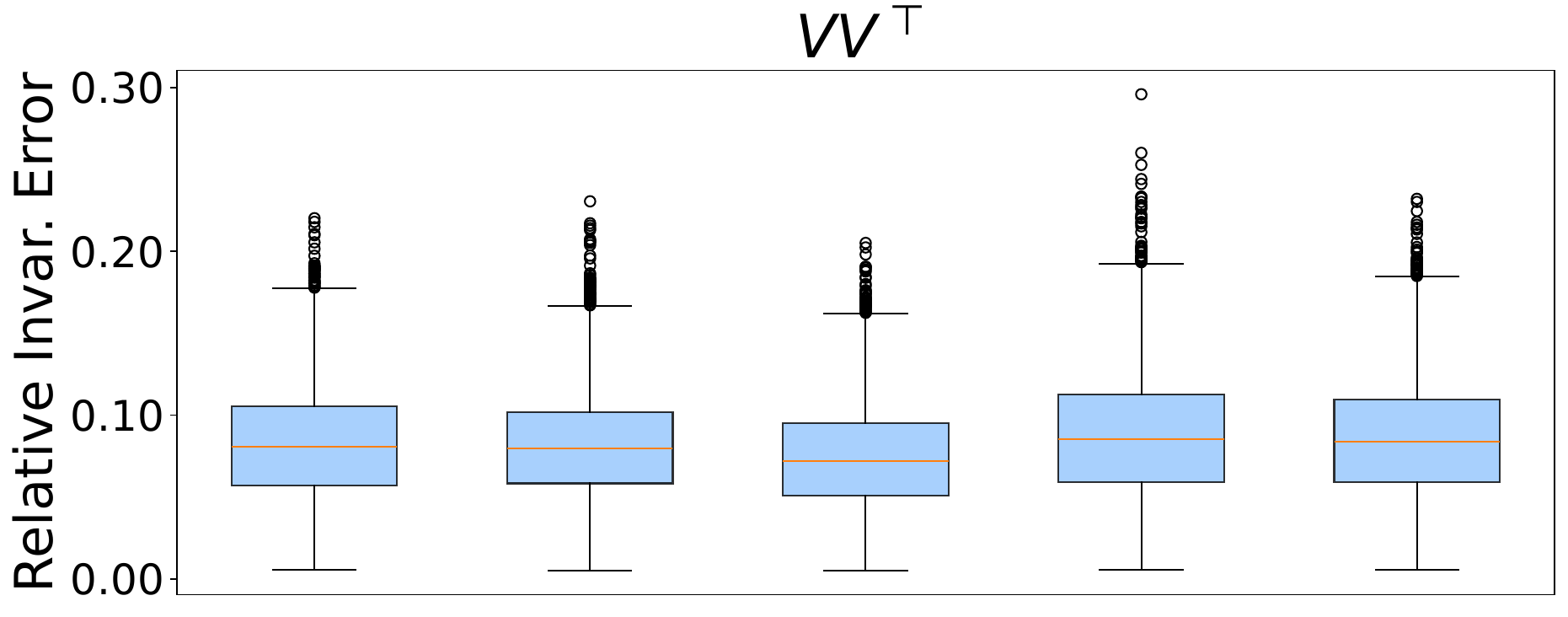} \hfill
    \includegraphics[width=0.48\linewidth]{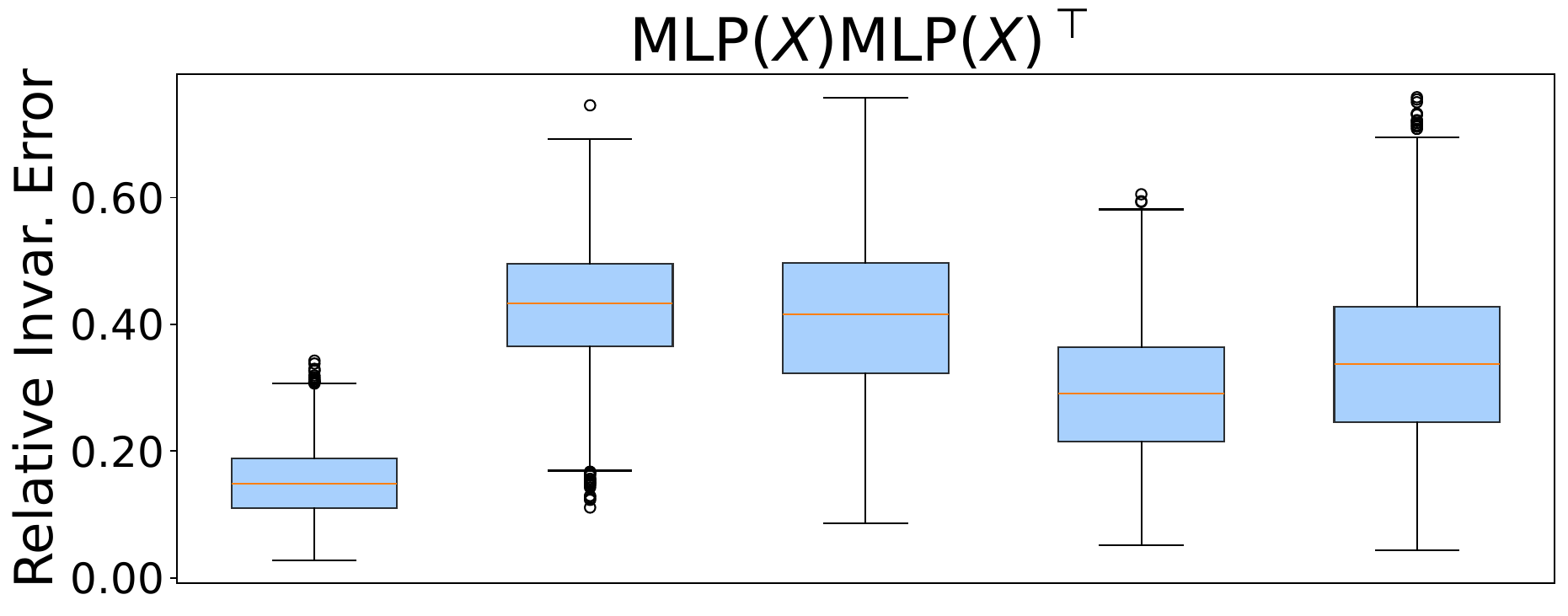}
    \caption{Comparison of equivariance of $\mV$ and $\operatorname{MLP}(\mX)$ in the first block of a Transformer trained on ShapeNet normals prediction with $16$ $\operatorname{O}(3)$ augmentations. As a proxy, we measure the invariance of inner products, divided by a random baseline. Latent representations after the MLP appear less equivariant than those after the value projection, but more equivariant than the random baseline. }
    \label{fig:comparison_V_MLP}
\end{figure}

\begin{wrapfigure}[15]{r}{0.4\textwidth}
\vspace{-6mm}
    \centering
    \begin{subfigure}[t]{0.48\linewidth}
        \centering
        \includegraphics[width=\linewidth]{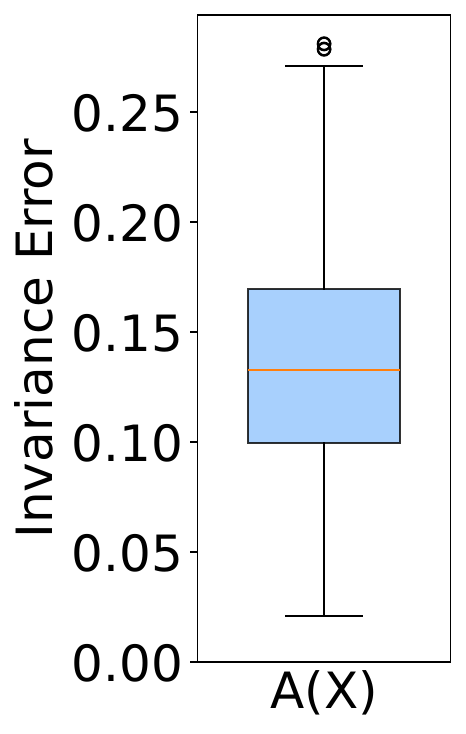}
        \caption{Scale}
    \end{subfigure}
    \hfill
    \begin{subfigure}[t]{0.45\linewidth}
        \centering
        \includegraphics[width=\linewidth]{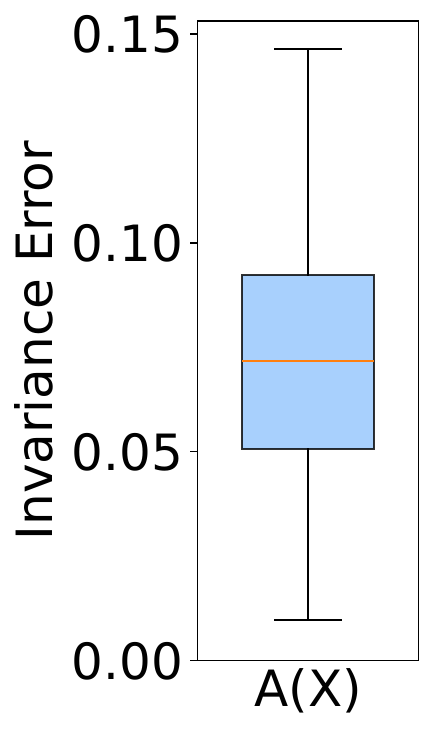}
        \caption{Translation}
    \end{subfigure}
    \caption{Results for first block of Transformers trained on ShapeNet per-part centroid prediction with $16$ $\operatorname{Scale}(3)$ (left) or $\operatorname{T}(3)$ augmentations (right) per sample. Both attention maps are approximately invariant to the respective group.}
    \label{fig:equivariance_T_Scale_A_invariant}
\end{wrapfigure}

Overall, our findings suggest that approximate equivariance is maintained by each submodule individually, instead of being introduced towards the end of the forward pass, potentially by a shared mechanism spanning all components. Wherever equivariance on the full ambient space of the input would be too restrictive on the parameter space, models rely on data-dependent approximations instead. 

\subsubsection{Other base groups:\\ attention maps are invariant}

For other base angle preserving groups, namely translation and scale, we also observe approximately invariant attention maps, see Figure \ref{fig:equivariance_T_Scale_A_invariant}. This further supports our claim that learning invariant building blocks can be helpful in learning equivariance. However as of now, the actual mechanisms and the role of the invariant attention maps remain to be identified in future work.

\section{Proofs}
\label{app:proofs}

We include the proof of our theoretical result in \S\ref{sec:mechanisms}. The numbering of the statements follows the numbering used in the paper.

\subsection{Proof of Lemma \ref{lem:orbit-wise-attention-invariance}}

\noindent\textbf{Lemma \ref{lem:orbit-wise-attention-invariance}.}
\textit{    Let $\mX\in\RR^{n\times d}$ and let $\mB = \mW_Q\mW_K^\top\in\RR^{d\times d}$ and $\mS_{\mB}(\mX)
    =
    \frac{1}{\sqrt h}\mX\mB\mX^\top$. 
    Let $\mP_n
    =
    \mI_n - \frac{1}{n}\vone_n\vone_n^\top$ denote the centering matrix acting on the columns of a score matrix. Then, for
any $\mR\in\operatorname{O}(d)$, the following are equivalent: $
     \A(\mX\mR)=\A(\mX)$, and $
        \left(
        \mS_{\mB}(\mX\mR)-\mS_{\mB}(\mX)
    \right)\mP_n
    =
    \mZeros$.
Equivalently, we have
\begin{equation*}
        \mX
    \left(
        \mR\mB\mR^\top-\mB
    \right)
    \mX^\top
    \mP_n
    =
    \mZeros .
\end{equation*}
In row-wise notation, this means that for all query indices $i$ and all key indices $j,k$, $    \vx_i^\top\left(\mR\mB\mR^\top-\mB\right) (\vx_j-\vx_k)=0 .$
Thus attention-map invariance along an augmentation orbit only requires the row-wise score differences to be invariant. }

\begin{proof}
Let $ \mS = \mS_{\mB}(\mX)$, $\widetilde{\mS} = \mS_{\mB}(\mX\mR),$ and $\mD = \widetilde{\mS}-\mS$.
Note that  for two row vectors
$\vu,\vv\in\RR^n$,
\[
    \operatorname{softmax}(\vu)
    =
    \operatorname{softmax}(\vv)
\]
if and only if there exists a scalar $c\in\RR$ such that $
    \vu-\vv = c\vone_n^\top$.
If $\vu=\vv+c\vone_n^\top$, then the factor $e^c$ cancels between the numerator and denominator of the softmax. Conversely, if the softmaxes are equal, then for every coordinate $j$, $
    \frac{e^{u_j}}{\sum_{\ell=1}^n e^{u_\ell}}
    =
    \frac{e^{v_j}}{\sum_{\ell=1}^n e^{v_\ell}}$.

Taking logarithms gives $
    u_j-v_j
    =
    \log\sum_{\ell=1}^n e^{u_\ell}
    -
    \log\sum_{\ell=1}^n e^{v_\ell}$,
which is independent of $j$. Hence $\vu-\vv$ is constant across coordinates.

Since $\A(\mX)$ is obtained by applying softmax row-wise to $\mS_{\mB}(\mX)$, we therefore have $
    \A(\mX\mR)=\A(\mX)$
if and only if, for every row index $i$, there exists a scalar $c_i$ such that $
    \mD_{i,:} = c_i\vone_n^\top$.
Equivalently, every row of $\mD$ is constant across columns.

Now recall that $
    \mP_n
    =
    \mI_n - \frac{1}{n}\vone_n\vone_n^\top$
is the centering matrix. For any row vector $\vz\in\RR^n$, $
    \vz\mP_n
    =
    \vz - \frac{1}{n}\left(\sum_{j=1}^n z_j\right)\vone_n^\top $.
Thus $\vz\mP_n=\vzero^\top$ if and only if $\vz$ is constant across coordinates. Applying this row by row, we obtain 
\begin{equation*}
    \A(\mX\mR)=\A(\mX)
    \quad\Longleftrightarrow\quad
    \mD\mP_n=\mZeros.
\end{equation*}

By substituting the definition of $\mD$, we have $
    \left(
        \mS_{\mB}(\mX\mR)-\mS_{\mB}(\mX)
    \right)\mP_n
    =
    \mZeros$.

Since $
    \mS_{\mB}(\mX)
    =
    \frac{1}{\sqrt h}\mX\mB\mX^\top$  and $
    \mS_{\mB}(\mX\mR)
    =
    \frac{1}{\sqrt h}
    (\mX\mR)\mB(\mX\mR)^\top
    =
    \frac{1}{\sqrt h}
    \mX\mR\mB\mR^\top\mX^\top$,
we have $
    \mS_{\mB}(\mX\mR)-\mS_{\mB}(\mX)
    =
    \frac{1}{\sqrt h}
    \mX
    \left(
        \mR\mB\mR^\top-\mB
    \right)
    \mX^\top$. 

Since the scalar factor $1/\sqrt h$ is nonzero, the condition above is equivalent to $
    \mX
    \left(
        \mR\mB\mR^\top-\mB
    \right)
    \mX^\top
    \mP_n
    =
    \mZeros$.

Finally, the row-wise form follows by looking at differences between two columns $j$ and $k$ of the same row. The $(i,j)$ entry of $
    \mX
    \left(
        \mR\mB\mR^\top-\mB
    \right)
    \mX^\top$
is $
    \vx_i^\top
    \left(
        \mR\mB\mR^\top-\mB
    \right)
    \vx_j$.
Henece, each row is constant across columns if and only if, for all $i,j,k$, $
    \vx_i^\top
    \left(
        \mR\mB\mR^\top-\mB
    \right)
    \vx_j
    -
    \vx_i^\top
    \left(
        \mR\mB\mR^\top-\mB
    \right)
    \vx_k
    =
    0$,
which is  $
    \vx_i^\top
    \left(
        \mR\mB\mR^\top-\mB
    \right)
    (\vx_j-\vx_k)
    =
    0$.

\end{proof}

\subsection{Proof of Lemma \ref{lemma:translation}}

\noindent\textbf{Lemma \ref{lemma:translation}.}
\textit{    Let $\operatorname{\Lambda}: \RR^{d \times n} \to \R^{n \times n}$ be translation invariant with row-stochastic outputs. Set $\mLambda = \operatorname{\Lambda}(\mX)$. Then $f(\mX) = g(\mX - \mLambda^\top \mX)$ is translation invariant for any function $g$.}

\begin{proof}
    Let $\vt \in \R^d$ and $\mT = \vone_n \vt^\top$ the matrix with $\vt$ in each row. Then $\operatorname{\Lambda}(\mX + \mT) = \operatorname{\Lambda}(\mX) = \mLambda$, such that
    \begin{equation*}
        f(\mX + \mT) = \mX + \mT - \mLambda^\top (\mX + \mT)
        = \mX - \mLambda^\top \mX - \mLambda^\top \mT + \mT = \mX - \mLambda^\top \mX = f(\mX).
    \end{equation*}
\end{proof}

\subsection{Proof of Lemma \ref{lemma:scale}}

\noindent\textbf{Lemma \ref{lemma:scale}.}
\textit{    Let $\SA(\mX)=\A(\mX)\mX\mW_V$ be a single attention-only layer. 
    If for some $c>0$, $\A(c\mX)=\A(\mX)$, then $\SA(c\mX)=c\,\SA(\mX)$.
Hence, an attention-only stack whose attention maps are stable over the scale augmentation range is positively homogeneous $\Phi(c\mX)=c\,\Phi(\mX)$.
If the final prediction is made by taking an $\operatorname{argmax}$ over logits that are scaled by this same positive factor, then the predicted class is invariant to $c$.}
\begin{proof}
For a single attention-only layer, we compute directly. Since 
    $\SA(\mX)=\A(\mX)\mX\mW_V$,
we have $
    \SA(c\mX)
    =
    \A(c\mX)(c\mX)\mW_V$.
By assumption, $\A(c\mX)=\A(\mX)$. Therefore $
    \SA(c\mX)
    =
    \A(\mX)(c\mX)\mW_V
    =
    c\,\A(\mX)\mX\mW_V
    =
    c\,\SA(\mX)$.
Thus the layer is positively homogeneous of degree one on any scale orbit on which its attention map is unchanged.

We now extend this argument to a stack of attention-only layers. Let $
    \mH^{(0)}(\mX)=\mX$,
and $    \mH^{(\ell+1)}(\mX)
    =
    \SA_\ell(\mH^{(\ell)}(\mX))$,
where $
    \SA_\ell(\mH)
    =
    \A_\ell(\mH)\mH\mW_{V,\ell}$.
Assume that, over the scale augmentation range under consideration, the attention maps are stable at every layer, meaning $
    \A_\ell(\mH^{(\ell)}(c\mX))
    =
    \A_\ell(\mH^{(\ell)}(\mX))$
for every layer $\ell$.

We prove by induction that $
    \mH^{(\ell)}(c\mX)
    =
    c\,\mH^{(\ell)}(\mX) $
for all layers $\ell$. Then, we have $
    \mH^{(0)}(c\mX)=c\mX=c\,\mH^{(0)}(\mX)$.
Now suppose that $
    \mH^{(\ell)}(c\mX)
    =
    c\,\mH^{(\ell)}(\mX)$.
Then, 
\[
\begin{aligned}
    \mH^{(\ell+1)}(c\mX)
    &=
    \SA_\ell(\mH^{(\ell)}(c\mX)) \\
    &=
    \A_\ell(\mH^{(\ell)}(c\mX))
    \mH^{(\ell)}(c\mX)
    \mW_{V,\ell} \\
    &=
    \A_\ell(\mH^{(\ell)}(\mX))
    \left(c\,\mH^{(\ell)}(\mX)\right)
    \mW_{V,\ell} \\
    &=
    c\,
    \A_\ell(\mH^{(\ell)}(\mX))
    \mH^{(\ell)}(\mX)
    \mW_{V,\ell} \\
    &=
    c\,\mH^{(\ell+1)}(\mX).
\end{aligned}
\]
Hence, by induction, the whole attention-only stack is positively homogeneous $ 
    \Phi(c\mX)=c\,\Phi(\mX)$.

The same conclusion holds if the attention-only stack includes residual connections of the form $
    \mH^{(\ell+1)}(\mX)
    =
    \mH^{(\ell)}(\mX)
    +
    \SA_\ell(\mH^{(\ell)}(\mX))$.
    
We see that using the same induction hypothesis and the single-layer homogeneity above,
\[
\begin{aligned}
    \mH^{(\ell+1)}(c\mX)
    &=
    \mH^{(\ell)}(c\mX)
    +
    \SA_\ell(\mH^{(\ell)}(c\mX)) \\
    &=
    c\,\mH^{(\ell)}(\mX)
    +
    c\,\SA_\ell(\mH^{(\ell)}(\mX)) \\
    &=
    c\,\mH^{(\ell+1)}(\mX).
\end{aligned}
\]
Thus residual addition preserves positive homogeneity.

Finally, suppose the final logits satisfy $
    \vz(c\mX)=c\,\vz(\mX)$  with $c>0$. Then, we have $
    \argmax_k z_k(c\mX)
    =
    \argmax_k c\,z_k(\mX)$.
Multiplication by a positive scalar preserves the ordering of the coordinates, thus, $
    \argmax_k c\,z_k(\mX)
    =
    \argmax_k z_k(\mX)$.
Therefore the predicted class is invariant to the scale factor $c$.
\end{proof}


\end{document}